\documentclass[11pt]{article}
\usepackage[OT1]{fontenc}
\usepackage[dvipsnames]{xcolor}        
\usepackage{smile}
\usepackage{fullpage}

\def\pa{{\mathtt{pa}}}
\newcommand{\head}[1]{{({#1})}}
\newcommand{\h}{{\head{h}}}
\newcommand{\hprime}{{\head{h'}}}

\newcommand{\veczero}{\mathbf{0}}
\newcommand{\LayerNorm}{\text{LN}}

\newcommand{\HleqD}{[H]_{\leq D}}
\newcommand{\pifromP}{{\pi\sim\cP}}
\newcommand{\EEpifromP}{{\EE_\pifromP}}
\newcommand{\seq}{{\mathtt{seq}}}

\acrodef{RPE}{Relative Positional Embedding}

\crefformat{section}{\S#2#1#3}

\crefformat{appendix}{\S#2#1#3}

\crefname{figure}{Figure}{Figures}
\Crefname{figure}{Figure}{Figures}

\usepackage{wrapfig}
\usepackage{tabularx}
\usepackage{caption}
\title{Unveiling Induction Heads: Provable Training Dynamics and Feature Learning in Transformers}
\author{
    Siyu Chen$^1$
\and 
    Heejune Sheen$^1$ 
\and 
    Tianhao Wang$^2$
\and 
    Zhuoran Yang$^1$
\and 
{\small\textit{$^1$Department of Statistics and Data Science, Yale University}} 
\and
{\small\textit{$^2$Toyota Technological Institute at Chicago}}
\and
{
    \small\texttt{\{siyu.chen.sc3226, heejune.sheen, zhuoran.yang\}@yale.edu}\quad \texttt{tianhao.wang@ttic.edu}
}
}
\date{}
\begin{document}
\maketitle

\begin{abstract}
In-context learning (ICL) is a cornerstone of large language model (LLM) functionality, yet its theoretical foundations remain elusive due to the complexity of transformer architectures. In particular, most existing work only theoretically explains how the attention mechanism facilitates ICL under certain data models. It remains unclear how the other building blocks of the transformer contribute to ICL. To address this question, we study how a two-attention-layer transformer is trained to perform ICL on $n$-gram Markov chain data, where each token in the Markov chain statistically depends on the previous $n$ tokens. 
We analyze a sophisticated transformer model featuring relative positional embedding, multi-head softmax attention,  and a feed-forward layer with normalization. 
We prove that the gradient flow with respect to a cross-entropy ICL loss converges to a limiting model that performs a generalized version of the ``{\bf \textit{induction head}}'' mechanism with a learned feature, resulting from the congruous contribution of all the building blocks. 
In the limiting model, the first attention layer acts as a {\bf \textit{copier}}, copying past tokens within a given window to each position, and the feed-forward network with normalization acts as a {\bf 
 \textit{selector}} that generates a feature vector by only looking at informationally relevant parents from the window. 
Finally, the second attention layer is a {\bf \textit{classifier}} that
compares these features with the feature at the output position, and uses the resulting similarity scores to generate the desired output. Our theory is further validated by experiments.

\begin{figure}[h!]
    \centering
    \includegraphics[width=0.75\textwidth]{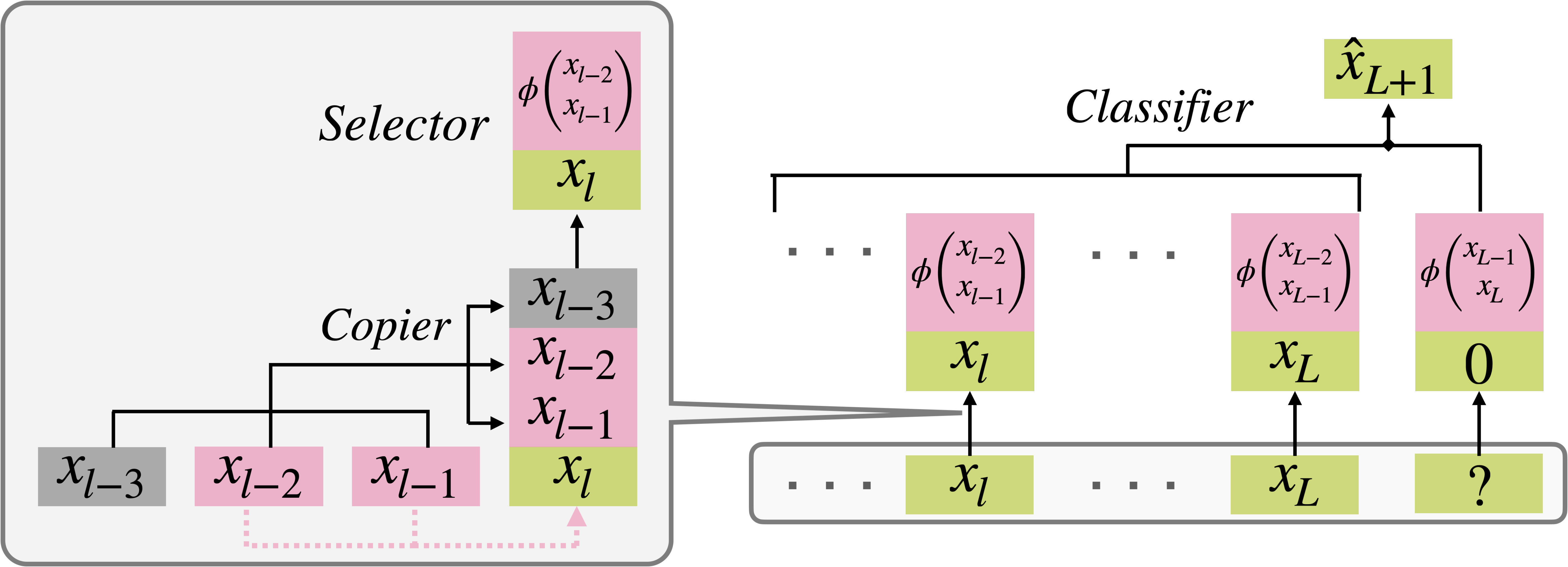}
     \caption{An illustration of the Generalized Induction Head (GIH) mechanism. Given the first $L$ tokens $\{x_{l}\}_{l\in [L]} $ of a Markov chain, we aim to predict $x_{L+1}$.  In this figure, the {\bf copier} copies a partial history of size three containing the parents at each token position, then the {\bf selector} chooses a subset of these tokens to form a feature. Finally, each $x_{l}$ and the feature at the $l$-th position are fed into a {\bf classifier} to obtain the prediction $\hat x_{L+1}$. }
    \label{fig:enter-label}
\end{figure}

\end{abstract}


\section{Introduction}\label{sec:intro}

In-context learning (ICL) \citep{brown2020language} has emerged as a crucial aspect of large language model (LLM) \citep{radford2019language,brown2020language,achiam2023gpt,anthropic2023claude,team2023gemini} functionality, enabling pre-trained LLMs to solve user-specified tasks during inference without updating model parameters. In ICL, a pre-trained LLM, typically a transformer, receives prompts containing a few demonstration examples sampled from a task-specific distribution and produces the desired output for that task. This capability is noteworthy because the tasks addressed during the ICL might not be part of the original training data set. The success of ICL requires the LLM to perform certain learning processes during inference.

Although many previous works aim to demystify ICL from either empirical or theoretical perspectives, the theoretical foundations of ICL remain elusive. 
This is primarily due to the complexity of transformer architectures, which integrate token and position embeddings, multiple layers of multi-head softmax attention, layer normalization, and feedforward neural networks. 
When it comes to understanding how the ICL ability emerges in transformers after training, existing works often focus on simplified models, such as linear attention mechanisms or single-layer transformers \citep{von2023transformers}, and ICL tasks are typically confined to linear regression \citep{akyurek2023what}. 
This leaves a gap in understanding how full-fledged transformer architectures facilitate ICL of more complex tasks, especially when latent causal structures exist among the tokens in a sequence.

In this paper, our aim is to narrow this gap by studying {\color{OrangeRed} \textbf{how a two-attention-layer transformer is trained to perform ICL of an $n$-gram Markov chain model}}, where each token in the Markov chain statistically depends on the $n$ tokens before it, known as the parent set. 
Specifically, we consider a transformer model with relative positional embedding (RPE)~\citep{he2020deberta}, multi-head softmax attention, and a feed-forward network (FFN) layer with normalization. 
We employ such a transformer model to predict the $(L+1)$-th token of an $n$-gram Markov chain, with the first $L$ tokens given as the prompt, where $L+1$ is the sequence length. 
Here the $L$-token sequence is sampled from a random Markov chain model, where a random transition kernel obeying the $n$-gram Markov property is used to generate sequences. The token sequence is fed into the transformer model, which outputs a probability distribution over the vocabulary set to predict the $(L+1)$-th token.
To train the transformer model, we sample token sequences from these random Markov chain models and minimize the cross-entropy  loss between the predicted token distribution and the true token distribution.

Under this setting, we aim to answer the following three questions: {\color{OrangeRed}(i)} \emph{Does the gradient flow with respect to the cross-entropy loss converge during training?} {\color{OrangeRed}(ii)} \emph{If yes, how does the limiting model perform ICL?} {\color{OrangeRed}(iii)} \emph{How do the building blocks of the transformer model contribute to ICL?}

\paragraph{Main Results} 
We provide an affirmative answer to the Question {\color{OrangeRed}(i)}  by proving that the gradient flow converges during training. 
In particular, we identify three phases of training dynamics: in the first stage, FFN learns the potential parent set; in the second stage, each attention head of the first multi-head softmax attention layer learns to focus on a single parent token selected by FFN; and in the final stage, the parameter of the second attention layer increases, and the transformer approaches the limiting model.
Moreover, for Questions {\color{OrangeRed}(ii)} and {\color{OrangeRed}(iii)}, we show that the limiting model performs a specialized form of exponential kernel regression, dubbed ``{\textbf{generalized induction head}}'', which requires the congruous contribution of all the building blocks.
Specifically, the first attention layer acts as a \textit{\color{OrangeRed}copier}, copying past tokens within a given window to each position. The FFN layer acts as a \textit{\color{OrangeRed} selector} that generates a feature vector by only looking at informationally relevant parents from the window according to a modified $\chi^2$-mutual information.  
Finally, the second attention layer is an \textit{exponential kernel classifier} that
compares the features at each position with those created for the output position $L+1$, and uses the resulting similarity scores to generate the desired output.  
When specialized to the case where $n=1$, the limiting model selects the true parent token and implements the \emph{induction head} mechanism \citep{elhage2021mathematical}. 
In this case, we recover the theory in \citet{nichani2024transformers}. Our theory is complemented by numerical experiments, which validate the three-phase training dynamics and mechanism of generalized induction head. 

To our best knowledge, our work is the first to provide a comprehensive understanding of how ICL is empowered by a collaboration of different building blocks in a transformer model. In particular, we identify the pivotal roles played by RPE in the {\color{OrangeRed}copier} component, the FFN layer with normalization in the {\color{OrangeRed}selector} component, and attention in the {\color{OrangeRed} classifier} component. 
We believe our work will shed light on the theoretical understanding of ICL for more complicated tasks. 

\subsection{Related Works}\label{sec:related_works}
Our work adds to the rapidly growing literature on understanding ICL by transformers. 





\paragraph{In Context Learning (ICL)} 
Commercial Large Language Models (LLMs) such as ChatGPT \citep{brown2020language}, GPT-4 \citep{achiam2023gpt}, and Gemini \citep{team2023gemini} typically operate in an autoregressive manner. These models exhibit remarkable ICL capabilities, without requiring further training. Previous research explores various aspects of the in-context learning (ICL) ability of these models. This includes their performance in zero-shot and few-shot learning scenarios \citep{honovich2022instruction, wei2021finetuned}, the use of the chain of thought method to enhance reasoning \citep{wei2022chain, zhou2022least}, and learning with multi-modalities \citep{alayrac2022flamingo}.
Moreover, recent research highlights the properties and advantages of using transformers beyond the traditional ICL setting, thereby broadening our understanding of their capabilities and applications \citep{edelman2022inductive,Li2023HowDT,jelassi2022vision,sanford2023representational,pmlr-v202-giannou23a,Liu2022TransformersLS,Tarzanagh2023TransformersAS,tarzanagh_max-margin_2023,tian2023joma,tian2023scan,song2023uncovering,deora2023optimization,chen2024provably, rajaraman2024toward}.

There is a large and growing body of literature on understanding how transformer architecture enables ICL. 
One strand of research proposes to understand ICL by casting it as a version of Bayesian inference expressed by the transformer architecture. See, e.g., 
\cite{xie2021explanation, Muller2021TransformersCD, zhang2022analysis, zhang2023and, ahuja2023context, jeon2024information, he2024words} and the references therein.
Another line of work investigates how transformers internally emulate specific algorithms to solve ICL tasks, where \citet{akyurek2023what, von2023transformers, fu2023transformers, ahn2023transformers, mahankali2023one, giannou2024well, wu2023many} focus on learning with linear regression tasks and \citet{bai2023transformers, cheng2023transformers, collins2024context, guo2023transformers} investigate transformers' capabilities in learning with nonlinear functions.
However, all of these works above focus on regression tasks where token (or token pairs) in the prompt sequences are i.i.d. or uncorrelated, which may not capture the more sophisticated data structures in real-world applications.


In addition, 
to study ICL with correlated data, 
there is also substantial interest in understanding how ICL operates over data drawn from Markov chains, providing insight into how transformer architectures contribute to ICL in these settings \citep{edelman2024evolution, makkuva2024attention,chen2024can}. Furthermore,  \cite{lin2023transformers,sinii2023context} show how transformers can solve reinforcement learning problems in an in-context  fashion. 

While many of the aforementioned works focus on the expressivity of the transformer model on different ICL tasks and the statistical properties of the learned models, 
understanding training dynamics from an optimization perspective is also crucial for comprehending ICL by transformers. 
The training dynamics for one-layer attention models have been investigated under different data models for both regression and classification tasks \citep{zhang2023trained, huang2023context, Tarzanagh2023TransformersAS,tarzanagh_max-margin_2023,kim2024transformers,chen2024training,vasudeva2024implicit,li2024mechanics,thrampoulidis2024implicit,sheen2024implicit}.
These studies offer a thorough characterization of the training process, yet they have limitations --- they are not directly applicable to data drawn from Markov processes and are confined to single-layer attention.
Our work belongs to this line of research and we adopt a two-attention-layer transformer architecture, which is more complicated than the transformer studied in these works.

\paragraph{Induction Head}
\cite{elhage2021mathematical} introduce the concept of ``induction heads'' as the mechanism underlying the ICL capabilities of transformers.
Since then, there has been a surge of interest in understanding the induction head mechanism and its role in ICL.
At a high level, the induction head mechanism works by matching the history of the current token with those seen previously in the sequence and then predicting the next token based on the matched historical sub-sequences.
\cite{olsson2022context} provide empirical evidence highlighting that induction heads are crucial in facilitating the ICL capabilities of transformers.
\cite{bietti2024birth,edelman2024evolution} conduct a further empirical investigation into the development of induction heads specifically tailored for the ICL of bi-gram data models. 
\citet{rajaraman2024transformers} provide explicit constructions of single-head transformers with constant depths that can learn $n$-gram data.
Also, a wider range of functionalities exhibited by induction heads that interact with various other mechanisms have been observed by \cite{wang2022interpretability}.

From a theoretical perspective, 
\cite{nichani2024transformers} study the ICL of \emph{first-order} Markov chains using a two-layer transformer and demonstrate the formation of the induction head mechanism.
\citet{makkuva2024local} also prove that training a single layer attention with a feed-forward layer on \emph{first-order} Markov data (with $\{0, 1\}$ vocabulary) can converge to either to global or local minima depending on the initialization. 
However, the \emph{first-order} assumption seems to be quite restrictive, especially when modeling the natural language, where the tokens can depend on multiple previous tokens.
Most related to our work is \citet{nichani2024transformers}, where they analyzed how training by gradient descent enables a two-layer transformer to learn the latent causal graph underlying the ICL data.
However, the analysis in \citet{nichani2024transformers} applies to Markov chains where each token has at most one parent, and it remains unclear how to extend the analysis to more general $n$-gram Markov chains.

In this work, we show that a generalized version of the induction head mechanism can emerge when training a two-layer transformer on $n$-gram Markov chains.
Moreover, our transformer models are more sophisticated, incorporating features like relative positional embedding, multi-head attention,  an FNN layer, and normalization. Notably, we provide an in-depth dynamics analysis of the corresponding FFN layer and two-layer multi-head attention.

\paragraph{Roadmap}
The rest of the paper is organized as follows:
We introduce the problem setup of ICL of Markov chains in \Cref{sec:setup}.
Then in \Cref{sec:theory}, we present the main theoretical results and related discussions.
A proof sketch is provided in \Cref{sec:proof_sketch}.
Finally, we present corresponding experiment results in \Cref{sec:experiments}, and the detailed proofs are deferred to the Appendix.

 \paragraph{Notation}
We denote by $e_1, \dots, e_d$ the standard basis vectors in $\RR^d$ and by $\bm{1}$ the all-one vector in $\RR^d$. 
We denote by $\sigma(\cdot)$ the softmax function such that the $i$-th coordinate of $\sigma(x)$ is $\sigma_i(x) = {\exp(x_i)}/{\sum_{l=1}^L \exp(x_l)}$ for $x\in\RR^L$.
By default, the softmax operation will always be applied row-wise.
For any integer $n>0$, we denote $[n]:=\{1, \dots, n\}$.
For a vector $w\in\RR^M$, we denote by $w_i$ the $i$-th entry of $w$ and $w_{-i}$ the $(M+1-i)$-th entry of $w$ for positive integer $i\in[M]$.
For a matrix $W$, we denote by $W(i,j)$ the entry at the $i$-th row and $j$-th column of $W$.
For two vectors $u$ and $v$, we write $u/v$ as the vector obtained by taking element-wise division between $u$ and $v$.
We denote by $a\lor b$ and $a\land b$ the maximum and minimum of $a$ and $b$, respectively.
We denote by $x_{s:t}$ the sequence $\{ x_{s}, x_{s+1}, \ldots, x_{t}\}$.
For a class $\cX$, we denote by $\Delta(\cX)$ the space of probability measures over $\cX$.
We use the standard big O notation throughout the paper.

\section{Problem Setup: In-Context Learning of Markov Chains}
\label{sec:setup}
In this section, we present the details of the problem setting. In particular, we first introduce the statistical problem of ICL of $n$-gram Markov chains in \cref{sec:icl_mc} and then lay out the details of the transformer model in \Cref{sec:transformer_model}.

\subsection{In-Context Learning and \(n\)-Gram Markov Chains} \label{sec:icl_mc}
We study how autoregressive transformers are trained to perform in-context learning (ICL). 
A pre-trained transformer can be viewed as a conditional distribution $f_{\mathtt{tf}}(\cdot \given \mathtt{prompt})$ over a finite vocabulary set $\cX$, where $\mathtt{prompt}$ is a sequence of tokens in $\cX$. 
We consider an in-context unsupervised learning problem where the pre-trained transformer $f_{\mathtt{tf}} $ is used to predict the $(L+1)$-th token $x_{L+1}$ with the first $L$ tokens being the prompt. 
Here $L$ is a fixed number and the joint distribution of the sequence $x_{1:(L+1)}$ is sampled from a random $n$-gram Markov chain.
In other words, with $x_{1:(L+1)}$ sampled from some distribution, we  evaluate how well $f_{\mathtt{tf}}(\cdot \given x_{1:L})$ predicts the distribution of $x_{L+1}$. 

\begin{wrapfigure}{r}{0.45\textwidth}
\centering
  \begin{minipage}[h]{0.45\textwidth}
        \centering
        \includegraphics[width=\textwidth]{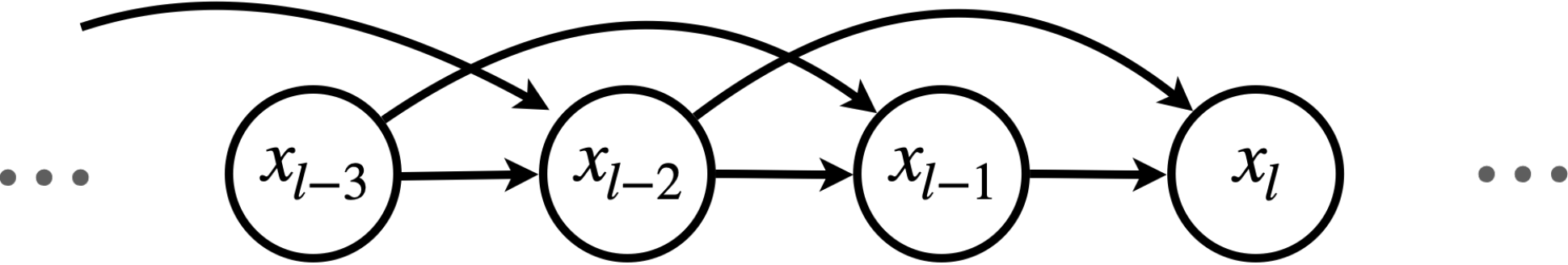}
    \caption{A two-gram Markov chain with parent set $\pa = \{-1, -2\}$.}
    \label{fig:n-gram}
    \end{minipage}
\end{wrapfigure}
\paragraph{$n$-Gram Markov Chains} We assume the data comes from a mixture of $n$-gram Markov chain model, denoted by a tuple $(\cX, \mathtt{pa}, \cP, \mu_0)$, where $\cX$ is the state space and  
$ \pa=(-r_1, \dots, -r_n)$ is the parent set with positive integers $r_1 < r_2 < \dots < r_n$.
That is, for each $l > r_n$, $x_{l}$   only statistically depends on $( x_{l-r_n}, \dots, x_{l-r_1})$, which is denoted by $X_{\pa(l)} $ and referred to as the parent tokens of $x_l$. 
We let $d = |\cX|$ denote the vocabulary size. 
Moreover, 
$\cP$ is a probability distribution over the set of Markov transition kernels respecting the parent structure specified by $\mathtt{pa}$, and $\mu_0$ is the joint distribution of the first $r_n$ tokens $x_{1:r_n}$. 
Note that the size of the parent set $n$ can be smaller than or equal to $r_n$.
Thus, the sequence $x_{1:(L+1)}$ is generated as follows: (i)~sample initial $r_n$ tokens $(x_1, \dots, x_{r_n}) \sim \mu_0$, (ii) sample a random transition kernel $\pi\sim \cP$, where $\pi \colon \cX^n \rightarrow \Delta(\cX)$, and (iii) sample token $x_l \sim \pi(\cdot \given X_{\pa(l)})$ for $l = r_n+1, \ldots, L+1$. 
See Figure \ref{fig:n-gram} for an illustration of the generating model of $x_{1:(L+1)}$.

\paragraph{Cross-Entropy Loss} When $x_{1:(L+1)}$ is generated, $x_{1:L}$ is fed into the transformer $f_{\mathtt{tf}}$ to predict $x_{L+1}$. To assess the performance of ICL, we adopt the population cross-entropy (CE) loss
\begin{align} \label{eq:cross_entropy}
    \cL(f_{\mathtt{tf}}) = - \EE_{\pi \sim \cP, x_{1:(L+1)}} \bigl[\log \bigl(f_{\mathtt{tf}}(x_{L+1}  \given x_{1:L})  + \epsilon \bigr)  \bigr] ,
\end{align}
where $\epsilon>0$ is a small constant introduced for numerical stability and in the sequel we will take $\varepsilon = O(L^{-1/2})$. 
Here, the expectation is taken with respect to the joint distribution of $x_{1:(L+1)}$ (including the randomness of $\pi \sim \cP$).
When setting $\epsilon=0$, we note that minimizing this cross-entropy loss is equivalent to minimizing the KL divergence 
\begin{align}
\label{eq:kl_divergence_loss}
\EE_{\pi \sim \cP, x_{1:L}} \bigl[ \mathtt{KL} ( \pi(\cdot \given X_{\mathtt{pa}(L+1)}) \,\|\, f_{\mathtt{tf}}( \cdot \given x_{1:L}) ) \bigr].
\end{align}
As a remark, we also relax a condition in \citet{nichani2024transformers} where the last token $x_L$ has to be resampled from a uniform distribution.
In addition, our analysis can also be extended to sequential CE loss, which corresponds to predicting every token in the sequence given the past rather than just the last token $x_{L+1}$. 
This is closer to the training paradigm used in practice \citep{brown2020language}.
See \cref{sec:sequential_ce_loss} for a further discussion on the sequential CE loss.
 



\subsection{A Two-Layer Transformer Model} \label{sec:transformer_model}

We consider a class of two-attention-layer transformer model, denoted by   $\mathtt{TF}(M, H, d, D)$, which incorporates Relative Positional Embedding (RPE) \citep{he2020deberta}, Multi-Head Attention (MHA) \citep{vaswani2017attention}, and a Feed-Forward network (FFN)   with normalization.
Here $M$ is an integer that specifies the window size of RPE, $H$ is the number of heads in the first attention layer, $d$ is the vocabulary size, and $D$ is an integer that controls the complexity of FFN. The details of $\mathtt{TF}(M, H, d, D)$ are as follows. 

\paragraph{Token Embedding, Input and Output} Note that each token takes values in $\cX$ with $d = |\cX|$. 
We embed the tokens into one-hot vectors in $\RR^d$, and thus we can identify $\cX $ as the canonical basis in $\RR^d$, i.e.,  $\cX = \{ e_1, \ldots, e_d\}$. 
A transformer model can be viewed as a mapping from $\RR^{(L+1)\times d}$ to $\Delta(\cX)$. 
In particular, given the input sequence $x_{1:L}$, we denote $X = (x_1, \dots, x_L)^\top \in \RR^{L\times d}$, and we append a zero vector $\zero\in\RR^d$ to the sequence, and define $\tilde X = (x_1, \dots, x_L, \veczero)^\top \in \RR^{(L+1)\times d}$. The transformer takes $\tilde X$ as input and outputs a probability distribution over $\cX$. 

\begin{wrapfigure}{L}{0.4\textwidth}
    \centering
    \vspace{-0.75cm}
    \begin{minipage}[t]{0.4\textwidth}
    \centering
    \includegraphics[width=\textwidth]{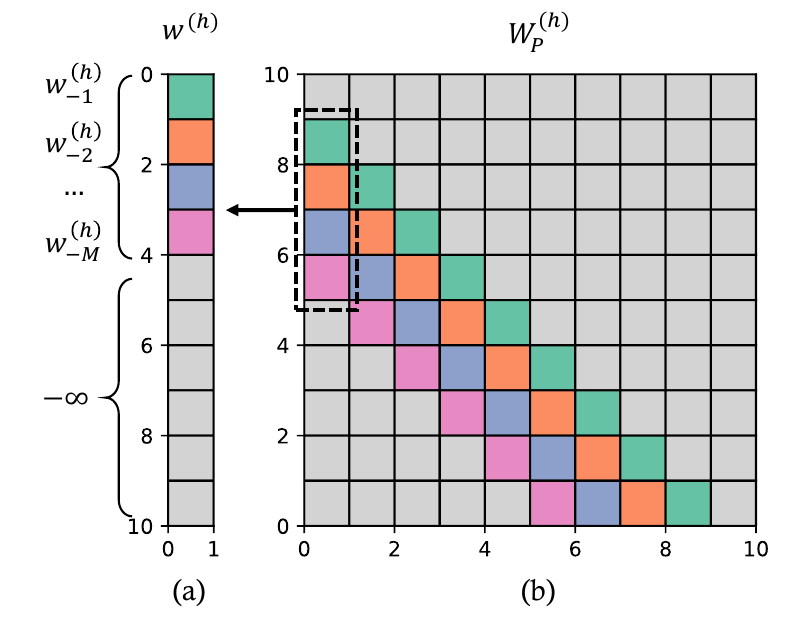}
    \setlength{\abovecaptionskip}{-10pt}
    \setlength{\belowcaptionskip}{-20pt}
    \caption{Illustration of the relationship between RPE vector $w^{(h)}$ and corresponding matrix $W_P^{(h)}$.}
    \label{fig:RPE}
    \end{minipage}
\end{wrapfigure}
\paragraph{Relative Positional Embedding}
In each head of the first attention layer, we adopt RPE to incorporate positional information. 
Specifically, 
RPE is parameterized by a vector $\allowbreak w = (w_{-M}, \dots, w_{-1})^\top \in \RR^M$, and it 
assigns a scalar value  $W_P(i,j) $ to a pair of positions $(i,j)$ satisfying 
\begin{align} \label{eq:rpe_def}
W_P(i,j) &= w_{j-i} ~~\mathrm{if} ~~i-j \in \{1,\ldots, M\}, \\  W_P(i,j)  &= -\infty  ~~\mathrm{if}~~ j \geq i~~\mathrm{or} ~~|j- i | > M.  
\end{align}
In other words, as illustrated in Figure~\ref{fig:RPE}, the $i$-th token only attends to tokens with indices in $\{i-1, \ldots, i-M\}$, referred to as the \emph{length-$M$ window of the $i$-th token},  and the trainable vector $w$ determines the value of positional embedding. 
Here, we use $-k$ to index the last $k$-th position.

\paragraph{The First Attention Layer} 
The input sequence is processed by the first attention layer with $H$ parallel heads. 
In all heads, we discard the token information and only use RPE to compute the attention score. 
Specifically, each attention head $h$  maps $\tilde X$ into a sequence in $\RR^d$ with length $L+1$, denoted by $V^\h = (v_1^\h, \ldots, v^\h_{L+1})^\top \in \RR^{(L+1)\times d}$. 
For any $l \in [L+1]$, $v^\h_{l}$ is computed via
\begin{align}\label{eq:define_V_h}
v^\h_{\color{PineGreen}l}  = \sum_{{\color{OrangeRed}j} = 1}^L \sigma_{\color{OrangeRed}j }\bigl (W_P^\h({\color{PineGreen} l}, \cdot)\bigr)\cdot 
 x_{\color{OrangeRed}j}   =    \sum_{{\color{OrangeRed}j} =1}^L \frac{\exp\bigl (W_P^\h({\color{PineGreen} l} ,{\color{OrangeRed}j}) \bigr ) \cdot x_{\color{OrangeRed}j}}{\sum_{k =1}^L \exp\bigl (W_P^\h({\color{PineGreen} l}, k) \bigl)}  . 
\end{align}
That is, we use the RPE parameter $W_P^\h $ to construct a weighted sum over the input sequence at each position $l \in [L+1]$. Here $W_P^\h$ is the RPE matrix of the $h$-th head. 


\paragraph{Feed-Forward Network with Normalization}
Following the first attention layer, we concatenate the outputs of the $H$ attention heads and define $V = (
         V^\head{1}, \ldots,  V^\head{H}
    )   \in \RR^{(L+1)\times Hd}$. 
    Here we abuse the notation and write $V = (v_1, \ldots, v_{L+1})^\top$, i.e., each $v_l$ is the $l$-th row of $V$. 
For any vector $v\in\RR^{Hd}$, we can split it into $(v^{(1)\top}, \ldots, v^{(H)\top})^\top$ where each block $v^{(h)} \in \RR^d$.
For embedding dimension $d_e$, each vector of $V$ is passed through an FFN $\phi(\cdot):\RR^{Hd}\to\RR^{d_e}$, which specifies a polynomial kernel such that 
for any $v,v'\in\RR^{Hd}$,  we have 
\begin{align}\label{eq:feed-forward}
    \langle \phi(v), \phi(v') \rangle = \sum_{{\color{PineGreen}\cS}\in \HleqD} c_{{\color{PineGreen}\cS}}^2 \cdot \prod_{h\in {\color{PineGreen}\cS}} \langle v^\head{h}, {v'}^\head{h} \rangle. 
\end{align}
Here, the low-degree parent set $\HleqD := \{ \cS \subseteq [H]: |\cS| \le D\}$ contains all    subsets of $[H]$ with cardinality at most $D$, and $\{c_\cS : \cS\in \HleqD\}$ are the corresponding trainable parameters of $\phi (\cdot)$.
Therefore, the FFN $\phi (\cdot) $ specifies a kernel on the output of the multihead attention which induces a special inner product structure.
While  \eqref{eq:feed-forward} characterizes $\phi (\cdot) $ implicitly, we provide an explicit construction of $\phi(\cdot)$ in  \Cref{lemma:poly_kernel} as a vector-valued mapping whose entries are monomials of the input's entries.
Moreover, the complexity of $\phi(\cdot) $ is controlled by the maximum degree $D$, which also influences the embedding dimension $d_e$ as we show in the construction.

Furthermore, to control the magnitude of the FFN outputs, we normalize $\phi(\cdot)$ by letting $u_{l} = \phi(v_l)/\sqrt{C_D}$ for all $l \in [L+1]$, where we define $C_D = \sum_{\cS \in \HleqD} c_{\cS}^2 $. 
Such a normalization scheme is motivated by the standard layer normalization \citep{ba2016layer} in transformer architectures.
To motivate the use of $\sqrt{C_D}$ as the normalization, consider a special case where the positional embeddings, after the softmax function, produce attention weights that are close to one-hot for each head. 
Then $v_l^\h$ in \eqref{eq:define_V_h} is equal to  some token in $x_{1:L}$. 
As a result, each $v_l$ consists of $H$ tokens and 
\begin{align}
    \|\phi (v_l)\|_2 = \sqrt{\sum_{{\cS}\in \HleqD} c_{{\cS}}^2 \cdot \prod_{h\in {\cS}} \langle v_l^\head{h}, v_l^\head{h} \rangle} = \sqrt{C_D}.
\end{align}
Thus, $u_{l}$ is roughly equivalent to the output of the layer normalization $\phi(v_l)/ \|\phi (v_l)\|_2$ (without trainable parameters). 
Although our theoretical analysis and simulations focus on this simplified version of layer normalization, our additional experiments in \cref{sec:exp_supp} demonstrate that it aligns well with the performance of the actual layer normalization.

\paragraph{The Second Attention Layer} The normalized vector sequence $U = (u_1, \ldots, u_{L+1})^\top$ and the original sequence $\tilde X$ are then fed into the second attention layer to generate the final output.
In particular, $u_{L+1} $ is used as the query to compare with the keys $\{ u_{M+1}, \ldots, u_L\} $, and the resulting attention scores are used to aggregate the values $x_{(M+1):L}$. 
This attention layer has a single head and a scalar trainable parameter $a$. 
We let $U_{1:L} = (u_1, \ldots, u_L)^\top\in\RR^{L\times d_e}$ and denote by $\mathtt{Mask}(\cdot)$ the mask that sets every entry of the  first $M$ rows of a matrix to be $-\infty$.
The final output is given by 
\begin{align}\label{eq:attention-output} 
y = 
\sum_{ {\color{OrangeRed}j} =M+1 }^L \sigma_{\color{OrangeRed} j} \bigl (a \cdot u_{L+1}^\top \mathtt{Mask} (U_{1:L}^\top ) \big)  \cdot x_{{\color{OrangeRed} j} }=
  \sum_{ {\color{OrangeRed} j} =M+1}^L \frac{\exp\bigl (a \cdot u_{L+1}^\top u_{\color{OrangeRed} j} )\cdot x_{\color{OrangeRed} j}}{\sum_{k =M+1 }^L \exp\bigl ( a \cdot u_{L+1}^\top u_k  \bigl) }.
\end{align}
Note that the softmax function in \eqref{eq:attention-output} yields a probability distribution over $[L]$ and that $x_{1:L}$ is a sequence of one-hot vectors. Thus $y $ in \eqref{eq:attention-output} is  a probability distribution over $\cX$. 
The mask operator is included here just to simplify our analysis while in the experiments we are not using the mask. 

In summary, given the input 
$\tilde X\in \RR^{(L+1)\times d}$, in the matrix form, our transformer model $\mathtt{TF}(M,H, d, D)$ consecutively applies the following operations: 
\begin{align}\label{eq:transformer}
\begin{aligned}
&\textbf{First Attention:} && \qquad V^\h = \sigma(W_P^\h) \tilde X &&\in \RR^{(L+1)\times d}, ~\forall  h\in[H] ;\\
    &\textbf{Concatenate:} && \qquad V = [
         V^\head{1}, \dots,  V^\head{H}
    ]  &&\in \RR^{(L+1)\times Hd} ;\\
    &\textbf{FFN \& Normalize:} &&\qquad U = \phi(V)/\sqrt{C_D} &&\in \RR^{(L+1)\times d_e}; \\
    &\textbf{Second Attention:} && \qquad y^\top = \sigma\bigl(
        a \cdot u_{L+1}^\top \mathtt{Mask}(U_{1:L}^\top)  
    \bigr) X &&\in \RR^{1\times d}.
\end{aligned}
\end{align}
The trainable parameters of the above transformer model are denoted by 
\begin{align}\label{eq:transformer_parameters}
    \Theta = \big\{a, \{w_{-1}^\h,\ldots,w_{-M}^\h\}_{h \in [H]}, \{c_\cS: \cS\in\HleqD \} \big\}.
\end{align}
We remark that the transformer model in \eqref{eq:transformer} is known as a disentangled transformer \citep{friedman2024learning}, which is a version of the transformer model that is more amenable for theoretical analysis. 
One thing to be noted is that there is a residual connection that directly copies $\tilde X$ to the output of the FFN \& Normalize block, which gives us $[U, \tilde X]$, and the second attention layer will treat the copied $\tilde X$ as the value in the attention mechanism.
We omit the residual connection in the above paradigm for notation simplicity.
As shown in \citet{nichani2024transformers}, any standard transformer model can be expressed as a disentangled transformer by specializing the attention weights to allow feature concatenation.


Our goal is to investigate whether the transformer model $\mathtt{TF}(M, H, d, D)$ can perform ICL over $n$-gram Markov chains and further, whether such capability can be learned from data with common training algorithms like gradient descent.




\section{Theoretical Results}\label{sec:theory}

In this section, we present the theoretical results. 
We first show in \Cref{sec:induction_head} and \Cref{sec:transformer&GIH} that there exists a transformer  in  $\mathtt{TF}(M,H, d, D)$ that implements a generalized ``induction head'' mechanism \citep{olsson2022context} with a learned feature, which serves as a natural algorithm for learning $n$-gram Markov chains.
Then in \Cref{sec:convergence} we prove that the gradient flow in \eqref{eq:grad_flow} finds such a desired model asymptotically. 

\subsection{Generalized Induction Head Mechanism for Learning $n$-Gram Markov Chains} \label{sec:induction_head}

Recall that we define the mixture of $n$-gram Markov chain model $(\cX, \mathtt{pa}, \cP, \mu_0)$ in \cref{sec:icl_mc}, where $\cP$ is a distribution over the Markov transition kernels. 
For regularity, we assume existence of a unique stationary distribution for any $\pi\in\supp(\cP)$, where a rigorous statement is deferred to \cref{asp:Markov_chain}. 
We also assume the window size $M>r_n$.
For any $n$-gram Markov chain with transition kernel $\pi \sim \cP$, we let $\mu^\pi \in \Delta(\cX^{M+1})$  denote the stationary distribution of the Markov chain over a window of size $M+1$. 
Here we use $\{ z_{\ell} \}_{l\geq 1}$ to denote a random sequence of tokens generated by the Markov chain. 
Then  $\mu^\pi $ denotes the joint distribution of a block of $M+1$ tokens $(z_{l-M}, \ldots, z_{l-1}, z_{l})$ under the stationary distribution of $\pi$, where $l > M$ is an integer. 

In the following, we introduce a generalized induction head (GIH) estimator for the task of predicting $x_{L+1}$ given $x_{1:L}$, which is based on the following simple idea: \emph{$x_{L+1}$ should be similar to a previous token  $x_{l}$ if their parents are similar}. 
As the parent set $\mathtt{pa}$ is unknown, GIH adopts an information-theoretic criterion to select a subset of previous tokens as a proxy of the parents. 
Specifically, GIH uses a modified version of $\chi^2$-mutual information, which is defined as follows. 

\begin{definition}[Modified $\chi^2$-Mutual Information]\label{def:modified_chi_square_mi}
We take a length-$(M+1)$ windows $ (z_{l-M}, \ldots, z_{l-1}, z_l)$ for some $l>M$ and suppose the sequence is sampled from stationary distribution $\mu^\pi$ with $\pi\sim \cP$.
Let $Z = (z_{l-M}, \ldots, z_{l-1})$.
For any subset $\cS \subseteq [M]$, we use $Z_{-\cS} $ to denote the subvector of $Z$ containing entries of the form $z_{l-s}$, $\forall s \in \cS$. 
For instance, suppose $\cS = \{2,5\}$, then $Z_{-\cS} = (z_{l -5}, z_{l-2}) $. 
The modified $\chi^2$-mutual information for $\cS$ is defined~as 
\begin{align}\label{eq:chi_squared_MI}
   \tilde I_{\chi^2}(\cS) = \EE_{\pi\sim\cP, (z,Z) \sim \mu^\pi}\bigg[\bigg(\sum_{e\in\cX} \frac{ [\mu^\pi(z = e \given Z_{-\cS})]^2}{\mu^\pi(z=e)} - 1 \bigg) \cdot  \mu^\pi(Z_{-\cS})\bigg], 
\end{align}
where $\mu^\pi(z=\cdot\mid Z_{-\cS})$ is the conditional distribution of $z$ induced by $\mu^\pi$ given the partial history $ Z_{-\cS}$, and $\mu^\pi( Z_{-\cS}), \mu^\pi(z)$ are the marginal distributions of $ Z_{-\cS}$ and $z$ under~$(z,Z) \sim \mu^\pi$.
\end{definition}

Intuitively, $\tilde I_{\chi^2}(\cS)$ is modified from the vanilla $\chi^2$-mutual information ($\chi^2$-MI) between two random variables \citep{Polyanskiy_Wu_2024} and quantifies how much information the partial history $Z_{-\cS}$ contains about $z$.
In particular, we incorporate an additional $\mu^\pi(Z_{-\cS})$ term that decreases with the growing size of $\cS$.
To see the rationality, we first introduce a GIH estimator based on the modified $\chi^2$-mutual information.

\begin{definition}[Generalized Induction Head] \label{def:gih}
A GIH estimator with window size  $M \in \NN$, feature size $D \in \NN$ is denoted by $ \mathtt{GIH}(\cdot; M, D)$, which maps $x_{1:L}$ to a distribution over $\cX$.
We let $\cS^\star$ be the information-optimal subset (referred to as the ``information set'' in the sequel\footnote{With a slight abuse of notation, we also call $X_{l-\cS^\star}:=(x_{l-s}:s\in\cS^\star)$ the information set of the $l$-th token $x_l$.}) of  $[ M]$ with size no more than $D$ that maximizes the modified $\chi^2$-mutual information  $\tilde I_{\chi^2}(\cdot ) $  defined in \eqref{eq:chi_squared_MI}. 
That is, we define the information set $  S^*$ as
\begin{align} \label{eq:define_set_S_star}
    \cS^\star = {\textstyle \argmax_{\cS\in [M]_{\le D} }}  \tilde I_{\chi^2}(\cS).
\end{align}
Then $\mathtt{GIH}(x_{1:L}; M, D)$ outputs
\begin{align}
    y^\star \defeq \begin{cases}
        N^{-1} \cdot {\sum_{l=M+1}^L x_l \cdot \ind(X_{l-\cS^\star} = X_{L+1-\cS^\star}) }, ~\text{if}~ N \ge 1, \\
        (L-M)^{-1} \cdot \sum_{l=M+1}^L x_l, \quad \text{otherwise}.
    \end{cases}
    \label{eq:gih}
\end{align}
Here, we define $X_{l-\cS^\star}$ as the set $\{x_{l-s}: s\in \cS^\star\}$ and $N = \sum_{l=M+1}^L \ind(X_{l-\cS^\star} = X_{L+1-\cS^\star})$.
\end{definition}

Note that $\cS^\star$ defined in \eqref{eq:define_set_S_star} depends on the choices of $M$ and $D$ and serves as a proxy of the unknown parent set $\mathtt{pa}$ based on $\tilde I_{\chi^2}(\cdot)$ defined in \eqref{eq:chi_squared_MI}.
In a nutshell, the 
GIH estimator checks whether the partial histories of $X_{l-\cS^\star} $ and $X_{L+1-\cS^\star}$ match and aggregate all the tokens $x_l$ that have a matching partial history to predict $x_{L+1}$.
As a remark, using the modified $\chi^2$-MI as the information criterion rules out redundancy in the information set $\cS^\star$ in the following sense:
\vspace{2pt}

{\noindent \bf $\bullet$ $\cS^\star$ cannot be a superset of the true parents.} Note that if $\cS $ is a superset of the true parent set,  
    by the Markov property, $z$ and $Z_{-S}$ are conditionally independent given 
    the true parents   $Z_{\pa}$. 
    Thus, maximizing the vanilla $\chi^2$-mutual information yields multiple maximizers, i.e., all the supersets of the true parent set. 
    However, with the modification in \eqref{eq:chi_squared_MI}, 
    any superset yields a strictly smaller $\tilde I_{\chi^2}$ compared to the exact parent set, making them suboptimal.

\vspace{2pt}
{\noindent \bf $\bullet$ The modified $\chi^2$-MI selects informative partial history.} 
Even a true parent may bear relatively little information about the target compared to other parents sometimes. 
    Meanwhile, exact match of a larger set of partial history becomes much harder as it tends to appear less frequently in the context sequence, leading to poor estimation accuracy for the estimator in \eqref{eq:gih}. 
    The modified $\chi^2$-MI reaches a balance by selecting the informative partial history while penalizing the size of the information set.

\vspace{2pt}

The term involving $\mu^{\pi} (z =\cdot \given Z_{-\cS})$ can be viewed as the \emph{signal} part which helps us to find an informative subset $\cS$. 
The term $\mu^{\pi} (Z_{-\cS})$ can be viewed as \emph{penalty on the model complexity} which favors smaller subsets. 
Thus, 
the modified $\chi^2$-MI strikes a balance between these two objectives and enables us to find a good proxy $\cS^\star$ of $\mathtt{pa}$ when $L$ is finite.  
Moreover, when $L$ is sufficiently large, we identify two scenarios in which
maximizing $\tilde I_{\chi^2}(\cdot)$ yields the true parent set (see \cref{sec:discussion} for details).
Moreover, the GIH estimator is a generalization of the induction head mechanism \citep{elhage2021mathematical} to the stochastic setting with multiple parents, where we give the model more flexibility to learn based on a partial history that does not necessarily correspond to the true parent set.

\subsection{How Does Transformer Implement the GIH Mechanism?}
\label{sec:transformer&GIH}
In the following, we briefly illustrate how a two-attention-layer transformer model as introduced in \cref{eq:transformer} implements the GIH mechanism.
As we will show in \cref{sec:convergence}, gradient flow with respect to the cross-entropy loss converges to this transformer in the limit.  
\begin{figure}[h]
    \centering
    \includegraphics[width=0.9\textwidth]{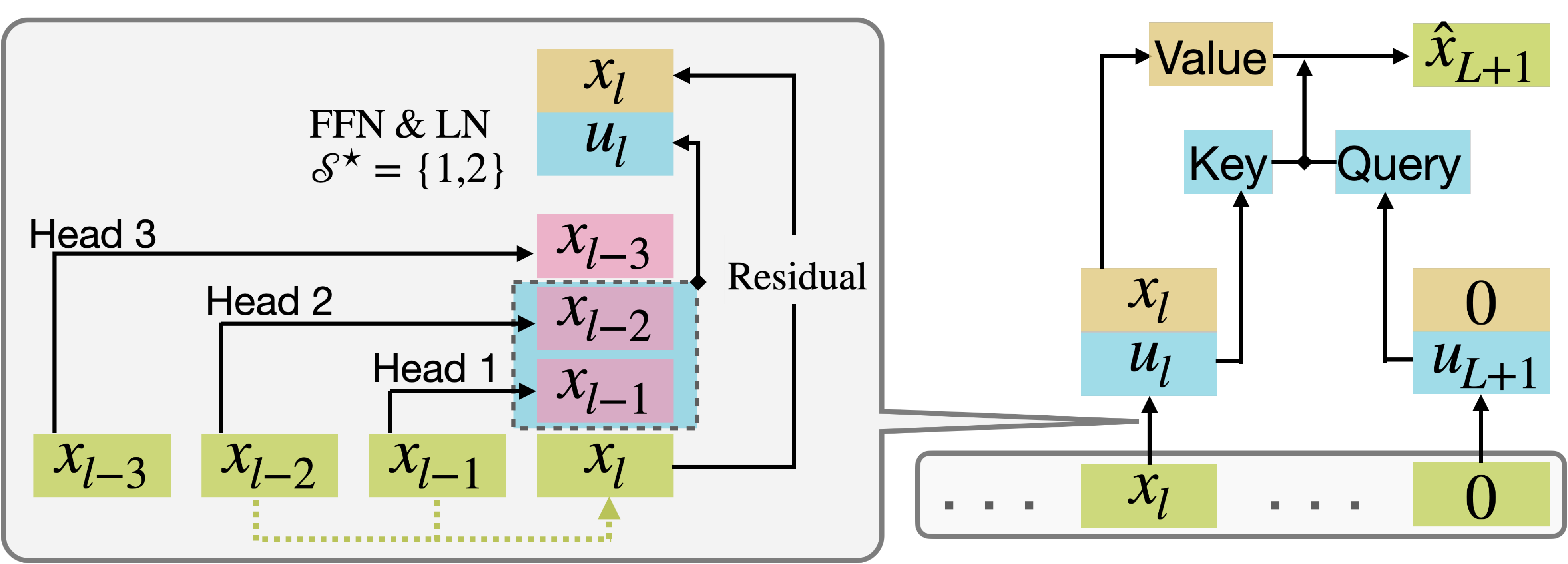}
    \caption{Illustration of the GIH mechanism in a two-attention-layer transformer model. 
    Here, $\pa=\{-1, -2\}$, $M=3$ and $\cS^\star=\{1, 2\}$.
    The first attention layer copies the parents (including the information set $\cS^\star$) to the current position. 
    Then the FFN layer together with layer normalization generates the features $u_l$ using the parent tokens within the information set $\cS^\star$.  
    The second attention layer treats each $x_l$ as the value, and aggregates $x_l$ as the prediction by matching the keys and query that come from the learned features using the attention mechanism. The $L+1$-th token is padded with zeros in the input.
    }
    \label{fig:GIH_illustration}
\end{figure}

\paragraph{Step \RNum{1}: The First Attention Layer Copies the Information Set $\cS^\star$ to the Current Position}
Suppose the number of heads is equal to the window  size for simplicity, i.e.,
$H=M$. 
Then, attention head $h\in\cS^\star$ can attend to the $h$-th parent token by setting the RPE weights in the softmax function to be $w^\h = \rho \cdot e_{-h}$ for a sufficiently large $\rho$, where $e_{-h}\in \RR^{M}$ is the canonical basis vector with the $(M+1 - h)$-th entry being one and all other entries being zero.
As a result, each $v_l^\h$ for $h\in\cS^\star$ satisfies
\(
    v_l^\h \approx x_{l - h}.
\)

\paragraph{Step \RNum{2}: FFN Generates the Polynomial Features of the Information Set $\cS^\star$}
As we have introduced in \eqref{eq:feed-forward}, each learnable $c_{\cS}$ in the FFN layer determines the contribution of the corresponding subset $\cS$ to the output feature.
To let the optimal information set $\cS^\star$ dominate the output, we set $c_{\cS^\star}=1$ whereas $c_{\cS}=0$ for all $\cS\neq \cS^\star$.
The exact form of the output of the FFN layer, $\phi(v_l)$,  is deferred to \cref{sec:ffn_polynomial}. Here  the only property we require is that 
\begin{align}
    s_l\defeq \langle \phi(v_l), \phi(v_{L+1}) \rangle  = \prod_{h\in\cS^\star} \langle v_l^\head{h}, v_{L+1}^\head{h} \rangle \approx \ind(X_{l-\cS^\star} = X_{L+1-\cS^\star}), 
    \label{eq:matching-infoset}
\end{align}
Here $X_{l-\cS^\star}:=(x_{l-s}: s\in \cS^\star)$ and in the last equation we use the orthogonality and normalization of the vocabulary embeddings.

\paragraph{Step \RNum{3}: The Second Attention Layer Aggregates Tokens with Matching History on $\cS^\star$}
We can interpret $s_l$ in \eqref{eq:matching-infoset} as an indicator for whether the information set of a token $x_l$ matches the information set of the token $x_{L+1}$. 
Then for the second attention layer, by setting $a$ to be sufficiently large, the output will become 
\begin{align} \label{eq:stage3_expression}
    y = \sum_{l=M+1}^L \frac{\exp\bigl (a \cdot s_l)\cdot x_l}{\sum_{k=M+1}^L \exp\bigl (a \cdot s_k  \bigl)} \approx \begin{cases}
        N^{-1} \cdot {\sum_{l=M+1}^L x_l \cdot \ind(X_{l-\cS^\star} = X_{L+1-\cS^\star}) }, ~\text{if}~ N \ge 1, \\
        (L-M)^{-1} \cdot \sum_{l=M+1}^L x_l, \quad \text{otherwise}, 
    \end{cases}
\end{align}
where $N = \sum_{l=M+1}^L \ind(X_{l-\cS^\star} = X_{L+1-\cS^\star})$.
That is, if at least one token $x_{l}$ has a matching information set as $x_{L+1}$, i.e., their histories restricted to $\cS^\star$ are the same, 
 the second attention layer outputs the average of such tokens. 
 Otherwise, it outputs the average of previous tokens from $x_{M+1}$ to $x_{L}$. 
In \Cref{lem:gih_approx} in the appendix, we will show that the model learned by gradient flow implements the GIH mechanism up to a diminishing approximation error. 

The weights of the transformer constructed above are illustrated in \Cref{fig:convergence}. 
We consider the transformer model with $M = H = 3$, $d = 3$, and $D = 2$. 
In this case,
in the first attention layer, for each $h \in [3]$,  $W^{(h)}_P$ has three finite parameters $w_{-1}^{(h)}, w_{-2}^{(h)}$, and $w_{-3}^{(h)}$.
By our construction, we have $w_{-h}^{(h)} = \rho$ for all $h \in [3]$ and the rest of the entries of $\{ w^{(h)} \}_{h\in [3]}$ are all equal to zero.  In \Cref{fig:convergence}-(a) we plot the top ten  by ten block of $W_{P}^{(1)}$, where $w^{(1)}_{-1} = \rho $ is shown in yellow and $w^{(1)}_{-2} = w^{(1)}_{-3}$ are shown in purple. 
The gray color stands for $-\infty$ entries. 
In \Cref{fig:convergence}-(b) we  plot $\{ w^{(h)} \}_{h\in [3]}$. 
In \Cref{fig:convergence}-(c) we plot the parameters of the FFN. 
Since $H = 3$ and  $D = 2$, $\HleqD$ contains seven elements: $\emptyset$, $\{1\}, \{2\}, \{3\}$, $\{1,2\}$, $\{1,3\}$, and $\{2,3\}$. 
We use binary strings of length $3$ to index these seven subsets, where the $i$-th bit indicates whether element $i$ is included in the subset. For instance,  ``110'' represents $\{ 1, 2\}$.  We set $\cS ^\star 
 = \{1,2\}$, $c_{\cS ^\star} = 1$, and $c_{\cS} = 0$ for any other $\cS$.

\begin{figure}[h]
    \centering
    \includegraphics[width=0.95\textwidth]{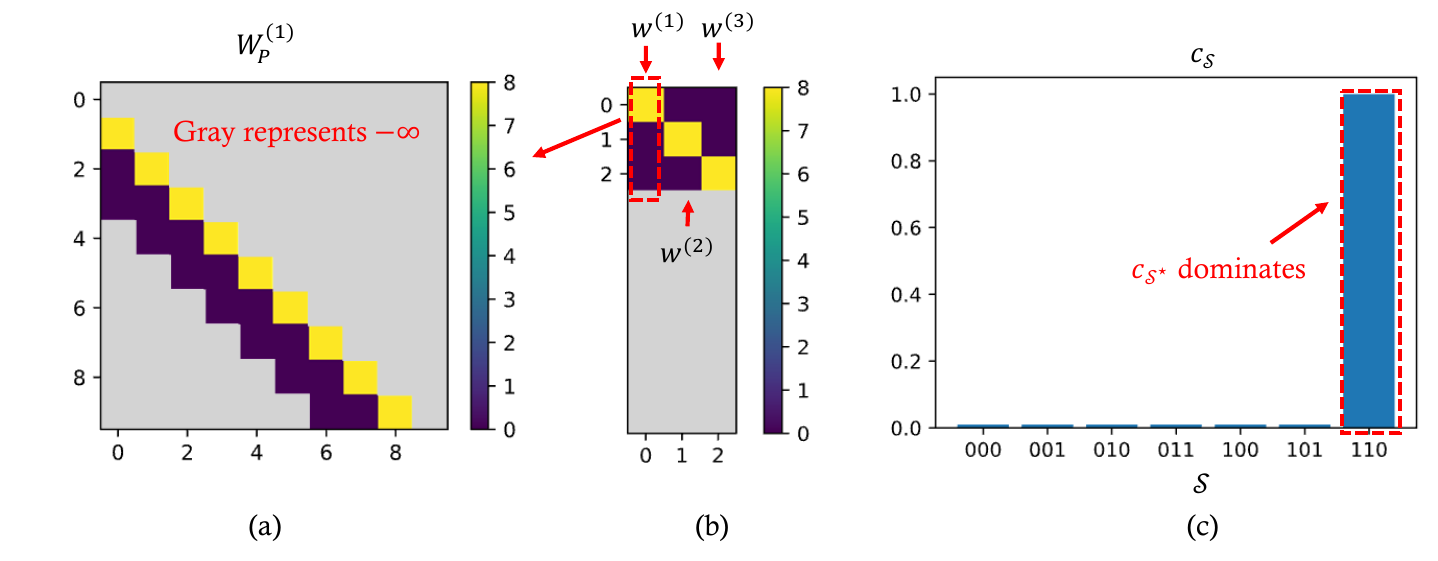}
    \caption{Limiting model of \(\mathtt{TF}(M=3,H=3, d=3, D=2)\) that implements the GIH mechanism with $L=100$, $\pa=\{-1, -2\}$. (a): The top left $10$ by $10$ block of $W_P^{(1)}$ that attends to the $-1$ parent. (b): The RPE weight heatmap for all 3 heads, where the $h$-th column corresponds to the RPE weight vector of head $h$.
     (c): In the GIH mechanism, only one $c_\cS^\star$ for the optimal information set $\cS^\star$ dominates.
    For the label of the $x$-axis, we use a binary coding $\{0, 1\}^3$ to indicate each subset $\cS$.
    Here, $\cS^\star = \{1,2\} $ is the parent set, which is represented by ``110''. 
    }
    \label{fig:convergence}
\end{figure}

\subsection{Convergence Guarantee of Gradient Flow} \label{sec:convergence}



In the following, we present the convergence guarantee for gradient flow. 
To simplify the discussion, we consider the case where $H = M$, meaning there are enough heads to implement the GIH mechanism by having each head copy a unique parent token from a window of size $M$.
Let us first introduce the paradigm of training by gradient flow. 

\paragraph{Training Paradigm}
Consider training a transformer $\mathtt{TF}(M, H, d, D)$ in \eqref{eq:transformer} with $M=H$ to perform ICL on the $n$-gram Markov chain model introduced in \Cref{sec:icl_mc}. 
Specifically, we define $\cL(\Theta) $ as the population cross-entropy loss in \eqref{eq:cross_entropy}, where the transformer model $f_{\mathtt{tf}}$ is given by \eqref{eq:transformer} with a parameter $\Theta$. 
Ideally, when training the parameter $\Theta$ with gradient flow, the dynamics with respect to the loss $\cL(\Theta)$ is given by:
\begin{align}\label{eq:grad_flow}
    \partial_t \Theta(t)  = - \nabla \cL\big(\Theta (t)  \big) .
\end{align}
We consider a three-stage training paradigm where, in each stage, only a specific subset of the weights is trained by gradient flow. The three stages are outlined in \cref{tab:training_stages}.
Specifically, in the first stage, we only train the  FFN layer via gradient flow while keeping other weights fixed. 
We then only train the RPE weights in the first attention layer in the second stage. Finally, we only train the weight $a$ in the second attention layer in the last stage, while fixing the rest of the parameters. 
This training approach is primarily used for analytical convenience; in practice, the entire model can be trained simultaneously, and similar convergence results are reported in \cref{sec:exp_supp}.
From a theoretical standpoint, we will also justify the three-stage paradigm in the discussion following \cref{thm:convergence}.

\begin{table}[!htb]
    \centering
    \begin{tabularx}{0.9\textwidth}{
        >{\centering\arraybackslash\hsize=0.2\hsize}X
        >{\hsize=0.9\hsize}X
        >{\hsize=1.2\hsize}X}
        \hline \hline
        Stage & Weights to Train & Description \\
        \hline 
         \RNum{1} & $\{c_\cS\}_{\cS\in\HleqD}$ in the FFN layer & Ratio $c_{\cS^\star}(t)/c_{\cS}(t)$ grows exponentially, \newline learning the low-degree features with $\cS^\star$, 
         \\
        \hline
         \RNum{2} & $\{w^\h\}_{h\in[H]}$ in the RPE of the first attention layer, & $1 - \prod_{h\in \cS^\star} (\sigma_{-h}^\h(t))^2$ decays polynomially \newline training each head in $\cS^\star$ to be a copier, 
         \\
        \hline
         \RNum{3} & $a$ in the weight of the second attention layer & 
        $a(t)$ experiences a two-stage growth, \newline learning the softmax aggregator for GIH, 
        \\
        \hline\hline \\
    \end{tabularx}
    \caption{Three-stage training paradigm for gradient flow. Here, the ``Weights to Train'' column indicates the weights updated in each stage, and the ``Description'' column summarizes the corresponding results from \cref{thm:convergence}.}
    \label{tab:training_stages}
\end{table}

\paragraph{Initialization Conditions}
Before presenting our main results about how training by gradient flow induces the GIH structure, let us introduce the following assumption on the initialization of the weights. We define the \emph{information gap} within the $D$-degree parent set $\HleqD$ as 
\begin{align}
    \Delta \tilde I_{\chi^2} = \tilde I_{\chi^2}(\cS^\star) -  \max_{S\in\HleqD \backslash \{\cS^\star\}} \tilde I_{\chi^2}(\cS), 
    \label{eq:Delta_I_chi2}
\end{align}
where we recall that $\cS^\star$ defined in \eqref{eq:define_set_S_star} maximizes the modified $\chi^2$ mutual information.
\begin{assumption}[Initialization]
    \label{asp:initialization}
    We assume that the following holds at initialization: 
    \begin{enumerate}
        \item For the first attention layer's RPE weights, $w_{-h}^\h \ge w_{-j}^\h + \Delta w$ for all $h,j\in[H]$ with $j\neq h$, where $\Delta w  > 0$ is a positive scalar satisfying
        \begin{align} \label{eq:def_delta_w}
            {\Delta w \geq \log ({M-1}) - \log \Big[\Big(1 + \Delta \tilde I_{\chi^2}/(14 \tilde I_{\chi^2}(\cS^\star))\Big]^{\frac{1}{2H}} - 1\Big)}. 
        \end{align}
        \item The scalar parameter $a$ in the second attention layer satisfies $0 < a \le O(L^{-3/2})$.
    \end{enumerate}
\end{assumption}
The first assumption on the RPE is used to induce the correspondence between parents and heads during the training by slightly breaking the symmetry between different attention heads. 
The second assumption on the scale of $a$ ensures that the attention probability given by the second attention layer is close to the uniform distribution over $[L]$.
These initialization conditions enable us to derive clean descriptions for the dynamics of the first attention layer and the FFN, shedding light on their respective roles in executing ICL.

We now outline our assumptions on the Markov chain used in the data generation process. Recall that \( r_n \) is the largest absolute integer in the parent set \( \mathtt{pa} \). For any position \( l \), we define the history \( Z = (z_{l-r_n}, \dots, z_{l-1}) \) as the last state and \( Z' = (z_{l-r_n+1}, \dots, z_{l}) \) as the current state. Since the parent of the new token \( z_l \) is already included in \( Z \), \( Z' \) is independent of all prior history given \( Z \), forming a Markov chain.

We define \( P_\pi \) as the \( d^{r_n} \times d^{r_n} \) transition matrix for this Markov chain, where states are successive \( r_n \)-tokens. Each row of \( P_\pi \) is indexed by \( Z' \) and each column by \( Z \). The matrix element \( P_\pi(Z', Z) \) is thus given by 
\[
P_\pi(Z', Z) = \pi(z'_{l} \mid Z_{\mathtt{pa}(l)}) \cdot \ind (Z'_{l-r_n + 1:-1} = Z_{l-r_n+1:-1}).
\]
This means that to transition from \( Z \) to \( Z' \), all elements of \( Z' \) except for \( z'_{-1} \) must match the last \( r_n-1 \) tokens of \( Z \). The token \( z'_{l} \) is then sampled according to the transition kernel \( \pi \) and depends only on the parent \( Z_{\pa(l)} \).
The above definition is in fact independent of the position \( l \) as the transition kernel \( \pi \) is the same across all positions.
Note that $P_\pi$ is also a stochastic matrix but with zero entries due to the indicator.
To proceed, we need the following notion of primitive matrix to state our assumption on $P_\pi$.
\begin{definition}[Primitive Matrix]\label{def:primitive}
    A nonnegative and irreducible square matrix $P$ is called primitive if there exists a positive integer $k$ such that all entries of $P^k$ are positive.
\end{definition}
We defer more details about the above definition to \cref{sec:background_perron_frobenius}.
By the celebrated Perron-Frobenius theorem, if a stochastic matrix $P_\pi$ is also primitive, then (i) there exists a unique stationary distribution for the Markov chain; (ii) $P_\pi$ has a unique leading eigenvalue equal to $1$, and the corresponding eigenvector is the stationary distribution.
Next, we state the assumptions on the mixture of Markov chains for data generation.
\begin{assumption}[Markov Chain]
    \label{asp:Markov_chain}
    For any $\pi\in\supp(\cP)$, we assume that:
    \begin{enumerate}
        \item The transition matrix $P_\pi$ is {primitive}. 
        In particular, we assume that there exists $\lambda < 1$ such that the eigenvalue of $P_\pi$ with the second largest magnitude satisfies $|\lambda_2(P_\pi)| \le \lambda$. Note that $\lambda_2(P_\pi)$ can be complex-valued.
        \item There exists $\gamma>0$ such that the transition kernel satisfies $\pi(x\given X_{\pa})\geq \gamma$ for any $(x, X_{\pa})$. 
    \end{enumerate}
\end{assumption}
In fact, the second condition $\pi(\cdot \given X_{\pa}) >\gamma$ already ensures that $P_\pi$ must be primitive, as is required by the first condition.  See \Cref{cor:primitive} for details.
On the high level, the first assumption guarantees a unique stationary distribution as well as a fast mixing rate of the Markov chain by ensuring a spectral gap for $P_\pi$.
The second assumption implies a lower bound on the probability for any set $\cS\subseteq[M]$ under the stationary distribution, i.e., $\mu^\pi(X_{l-\cS})\ge \gamma^{|\cS|}$ for any $l > M$. See \Cref{cor:station_lb} for details.

Now we are ready to present our main theoretical result on training transformers by gradient flow.

\begin{theorem}[Convergence of Gradient Flow]
    \label{thm:convergence}
    Suppose \cref{asp:initialization} and \cref{asp:Markov_chain} hold. 
    Consider $H\ge M$.
    We set $\varepsilon = L^{-1/2}$ for the cross-entropy loss and assume $L$ is sufficiently large. 
    Then the following holds for the three-stage training of gradient flow:
    \begin{description}[leftmargin=0.3in, labelindent=0in, labelsep=0in]
        \item[{\color{OrangeRed} Stage \RNum{1}: Parent Selection by FFN.} ]
        Let $C_D(t) = \sum_{\cS\in\HleqD} c_{\cS}(t)^2$ and $p_{\cS^\star}(t)  = c_{\cS^\star}^2(t)/ C_D(t)$.
        Then in the first stage with duration $t_1 \asymp {C_D(0)\log L /(a(0)\Delta \tilde I_{\chi^2})}$, the ratio $c_{\cS^\star}/c_{\cS}$ grows exponentially fast for any $\cS\neq \cS^\star$, and $\cS^\star$ dominates exponentially fast in the sense that, 
        \begin{align}
            1 - p_{\cS^\star}(t) \leq (1 - p_{\cS^\star}(0)) \cdot \exp\bigl(-(2C_D)^{-1} \cdot  a(0) \cdot p_{\cS^\star}(0) \cdot  \Delta \tilde I_{\chi^2} \cdot t \bigr),\quad \forall t\in [0, t_1).
        \end{align}
        \item[{\color{OrangeRed}Stage \RNum{2}: Concentration of The First Attention.} ]  
        Define $\sigma^\h(t) = \sigma(w^\h(t))\in\RR^M$, and let $\sigma_{\min}(t) := \min_{h \in \cS^\star} \sigma_{-h}^\h(t)$.
        Then in the second stage with duration $t_2-t_1 \asymp L /(a(0)\Delta \tilde I_{\chi^2})$, the first layer's attention heads have attention probabilities concentrated on the optimal information set $\cS^\star$ in the sense that for any $t\in[t_1, t_1+t_2)$,
        \begin{align}
             1 - \prod_{h\in \cS^\star} (\sigma_{-h}^\h(t))^2 \leq \frac{2|\cS^\star|\cdot (M-1)}{a(0)\cdot \Delta \tilde I_{\chi^2}\cdot \sigma_{\min}(0) \cdot (t - t_1)/2 + \exp(\Delta w)+(M-1)} \land 1.
        \end{align}
        \item[{\color{OrangeRed}Stage \RNum{3}: Growth of The Second Attention.} ] 
        For some constants $c_1,c_2$ depending on $(\cP, \cS^\star)$ with $0<c_1<c_2$,
        there exists a small constant $\delta>0$ such that the growth of $a(t)$ exhibits the following two sub-stages: (i) When $a(t) \leq \log(c_1/\delta)$,
            it holds that $\partial a(t) \asymp e^{a(t)}$;
        (ii) After $a(t)$ has grown such that $a(t)\geq \log(c_2/\delta)$, then $\partial_t a(t) \asymp 1/a(t)$
            until it reaches the value $\log L/8$.
    \end{description}
\end{theorem}
See \cref{sec:proof_sketch} for a proof sketch and \cref{sec:dynamics proof} for the detailed proof. We require that $L$ is sufficiently large, and the specific conditions for $L$ are deferred to \Cref{sec:L-condition}.

\vspace{5pt}
{\bf \noindent Interpretation of Training Dynamics.} We empirically verify \cref{thm:convergence} by conducting a simulation experiment. In particular, we train a transformer with $H=M=3$ and $D=2$ based on Markov chain data with $d = 2$, $L=100$ and $\pa = \{-1, -2\}$. 
We sample the transition kernel from a Dirichlet prior such that $\cS^\star = \{ 1, 2\}$ also matches the parent set. 
For more details on this simulation, see \Cref{sec:experiments}.
The results are shown in \Cref{fig:train_C_W_a} and align perfectly with \cref{thm:convergence}. 
From \cref{thm:convergence}, we can interpret the three stages of training dynamics as follows. 
\begin{enumerate}
    \item[$\bullet$] In the first stage, the training of FFN parameters learns a \textbf{\color{OrangeRed} \emph{selector}} that selects an informative set $\cS^\star$ by realizing the corresponding feature embedding through the polynomial kernel. 
    That is, when $t$ is sufficiently large, we have $p_{\cS^\star }(t) \approx 1$ and $p_{\cS} (t) \approx 0$ for all $\cS \neq \cS^\star$. 
    In this case, for any input vectors $v, v' \in \RR^{Hd}$, the inner product in \eqref{eq:feed-forward} reduces to 
    $$\langle \phi( v), \phi(v') \rangle \approx c_{\cS^\star} ^2 \cdot  \prod_{h \in \cS^\star } \langle v^{(h)}, v'^{(h)}\rangle.  $$
    That is, FFN only selects the blocks in $\cS^\star $  as the feature. We observe this phenomenon in the experiment, where we set   $\cS^\star=\{1, 2\}$.
    As shown in  \Cref{fig:train_C_W_a}-(a), it is clear that $c_{\cS^\star}$ immediately dominates the rest of $c_{\cS}$'s within only a few gradient epochs.
    
    \item[$\bullet$] 
    In the second stage, we update the parameters of the RPE. 
    This stage turns the first attention layer into a \textbf{\color{OrangeRed}\emph{copier}} by establishing the correspondence between the attention heads and the parents in the selected $\cS^\star$. 
    That is, each attention head copies a particular parent in $\cS^\star$. 
    Specifically, when $t $ is sufficiently large, 
    for any $h\in \cS^\star$, $\sigma^{(h)} (t) = \sigma (w^{(h)} (t) )   \approx 1$. 
    Recalling the construction of RPE, this implies that $v_l ^{(h)} $ in \eqref{eq:define_V_h} becomes $x_{l - h}$ for all $h\in \cS^\star$. 
As shown in  \cref{fig:train_C_W_a}-(b), in the experiment, the first two heads initialized towards the first two parents will deterministically copy parents $-1$ and $-2$ eventually. The third head stays close to its initial value. 
This head has a negligible effect on the output because  
$3\notin\cS^\star$ and $p_{\cS^\star} \approx 1$. 
    \item[$\bullet$] 
    After the first two stages are completed, 
we know that the features constructed approximately satisfy 
\eqref{eq:matching-infoset}  up to a proportionality factor. 
Then, in the final training stage, the scalar weight $a$ in the second attention layer keeps increasing. 
Thus, this stage learns an exponential kernel \textbf{\color{OrangeRed}\emph{classifier}} as specified in \eqref{eq:stage3_expression}. 
When $a(t)$ is sufficiently large, the learned transformer is close to a classifier that uses covariate-label pairs of the form $(X_{l -\cS^\star}, x_{l})$ to predict $x_{L+1}$. 
In particular, when $a(t)$ goes to infinity, the transformer exactly becomes the GIH mechanism given in 
\cref{def:gih}. 
Moreover, we theoretically prove that the increasing trajectory of $a(t)$ has two stages, where $\mathrm{d} a(t)/\mathrm{d} t$ is initially large and gradually decays, this is also clearly observed in the experiment.  See \cref{fig:train_C_W_a}-s(c) for details. 

\end{enumerate} 

In summary, we theoretically show that the limiting model obtained by three-stage training approximately implements 
 the GIH mechanism.
 We will prove that the difference between these two estimators is at most 
 $O(L^{-1/8})$.  
We defer the formal statement and proof to \cref{sec:proof_of_gih_approx}.
Moreover, as an answer to the Question {\color{OrangeRed}(iii)} raised in \Cref{sec:intro}, the different components of the transformer architecture are all critical for achieving this:
FFN with normalization realizes the \textbf{\color{OrangeRed} \emph{selector}}, the multi-head design of attention supports the \textbf{\color{OrangeRed}\emph{copier}}, and finally, the softmax operation facilitates the exponential kernel \textbf{\color{OrangeRed}\emph{classifier}}.
These components work organically as a whole system, yielding the trained transformer's capability of ICL of $n$-gram Markov chains.

Another takeaway from \cref{thm:convergence} is a strict separation in the growth rate of these three stages. 
In particular, the convergence rates of the corresponding components of the transformer model in these three stages range from exponentially fast (Stage I), polynomially fast (Stage II), to logarithmically slow (Stage III).
With such two exponential separations of convergence rates, we expect that these three stages naturally arise when we simultaneously train the whole model via gradient descent/flow. 
We empirically verify this argument and the details are deferred to \cref{sec:exp_supp}. 


\begin{figure}[tb!]
    \centering
    \includegraphics[width=0.95\textwidth]{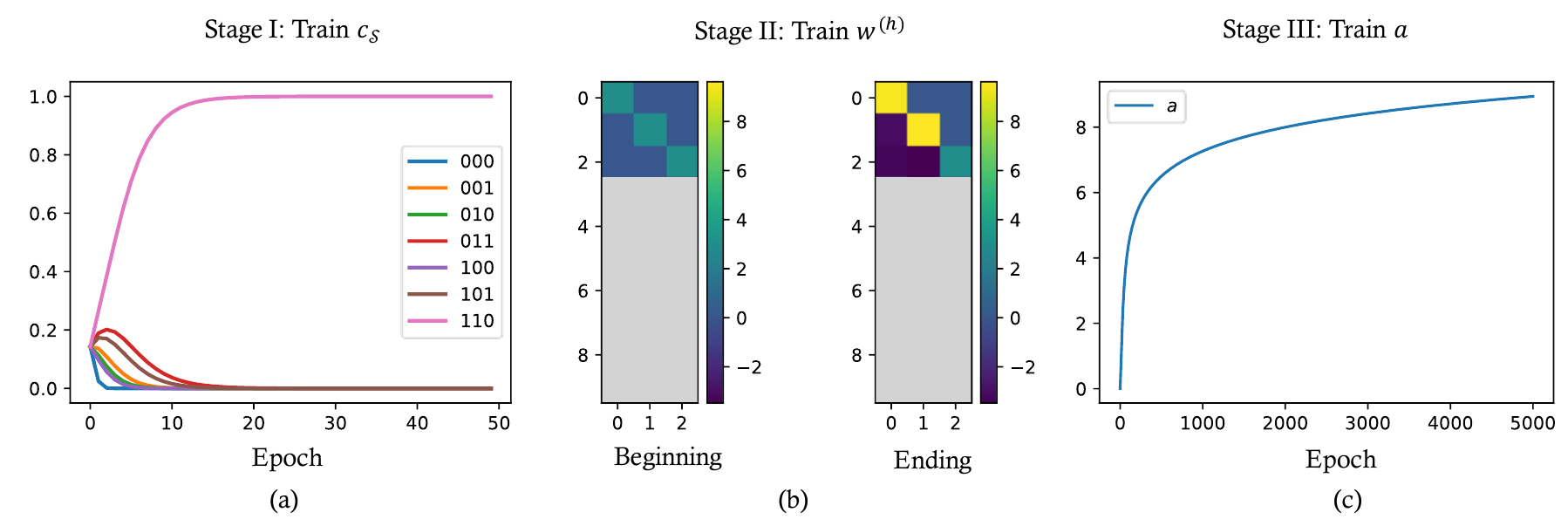}
    \caption{An illustration of the transformer parameters during the three-stage training.  We train a transformer in  \(\mathtt{TF}(M=3,H=3, d=3, D=2)\)  with $L=100$, $\pa=\{-1, -2\}$. See \Cref{sec:experiments} for more details of the simulation.  
    In (a) we show the evolution of $\{p_\cS\}_{\cS \in \HleqD }$ in the first stage of training where $p_{\cS} = c_{\cS}^2/ \sum_{\cS'\in\HleqD} c_{\cS'}^2$.
    We use a binary coding in $\{0, 1\}^3$ to indicate each subset $\cS$. 
    Recall that ``110'' represents $ = \{1,2\}$, which is exactly $\cS^\star$. This figure shows that 
   $p_{\cS^\star}$ gradually increases to one while the any other $p_{\cS}$ decays to zero.   
    In (b) we plot the  RPE weights of the first attention layer before and after the second stage of training. 
    Here the $h$-th column corresponds to the RPE weight vector of head $h$.  
    This figure shows that  $w^{(1)} _{-1} $  and $w^{(2)} _{-2} $  increase to a large number after training, while $w^{(3)} _{-3} $ stays close to its initial value. 
    Thus, we have $\sigma (w^{(1)}) \approx \sigma(w^{(2)}) \approx 1 $.  
    That is, the first two heads are trained to attend to parents $-1$ and $-2$, respectively.
In (c) we plot the 
evolution of $a$ in the last stage of training.   This figure clearly exhibits a two-step growth pattern and $a$ keeps increasing throughout this stage. In summary, the results of the simulation experiments coincide with the theoretical results. }
    \label{fig:train_C_W_a}
\end{figure}


\subsection{Further Discussions on the GIH Mechanism} \label{sec:discussion}
We conclude this section with further discussions on the modified $\chi^2$-mutual information and low-degree polynomial kernel for the FFN within the GIH mechanism.

\paragraph{On the Modified $\chi^2$-Mutual Information}
Now that we have shown how gradient flow approaches the desired GIH model, it is natural to ask the following questions: What is the optimal subset $\cS^\star$ that the model selects? How well does the model perform? 
For the purpose of illustration, let us consider a symmetric case where the stationary distribution $\mu^\pi$ over a length-$r_n$ window is uniform over $\cX^{r_n}$. 
One can verify that in this case, the stationary distribution over a window of any other length is uniform as well, and the modified mutual information can be simplified into
\begin{align}
    \log \tilde I_{\chi^2}(\cS) = \log I_{\chi^2}(\cS) - |\cS| \log {d}, 
        \label{eq:MI-symmetry}
\end{align}
where $I_{\chi^2}(\cS)$ is the standard $\chi^2$ mutual information between $\mu^\pi(z\given Z_{-\cS})$ and $\mu^\pi(z)$, and the second term $|\cS| \log {d}$ serves as a penalty on the \textit{model complexity}.
Thus, the GIH mechanism is \emph{reaching a balance between the model complexity and the information richness}.
Below we characterize two scenarios where the model will select the exact parent set, i.e., $\cS^\star=\texttt{pa}$.
\begin{enumerate}[
    leftmargin=0.3in,nolistsep
]
    \item If $n=1$, i.e., each token only has one parent, then $\cS^\star=\texttt{pa}$. 
    This is because $\cS^\star$ simultaneously maximizes both terms in \eqref{eq:MI-symmetry}, thus reproducing the results in \cite{nichani2024transformers}. 
    \item If $n$ is known a priori and restricting the polynomial kernel to $\cS\in[H]_{=n}=\{\cS\in[H]: |\cS|=n\}$ for the FFN layer, then $\cS^\star=\texttt{pa}$.
    Here, the penalty term does not influence the selection and the exact parent set maximizes the mutual information by the data-processing inequality.
\end{enumerate}
In the general case, however, the model could be much more flexible, and it is possible that the model selects only a subset of the true parent set or even some non-parent tokens that are also informative.
The rationale is that with a more complex model, e.g., selecting a large $\cS$, the model are able to make more accurate predictions for large $L$ but may endure a large estimation error for small $L$, as the exact matching $X_{l-\cS} = X_{L+1-\cS}$ may appear rarely in the sequence. 

\paragraph{On the Low-Degree Polynomial Kernel}
The goal of using a low-degree polynomial kernel in \eqref{eq:feed-forward} is to strike a balance between model complexity (which is also related to computational cost) and the model's accuracy.
In this regard, we have the following corollary. 
\begin{corollary}
    We always have
    $|\cS^\star| \le n$ regardless of the choice of  $D$, where $\cS^\star = \argmax_{\HleqD} \log \tilde I_{\chi^2}(\cS)$ for $\tilde I_{\chi^2}(\cS)$ in \eqref{eq:MI-symmetry}
\end{corollary}
The reasoning behind this corollary is as follows. 
Consider any set $\cS$ with $|\cS| > n$, 
we have 
 $I_{\chi^2}(\cS) \leq I_{\chi^2}(\mathtt{pa})$ as the true parent set is the most informative.  
Moreover, since $|\mathtt{pa}| = n < |\cS|$, $\cS$ suffers from a larger penalty. 
As a result, we have 
   $\log \tilde I_{\chi^2}(\cS) < \log \tilde I_{\chi^2}(\pa)$ when $\cS $ has more than $n$ elements.
In other words, it is without loss of generality to set $ D \le n$.


\section{Proof Sketch}\label{sec:proof_sketch}
In this section, we discuss the main ingredients of analysis of gradient flow.
First,  we show in \Cref{sec:proof_sketch_simplification} how to simplify the model based on our choice of the initialization and the structure of the disentangled transformer.
We then proceed to present the main proof ideas for the three stages of the gradient flow dynamics, where the training yields the following behaviors:
\begin{enumerate}
\item [$\bullet$] {\color{OrangeRed} Stage I}: A unique $\cS^\star \in \HleqD$ stands out such that the associated parameter $c_{\cS^\star}$ dominates those of the other sets. As a result,  $p_{\cS}^* (t) = c_{\cS^*}^2 (t)/ C_{D} (t) $ approaches to one. 

\item[$\bullet$] {\color{OrangeRed} Stage II}: For each $h \in \cS^\star$, $\sigma(w^\h)$ approaches a one-hot vector $e_{M+1-h}\in\RR^M$, where $w^\h$ contains the parameters of RPE of the $h$-th head.
During this stage, each head concentrates on copying a particular parent. 

\item[$\bullet$] {\color{OrangeRed} Stage III}: Finally, $a$ grows and reaches $\cO(\log L)$. As a result, the trained model approximately implements the GIH mechanism $\mathtt{GIH}(x_{1:L}; M, D)$. 
\end{enumerate}

\subsection{Simplification of the Transformer Model at Initialization}\label{sec:proof_sketch_simplification}

We first simplify the expression of the transformer model at initialization under \Cref{asp:initialization}, by showing that the attention scores of the second attention layer admit a simpler form.

For the second attention layer, we write the output as $y^\top = \sigma(as) X$ where $s := u_{L+1}^\top\mathtt{Mask}(U_{1:L}^\top) \in\RR^{1\times L}$ is the row vector of the similarity scores.
Recall from \eqref{eq:transformer} that the FFN layer with normalization outputs $U = \phi(V)/\sqrt{C_D}\in \RR^{(L+1)\times d_e}$, and we denote the $l$-th row of $U$ by $u_l = \phi(v_l) / \sqrt{C_D}$.
For $l=M+1,\ldots,L$, the $l$-th entry of $s$ is given by $$s_l = \langle u_l, u_{L+1}\rangle = \langle \phi(v_l), \phi(v_{L+1})\rangle / C_D,$$ and the other entries are all $-\infty$.
By the property of the FFN layer in \eqref{eq:feed-forward} and the definition $C_D=\sum_{\cS\in\HleqD} c_\cS^2 $, we can rewrite the above attention score as
\begin{align}\label{eq:simplify_s_l}
    s_l = \frac{\sum_{\cS\in \HleqD} c_{\cS}^2 \cdot \prod_{h\in \cS} \langle v_l^\head{h}, v_{L+1}^\head{h} \rangle}{\sum_{\cS\in\HleqD} c_\cS^2}, \quad \text{for } l=M+1,\ldots,L.
\end{align}
Note that under \Cref{asp:initialization},  by the definition of $\Delta w$ in \eqref{eq:def_delta_w}, we have a sufficiently large gap $w_{-h}^\h - w_{-j}^\h$ for all $j\neq h$ at initialization.
Thus, $\exp(w_{-h}^\h)\gg \exp(w_{-j}^\h)$ for all $j\neq h$, which implies the following approximation:
\begin{align}\label{eq:approx_first_layer}
    v_l^\h = \sum_{k=1}^{M} \frac{\exp(w_{-k}^\h)}{\sum_{j=1}^{M} \exp(w_{-j}^\h)} \cdot x_{l-k} \approx x_{l-h}, \quad \text{for } l=M+1,\ldots,L.
\end{align}
This further implies that for $l=M+1,\ldots,L$, we have 
\begin{align}\label{eq:approx_subset}
    \prod_{h\in \cS} \langle v_l^\head{h}, v_{L+1}^\head{h} \rangle \approx \prod_{h\in \cS} \langle x_{l-h}, x_{L+1-h} \rangle = \ind \{ x_{l-i} = x_{L+1-i} \text{ for } i\in \cS\},
\end{align}
which is a binary value indicating whether the query and the key token's history match on the subset $\cS$.
Combining \eqref{eq:simplify_s_l} and \eqref{eq:approx_subset}, we obtain the following simplified expression for $s_l$:
\begin{align}
    s_l \approx \frac{\sum_{\cS\in\HleqD} c_\cS^2 \cdot \ind \{ x_{l-i} = x_{L+1-i} \text{ for } i\in \cS\} }{\sum_{\cS\in\HleqD} c_\cS^2}
    = \sum_{\cS\in\HleqD} p_\cS \cdot \ind \{ x_{l-i} = x_{L+1-i} \text{ for } i\in \cS\}
\end{align}
where we denote $p_\cS = c_\cS^2 / \sum_{\cS\in\HleqD} c_\cS^2$ for $\cS\in\HleqD$.

In summary, when $\Delta w$ is sufficiently large, $v_{l}^\h$ approximately copies the token $x_{l-h}$. 
As a result, the attention score $s_l$ satisfies 
\begin{align} \label{eq:approx_attn_score}
    s_l \approx \sum_{\cS\in\HleqD} p_\cS \cdot \ind \{ x_{l-i} = x_{L+1-i} \text{ for } i\in \cS\}.
\end{align}

\subsection{Analysis for Training the FFN and the First Attention Layer}
The first two training stages involve the dynamics of the weights of the FFN, 
$\{c_\cS\}_{\cS\in\HleqD}$, 
and the weights of the first attention layer, $\{w^\h\}_{h=1}^H$.
The analyses of these two stages have similar structures and contain the following essential steps:
\begin{enumerate}
    \item Derive the explicit expression of the dynamics of the weights, via direct calculations.
    \item Unveil the key quantities (related to the modified $\chi^2$-MI) that dominantly drive the dynamics, by replacing the empirical average over the context sequence with the expectation over the stationary distribution, along with other approximations.
    \item Then based on the above characterization of the dynamics, we can show the convergence of the weights to the desired values.
\end{enumerate}

\subsubsection{Training the FFN: Identification of the Information Set $\cS^\star$}\label{sec:proof_sketch_stage1}
In the first stage, we track the dynamics of $c_\cS^2(t)$ for each $\cS\in\HleqD$.
For convenience, we drop the dependence on $t$ in the sequel.

Recall the output of the model is $y = (\sigma(a\cdot s) X)^\top$ and the cross-entropy loss function is 
$\cL(\Theta) = \EE_{\pi \sim \cP, x_{1:L}} [\ell(\Theta)]$, where $\ell(\Theta)$ can be written as  $\ell(\Theta) = -\left \langle x_{L+1}, \log (y+\varepsilon \bm{1}) \right \rangle.$
We ignore the small constant $\varepsilon$ in the following proof sketch for simplicity.
We also abbreviate the vector of attention probabilities in the second attention layer as $\sigma\in \RR^L$.

\paragraph{Calculation of the Dynamics of $c_\cS^2$}
By a direct calculation for the loss $\ell$ and $s_l$ in  \eqref{eq:simplify_s_l},
\begin{align}
    \frac{\partial \ell}{\partial s_l} =  -a \cdot \sigma_l(a \cdot s) \cdot \rbr{\frac{x_{L+1}}{y}}^\top \rbr{x_l -y}, \quad 
    \frac{\partial s_l}{\partial c_\cS}& = \frac{2 c_\cS \prod_{h\in \cS} \langle v_l^\head{h}, v_{L+1}^\head{h} \rangle}{C_D} - \frac{2 c_\cS s_l}{C_D}.
\end{align}
Here the vector $x_{L+1}/y$ is obtained by element-wise division and $\sigma_{l}(a\cdot s)$ is the $l$-th entry of $\sigma(a\cdot s)$.
Then applying the chain rule, we obtain the following dynamics for $c_\cS^2$ along the gradient flow:
\begin{align}
    \partial_t \log c_\cS^2
    = -\frac{2}{c_S} \sum_{l=M+1}^L \EE\bigg[\frac{\partial\ell}{\partial s_l} \frac{\partial s_l}{\partial c_S}\bigg]
    &= \frac{4a}{C_D}   \sum_{l=M+1}^{L} \EE \bigg[\sigma_l(a\cdot s) \cdot \prod_{h\in \cS} \langle v_l^\head{h}, v_{L+1}^\head{h} \rangle \cdot \bigg(\frac{x_{L+1}}{y}\bigg)^\top \rbr{x_l -y}\bigg]\\
    &\qqquad - \underbrace{\frac{4a}{C_D}   \sum_{l=M+1}^{L} \EE \bigg[\sigma_l(a \cdot s) \cdot s_l \cdot \bigg(\frac{x_{L+1}}{y}\bigg)^\top \rbr{x_l -y}\bigg]}_{\ds f(t)}.
\end{align}
Note that here the second term $f(t)$ is independent of $\cS$, and it will be canceled out when we consider the difference of the derivatives, $\partial_t\log c_{\cS}^2 - \partial_t \log c_{\cS'}^2$, for two sets $\cS,\cS'\in\HleqD$.
This is why we focus on the time derivative of $\log c_{\cS}^2$.

\paragraph{Relate the Dynamics to the Modified $\chi^2$-MI by Approximations}
Now using the approximation in \eqref{eq:approx_subset} for $\prod_{h\in \cS} \langle v_l^\head{h}, v_{L+1}^\head{h} \rangle$, expanding $(x_{L+1}/y)^\top (x_l-y)$ coordinate-wise, and noting that $\sigma_l(a\cdot s)\approx 1/(L-M)$ as we have small $a$ in the second attention layer, we arrive at
\begin{align}
    \partial_t \log c_\cS^2 &\approx \frac{4a}{(L-M)C_D} \sum_{l=M+1}^L \EE\Bigg[ \ind(X_{l-\cS} = X_{L+1-\cS}) \cdot \biggl(\sum_{k=1}^d\frac{\ind(x_{L+1} = x_l = e_k)}{y(k)} - 1 \biggr)\Bigg] - f(t). \label{eq:log_cS_approx}
\end{align}
where $y(k)$ denotes the $k$-th entry of $y$ and $X_{l-\cS}:=(x_{l-i}:i\in\cS)$ denotes the history of $x_l$ on the set $\cS$, similar for $X_{L+1-\cS}$.
Note that $y(k)\approx (L-M)^{-1}\sum_{l=M+1}^L \ind(x_l=e_k) \approx \mu^\pi(e_k)$, which follows from the mixing assumption of the Markov chain that allows us to replace the average over $l=M+1,\ldots,L$ by the expectation over the stationary distribution.
Also for the same reason, we can replace $(x_{l}, X_{l-\cS}), (x_{L+1}, X_{L+1-\cS})$ with \emph{two independent copies} from the stationary distribution $\mu^\pi$, i.e.,
\begin{align}
    \partial_t \log c_\cS^2 &\approx \frac{4a}{C_D} \EE_{(x, X), (z, Z) \sim \mu^\pi \times \mu^\pi}\Bigg[ \ind(Z_{-\cS} = X_{-\cS}) \cdot \biggl(\sum_{k=1}^d\frac{\ind(x  = z  = e_k)}{\mu^\pi(e_k)} - 1 \biggr)\Bigg] - f(t). \label{eq:log_cS_approx-1}
\end{align}
See the approximation from $g_{2, \cS}$ to $g_{3, \cS}$ in \Cref{sec:proof_stage1}.
Indeed, the first term in \eqref{eq:log_cS_approx-1} becomes the modified $\chi^2$-MI, $\tilde I_{\chi^2}(\cS)$, which is defined in \Cref{def:modified_chi_square_mi}.
This gives rise to the following approximation:
\begin{align}
    \partial_t \log c_\cS^2 
    &\approx \frac{4a}{C_D} \tilde I_{\chi^2}(\cS) - f(t).
    \label{eq:log_c_S_gf}
\end{align} 
Since the value of $f(t)$ is independent of the specific choice of  set $\cS$, it is clear that the set $\cS$ achieving the fastest growth rate is the information-optimal set $\cS^*=\argmax_{\cS\in \HleqD} \tilde I_{\chi^2}(\cS)$ that maximizes the modified $\chi^2$-MI within $\HleqD$.

\paragraph{Convergence of $p_{\cS^\star}$}
Note that $p_\cS = c_\cS^2 / \sum_{\cS'\in\HleqD} c_{\cS'}^2$
quantifies the contribution of the set $\cS$ to the feature produced by the FFN layer.
Thus, it is the \emph{relative growth rate} of $c_\cS^2$ that matters. 
Towards this end, it follows from \eqref{eq:log_c_S_gf} that, for all $\cS\in\HleqD\backslash\{\cS^\star\}$,
\begin{align}
    \partial_t \log \frac{c_{\cS^\star}^2}{c_\cS^2} \approx \frac{4a}{C_D} \cdot \rbr{\tilde I_{\chi^2}(\cS^\star) - \tilde I_{\chi^2}(\cS)} \ge \frac{4a}{C_D} \cdot \Delta \tilde I_{\chi^2}.
    \label{eq:sketch-1}
\end{align}
Here we recall from \eqref{eq:Delta_I_chi2} that $\Delta \tilde I_{\chi^2}$ quantifies the minimal gap between the modified $\chi^2$-MI of $\cS^\star$ and any other set in $\HleqD$.
The lower bound given by \eqref{eq:sketch-1} ensures that for all $\cS\neq\cS^\star$, the ratio $c_{\cS^\star}^2/c_\cS^2$ grows exponentially fast, which further implies that $p_{\cS^\star}$ approaches one exponentially fast.
This concludes the first stage of the training dynamics.

\subsubsection{Training the First Attention Layer: Convergence of $\sigma(w^\h)$ to One-Hot Vector}\label{sec:proof_sketch_stage2}
As we proceed to the second stage after $p_{\cS^\star}\approx1$, it suffices to show how $\sigma(w^\h)$ converges to a one-hot vector $e_{M+1-h}$ for $h\in\cS^\star$ in order to show that the model converges to the GIH mechanism.
Recall that we denote $X = (x_1, \ldots, x_L)\in \RR^{L\times d} $. 
For notational convenience, we denote $\sigma^\h := \sigma(w^\h)$ and let $X_{(l-M):(l-1)} \in \RR^{M \times d}$ denote the submatrix of $X$ with rows $l-M, \ldots, l-1$ for any $l$. 
Following our convention, we let $\sigma_{-i}^\h$ denote the $(M+1-i)$-th entry of $\sigma^\h$ and similarly for $w_{-i}^\h$.

\paragraph{Calculation of the Dynamics of $w^\h$}
The main idea for analyzing $\{w^\h\}_{h=1}^H$ is the same as that in the previous stage: It suffices to analyze the \emph{difference between the growth rates} of different coordinates of $w^\h$ for $h\in\cS^\star$.
In particular, we care about how quickly $w_{-h}^\h$ grows compared to other coordinates if $w_{-h}^\h$ is initialized to be larger than the remaining coordinates:
\begin{align}
    \partial_t w_{-h}^\h - \partial_t w_{-i}^\h &= \sum_{l=M+1}^L \EE\bigg[\frac{\partial\ell}{\partial s_l} \bigg(\frac{\partial s_l}{\partial w_{-h}^\h} - \frac{\partial s_l}{\partial w_{-i}^\h}\bigg)\bigg]\label{eq:main_gradient_diff20}\\
    &= a \sum_{l=M+1}^L \EE\bigg[\sigma_l(as) \left(\sum_{k=1}^d \frac{\ind(x_{L+1}=x_l = e_k)}{y(k)} - 1 \right) \bigg(\frac{\partial s_l}{\partial w^\h_{-h}} 
    - \frac{\partial s_l}{\partial w^\h_{-i}} \bigg) \bigg].    
\end{align}
Now, we invoke the result obtained in the previous stage that $p_{\cS^\star}\approx1$, which gives us $s_l\approx\prod_{h\in\cS^\star}\langle v_l^\h,v_{L+1}^\h\rangle$.
Consequently, for any $h\in\cS^\star$, we have
\begin{align}\label{eq:gradient_s_w}
    \frac{\partial s_l}{\partial w^\h_{-i}} &\approx \frac{\partial}{\partial w_{-i}^\h} \prod_{h'\in\cS^\star} \langle v_l^{(h')},v_{L+1}^{(h')}\rangle
    = \bigg(\prod_{h'\in\cS^\star\setminus\{h\}}  \langle v_l^\head{h'}, v_{L+1}^\head{h'} \rangle\bigg) \cdot b_l^\top
    (e_{M+1-i} - (\sigma^\h)^\top ) \sigma_{-i}^\h
\end{align}
where the equality follows from the fact that $w_{-i}^\h$ only affects $(v_l^\h,v_{L+1}^\h)$ and differentiating through the softmax function.
Here we define $b_l :=  X_{(l-M):(l-1)}v_{L+1}^\head{h} +  X_{(L+1-M):L}v_{l}^\head{h}$ to simplify the notation.
Combining \eqref{eq:main_gradient_diff20} and \eqref{eq:gradient_s_w}, we obtain 
\begin{align}
     \!\!\!\!\! \partial_t w_{-h}^\h - \partial_t w_{-i}^\h 
    \approx a  g_h^\top \bigg(\sigma_{-i}^\h\left(e_{M+1-h} - e_{M+1-i} \right) + (\sigma_{-h}^\h - \sigma_{-i}^\h ) \sum_{j\neq h} \sigma_{-j}^\h (e_{M+1-h}-e_{M+1-j}) \bigg), \label{eq:main_gradient_diff2}
\end{align}
where we introduce the following notation 
\begin{align}
    g_{h} :=  \sum_{l=M+1}^L \EE\bigg[\sigma_l(a\cdot s) \cdot \left(\sum_{k=1}^d \frac{\ind(x_{L+1}=x_l = e_k)}{y(k)} - 1 \right) \cdot
    \prod_{h'\in \cS\setminus\{h\}}  \langle v_l^\head{h'}, v_{L+1}^\head{h'} \rangle b_l\bigg].
\end{align}
A detailed deviation of \eqref{eq:main_gradient_diff2} can be found in \eqref{eq:diff_stage2-0}.
Notice that $\sigma_{-h}^\h - \sigma_{-i}^\h$ is positive at initialization. 
Now suppose $\sigma_{-h}^\h - \sigma_{-i}^\h>0$ holds at current time $t$.
Then, lower bounding $\partial_t w_{-h}^\h - \partial_t w_{-i}^\h $ boils down to lower bounding \(
    g_h^\top \left(e_{M+1-h} - e_{M+1-i} \right)
\)
for $i \neq h$.
Furthermore, if we can show that $\partial_t w_{-h}^\h - \partial_t w_{-i}^\h$ is lower bounded by some positive value, the gap $\sigma_{-h}^\h - \sigma_{-i}^\h$ will further increase.
Since $\sum_{i=1}^M \sigma_{-i}^\h\equiv 1$, this will create a reinforcing loop that makes $\sigma_{-h}^\h$ monotonically increase.

\paragraph{Relate the Dynamics to the Modified $\chi^2$-MI by Approximations}
We demonstrate next that 
\(
    g_h^\top \left(e_{M+1-h} - e_{M+1-i} \right)
\) 
for $i\neq h$ admits a lower bound depending on the information gap $\Delta\tilde I_{\chi^2}$.
Specifically, using the same strategy for \eqref{eq:log_cS_approx}, we have by definition that 
\begin{align}
    &g_h^\top e_{M+1-i} 
    \label{eq:approx_g_h-sketch}\\
    &\quad \approx \frac{1}{L-M}\sum_{l=M+1}^L \EE\left[\left(\sum_{k=1}^d \frac{\ind(x_{L+1}=x_l = e_k)}{y(k)} - 1 \right) \cdot \ind(x_{l-j}=x_{L+1-j}, j\in\cS^\star\setminus\{h\}) \cdot b_l^\top e_{M+1-i} \right]
\end{align}
where for $b_l$ we have by the same approximation $v_l^\h \approx x_{l-h}$ and $v_{L+1}^\h \approx x_{L+1-h}$ as in \eqref{eq:approx_subset} that
\begin{align}
    b_l^\top e_{M+1-i} = {v_{L+1}^\h}^\top x_{l-i} + {v_{l}^\h}^\top x_{L+1-i} \approx \ind(x_{L+1-h} = x_{l-i}) + \ind(x_{l-h}=x_{l-i}). 
    \label{eq:approx_b e}
\end{align}
Now we consider the case $i = h$ and $i\neq h$ separately:
\begin{enumerate}
    \item[(i)] ($i = h$) For \( g_h^\top e_{M+1-h} \), we simply set \( i = h \) in \eqref{eq:approx_b e}, and the indicator \(\ind(x_{L+1-h} = x_{l-h})\) will exactly compensate for the exclusion of \( h \) in the indicator function of \eqref{eq:approx_g_h-sketch}.
    Drawing an analogy to how we go from \eqref{eq:log_cS_approx} to \eqref{eq:log_c_S_gf}, we obtain
    \[
        g_h^\top e_{M+1-h} \approx 2\tilde I_{\chi^2}(\cS^\star).
    \]
    \item[(ii)] ($i\neq h$) For \( g_h^\top e_{M+1-i} \) with \( i \neq h \) in \eqref{eq:approx_b e}, we apply the same reasoning as in the previous case. Additionally, by using the Cauchy-Schwarz inequality, the following inequality holds up to a small error (see \Cref{lem:cross_mutual} for a detailed derivation):
    \begin{align} 
        g_h^\top e_{M+1-i} \leq \tilde I_{\chi^2}(\cS^\star) + \tilde I_{\chi^2}(\cS^\star\backslash\{h\}\cup\{i\}) \leq 2\tilde I_{\chi^2}(\cS^\star) - \Delta \tilde I_{\chi^2}.
    \end{align}
\end{enumerate}
Plugging this back into the dynamics in \eqref{eq:main_gradient_diff2}, we conclude that for all $i\neq h$,
\begin{align}
    \partial_t w_{-h}^\h - \partial_t w_{-i}^\h \geq a \cdot \sigma_{-i}^\h \cdot \Delta  \tilde I_{\chi^2}.
\end{align}

\paragraph{Convergence of $\sigma(w^\h)$} 
Combining the arguments in the previous two steps, we can now say that $\sigma_{-h}^\h$ will monotonically increase. 
It remains to show that $\sigma_{-h}^\h$ converges to one.
Note that $\log (\sigma_{-h}^\h/\sigma_{-i}^\h) = w_{-h}^\h - w_{-i}^\h$ by the definition of the softmax function.
Therefore,
\begin{align}
    \partial_t \log \bigl ( \sigma_{-h}^\h \bigl / \sigma_{-i}^\h \bigr) = \partial_t w_{-h}^\h - \partial_t w_{-i}^\h 
    \ge a\cdot \sigma_{-i}^\h \cdot \Delta  \tilde I_{\chi^2}
    = a\cdot \Delta  \tilde I_{\chi^2} \cdot \sigma_{-h}^\h(0) \cdot \bigl( \sigma_{-i}^\h \big / \sigma_{-h}^\h \bigr ) 
\end{align}
where $\sigma_{-h}^\h(0)$ is the initial value of $\sigma_{-h}^\h$ at time $t=0$.
One can now rearrange the term and pick the ratio $\sigma_{-i}^\h/\sigma_{-h}^\h$ as the variable to track in the dynamics. 
A refined analysis in the convergence analysis in \Cref{sec:proof_stage2} shows that $\sigma^\h$ converges to a one-hot vector with $\sigma_{-h}^\h$ going to one.
In particular, the convergence rate is determined by the information gap $\Delta \tilde I_{\chi^2}$ according to the above formula.

\subsection{Analysis for the Training of the Second Attention Layer}\label{sec:proof_sketch_stage3}
In the last stage, we turn to the training of $a$ given that all $\sigma^\h$'s for $h\in\cS^\star$ are approximately one-hot vectors.
The following approximation of the dynamics of $a(t)$ is performed in the region $a \le O(\log L)$, where the signal term in the dynamics dominates the approximation error.

\paragraph{Calculation of the Dynamics of $a$}
After Stages \RNum{1} and \RNum{2}, the output is approximated as $y(k)\approx y^\star (k) \defeq \sum_{l=1}^L \sigma_l^\star \ind(x_l = e_k)$ for each $k\in[d]$.
Here the weighting coefficients $\sigma_1^\star,\ldots,\sigma_L^\star$ satisfy
\begin{align}
    \sigma_l^\star \propto \exp\left(a \cdot \ind(X_{l -  \cS^\star} = X_{L+1-\cS^\star})\right).
\end{align}
Note that for each $l\in[L]$, $\sigma_l^\star$ indicates the importance assigned to the $l$-th token based on the corresponding history of $x_l$ over the information set $\cS^\star$.
In the population counterpart, when the chain has sufficiently mixed, for given $X_{L+1-\cS^\star}$, we can roughly view each $(x_l, X_{l-\cS^\star})$ as being sampled from a \emph{reweighed version of the stationary distribution}:
\begin{align}
    \tilde{\mu}^\pi(x_l, X_{l-\cS^\star} \given X_{L+1-\cS^\star}) \propto \mu^\pi(x_l, X_{l-\cS^\star}) \cdot \exp\left(a \cdot \ind(X_{l -  \cS^\star} = X_{L+1-\cS^\star})\right).
\end{align}
Following the same argument as those in the previous stages, replacing the sum over $l$ with the expectation over the stationary distribution, we arrive at
\begin{align}
    \partial_t a \approx \EE_{ \pi\sim \cP, (x, X_{-\cS^\star}, z, Z_{-\cS^\star}) \sim q^\pi}\bigg[  \ind(X_{-\cS^\star} = Z_{-\cS^\star}) \cdot \bigg(\sum_{k=1}^d\frac{\ind(x = z = e_k)}{\tilde \mu^\pi(z=e_k\given X_{-\cS^\star})} - 1 \bigg) \bigg].\label{eq:gradient_a}
\end{align}
See detailed derivations of the above approximation in \Cref{sec:proof_stage3}.
Comparing the above expression with \eqref{eq:log_cS_approx-1} in Stage I, one can see that
here $(x, X_{-\cS^\star})$ and $(z, Z_{-\cS^\star})$ are no longer independent because now the model has learned to perform a \emph{information-theoretic feature selection}, i.e., focusing on tokens sharing the same set of features based on the information set $\cS^\star$, which is defined according to the modified $\chi^2$-mutual information. 
In fact, the underlying joint distribution $q^\pi$ is given by 
$
    q^\pi = \mu^\pi(x, X_{-\cS^\star}) \cdot \tilde \mu^\pi(z, Z_{-\cS^\star}\given X_{-\cS^\star}).
$

\paragraph{Divergence of $a$}
As the dynamics of $a$ has no closed-form expression due to the nonlinearity in the reweighed distribution $\tilde\mu^\pi$, we resort to providing characterization for cases where $a$ is either sufficiently small or large.
In both cases, the lower and upper bounds of \eqref{eq:gradient_a} can be derived, respectively.
Using these bounds, we can argue rigorously that for small $a$, it undergoes super-exponential growth until it reaches a critical ``elbow'' value. 
After that, when $a$ becomes even larger, it grows logarithmically until it reaches $\Omega(\log L)$.


\section{Experiments} \label{sec:experiments}
In this section, we first detail the setup for the experiment in \Cref{fig:train_C_W_a}, and then provide additional results for training a model that also incorporates the word embedding matrices $W_Q$, $W_K$, $W_V$ and the output embedding matrix $W_O$ in the first attention layer.
Let us first detail the data setup that is used for all the experiments in this work.

\paragraph{Data generation}
The dataset for the ICL task is generated as $n$-gram Markov chains as described in \cref{sec:icl_mc}. 
We take $\pa = \{-1, -2\}$ as the parent set. Thus, the number of parents is $n=2$ and the token embedding dimension is $d=3$.
Note that for each sequence, the transition matrix $\pi(x\given x_\pa)$ is of shape $d\times d^{n}$. 
We assign a prior distribution $\mathcal{P}$ for the transition matrix, which is defined such that each column of the transition matrix of kernel $\pi$ is independently drawn from a symmetric Dirichlet distribution with parameter $\alpha=0.01$, i.e., $\pi(\cdot | x_{\text{pa}}) \sim \text{Dir}(\alpha \cdot \mathbf{1}_{d})$. 
Note that each chain has different transition kernel $\pi$ but follows the same prior distribution $\cP$.
We randomly sample 10,000 Markov chains with $L=100$ from the prior distribution $\mathcal{P}$; 9,000 are used for training and 1,000 for validation. 

\subsection{Training with Stage Splitting}
we present the simulation results with model $\mathtt{TF}(M, H, d, D)$ in \eqref{eq:transformer} and training in the three-stage  manner.  
We configure the model with window size $M=3$, number of heads $H=3$, vocabulary size $d=3$ and maximal FFN degree $D=2$.

\paragraph{Model initialization}
The RPE weight matrix $W^{(h)}_P$ is initialized such that the $(-i)$-th diagonal of $W^{(h)}_P$ has value $w^{(h)}_{-i}$ for $i=1, 2, \ldots, M$, while all other entries are initialized to $-\infty$. See \Cref{fig:RPE} for an interpretation.
We initialize $w_{-h}^{(h)} = 3$ and set the remaining entries within the size-$M$ window to $0.01$ to ensure symmetrization-breaking and some initial correspondence between heads and parents.
For the FFN layer that learns the polynomial features, all $c_\cS$ for $\cS \in \HleqD$ are initialized to $0.01$. The initial value of $a$ in the second attention layer is set to $0.01$.

\paragraph{Training settings} The models are trained using gradient descent with respect to the cross-entropy loss and a constant learning rate that is set to one for all stages. We train the model in Stage I (update parameters $\{ c_\cS\} $ only) for 2000 epochs, in Stage II (update parameters $\{w^{(h)}\}$ only) for 50,000 epochs, and in Stage III (update parameter $a$ only) for 5000 epochs, respectively. All experiments are conducted using a single Nvidia A100 GPU.
The results are already shown in \cref{fig:train_C_W_a}, which matches our theoretical results.

\subsection{Training without Stage Splitting}

We also tested training the whole model without stage splitting. The data generation is the same as described above.
For the model, we additionally include the word embedding matrices $W_Q$, $W_K$, $W_V$ in the first attention layer.
The training setup is the same as the one described above with additional configurations specified in \cref{sec:exp_supp}. The result is shown in \cref{fig:train_full}.
We observe similar patterns, i.e., the dominating $c_{\cS^\star}$ and the focus of the attention heads on the parents, as well as the growth of the weight $a$ in the second attention layer.
However, the training dynamics of the model are not as ``ideal'' as the one with stage splitting, as the model tends to learn a false parent set at the beginning as shown in \cref{fig:train_full}-(b).
But after a sufficient number of training steps, the correct information set $\cS^\star$ starts to dominate, and the loss experiences a sharp decrease.
We further plot the $W_Q$, $W_K$ and $W_V$ matrices after the training in \Cref{fig:train_full-W}. 
The fact that the model eventually has $W_Q$ and $W_K$ close to zero and $W_V$ close to the identity matrix (up to a scaling factor) justifies our simplified model, where we remove $W_Q$ and $W_K$ and set  $W_V=I$ in the first attention layer.

\begin{figure}[!htb]
    \centering
    \includegraphics[width=0.8\textwidth]{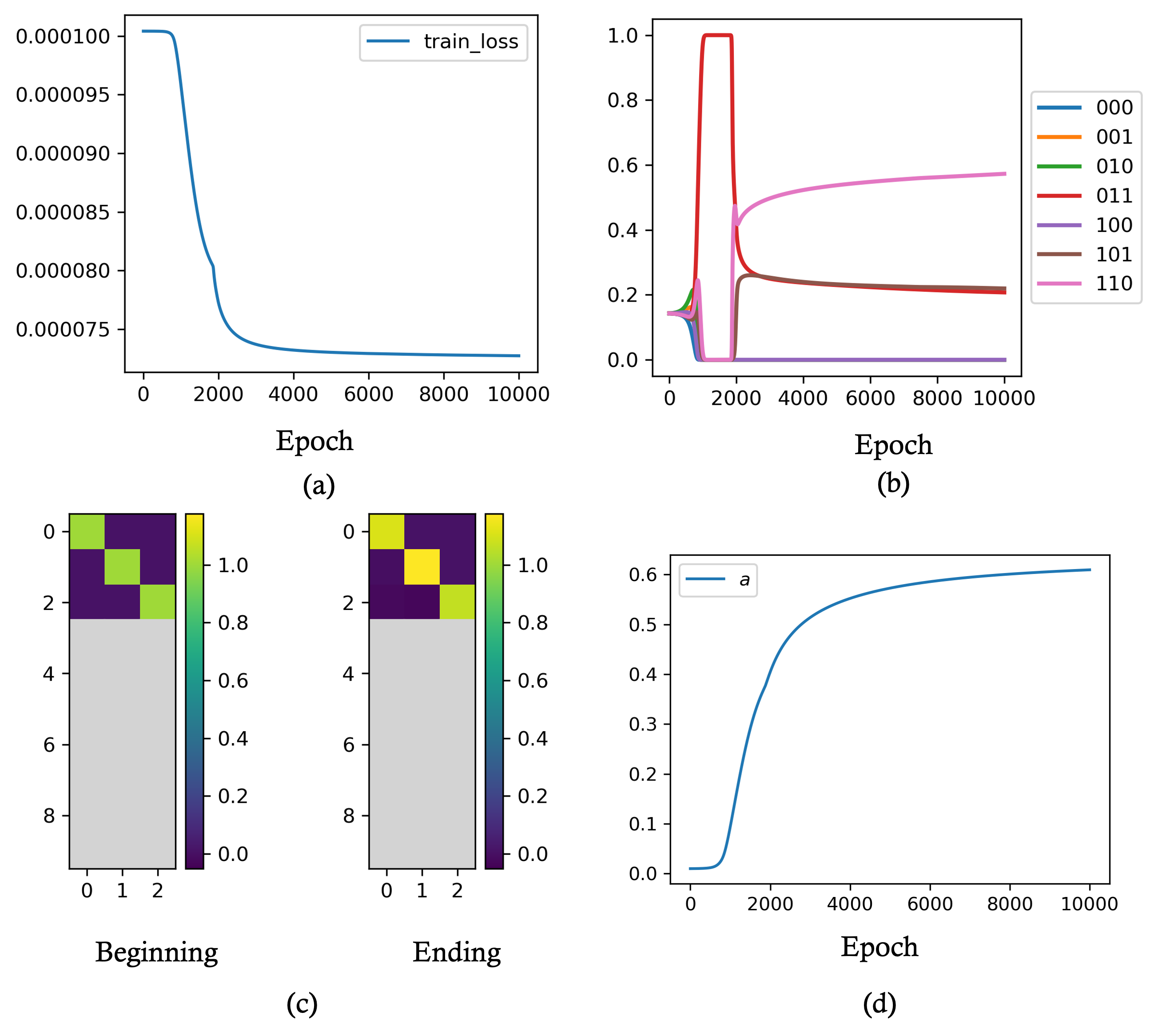}
    \caption{An illustration of the evolution of gradient descent dynamics when training a transformer model specified in \Cref{sec:exp_supp} with word embedding matrices $\{W_Q, W_K, W_V\}$.
    Here the dynamics are not split into three stages and each gradient descent step updates all parameters. 
    We set $M = H = 3$, $d = 3$, and $D = 2$, the number of input token is $L = 100$, and Markov chain has parent set $\pa = \{-1, -2\}$.
    In (a) we show the training loss of the model, which shows that the loss decreases and converges to some value. 
    In (b) we show 
   the evolution of $p_{\cS}$ where we use binary coding $\{0, 1\}^3$ to indicate each subset $\cS$. 
    Here, $p_{\cS^\star}$ has code ``$110$'', which corresponds to the true parent set. 
This figure shows that initially a wrong $p_{\cS}$ dominates at the early stage of training, which corresponds to $\cS = \{2,3\}$ (code ``011''). Then eventually $p_{\cS^\star}$ increases and becomes dominant. However, $p_{\cS^\star} $ does not increase to one and is about $0.6$, and there are two $p_{\cS}$'s  that are about $0.2$. 
    In (c) we show the 
    RPE weights of the first attention layer before and after training.
    The entries corresponding to the true parents, $w^{(1)}_{-1} $ and $w^{(2)}_{-2} $, significantly increase after training, while $w^{(3)}_{-3} $ slightly increases from initialization.  
    This figure shows that each attention head focuses on copying a single previous token. 
    In (d) we show the evolution of the weight $a$ in the second attention layer. We observe a similar ``elbow'' curve as in \Cref{fig:train_C_W_a}-(c).
    }
    \label{fig:train_full}
\end{figure}

\subsection{Prior and Length Generalization}
We further test the model learned by the three-stage training on sequences coming from different priors and of different lengths.
Note that our pre-trained transformer learns to perform GIH.
As introduced in \Cref{sec:induction_head}, the GIH estimator can be applied to a sequence with an arbitrary length and does not concern the prior distribution of the underlying Markov chain.  
Thus, it is natural to see if the pre-trained transformer can also generalize to different lengths and prior distributions. 

Recall that we train the transformer model with sequence length $L=100$ and the concentration parameter of the Dirichlet prior is  $\alpha=0.01$. 
Here, we test the pre-trained transformer on new sequences of different lengths and sampled from different prior distributions. 
That is, with a different concentration parameter $\alpha$, we sample a random Markov chain, and generate a sequence of length $L$, and evaluate of cross-entropy loss for predicting $x_{L+1}$. Here we choose $\alpha \in \{0.05, 0.1, 0.2\}$ and range $L$ from $10$ to $1000$. 
When generating the data, the Markov chains share the same parent set $\pa = \{-1, -2\}$ with the pre-training data. 
The results are shown in \Cref{fig:generalization}.
The results show a decreasing trend in testing loss as the sequence
length increases.
For $\alpha = 0.2$, we observe first a small increase in the test loss when $L$ just exceeds $100$, but then the loss decreases as $L$ increases further.
This experiment shows that the pre-trained transformer indeed generalizes in length and is robust to the change of prior distribution.

\begin{figure}[!htb]
    \centering
    \includegraphics[width=0.45\textwidth]{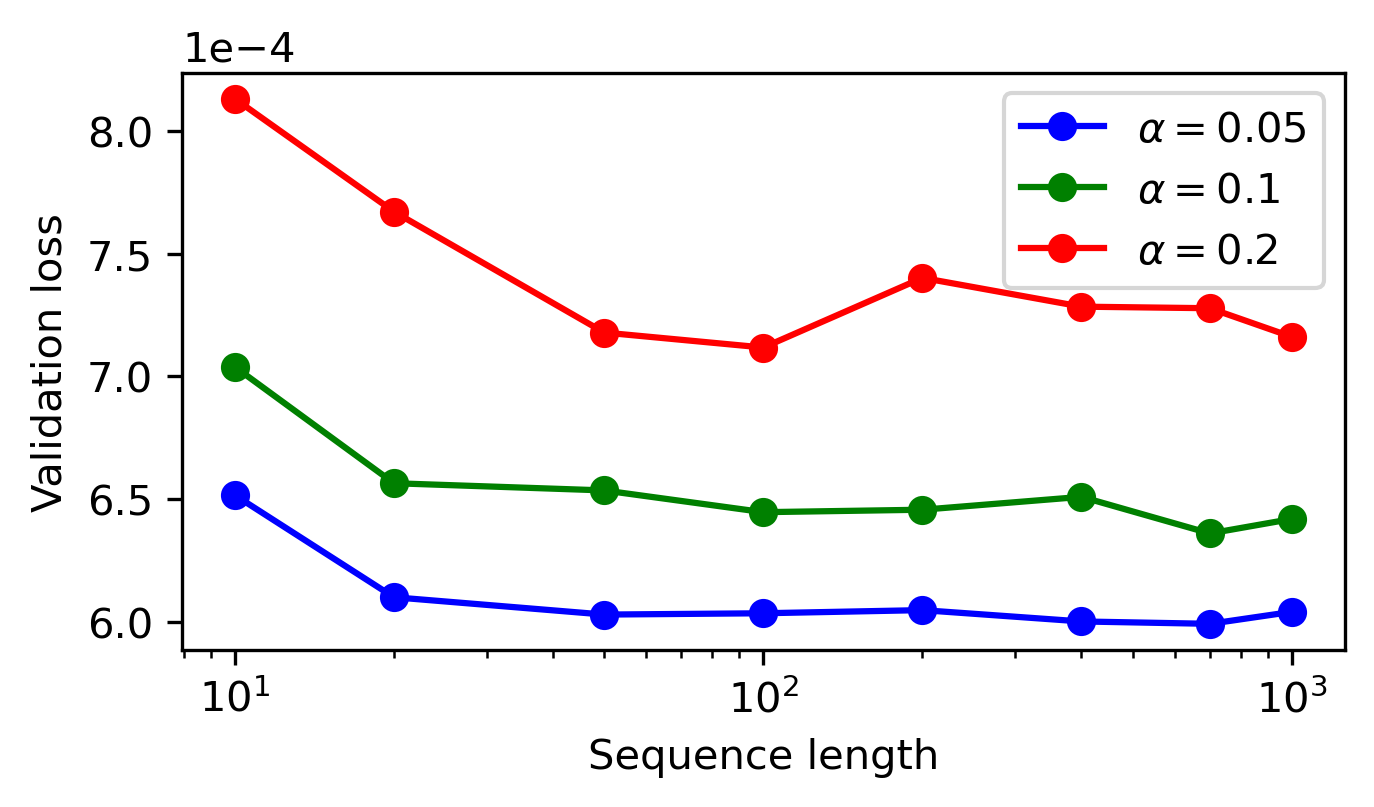}
    \caption{Generalization capability of our model
    to different sequence lengths and prior distributions.
    We plot the cross-entropy loss of the pre-trained transformer model on sequences with different lengths sampled from Markov chains with different prior distributions. 
    The prior is Dirichlet distribution with $\alpha \in \{0.05, 0.1, 0.2\} $ and we vary the length $L$ in $\{10, 20, 50, 100, 200, 400, 700, 1000\}. $  
    The pre-training data contains sequences of length $L = 100$ and $\alpha = 0.01$. 
    For different $\alpha$, we see that the error has a decreasing trend as $L$ increases. This shows that the pre-trained transformer can generalize in length and is robust to the distributional shift due to a change of prior. 
    }
    \label{fig:generalization}
\end{figure}

\section{Conclusion and Future Work}
In this paper, we have studied the training dynamics of a two-attention-layer transformer model for learning $n$-gram Markov chains in an in-context way. 
Our theoretical analysis underscores a congruous interplay between the multihead attention mechanism, the feed-forward network, and layer normalization that yields a generalized version of the induction head mechanism during the training. 
In particular, we prove that the generalized induction head mechanism adopts a modified $\chi^2$-mutual information criterion for parent selection that strikes a balance between information richness and model complexity.
To our best knowledge, our work gives the first theoretical evidence for learning an induction head mechanism with $n$-gram Markov data, which potentially sheds light on the inner workings of large-scale transformer models. 

Our work opens new directions for developing a rigorous understanding of the transformer models. 
A natural direction would be that if one can find such a mechanism with standard FFN layer using multi-layer perceptron and standard layer normalization in the more practical transformer model.
The intuition is that our FFN layer in \cref{eq:feed-forward}, which is further instantiated in \eqref{eq:feed-forward-explicit}, lies in the space of low-degree polynomials and can be well represented by a MLP with sufficient dimensions and proper activation functions. 
Initial attempts to learn nonlinear features have also been made by \citet{kim2024transformers}. 
Another direction is to investigate the training dynamics beyond a single loop of this induction head mechanism, e.g., iteration head with recursively refined predictions \citep{cabannes2024iteration}, and how the induction head mechanism occurs in multi-layer transformer models.

\section{Acknowledgement}
We acknowledge Shaobo Wang for his help with the experiments.
We also thank Jason D. Lee, Alex Damian, and Eshaan Nichani for their helpful discussions.   
Zhuoran Yang acknowledges the support of NSF under the award DMS-2413243. 

\newpage
\bibliographystyle{ims}
\bibliography{reference}

\newpage 
\appendix

\tableofcontents

\section*{Organization of The Appendix}
The appendices are organized as follows: 
\begin{itemize}
    \item In \Cref{sec:exp_supp}, we discuss additional experimental details.
    \item In \Cref{sec:additional_background}, we provide explicit expressions for the FFN realizing a low-degree polynomial kernel, and review basics related to concepts mentioned in the main text.
    \item In \Cref{sec:dynamics proof}, we present the proof for \Cref{thm:convergence}.
    \item In \Cref{sec:auxiliary_results}, we collect auxiliary results used in the proof of \Cref{thm:convergence}.
\end{itemize}


\section{Additional Experiments}
\label{sec:exp_supp}

\subsection{Training without Stage Splitting}
Previously in \Cref{sec:experiments}, we show the simulation results on the simplified model \eqref{eq:transformer}. Now we present the results of additional experiments based on the full model defined as follows.
\begin{equation}\label{eq:transformer_original}
\begin{aligned}
    &\textbf{First Attention:} && \tilde V^\h =\sigma\bigl(\tilde X W_{Q}^\h {W_{K}^\h}^\top \tilde X^\top + W_{P}^\h\bigr) \tilde X {W_{V}^\h}^\top && \in \RR^{(L+1)\times d}; \\
    &\textbf{Concatenate \& Normalize:} && V = \LayerNorm\bigl([
        \tilde V^\head{1}, \dots, \tilde V^\head{H}, \tilde X
    ] \bigr) && \in \RR^{(L+1)\times (H+1)d}; \\
    &\textbf{FFN \& Normalize:} && \tilde U = \phi(V)/\sqrt{C_D} &&\in \RR^{(L+1)\times d_e}; \\
    &\textbf{Concatenate} && \tilde X' = [\tilde U, V] &&\in \RR^{(L+1)\times ((H+1)d + d_e)};\\
    &\textbf{Second Attention:} && Y = 
    \sigma\bigl(
        a \cdot (\tilde x_{L+1}')^\top (\tilde X_{1:L}')^\top  
    \bigr) X
    && \in \RR^{(L+1)\times d}.
\end{aligned}
\end{equation}
In head $h$ of the first attention layer, $W_P^\h$ is the relative positional embedding matrix, and we include $W_{Q}^\h\in\RR^{d\times d}$, $W_{K}^\h\in\RR^{d\times d}$ and $W_{V}^\h\in\RR^{d\times d}$ as the weight matrices for the query, key, and value projections, respectively.
That is, in the full model, we the attention heads has more weight matrices than the simplified model.  
Another difference is that we also explicitly include the residual link that copies $\tilde X$ to the output of the first attention layer.
For the FFN layer, $\phi:\RR^{(H+1)d}\rightarrow \RR^{d_e}$ is the same feed-forward network specified in \eqref{eq:feed-forward}. 
Here, we use a standard $\ell_2$-layer-normalization $\LayerNorm (\cdot) $, defined as 
\begin{align}
    \LayerNorm([x, y]) = \left[
        \frac{x}{\norm{x}_2}, \frac{y}{\norm{y}_2}
    \right].
\end{align}
The second attention layer takes $X$ as the value, which comes from the residual link (i.e., concatenation of $\tilde U$ and $V$ while $\tilde X$ in $V$ remains the same after $\ell_2$-normalization). 
In comparison to the simplified model in \eqref{eq:transformer}, here we incorporate the query, key and value projections for the first layer as in a standard transformer architecture.

Our training setup is similar to that in \Cref{sec:experiments}. We use the same dataset and a similar training settings.
All these weight matrices $W_{Q}^{(h)}$, $W_{K}^{(h)}$ and $W_{V}^{(h)}$ are initialized as identity matrices scaled by 0.001. 
We initialized the RPE vector $w^{(h)}$ as $w_{-h}^{(h)} = 1$ for $h=1, 2, 3$, and leave the remaining entries within the length-$M$ window to 0.01.
We trained the model with all parameters together for 10,000 epochs with the same loss function and learning rate. As illustrated in \Cref{fig:train_full}, the full model converged to a state comparable to our simplified model.
We further plot the $W_Q^{(1)}, W_K^{(1)}, W_V^{(1)}$ for the first head after training in \Cref{fig:train_full-W}.
The results demonstrate that the model converges to a point where the query and key projections are close to zero, which leaves the RPE weights to dominate the attention mechanism.
This fact justifies our simplification in \eqref{eq:transformer} where we remove the query and key projection weights and set $W_{V}^{(h)}$ to be identity matrix. 
\begin{figure}[!htb]
    \centering
    \includegraphics[width=1\textwidth]{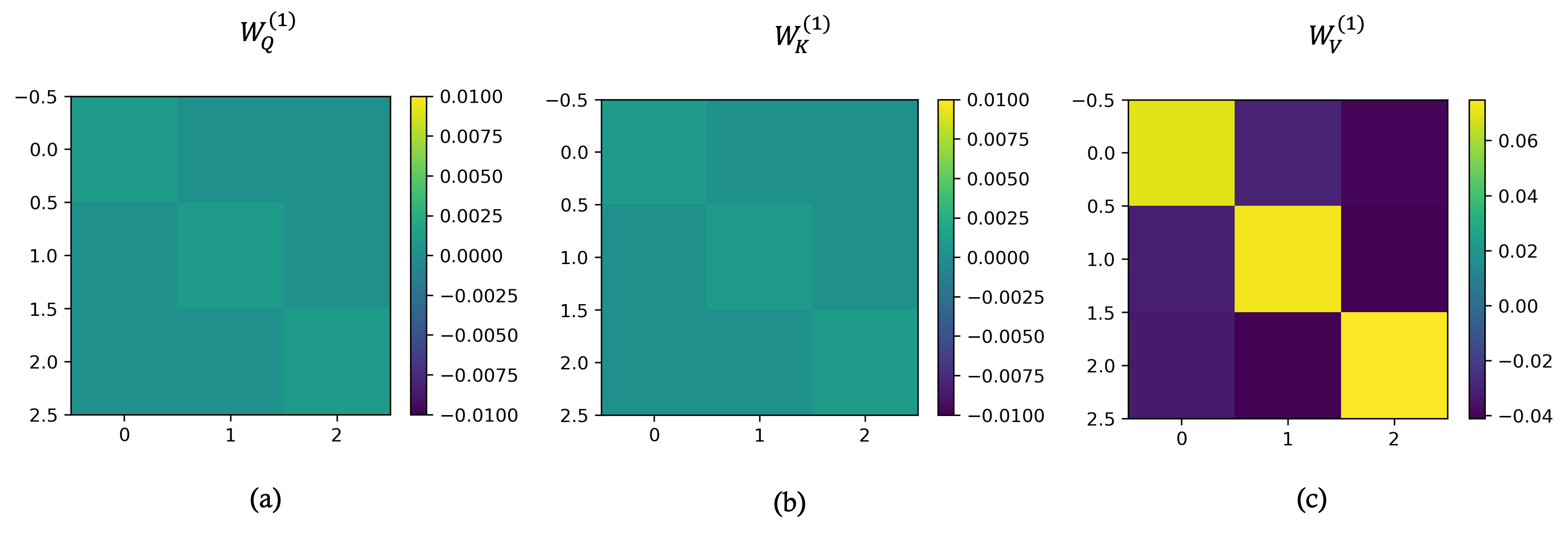}
    \caption{A visualization of the word embedding matrices $W_Q^{(1)}$, $W_K^{(1)}$, $W_V^{(1)}$ of
    a pre-trained transformer with $M = H = 3$, $d = 3$, and $D = 2$. 
    These are the parameters in of 
    the first attention head in the first attention layer. 
    Since $d = 3$,  
    all word embedding matrices are of shape $3\times 3$. As shown in (a) and (b), $W_Q^{(1)}$ and $W_K^{(1)}$  do not change much compared to their initialization value $0.001$. 
    Thus, they are both close to the zero matrix and play a negligible role in the first attention layer. 
    Besides, in (c) we plot   $W_V^{(1)}$, which establishes a clear diagonal structure, with the diagonal entries growing to $0.07$ compared to the initialization value $0.001$. Thus, $W_{V}^{(1)} $ is proportional to the identity matrix. 
    }
    \label{fig:train_full-W}
\end{figure}


\section{Additional Background and Discussions}\label{sec:additional_background}

\subsection{Feed-Forward Network for Polynomial Kernel}\label{sec:ffn_polynomial}

\begin{lemma}\label{lemma:poly_kernel}
Recall the FFN satisfying \eqref{eq:feed-forward}, which maps a vector $z \in \RR^{dH}$ to a vector in $\RR^{d_e}$. 
We write $z$ as $(z^\head{1}, \ldots,  z^\head{H}) $ where $z^\h \in \RR^d$ for all $h\in [H]$. 
Let $z_i^\h$ be the $i$-th entry of $z^\h$.
Then we can explicitly construct $\phi(\cdot)$ by letting 
    \begin{align}\label{eq:feed-forward-explicit}
    \phi\bigl ((z^\head{1}, \ldots,  z^\head{H}) \bigr ) = \bigg(c_{\color{PineGreen}\cS} \cdot \prod_{h\in{\color{PineGreen}\cS} } z^\h_{\color{OrangeRed}i_h}: \{{\color{OrangeRed}i_h}\}_{{h}\in  {\color{PineGreen}\cS}}\subseteq[d] , {\color{PineGreen}\cS}\in\HleqD \Big), 
\end{align}
which is equivalent to 
\begin{align}
    \phi\bigl ((z^\head{1}, \ldots,  z^\head{H}) \bigr ) = \left(c_\cS\cdot \Vec\bigl(\otimes_{h\in\cS} (z^\h)\bigr)\right)_{\cS\in \HleqD}, 
\end{align}
where $\Vec(\cdot)$ is the vectorization operator that transforms a tensor into a vector by stacking all the entries in the tensor.
That is, for any $\cS$, we consider the $|\cS|$ vectors in $\RR^d$, $\{z^{(h)}\}_{h\in \cS} $. In \eqref{eq:feed-forward-explicit} we compute all possible products of the entries of these vectors and multiply them by $c_{\cS}$.
In particular, for each $\cS  \in \HleqD$, we enumerate $i_h \in [d]$ for all $h \in \cS$. 
Therefore, the output dimension of  $\phi$ is given by 
\begin{align}\label{eq:feed-forward-output-dimension}
    d_e = \sum_{\cS \in \HleqD} d^{|\cS|}.  
\end{align}

\end{lemma}

\begin{proof}
First, we note that the indices of $\phi(\cdot)$ have a grouped structure --- we first enumerate all subsets in $\HleqD$ and then enumerate all monomials with superscripts in $\cS$.
Since there are $d^{|\cS|}$ monomials, the output dimension is given by \eqref{eq:feed-forward-output-dimension}. 

It remains to verify \eqref{eq:feed-forward} 
with $\phi(\cdot)$ defined in \eqref{eq:feed-forward-explicit}.
To this end, we note that for any $u, v \in \RR^{dH}$ and any  $\cS \in \HleqD$,  we have 
\begin{align}
\sum_{(i_{h})_{h \in \cS } \in [d]^{|\cS|}} \bigg( \prod_{{ h}\in{ \cS} } u_{i_{ h}}^{({  h})} \cdot v_{i_{ h}}^{({  h})} \bigg)  = \prod_{{ h}\in{ \cS} } \biggl( \sum_{i_{h} \in [d] }u_{i_{ h}}^{({  h})} \cdot v_{i_{ h}}^{({  h})}  \bigg) = \prod_{{ h}\in{ \cS} } \langle u^\head{h}, v^\head{h} \rangle, 
\end{align}
which directly implies \eqref{eq:feed-forward}. Therefore, we conclude the proof of this lemma. 
\end{proof}

\subsection{Perron-Frobenius Theorem}\label{sec:background_perron_frobenius}
Next, we review the basics for the celebrated Perron-Frobenius theorem on non-negative matrices \citep[Chapter 7]{meyer2023matrix}. 
We consider the following class of irreducible matrices. 
\begin{definition}[Irreducible Matrix]
A non-negative square matrix $P\in\RR_+^{d\times d}$ is called irreducible if the induced directed graph $\cG$ is strongly connected, i.e., for any pair of nodes in the graph, there always exists a directed path that connects these two nodes. 
Here, the induced graph $\cG$ is defined based on $d$ nodes with adjacent matrix $A$ given by $A_{ij}= \ind(P_{ij}\neq 0)$. 
\end{definition}
In particular, if $P$ is a stochastic matrix that corresponds to a $d$-state Markov chain, then starting from any state, we can reach any other state with positive probability in a finite number of steps.
The irreducibility property also has an equivalent definition in the matrix form. 
That is, for any permutation matrix $T$, $T P T^{-1}$ cannot be written as an upper triangular block matrix with the following form 
\begin{align}
    \begin{bmatrix}
        M_1 & M_2 \\
        0 & M_3
    \end{bmatrix}.
\end{align}
In other words, an irreducible matrix does not have a nontrivial absorbing subspace that aligns with the standard basis. 

In this work, we require more than the irreducibility property from the transition matrix $P_\pi$ defined in \cref{sec:convergence}. 
In fact, we need the existence of a unique stationary distribution (which is not guaranteed by the irreducibility) so that the chain has a sufficiently fast mixing rate. This enables us to learn with a finite sequence length $L$. 
To achieve this, one typically needs the second largest magnitude of the eigenvalues of $P_\pi$, denoted by $\lambda$, to be bounded away from $1$, which is the leading eigenvalue of the transition matrix.
The difference $1- \lambda$ is also referred to as the spectral gap.
It is well-known that if all the entries of $P_\pi$ are positive, then $P_\pi$ is irreducible and there is only one leading eigenvalue on the spectral circle with the corresponding eigenvector given by the chain's stationary distribution $\mu^\pi$, and all the other eigenvalues have magnitude strictly less than $1$.
However, for our case, the transition matrix $P_\pi$ has zero entries by definition. 
Fortunately, the nice property on the existence of spectral gap can be generalized to a class called \emph{primitive} matrix.
\begin{definition}[Primitive Matrix]
    A nonnegative and irreducible square matrix $P$ is called primitive if there exists an integer $k$ such that all the entries of $P^k$ are positive.
\end{definition}
By definition of the primitive matrix, one can immediately see that for any $k'>k$, $P_\pi^{k'}$ is a positive matrix.
The following is the celebrated Perron-Frobenius theorem that characterizes the spectral structure of the primitive matrices.
\begin{theorem}[Perron-Frobenius Theorem for Primitive Matrices]
    Let $P$ be a primitive matrix. Then the following statements hold:
    \begin{enumerate}
        \item The leading eigenvalue of $P$ is real and positive, and it is the unique eigenvalue with the largest magnitude. In particular, if $P$ is a stochastic matrix, then the leading eigenvalue is $1$.
        \item The leading eigenvector of $P$ is positive and unique up to a scaling factor. In particular, if $P$ is a stochastic matrix, then the leading eigenvector is the stationary distribution of the Markov chain with transition kernel $P$.
    \end{enumerate}
\end{theorem}
The Perron-Frobenius theorem guarantees the existence of a unique stationary distribution $\mu^\pi$ when the transition matrix $P_\pi$ is primitive. 
In particular, when we further assume that the transition matrix $P_\pi$ has a spectral gap, the chain is sufficiently mixed, meaning that we can thus approximate sum over the entire sequence with an average with respect to the stationary distribution. 
In particular, the approximation error will decays with the sequence length $L$.


\subsection{Sequential CE Loss}
\label{sec:sequential_ce_loss}
In this work, we only consider the prediction error on the last token in the sequence as in \eqref{eq:cross_entropy}:
\begin{align} 
    \cL(f_{\mathtt{tf}}) = - \EE_{\pi \sim \cP, x_{1:(L+1)}} \bigl[\log \bigl(f_{\mathtt{tf}}(x_{L+1}  \given x_{1:L})  + \epsilon \bigr)  \bigr].
\end{align}
In practice however, people often train the transformer model by minimizing the cross-entropy (CE) loss over the entire sequence.
We demonstrate that our analysis can be extended to training on the entire sequence. 
In this vein, we define the sequential CE loss as 
\begin{align}
    \cL_{\seq}(f_{\mathtt{tf}}) = \sum_{l=1}^L - \EE_{\pi \sim \cP, X} \bigl[\log \bigl(f_{\mathtt{tf}}(x_{l+1}  \given x_{1:l})  + \epsilon \bigr)  \bigr].
    \label{eq:sequential_ce_loss}
\end{align}
One can equivalently view this sequential CE loss as an aggregation of the CE loss for sequence length ranging from $1$ to $L$.
We argue from the following two perspectives that our analysis can be extended to the sequential cross-entropy (CE) loss:
\begin{enumerate}
    \item Due to the use of relative positional embedding (RPE), the transformer's predictions are invariant to the \emph{absolute} positions of tokens within a sequence. Intuitively, this implies that even if we choose a different sequence length $L'$, the model can still handle the task in the same manner.
    
    \item By \cref{asp:Markov_chain}, the chain is sufficiently mixed for large $L$. In the analysis, we actually use $X_{l-M:l} = (x_l, x_{l-1}, \ldots, x_{l-M}) \sim \mu^\pi$, where $\mu^\pi$ is the stationary distribution over a length-$(M+1)$ window, to approximate the aggregation over $X_{l-M:l}$ for $l = M+1, \ldots, L$ in the sequence. For example, this approximation is reflected in the transition from \eqref{eq:log_cS_approx} to \eqref{eq:log_cS_approx-1} in the proof sketch in \cref{sec:proof_sketch}. Since changing the sequence length does not affect the underlying stationary distribution, the only issue is the approximation error. 
    In particular, for sufficiently large $L$, the CE loss at large $l$ constitutes the majority of the sequential CE loss in \eqref{eq:sequential_ce_loss}, making the CE loss at small $l$ negligible. 
\end{enumerate}



\subsection{Standard $\chi^2$-Divergence and Mutual Information}
The $\chi^2$-divergence (or $\chi^2$-distance) between two probability distributions \( P \) and \( Q \) in the same probability space is defined as:
\[
D_{\chi^2}(P \| Q) = \sum_{x\in\mathrm{supp}(Q)} \frac{(P(x) - Q(x))^2}{Q(x)},
\]
where the summation is taken over all elements \( x \) in the sample space where \( Q(x) > 0 \).
The $\chi^2$-mutual information between two random variables \( X \) and \( Y \) with joint distribution \( P_{XY} \) and marginal distributions \( P_X \) and \( P_Y \) is defined as:
\[
I_{\chi^2}(X; Y) = D_{\chi^2}(P_{XY} \| P_X \otimes P_Y) = \sum_{y} D_{\chi^2}(P_{X\given Y}(\cdot \given y) \| P_X(\cdot )) P_Y(y).
\]
where \( P_X \otimes P_Y \) is the product of the marginals, meaning \( (P_X \otimes P_Y)(x, y) = P_X(x) P_Y(y) \).
For a Markov chain \( X \rightarrow Y \rightarrow Z \), the $\chi^2$-mutual information satisfies the data processing inequality
\[
   I_{\chi^2}(X; Z) \leq I_{\chi^2}(Y; Z), 
\]
which follows from the observation that $\chi^2$-divergence is also an $f$-divergence.

\subsection{More Details on the Generalized Induction Head Mechanism} \label{sec:gih}

Recall that we define the Generalized Induction Head (GIH) estimator in \Cref{eq:gih}. 
Specifically, $\mathtt{GIH}(x_{1:L}; M, D  )$ is constructed in two steps. First, we find the information-optimal subset $\cS^\star$ of $[M]$ by solving \eqref{eq:define_set_S_star}. 
Second, we build a $d$-class kernel classifier to predict $x_{L+1}$, where the ``data'' used by such a classifier are $\{ \psi_{\cS^\star} (l), x_{l}\}_{l\in[M+1, L]}$.
Here $\{ \psi_{\cS^\star} (l) , l \in [M+1, L+1]\}$ are features constructed at each position based on the partial history given $\cS^\star$. 
In particular, similar to \eqref{eq:feed-forward-explicit}, for any subset $\cS$ of $[M]$, any input token sequence $x_{1:L}$, and any position $l \in [M+1, L+1]$, we define $\psi_{\cS} (l) = \psi_{\cS} (l; x_{1:L})$ as   
\begin{align}\label{eq:define_psi_S}
    \psi_{\cS} (l) = \Vec\Bigl(\outerprod_{s\in\cS} x_{l-s}\Bigr) = \bigg( \prod_{{\color{OrangeRed} s}\in{ \cS} } (x_{l - {\color{OrangeRed} s}})_{i_{\color{OrangeRed} s} } : \{i_{\color{OrangeRed} s}\}_{{\color{OrangeRed} s}\in   { \cS}}\subseteq[d]    \Big) \in \RR^{d^{|\cS|}}.  
\end{align}
In other word, $\psi_{\cS}(l)$ is given by expanding the rank-$1$ tensor spanned by $\{x_{l - s}\}_{s\in \cS}$ into a vector.
Here $x_{l-s} \in \cX$ is a vector in $\RR^d$ and we let $(x_{l - {\color{OrangeRed} s}})_{i_{\color{OrangeRed} s} }$ denote its $i_{\color{OrangeRed} s}$-th entry. 
The rationale behind $\psi_{\cS} (l)$  is similar to $\phi$ introduced in \eqref{eq:feed-forward-explicit}. 
We form a long vector containing all the products of the entries of vectors $\{ x_{l -  {\color{OrangeRed} s}} \}_{ {\color{OrangeRed} s} \in \cS} $. 
Here we omit the dependency of $\psi_{\cS}$ on the input sequence $x_{1:L}$ to simplify the notation. 
Furthermore,  $\psi_{\cS}$ induces a polynomial kernel such that for any $l, m \in [M+1, L+1]$, we have 
$$
\langle \psi_{\cS} (l), \psi_{\cS}(m) \rangle = \prod _{{\color{OrangeRed} s} \in \cS  } \langle x_{l - {\color{OrangeRed} s}} , x_{m-{\color{OrangeRed} s}} \rangle  = \ind \{ x_{l-{\color{OrangeRed} s}}  = x_{m-{\color{OrangeRed} s}}, \forall {\color{OrangeRed} s } \in \cS\}.
$$
That is, feature $\psi_{\cS} $ selects the token position pairs $(l, m)$ such that the partial histories induced by $\cS$ at position $l$ and $m$ are exactly the same. 

Based on $\{ \psi_{\cS^\star} (l), x_{l}\}_{l\in[M+1, L]}$, GIH forms a kernel classifier using the indicator kernel. 
Specifically, for any $j \in [d] $, by   \eqref{eq:gih}, $\mathtt{GIH}(x_{1:L}; M, D) $ outputs each $e_j \in \cX$ with probability 
$$
\PP \bigl( \mathtt{GIH}(x_{1:L}; M, D ) = e_j \bigr) =   \frac{\sum_{l = M+1}^L \ind\{ x_{l-s}  = x_{{\color{PineGreen}L+1}-s}, \forall s \in \cS^\star\} \cdot \ind\{ x_{l} = e_j\}  } { \sum_{m = M+1}^L \ind\{ x_{m-{  s}}  = x_{{\color{PineGreen}L+1}-   s}, \forall {s } \in \cS^\star\} } .
$$


\section{Analysis of the Training Dyanamics}\label{sec:dynamics proof}
\paragraph{Masking the Simplified Model} 
Recall that we apply a mask to the first $M$ position in the simplified model. 
Therefore, we only allow index $l$ to run from $M+1$ to $L$ in the following analysis.
In the following, we first specify the conditions on $L$ that are required for the analysis of the training dynamics and then present the proof of \Cref{thm:convergence}.

\subsection{Conditions on the Sequence Length}
\label{sec:L-condition}
We first introduce the following condition on $L$:
\begin{align}
    L\ge \Omega\bigg(\frac{1}{\Delta \tilde I_{\chi^2}^2 (1-\lambda) \gamma^{r_n+2 }}\bigg), \quad 
    L\ge (1-\lambda)^{-1}\gamma^{-D}, \quad \sqrt L \ge M \lor d, \label{eq:L_condition-1}
\end{align}
where $\Omega$ only hides a universal constant that does not depend on the model parameters.
The conditions in \eqref{eq:L_condition-1} will facilitate our analysis for Stage \RNum{1} and Stage \RNum{2}.
For the last stage, we require 
\begin{align}
    L \ge 2 M + r_n \frac{\log \gamma^{-1}}{\lambda^{-1}}, \quad
    \frac{L}{(\log L)^4 } \ge \Omega\bigg(\frac{1}{\kappa^4 \gamma^{8 + 2|\cS^\star|}} \cdot \left(\frac{\sqrt M + d}{(1-\lambda)^{1/2} \gamma^{|\cS^\star|+2+ r_n/4}}\right)^4\bigg),  
    \label{eq:L_condition-2}
\end{align}
where 
\begin{align}
    \kappa \defeq \EE\left[
        D_{\chi^2} (\mu^\pi(\cdot)\,\|\, \mu^\pi(\cdot \given X_{-\cS^\star}))\right] \land \EE\left[D_{\chi^2} (\mu^\pi(\cdot \given X_{-\cS^\star})\,\|\, \mu^\pi(\cdot))
    \right] \land 1, 
\end{align}
and $\Omega$ only hides universal constants that do not depend on the model parameters.
Here, $\mu^\pi(x, X_{-\cS^\star})$ denotes the stationary distribution of the Markov chain over token $x$ and its parents $X_{-\cS^\star}$, with $\cS^\star$ being the information set defined in \eqref{eq:define_set_S_star}.

\subsection{Analysis for Stage \RNum{1}} \label{sec:proof_stage1}
In this section, we analyze the dynamics of the parameters $\{c^2_\cS\}_{\cS\in\HleqD}$ in the first stage of training. 
We will show that there is a unique $\cS_*\in\HleqD$ such that $c^2_{\cS^\star}$ dominates all the other $c^2_{\cS}$'s at the end of the first stage. 
In addition, we will characterize how fast this happens and provide a corresponding convergence rate.

\paragraph{Proof Strategy} 
At a high level, the strategy is to analyze $\partial_t \log c_{\mathcal{S}^\star}^2 - \partial_t \log c_\mathcal{S}^2$ for all $\mathcal{S} \neq \mathcal{S}^\star$ via the following steps:
\begin{enumerate}
\item \textbf{\color{OrangeRed}Dynamics Calculation.} 
First, we calculate the dynamics of $\log c_\mathcal{S}^2$ for each fixed $\mathcal{S}$. 
By selecting sufficiently small values for $a$ and $\varepsilon$, and leveraging the mixing properties of the Markov chain with large $L$, the dynamics of $\log c_\mathcal{S}^2$ is approximately governed by the modified mutual information $\tilde I_{\chi^2}(\cS)$.

\item \textbf{\color{OrangeRed}Lower Bound for The Growth Rate.} 
Consequently, we are able to lower bound the difference between the growth rates, 
$\partial_t \log c_{\cS^\star}^2 - \partial_t \log c_\cS^2$, in terms of $\Delta \tilde I_{\chi^2}$, 
the gap between the modified mutual information of $\mathcal{S}^\star$ and the second-best set.

\item \textbf{\color{OrangeRed}Convergence.} 
Finally, we derive the convergence using the above lower bound.
\end{enumerate}

Before presenting the proof, we first remind the readers of a few definitions and notations.
Recall that our simplified model is given by 
\begin{align}
y = (\sigma(a s)X)^\top = \sum_{l=M+1}^{L} \sigma_l(a s) \cdot x_l, 
\where 
s_l = \frac{\sum_{\cS\in \HleqD} c_{\cS}^2 \cdot \prod_{h\in \cS} \langle v_l^\head{h}, v_{L+1}^\head{h} \rangle}{\sum_{\cS\in \HleqD} c_\cS^2}
\end{align}
Also recall that $C_D(t)=\sum_{\cS\in\HleqD} c_\cS^2(t)$ and $p_\cS(t)=c_\cS^2(t)/C_D(t)$ for each $\cS\in\HleqD$.
The loss function can be rewritten as
\begin{align}
\cL = \EE [\ell], \where \ell = -\langle x_{L+1}, \log(y+\varepsilon \bm{1})\rangle.
\end{align}
Here the expectation $\EE$ is taken over both the sequence $(x_1, \ldots, x_{L+1})$ and the Markov kernel $\pi\sim\cP$. 
We abbreviate $\sigma\equiv \sigma(as)$ for convenience and denote by $\sigma_l$ the $l$-th element of $\sigma$.
By direct calculation, we have
\begin{align}\label{eq:grad_1}
    \frac{\partial \ell}{\partial y} = - \frac{x_{L+1}}{y + \varepsilon \bm{1}}, \quad
    \frac{\partial y}{\partial \sigma} =  X^\top, \quad 
    \frac{\partial \sigma}{\partial s_l} =  a\cdot \sigma_l(a s)\cdot (e_l^\top -\sigma),
\end{align}
Then applying the chain rule,  we have 
\begin{align}\label{eq:grad_s_l}
    \frac{\partial \ell}{\partial s_l} = \frac{\partial\ell}{\partial y} \frac{\partial y}{\partial\sigma} \frac{\partial\sigma}{\partial s_l} 
    = -a \rbr{\frac{x_{L+1}}{y + \varepsilon \bm{1}}}^\top \rbr{x_l -y} \cdot \sigma_l(a s).
\end{align}
In addition, 
\begin{align}
    \frac{\partial s_l}{\partial c_\cS} = \frac{2 c_\cS \prod_{h\in \cS} \langle v_l^\head{h}, v_{L+1}^\head{h} \rangle}{\sum_{\cS'\in \HleqD} c_{\cS'}^2} - \frac{2 c_\cS s_l}{\sum_{\cS'\in \HleqD} c_{\cS'}^2}
    = \frac{2c_\cS}{C_D} \bigg(\prod_{h\in\cS}\langle v_l^\h,v_{L+1}^\h\rangle
    - s_l\bigg).
\end{align}

Now, we are ready to present the proof of \Cref{thm:convergence} for the first stage of training.
We remind readers that here only $\{c_\cS\}_{\cS\HleqD}$ are trained, and we 
omit the dependence on $t$ for convenience.

\begin{proof}[Proof of \Cref{thm:convergence}: Stage \RNum{1}]
As discussed in the proof strategy above, we first derive the dynamics of 
$\log c_\cS^2$ for each fixed $\cS\in\HleqD$.
Then we compare the growth rate of $c_{\cS^\star}^2$ with any other $c_\cS^2$.

\paragraph{Calculation of The Dynamics of $\log c_{\cS}^2$}
We fix a $\cS\in\HleqD$ and  apply the chain rule $\partial\ell/\partial c_\cS
=\sum_{l=M+1}^L \partial\ell/\partial s_l \cdot \partial s_l/\partial c_\cS$ and
the gradient flow formula that $\partial_t c_\cS^2  = - 2 c_\cS \cdot {\partial \cL}/{\partial c_\cS}$. We have   
\begin{align}\label{eq:dynamics_c_S}
\partial_t c_\cS^2 & = \frac{4a c_\cS^2}{C_D} \sum_{l=M+1}^{L} \EE \bigg[\sigma_l(a s) \cdot \rbr{\frac{x_{L+1}}{y + \varepsilon \bm{1}}}^\top \rbr{x_l -y} \cdot  \bigg(\prod_{h\in \cS} \langle v_l^\head{h}, v_{L+1}^\head{h} \rangle- s_l\bigg) \bigg]. 
\end{align}
\emph{In the following, we consider a fixed $\pi$ for error analysis and take expectation over $\pi$ again when plugging in everything back into the dynamics.}
To simplify the expression of $\partial_t c_{\cS}^2$, we define quantities $g_{0, \cS}$ and $f$ as
\begin{align}\label{eq:g0} 
\begin{aligned}
    g_{0, \cS} &:= \sum_{l=M+1}^{L} \EE_{X\mid\pi} \bigg[\sigma_l(a s) \sum_{k=1}^d\biggl(\frac{\ind(x_{L+1} = x_l = e_k)}{y(k) + \varepsilon} - \frac{y(k)\ind(x_{L+1} = e_k)}{y(k)+\varepsilon} \biggr) \prod_{h\in \cS} \langle v_l^\head{h}, v_{L+1}^\head{h} \rangle\bigg],\\
    f &:= \sum_{l=M+1}^{L} \EE_{X\mid\pi} \bigg[ \sigma_l(a s) \sum_{k=1}^d\biggl(\frac{\ind(x_{L+1} = x_l = e_k)}{y(k) + \varepsilon} - \frac{y(k)\ind(x_{L+1} = e_k)}{y(k)+\varepsilon} \biggr) \! \cdot \! s_l \bigg].
\end{aligned}
\end{align}
Note that here $f$ does not depend on $\cS$.
Based on the above definitions, we can rewrite \eqref{eq:dynamics_c_S} as
\begin{align}\label{eq:log_dynamics}
    \partial_t \log c_\cS^2 = \frac{1}{c_\cS^2} \cdot \partial_t c_\cS^2 
    = \frac{4a}{C_D}  \cdot \EE_{\pi\sim\cP} [g_{0,\cS} - f].
\end{align}
Using this, it can be shown that $C_D(t) $ does not change during the training, as described in the following lemma.
\begin{lemma}\label{lem:C-preserve}
The quantity $C_D(t) = \sum_{\cS\in \HleqD} c_\cS^2(t)$ is preserved along the 
gradient flow over $\{c_{\cS}\}_{\cS\in\HleqD}$, i.e., 
$\partial_t C_D(t)\equiv 0$.
\end{lemma}
This lemma will be useful in the following analysis, and we defer its proof to \Cref{sec:proof_stage1_add}.
Next, we proceed to further simplify the dynamics in \eqref{eq:log_dynamics} by approximating $g_{0,\cS}$.

\paragraph{Simplification of $\partial_t\log c_\cS^2$}
To approximiate $g_{0,\cS}$, we introduce the following quantities:
\begin{align}
    g_{1,\cS} &:= \frac{1}{L-M} \sum_{l=M+1}^{L}  \EE_{X\mid\pi} \biggl[ \biggl(\sum_{k=1}^d\frac{\ind(x_{L+1} = x_l = e_k)}{\bar y(k) + \varepsilon} - \frac{\bar y(k)\ind(x_{L+1} = e_k)}{\bar y(k)+\varepsilon} \biggr) \prod_{h\in \cS} \langle v_l^\head{h}, v_{L+1}^\head{h} \rangle \biggr], \label{eq:g1}  \\
    g_{2,\cS} &:= \frac{1}{L-M} \sum_{l=M+1}^{L}  \EE_{X\mid\pi} \biggl[ \biggl(\sum_{k=1}^d\frac{\ind(x_{L+1} = x_l = e_k)}{\mu^\pi(e_k)} - 1 \biggr) \prod_{h\in \cS} \langle v_l^\head{h}, v_{L+1}^\head{h} \rangle \biggr], \label{eq:g2} \\
    g_{3,\cS} &:= \EE_{(x,X), (z,Z) \sim \mu^\pi \otimes \mu^\pi} \biggl[ \biggl(\sum_{k=1}^d\frac{\ind(x = z = e_k)}{\mu^\pi(e_k)} - 1 \biggr) \prod_{h\in \cS} \langle v^\h(Z), v^\h(X) \rangle\biggr] \label{eq:g3}, 
\end{align}
where  $Z= (z_{-M}, \ldots, z_{-1})$ is independent of $X= (x_{-M}, \ldots, x_{-1})$
and we define 
\[v^\head{h}(X) := \sum_{i=1}^M \sigma_{-i_h}^\head{h} x_{-i_h}, \quad 
v^\head{h}(Z) := \sum_{i=1}^M \sigma_{-i_h}^\head{h} z_{-i_h}, \quad
\text{and } \bar y := \frac{1}{L-M}\sum_{l=M+1}^{L}x_{l}.\]
Here $\bar y(k) $ is the $k$-th entry of $\bar y$. 
We remark that each of $g_{1,\cS},g_{2,\cS},g_{3,\cS}$ is a function of $\pi$ 
and $t$, but we omit the dependence for brevity.

From $g_{0,\cS}$ to $g_{1,\cS}$, we replace attention probability $\sigma_l(as)$ by the uniform average with factor $1/L$, which yields $\bar y$. 
From $g_{1,\cS}$ to $g_{2,\cS}$, we replace the empirical distribution $\bar y$ with the stationary distribution $\mu^\pi$ and drop  the small constant $\varepsilon$. 
Finally, from $g_{2,\cS}$ to $g_{3,\cS}$, we replace the average over the sequence by the expectation over 
the stationary distribution $\mu^\pi$ of the underlying Markov chain.
We will show that the approximation error in each step is small, given that 
$a$ and $\varepsilon$ are sufficiently small and the Markov chain mixes well for
a large $L$.
\begin{itemize}
\item For the approximation of $g_{0,\cS}$ by $g_{1,\cS}$, note that when $a$ is small,
the attention probability $\sigma_l(a s) \approx 1/(L-M)$ for all $l\in[L]$.
More specifically, it follows from \Cref{lem:approximation1} that
\begin{align}
    |g_{0,\cS} - g_{1,\cS}| \leq \frac{8ad}{\varepsilon^2}.
\end{align}

\item For the approximation of $g_{1,\cS}$ by $g_{2,\cS}$, we leverage the approximation $\bar y(k) \approx \mu^\pi(e_k)$ due to the mixing of the Markov chain for large $L$.
The result in \Cref{lem:approximation2} implies that
\begin{align}
    |g_{1,\cS} - g_{2,\cS}| &\leq 4 \cdot \frac{(1-\lambda)^{-1/2}(D_{\chi^2}(\mu_0\,\|\, \mu^\pi) + 1)^{1/4} + 2\sqrt{M}}{L^{1/2}\gamma} + \gamma^{-1} \varepsilon
\end{align}
where $\mu_0(\cdot)$ is the initial distribution over the first $r_n$ tokens.
Here we abuse the notation of $\mu^\pi$ in $D_{\chi^2}(\mu_0\,\|\, \mu^\pi)$ to denote the stationary distribution over the last $r_n$ tokens.
Since $\mu_{\min}^\pi\geq\gamma$ by \Cref{asp:Markov_chain}, we have 
\begin{align} 
    D_{\chi^2}(\mu_0\|\mu^\pi ) = \sum_{X} {(\mu(X) - \mu^\pi(X))^2}/{\mu^\pi(X)} \leq \sum_{X} 1/\mu^\pi(X)  \leq ({2}/{\gamma})^{r_n}.
    \label{eq:chi2_init_bound}
\end{align}
Therefore, we can further simplify the above bound as
\begin{align}
    |g_{1,\cS}-g_{2,\cS}| = O\bigg(\frac{1}{\sqrt{L (1-\lambda) \gamma^{r_n+2}}} + \frac{\varepsilon}{\gamma}\bigg).
\end{align}

\item Finally, the approximation of $g_{2,\cS}$ by $g_{3,\cS}$ follows from the mixing property of the Markov chain.
In particular, it follows from \Cref{lem:approximation3} that
\begin{align}
    |g_{2,\cS} - g_{3,\cS}| \leq \frac{8M}{L\gamma} + \frac{16\sqrt{D_{\chi^2}(\mu_0\,\|\, \mu^\pi) + 1}}{L (1-\lambda) \gamma^{|\cS|/2+1}}
    \leq  O\bigg(\frac{1}{L (1-\lambda)\gamma^{|\cS|/2+r_n/2+1}} \bigg). 
\end{align}
\end{itemize}
Combining the above results, and by the assumption that $a=a(0)=O(1/L^{3/2})$ 
and $\varepsilon = 1/\sqrt{L}$, we obtain the following approximation error:
\begin{align}
    |g_{0,\cS} - g_{3,\cS}| 
    &= O\left(\frac{ad}{\varepsilon^2}\right) + O\bigg(\frac{1}{\sqrt{L (1-\lambda) \gamma^{r_n+2}}} + \frac{\varepsilon}{\gamma}\bigg) + O\bigg(\frac{1}{L (1-\lambda)\gamma^{|\cS|+2+r_n/2}} \bigg)\\
    & \le O\bigg(\frac{1}{\sqrt{L (1-\lambda)   \gamma^{r_n+2}}} + \frac{1}{L (1-\lambda)\gamma^{D/2+r_n/2+1}} \bigg) \le O\bigg(\frac{1}{\sqrt{L (1-\lambda)   \gamma^{r_n+2}}}\bigg), 
\end{align}
where we note that $|\cS|\leq D$ for any $\cS\in\HleqD$ and the last inequality holds by also noting our condition on $L$ in \eqref{eq:L_condition-1} that 
\(
    L \ge \Omega((1-\lambda)^{-1}\gamma^{-D}).
\)
As a result, the dynamics of $c_\cS^2$ in \eqref{eq:log_dynamics} can be approximated as follows:
\begin{align} \label{eq:log_dynamics_approx}
    \partial_t\log c_\cS^2 = \frac{4a}{C_D} \cdot \EE_{\pi\sim\cP} [g_{3,\cS}-f] + \cE, \quad \text{where } |\cE| 
    \le O\bigg(\frac{a}{C_D\sqrt{L(1-\lambda)\gamma^{r_n+2}}} \bigg),
\end{align}
where $\cO(\cdot)$ hides universal constants that do not depend on the model parameters.
Here and in the sequel,  we let $\cE$ denote an error term that is of the order $O(a / \sqrt{C_D^2 L(1-\lambda)\gamma^{r_n+2}})$ where the specific constant hidden in $O(\cdot) $ may change from line to line, but does not depend on the model parameters.
In fact, we can show $C_D$ remains constant by \Cref{lem:C-preserve} and $a$ is not updated during this stage. 
Thus, the error term $|\cE|$ is of scale $O(a L^{-1/2})$. 

\paragraph{Lower Bound for The Difference $\partial_t\log c_{\cS^\star}^2 - \partial_t\log c_{\cS}^2$}
The reason for approximating $g_{0,\cS}$ by $g_{3,\cS}$ in the previous step is that the latter is more interpretable, in the sense that we can relate it to the modified $\chi^2$ mutual information $\tilde I_{\chi^2}(\cS)$.
Recall that for each $\cS\in\HleqD$, the modified $\chi^2$-mutual information is
\begin{align}
    \tilde  I_{\chi^2}(\cS)= \EE_{\pi\sim\cP, (z,Z)\sim\mu^\pi}
    \bigg[\bigg(\sum_{e\in\cX} \frac{\mu^\pi(z=e \given Z_{-\cS})^2}{\mu^\pi(z=e)} - 1 \bigg) \cdot \mu^\pi(Z_{-\cS}) \bigg].
\end{align}
Note that $f $ in \eqref{eq:log_dynamics_approx} is independent of $\cS$, and will be canceled when computing 
$\partial_t\log c_{\cS^\star}^2 - \partial_t\log c_{\cS}^2$:
\begin{align}
    \partial_t\log c_{\cS^\star}^2 - \partial_t\log c_{\cS}^2 = \frac{4a}{C_D} \cdot \EE_{\pi\sim\cP}[g_{3,\cS^\star} -g_{3,\cS}] \pm 2|\cE|.
\end{align}
Thus, it suffices to consider $\EEpifromP [g_{3,\cS^\star } -g_{3,\cS}]$.
It follows from \Cref{lem:approximation4} that for each $\cS\in\HleqD$, $\EEpifromP [g_{3,\cS}]$ satisfies
\begin{align}
    \bigg| \EEpifromP [g_{3,\cS}] - \prod_{h\in\cS} (\sigma_{-h}^\h)^2 \cdot 
    \tilde I_{\chi^2}(\cS) \bigg| \leq         
    \bigg(1 - \prod_{h\in\cS} (\sigma_{-h}^\h)^2 \bigg)  \cdot 
    \tilde I_{\chi^2}(\cS^\star). 
\end{align}
This yields a lower bound for $\EE_{\pi\sim\cP}[g_{3,\cS^\star}]$ and an upper bound for $\EE_{\pi\sim\cP}[g_{3,\cS}]$ for each $\cS\neq\cS^\star$, i.e.,
\begin{align}\label{eq:approx_g3}
\begin{aligned}
    \EEpifromP [g_{3,\cS^\star}] &\geq \prod_{h\in\cS^\star} (\sigma_{-h}^\h)^2 \cdot \tilde I_{\chi^2}(\cS^\star) - \bigg(1 - \prod_{h\in\cS^\star} (\sigma_{-h}^\h)^2 \bigg)  \cdot  \tilde I_{\chi^2}(\cS^\star),\\
    \EEpifromP [g_{3,\cS}] &\leq \prod_{h\in\cS} (\sigma_{-h}^\h)^2 \cdot 
    \tilde I_{\chi^2}(\cS)
    + \bigg(1 - \prod_{h\in\cS} (\sigma_{-h}^\h)^2 \bigg)  \cdot 
 \tilde I_{\chi^2}(\cS^\star), \quad\text{for all } \cS\neq\cS^\star.
\end{aligned}
\end{align}
Consequently,
\begin{align}
    \partial_t \log c_{\cS^\star}^2 - \partial_t \log c_\cS^2 
    &= \frac{4a}{C_D} \cdot \EE_{\pi\sim\cP}[g_{3,\cS^*} - g_{3,\cS}] \pm 2|\cE|\\
    &\geq \frac{4a}{C_D} \bigg(\prod_{h\in\cS^\star} (\sigma_{-h}^\h)^2 \cdot \tilde I_{\chi^2}(\cS^\star) - \prod_{h\in\cS} (\sigma_{-h}^\h)^2 \cdot \tilde I_{\chi^2}(\cS)\bigg)\\
    &\qquad - \frac{4a}{C_D} \bigg(2 -\prod_{h\in\cS^\star} (\sigma_{-h}^\h)^2 -  \prod_{h\in\cS} (\sigma_{-h}^\h)^2 \bigg) \tilde I_{\chi^2}(\cS^\star) - 2|\cE| \\
    &\geq \frac{4a}{C_D} \bigg( \bigg(2\prod_{h\in\cS^\star} (\sigma_{-h}^\h)^2 -2\bigg) \tilde I_{\chi^2}(\cS^\star) +  \prod_{h\in\cS} (\sigma_{-h}^\h)^2 \cdot\Delta \tilde I_{\chi^2}\bigg) - 2|\cE|,
\end{align}
where the second inequality follows from the definition $\Delta \tilde I_{\chi^2} = \min_{S\in\HleqD \backslash \{\cS^\star\}}\tilde I_{\chi^2}(\cS^\star) - \tilde I_{\chi^2}(\cS)$.
Moreover, since each $(\sigma_{-h}^\h)^2\in(0,1)$,
we have $\prod_{h\in\cS} (\sigma_{-h}^\h)^2 \geq \prod_{h=1}^H (\sigma_{-h}^\h)^2$
for any $\cS\in\HleqD$.
Appling this to the above inequality, we obtain
\begin{align}\label{eq:diff_stage1}
    \partial_t \log c_{\cS^\star}^2 - \partial_t \log c_\cS^2 
    \geq \frac{4a}{C_D} \bigg(2\prod_{h=1}^H (\sigma_{-h}^\h)^2 \cdot \tilde I_{\chi^2}(\cS^\star) + \prod_{h=1}^H (\sigma_{-h}^\h)^2 \cdot\Delta \tilde I_{\chi^2} - 2\tilde I_{\chi^2}(\cS^\star) \bigg) - 2|\cE|, 
\end{align}


\paragraph{Exponential Growth of $c_{\cS^\star}^2$}
We proceed to show that the first term in \eqref{eq:diff_stage1} dominates the error term $\cE$ and thus leads to the exponential growth of $c_{\cS^\star}^2$.

Note that by \Cref{asp:initialization}, $w_{-h}^\h \ge w_{-j}^\h + \Delta w$ for all $j \neq h$ and $h \in[H]$, where the quantity $\Delta w$ satisfies
\begin{align} \label{eq:assump_deltaw}
    \Delta w \geq \log \left({M-1} \right) - \log \bigg(\biggl(1 + \frac{\Delta \tilde I_{\chi^2}}{14 \tilde I_{\chi^2}(\cS^\star)}\biggr)^{\frac{1}{2H}} - 1\bigg).
\end{align}
Recall that we are not updating the RPE parameters during this stage, so $\sigma^\h$ is fixed for all $h\in[H]$.
\textbf{\emph{So the gap condition \eqref{eq:assump_deltaw} holds throughout Stage \RNum{1}.}}
This conditions ensures that $w_{-h}^\h \gg w_{-j}^\h$, so $\prod_{h\in [H]} (\sigma_{-h}^\h)^2$ is sufficiently large.
More precisely, given that head $h$ is more focused on the $(-h)$-th position by having a gap $\Delta w$ in the initialization, 
we can further show by definition of the softmax function that
\begin{align}
    \sigma_{-h}^{(h)} \ge \frac{1}{1 + (M-1)\exp(-\Delta w)},\forall h\in[H] \Rightarrow
    \prod_{h=1}^H (\sigma_{-h}^\h)^2 \geq \frac{1}{\big(1+(M-1)\exp(-\Delta w)\big)^{2H}}.
    \label{eq:prod_sigma-1}
\end{align}
Plugging \eqref{eq:assump_deltaw} into \eqref{eq:prod_sigma-1}, we have by additionally noting that $\tilde I_{\chi^2}(\cS^\star) \geq \Delta \tilde I_{\chi^2} > 0$ that
\begin{align}
    \prod_{h=1}^H (\sigma_{-h}^\h)^2 
    \geq \biggl(1 + \frac{\Delta \tilde I_{\chi^2}}{14 \tilde I_{\chi^2}(\cS^\star)}\biggr)^{-1}
    > \frac{2\tilde I_{\chi^2}(\cS^\star) + 2/ 3 \cdot \Delta \tilde I_{\chi^2}}{2\tilde I_{\chi^2}(\cS^\star) + \Delta \tilde I_{\chi^2}},
\end{align}
which implies that 
\begin{align}
    2\prod_{h=1}^H (\sigma_{-h}^\h)^2 \cdot \tilde I_{\chi^2}(\cS^\star) 
    + \prod_{h=1}^H (\sigma_{-h}^\h)^2 \cdot\Delta \tilde I_{\chi^2}
    - 2\tilde I_{\chi^2}(\cS^\star) \geq \frac{2}{3}\Delta \tilde I_{\chi^2}. \label{eq:signal_delta1}
\end{align}
Moreover, when  $L$ is sufficiently large such that $L \geq 
 \Omega((\Delta \tilde I_{\chi^2}^2 (1-\lambda) \gamma^{r_n+2 })^{-1})$, $\cE $ in  \eqref{eq:diff_stage1} satisfy $|\cE|\leq 13a \Delta \tilde I_{\chi^2}/6C_D$, where $\Omega$ hides a universal constant that does not depend on the model parameters.
 Therefore, combining 
\eqref{eq:diff_stage1} and \eqref{eq:signal_delta1}, 
we conclude that
\begin{align}\label{eq:growth_c_S}
    \partial_t \log c_{\cS^\star}^2 - \partial_t \log  c_\cS^2 \geq \frac{8a \Delta \tilde I_{\chi^2} }{3C_D} - 2|\cE | \geq \frac{a \Delta \tilde I_{\chi^2} }{2C_D}.
\end{align}
This implies that $c_{\cS^\star}^2$ grows exponentially fast and becomes dominant.

\paragraph{Convergence of $p_{\cS^\star}$}
In this part, we treat all the model parameters as a function of time $t$. 
For simplicity, we omit the dependence on $t$ when it is clear from the context.
It remains to derive the convergence of $p_{\cS^\star}=c_{\cS^\star}^2/C_D$.
Expanding $C_D=\sum_{\cS\in\HleqD} c_\cS^2$, we can directly calculate the derivative of $p_{\cS^\star}$ as follows:
\begin{align}
    \partial_t\log(1-p_{\cS^\star}) &= \partial_t\log\bigg(1-\frac{c_{\cS^\star}^2}{\sum_{\cS\in\HleqD} c_{\cS}^2}\bigg) = \frac{C_D}{C_D-c_{\cS^\star}^2} \cdot \partial_t\bigg(1-\frac{c_{\cS^\star}^2}{\sum_{\cS\in\HleqD} c_{\cS}^2}\bigg)\\
    &= \frac{C_D}{C_D-c_{\cS^\star}^2} \cdot \frac{-(\sum_{\cS\in\HleqD}c_{\cS}^2) \cdot \partial_t c_{\cS^\star}^2 + c_{\cS^\star}^2  \cdot \sum_{\cS\in\HleqD}\partial_t c_{\cS}^2}{(\sum_{\cS\in\HleqD} c_\cS^2)^2}\\
    &= \frac{1}{C_D(C_D-c_{\cS^\star}^2)} \sum_{\cS\in\HleqD} (-c_\cS^2  \cdot \partial_t c_{\cS^\star}^2 + c_{\cS^\star}^2 \cdot \partial_t c_\cS^2)\\
    &= \frac{1}{C_D(C_D-c_{\cS^\star}^2)} \sum_{\cS\in\HleqD\setminus\{\cS^\star\}} c_{\cS^\star}^2 \cdot c_\cS^2\cdot  (-\partial_t\log c_{\cS^\star}^2 + \partial_t\log c_\cS^2)
\end{align}
where in the last equality we use the fact that $\partial_t \log c_\cS^2 = (\partial_t c_\cS^2)/c_\cS^2$.
Applying \eqref{eq:growth_c_S} to each $\cS\neq\cS^\star$, we further have
\begin{align}
    \partial_t\log(1-p_{\cS^\star}) &\leq \frac{1}{C_D(C_D-c_{\cS^\star}^2)} \sum_{\cS\in\HleqD\setminus\{\cS^\star\}} c_{\cS^\star}^2 \cdot c_\cS^2 \cdot \bigg(-\frac{a\Delta\tilde I_{\chi^2}}{2C_D}\bigg)\\
    &= \frac{1}{C_D(C_D-c_{\cS^\star}^2)} \cdot c_{\cS^\star}^2\cdot  (C_D-c_{\cS^\star}^2) \cdot \bigg(-\frac{a\Delta\tilde I_{\chi^2}}{2C_D}\bigg) = -\frac{c_{\cS^\star}^2\cdot  a \Delta\tilde I_{\chi^2}}{2C_D^2} <0.
\end{align}
This implies that $p_{\cS^\star}=c_{\cS^\star}^2/C_D$ monotonically increases, and thus $c_{\cS^\star}^2(t)\geq c_{\cS^\star}^2(0)$ for any $t\geq 0$ because $C_D$ is constant by \Cref{lem:C-preserve} and $c_{\cS^\star}^2(0)$ is the initial value for $c_{\cS^\star}^2$ at time $t=0$.
Therefore, we can further replace $c_{\cS^\star}^2$ by its initial value in the above inequality, which yields
\begin{align}
    \partial_t\log(1-p_{\cS^\star}) \leq - \frac{c_{\cS^\star}^2(0) a \Delta \tilde I_{\chi^2}}{2C_D^2} = - \frac{p_{\cS^\star}(0) a \Delta\tilde I_{\chi^2}}{2C_D}
\end{align}
We remark that the above upper bound is independent of $t$.
Finally, applying the Gr\"onwall's inequality to $\log(1-p_{\cS^\star})$, we obtain
\begin{align}
    1 - p_{\cS^\star}(t) \leq (1 - p_{\cS^\star}(0)) \cdot \exp\bigg(-\frac{p_{\cS^\star}(0) a \Delta \tilde I_{\chi^2} }{2C_D} \cdot t \bigg). \label{eq:P_converge_rate}
\end{align}
With training time $t_1 \ge  (2C_D(0)\log L)/(a \cdot p_{\cS^\star}(0)\Delta \tilde I_{\chi^2})$, we can guarantee that 
\begin{align}\label{eq:result_stage_1}
    1 - p_{\cS^\star}(t_1) \le L^{-1}.
\end{align}
This concludes the proof for the first stage of the training.
\end{proof}


\subsubsection{Additional Proofs for the Stage \RNum{1}}
\label{sec:proof_stage1_add}
We conclude this subsection with the proof of \Cref{lem:C-preserve}.
\begin{proof}[Proof of \Cref{lem:C-preserve}]
By \eqref{eq:log_dynamics}, we have 
\begin{align}
    \partial_t c_\cS^2 & = \EEpifromP [{4a \cdot p_\cS} (g_{0,\cS} - f)].
\end{align}
Moreover, by the definition of $g_{0,\cS}$ and $f$, it holds that 
$\sum_{\cS\in \HleqD} p_\cS g_{0,\cS} = f$.
Then,
\begin{align}
    \partial_t C_D = \sum_{\cS\in \HleqD} \partial_t c_\cS^2 = 4 a \cdot 
    \EEpifromP\bigg[ \sum_{\cS\in \HleqD} p_\cS g_{0,\cS} - f\bigg ] \equiv 0.
\end{align}
Thus, the quantity $C_D$ is preserved under the dynamics.
\end{proof}


\subsection{Analysis for Stage \RNum{2}}\label{sec:proof_stage2}

In this section, we provide the analysis of the dynamics of $\sigma^\h  \equiv \sigma(w^\h)$ for head $h \in \cS^\star$.
For head $h \notin \cS^\star$, the results from Stage \RNum{1} imply that $p_\cS \rightarrow 0$ for any $\cS\neq \cS^\star$. 
Consequently, any head $h\notin\cS^\star$ will be ignored when producing the output features of FFN. 
Conversely, for $h \in \cS^\star$, we establish the dominance of $w_{-h}^\h$ over $w_{-i}^\h$  for all $i \neq h$, yielding $\sigma_{-h}^\h \to 1$ as $t \to \infty$. In this limiting case,   head $h$ exactly copies the $(-h)$-th parent. 
We also provide the corresponding convergence rate.

\paragraph{Proof Strategy}
Similar to the proof for Stage \RNum{1}, our analysis for Stage~\RNum{2} characterizes the dynamics of the difference between the positional embedding weights, $\partial_t w_{-h}^\h - \partial_t w_{-i}^\h$ for all $i\neq h$, via the following steps:

\begin{enumerate}
    \item \textbf{\color{OrangeRed}Dynamics Calculation.} We initiate the analysis by deriving the dynamics of $w_{-i}^\h$ for any fixed $i$ and $h$.
    \item \textbf{\color{OrangeRed}Dynamics Approximation} Then we approximate the dynamics by identifying the dominant term controlled by the modified $\chi^2$ mutual information $\tilde I_{\chi^2}(\cS^\star)$.
    \item \textbf{\color{OrangeRed}Lower Bound for The Growth Rate} By comparing the corresponding modified $\chi^2$ mutual information, we establish a lower bound on $\partial_t w_{-h}^{(h)} - \partial_t w_{-i}^{(h)}$ for all $i\neq h$.
    \item \textbf{\color{OrangeRed}Convergence.} Finally, we derive the convergence rate of $\sigma_{-h}^\h$ using the above lower bound.
\end{enumerate}

Again, before proceeding with the detailed proof, we review the notations related to the dynamics of the positional embedding weights $\{w^\h\}_{h=1}^H$.
For the $h$-th head of the first attention layer, the positional embedding vector $w^\h$ induces the attention probability over a window of size $M$, i.e.,
\begin{align}
    \sigma(w^\h) =: \sigma^\h = (\sigma_{-M}^\h, \ldots, \sigma_{-1}^\h) \in\RR^{1\times M}.
\end{align}
Further recall the attention scores for the second attention layer, $as$, where $s=u_{L+1}^\top U_{M+1:L}^\top$.
Then for each $l\in[L]$, the $l$-th coordinate of $s$ is given by
\begin{align}
    s_l &= \sum_{\cS\in [H]_{\le D}} p_\cS \cdot \prod_{h\in \cS} \langle v_l^\head{h}, v_{L+1}^\head{h} \rangle, \quad \text{where each } v_l^\h = \sum_{i=1}^M \sigma^\h_{-i} x_{l-i} = \sigma^\h X_{(l-M):(l-1)}.
\end{align}
Here   $p_{\cS}$ is defined as in the analysis of Stage 1, and $X_{(l-M):(l-1)} \in \RR^{M \times d}$ is the submatrix of $X$ with rows $l-M, \ldots, l-1$.

By direct calculation, we have 
\begin{align}
    \frac{\partial \sigma^\h}{\partial w^\h} &= \diag(\sigma^\h) - (\sigma^\h)^\top \sigma^\h \in\RR^{M\times M}, \quad \frac{\partial v_l^\h}{\partial \sigma^\h} =  X_{(l-M):(l-1)}^\top \in\RR^{d\times M},
\end{align}
Then by chain rule, 
\begin{align}\label{eq:grad_v_l_w_h}
    \frac{\partial v_l^\h}{\partial w^\h} = \frac{\partial v_l^\h}{\partial \sigma^\h} \frac{\partial \sigma^\h}{\partial w^\h} =  X_{l-M:l-1}^\top \left(\diag(\sigma^\h) - (\sigma^\h)^\top \sigma^\h \right) \in\RR^{d\times M}.
\end{align}
Moreover, we can view each $s_{l} $ as a function of $\{ v_1^\h, \ldots, v_{L+1} ^\h \} _{h\in [H]} $.  Differentiating $s_l$ with respect to $v_l^\h$ and $v_{L+1}^\h$, 
we have
\begin{align}\label{eq:grad_s_l_v_l}
\begin{aligned}
    \frac{\partial s_l}{\partial v_l^\h} &= \sum_{\cS\in\HleqD\text{ s.t }h\in \cS} p_\cS \prod_{h'\in \cS\backslash \{h\}}  \langle v_l^\head{h'}, v_{L+1}^\head{h'} \rangle  v_{L+1}^\head{h} \in\RR^{d},\\
    \frac{\partial s_l}{\partial v_{L+1}^\h} &= \sum_{\cS\in\HleqD\text{ s.t }h \in \cS} p_{\cS} \prod_{h'\in \cS\backslash\{h\}} \!\! \langle v_l^\head{h'}, v_{L+1}^\head{h'} \rangle  v_{l}^\head{h} \in\RR^{d\times 1}.
\end{aligned}
\end{align}
  In the summation, we only add those $\cS$'s in $\HleqD$ containing $h$.
Also, recall from \eqref{eq:grad_s_l} that 
\begin{align}
    \frac{\partial \ell}{\partial s_l} &= - a \rbr{\frac{x_{L+1}}{y + \varepsilon \bm{1}}}^\top \rbr{x_l -y} \cdot \sigma_l\left(as\right).
    \label{eq:grad_s_l-1}
\end{align}
Now we are ready to proceed with the analysis for Stage \RNum{2}.

\begin{proof}[Proof of \Cref{thm:convergence}: Stage \RNum{2}]
We start by calculating the explicit expression of the dynamics of $\partial_t w_{-i}^\h$, and then derive approximation of the dynamics, which allows us to further show the convergence of $\sigma^\h$.

\paragraph{Calculation of The Dynamics of $\partial_t w^\h$}
First fix an $h\in[H]$. 
To simplify the notation, for each $l\in[L]$ we define 
\begin{align}
b_l :=  X_{(l-M):(l-1)} \cdot v_{L+1}^\head{h} +  X_{(L+1-M):L} \cdot v_{l}^\head{h}\in \RR^{M}.
\label{eq:def-b_l}
\end{align}
Note that $w^\h $ is the parameters of the $h$-th head and only enters each $v^{(h)} _{l} $, $l = 1, \ldots, L+1$. 
Recall that $s_l = \sum_{\cS\in\HleqD} p_\cS \prod_{h\in \cS }  \langle v_l^\head{h}, v_{L+1}^\head{h} \rangle$, and the RPE weight $w^\h$ for attention head $h$ only influences its outputs $v_l^\h$ and $v_{L+1}^\h$ in the sum. 
It thus follows from the chain rule that for each $i\in[M]$, we have 
\begin{align}
    &\frac{\partial s_l}{\partial w^\h_{-i}}  =   \left(\frac{\partial s_l}{\partial v_{L+1}^\h}\right)^\top \frac{\partial v_{L+1}^\h}{\partial w^\h_{-i}} + \left(\frac{\partial s_l}{\partial v_l^\h}\right)^\top \frac{\partial v_l^\h}{\partial w^\h_{-i}} \\
    &= \sum_{\cS\in\HleqD\text{ s.t }h\in \cS} p_\cS \prod_{h'\in \cS\backslash \{h\}}  \langle v_l^\head{h'}, v_{L+1}^\head{h'} \rangle\cdot  v_{L+1}^{(h)\top} X_{(l-M):(l-1)}^\top \left(\diag(\sigma^\h) - (\sigma^\h)^\top \sigma^\h \right) e_{M+1-i}\\
    &\quad + \sum_{\cS\in\HleqD\text{ s.t }h \in \cS} p_{\cS} \prod_{h'\in \cS\backslash\{h\}} \!\! \langle v_l^\head{h'}, v_{L+1}^\head{h'} \rangle\cdot v_{l}^{(h)\top} X_{(L+1-M):L}^\top \left(\diag(\sigma^\h) - (\sigma^\h)^\top \sigma^\h \right) e_{M+1-i}\\
       & =  \sum_{\cS\in [H]_{\le D}\text{ s.t } h \in \cS} p_\cS \prod_{h'\in \cS\backslash \{h\}}  \langle v_l^\head{h'}, v_{L+1}^\head{h'} \rangle \cdot b_l^\top
        \left(e_{M+1-i} - (\sigma^\h)^\top \right)\cdot \sigma_{-i}^\h,
\end{align}
where we remind readers that $e_{i} \in \RR^{M\times 1}$ is the $i$-th standard basis vector. 

Furthermore, along the gradient flow $\partial_t w_{-i}^\h = - \partial\cL/\partial w_{-i}^\h$, it follows from \eqref{eq:grad_s_l-1} that
\begin{align}
        \partial_t w_{-i}^\h &= -\EE_{\pi,X}\bigg[\sum_{l=M+1}^{L} \frac{\partial \ell}{\partial s_l} \frac{\partial s_l}{\partial w^\h_{-i}}\bigg]
        = a \sum_{l=M+1}^{L} \EE_{\pi,X}\bigg[\sigma_l\left(as\right) \rbr{\frac{x_{L+1}}{y + \varepsilon \bm{1}}}^\top \rbr{x_l -y} \frac{\partial s_l}{\partial w^\h_{-i}} \bigg] \\
        &= a \sum_{l=M+1}^{L} \EE_{\pi,X}\bigg[\sigma_l\left(as\right) \sum_{k=1}^d \biggl(\frac{\ind(x_{L+1} = x_l = e_k)}{y(k) + \varepsilon} - \frac{y(k)\ind(x_{L+1} = e_k)}{y(k)+\varepsilon} \biggr) \frac{\partial s_l}{\partial w^\h_{-i}} \bigg] \\
        &= a \cdot \EEpifromP \left[ g_{h,0}^\top \left(e_{M+1-i} - (\sigma^\h)^\top \right) \sigma_{-i}^\h \right],
\end{align}
where we plug in the expression of $\partial s_{l} / \partial w_{-i}^\h$ above  in the last equality. Here
the vector $g_{h,0}$ is defined as
\begin{align}
    g_{h,0} &:= \!\!\!\! \sum_{l=M+1}^{L} \! \sum_{\cS\in [H]_{\le D} \atop  \text{ s.t } h \in \cS} \!\!\! \EE_{X\mid\pi}\biggl[p_\cS \sigma_l \cdot \sum_{k=1}^d   \biggl(\frac{\ind(x_{L+1} = x_l = e_k)}{y(k) + \varepsilon} - \frac{y(k)\ind(x_{L+1} = e_k)}{y(k)+\varepsilon} \biggr) \cdot \!\!\!\!\!\! \prod_{h'\in \cS\backslash \{h\}}  \langle v_l^\head{h'}, v_{L+1}^\head{h'} \rangle b_l \biggr], 
\end{align}
where $\sigma_l$ is the softmax probability for the $l$-th token in the second attention layer.
Comparing $\partial_t w_{-i}^\h$ and $\partial_t w_{-h}^\h$, we have
\begin{align}
    \partial_t w_{-h}^\h - \partial_t w_{-i}^\h 
    &= a \cdot \EEpifromP \left[ g_{h,0}^\top \left(e_{M+1-h} - (\sigma^\h)^\top \right) \sigma_{-h}^\h - g_{h,0}^\top \left(e_{M+1-i} - (\sigma^\h)^\top \right) \sigma_{-i}^\h \right].\label{eq:diff_stage2-0.0}
\end{align}
Using the fact that $\sum_{j=1}^M\sigma_{-j}^\h=1$, we can rewrite
\begin{align}
    &\left(e_{M+1-h} - (\sigma^\h)^\top \right) \sigma_{-h}^\h - \left(e_{M+1-i} - (\sigma^\h)^\top \right) \sigma_{-i}^\h\\
    &\qquad = \sigma_{-i}^\h (e_{M+1-h}-e_{M+1-i}) + (\sigma_{-h}^\h - \sigma_{-i}^\h) (e_{M+1-h} - (\sigma^\h)^\top).\\
    &\qquad = \sigma_{-i}^\h (e_{M+1-h}-e_{M+1-i}) + (\sigma_{-h}^\h - \sigma_{-i}^\h) \sum_{j=1}^M \sigma_{-j}^\h (e_{M+1-h}-e_{M+1-j}), 
    \label{eq:diff_stage2-0}
\end{align}
where in the first identity, we add and then subtract term $\sigma_{-i}^\h e_{M+1-h}$. 
Combining \eqref{eq:diff_stage2-0.0} and \eqref{eq:diff_stage2-0} yields for each $i\in[M]$ that 
\begin{align}
    &\partial_t w_{-h}^\h - \partial_t w_{-i}^\h  \label{eq:diff_stage2}\\
    &\quad = a \cdot \EEpifromP 
    \bigg[g_{h,0}^\top \bigg(\sigma_{-i}^\h\left(e_{M+1-h} - e_{M+1-i} \right) 
    + \left(\sigma_{-h}^\h - \sigma_{-i}^\h \right) \sum_{j=1}^M \sigma_{-j}^\h (e_{M+1-h}-e_{M+1-j}) \bigg) \bigg].
\end{align}

\paragraph{Simplification of $\partial_t w_{-i}^\h$}
We proceed by deriving approximations to the vector $g_{h,0}$, which will help us identify the dominant term in the dynamics $\partial_t w_{-h}^\h - \partial_t w_{-i}^\h$.
Specifically, we define
\begin{align}
    g_{h,1} &:= \sum_{l=M+1}^{L} \EE_{X\mid\pi} \bigg[\sigma_l(as) \sum_{k=1}^d\biggl(\frac{\ind(x_{L+1} = x_l = e_k)}{y(k) + \varepsilon} - \frac{y(k)\ind(x_{L+1} = e_k)}{y(k)+\varepsilon} \biggr)
    \prod_{h'\in \cS^\star \backslash \{h\}} \langle v_l^\head{h'}, v_{L+1}^\head{h'} \rangle b_l \bigg], \\
    g_{h,2} &:= \frac{1}{L-M}\sum_{l=M+1}^{L}  \EE_{X\mid\pi}\bigg[ \sum_{k=1}^d\biggl(\frac{\ind(x_{L+1} = x_l = e_k)}{\bar y(k) + \varepsilon} - \frac{\bar y(k)\ind(x_{L+1} = e_k)}{\bar y(k)+\varepsilon} \biggr)
    \prod_{h'\in \cS^\star \backslash \{h\}} \!\! \langle v_l^\head{h'}, v_{L+1}^\head{h'} \rangle b_l \bigg], \\    
    g_{h,3} &:=  \frac{1}{L-M}\sum_{l=M+1}^{L}  \EE_{X\mid\pi}\bigg[ \biggl(\sum_{k=1}^d\frac{\ind(x_{L+1} = x_l = e_k)}{\mu^\pi(e_k)} - 1 \biggr)
    \prod_{h'\in \cS^\star \backslash \{h\}} \langle v_l^\head{h'}, v_{L+1}^\head{h'} \rangle b_l \bigg], \\    
    g_{h,4} &:= \EE_{(x,X), (z,Z) \sim \mu^\pi \otimes \mu^\pi} \bigg[ \biggl(\sum_{k=1}^d\frac{\ind(x = z = e_k)}{\mu^\pi(e_k)} - 1 \biggr)
    \prod_{h'\in \cS^\star \backslash \{h\}} \langle v^\hprime(Z), v^\hprime(X) \rangle b(X,Z) \bigg],
\end{align}
where $Z= [z_{-M}, \ldots, z_{-1}]^\top \in \R^{M\times d}$ is an independent copy of $X=[x_{-M},\ldots,x_{-1}]^\top\in\RR^{M\times d}$, and
\begin{gather}
    v^\head{h}(X) := \sum_{i=1}^M \sigma_{-i}^\h x_{-i}, \quad 
    v^\head{h}(Z) := \sum_{i=1}^M \sigma_{-i}^\h z_{-i}, \\
    b(X,Z) :=  Z(v^\head{h}(X)) +  X(v^\head{h}(Z)) , \quad 
    \bar y := \frac{1}{L-M}\sum_{l=M+1}^{L}x_{l}.
\end{gather}
The strategy of gradually approximating $g_{h,0}$ by $g_{h,1},g_{h,2},g_{h,3}$ and $g_{h,4}$ is similar to the analysis in Stage \RNum{1}.
To see the intuition, from $g_{h,0}$ to $g_{h,1}$, we use the fact that $p_{\cS^\star } \approx 1$ and $p_{\cS} \approx 0$ for any other $\cS$, which is a result of Stage 1. 
From $g_{h,1} $ to $g_{h,2}$, we replace $y$ by the empirical mean $\bar y$, thanks to the fact that $\sigma_{l}(a ) \approx 1/ L$ when $a $ is small. 
Then, from $g_{h,2} $ to $g_{h,3}$, we replace the empirical distribution $\bar y$ with the stationary distribution of the Markov chain. These two steps also appear in the analysis of Stage 1. Finally, to go from $g_{h,3}$ to $g_{h,4}$, we leverage the rapid mixing of the Markov chain. 

Note that the common structures in \eqref{eq:diff_stage2-0.0} are $g_{h,0}^\top ( e_{M+1-h} - e_{M+1-i} )$ for $i\neq h$. 
Hence, we only need to understand the approximation error in each step for $g_{h,0}^\top ( e_{M+1-h} - e_{M+1-i} )$.
Recall that we are focusing on $h\in \cS^\star$ in this stage.
\begin{itemize}
\item From $g_{h,0}$ to $g_{h,1}$, we remove the terms in the summation that are weighted down by $p_\cS$ for any $\cS\neq \cS^\star$ due to the rapid dominance of $p_{\mathcal{S}^\star}$ from Stage \RNum{1}. Recall that  $p_{\cS^*} $ converges to one at an exponential rate while all other $p_{\cS}$'s converge to zero. 
For simplicity, let us define 
\begin{align} 
    \rho(\cS) 
    &\defeq \sum_{l=M+1}^{L} \EE_{X\mid\pi} \biggl[\sigma_l(as) \sum_{k=1}^d\biggl(\frac{\ind(x_{L+1} = x_l = e_k)}{y(k) + \varepsilon} - \frac{y(k)\ind(x_{L+1} = e_k)}{y(k)+\varepsilon} \biggr)
    \\
    &\hspace{5cm} \cdot \prod_{h'\in \cS \backslash \{h\}} \langle v_l^\head{h'}, v_{L+1}^\head{h'} \rangle b_l \biggr]
    (e_{M+1-h} - e_{M+1-i}). 
\end{align} 
By the triangular inequality,  we have 
\begin{align}
&\big| (g_{h,0} - g_{h,1})^\top \left( e_{M+1-h} - e_{M+1-i} \right)\big|
= \bigg| \sum_{\cS\in [H]_{\le D}\setminus\{\cS^\star\} \atop \text{ s.t } h \in \cS} p_\cS \cdot \rho(\cS) - \rho(\cS^\star) \bigg| \\
&\hspace{3cm} \le (1 - p_{\cS^\star}) \cdot \big| \rho(\cS^\star)\big| + \sum_{\cS\in [H]_{\le D}\setminus\{\cS^\star\} } p_\cS \cdot \left| \rho(\cS)\right| \le 16 (1-p_{\cS^\star}), 
\label{eq:gh0_approx1}
\end{align}
where in the last line we use the claim that $|\rho(\cS)| \le 8$ for all $\cS$. 
To see this point, note that by definition of $b_l$ in \eqref{eq:def-b_l}, we have
\begin{gather}
    |b_l^\top \left( e_{M+1-h} - e_{M+1-i} \right)| = \bigl|\langle v_{L+1}^\head{h}, x_{l-h} -x_{l-i} \rangle - \langle v_{l}^\head{h},x_{L+1-h} -x_{L+1-i} \rangle\bigr| \le 4, \\
    \bigg|\prod_{h'\in \cS \backslash \{h\}} \langle v_l^\head{h'}, v_{L+1}^\head{h'} \rangle\bigg| \le 1, 
\end{gather}
since $v_l^\h$ and $x_l$ have norm at most $1$. Then, by \cref{lem:boundedness} where we plug in the upper bound $4$ for the function $f(\cdot)$ in the lemma, we conclude that
\(
    \rho(\cS) \leq 8, \quad \forall \cS \in \HleqD \setminus \{\cS^\star\}.
\)
Define $\Delta_1 := 1-p_{\cS^\star}(t_1)$, and $\Delta_1\leq 1/L$ by the results
from Stage \RNum{1}.
Thus, we obtain
\begin{align}\label{eq:gh0_approx_result}
    \big| (g_{h,0} - g_{h,1})^\top (e_{M+1-h} - e_{M+1-i})\big| \leq 16\Delta_1
    \leq 16/L.
\end{align}

\item For the approximation of $g_{h,1}$ by $g_{h,2}$, we use the fact that $\sigma_l(as)\approx 1/L$ when $a$ is sufficiently small. 
Specifically, we also take the absolute bound for $f(\cdot)$ as $4$ in \Cref{lem:approximation1} and obtain 
\begin{align}\label{eq:gh1_gh2_approx_result}
    \left|(g_{h,1} - g_{h,2})^\top \left( e_{M+1-h} - e_{M+1-i} \right) \right| \leq \frac{32ad}{\varepsilon^2}.
\end{align}

\item For the approximation of $g_{h,2}$ by $g_{h,3}$, we use the fact that $\bar y(k) \approx \mu^\pi(e_k)$ for large $L$.
More precisely, it follows from \Cref{lem:approximation2} with the upper bound $4$ for $f(\cdot)$ in the lemma that
\begin{align}
    & \left|(g_{h,2} - g_{h,3})^\top \left( e_{M+1-h} - e_{M+1-i} \right) \right| \leq 16 \cdot \frac{(1-\lambda)^{-1/2}(D_{\chi^2}(\mu_0\,\|\, \mu^\pi) + 1)^{1/4} + 2\sqrt{M}}{L^{1/2}\gamma} + 4\gamma^{-1} \varepsilon.
\end{align}

\item Finally, to go from $g_{h,3}$ to $g_{h,4}$, we leverage the rapid mixing of the Markov chain. Intuitively, when $l$ and $L+1$ are far apart, $x_{l} $ and its parents in $\cS^\star$ are independent of $x_{L+1}$ and its parents in $\cS^\star$. 
This observation yields the approximation of $g_{h,3}^\top \left( e_{M+1-h} - e_{M+1-i} \right)$ by $g_{h,4}^\top \left( e_{M+1-h} - e_{M+1-i} \right)$.
To simplify the notation, define two scalars
\begin{align}
    \tilde g_{h,l} &:= \biggl(\sum_{k=1}^d\frac{\ind(x_{L+1} = x_l = e_k)}{\mu^\pi(e_k)} - 1 \biggr)
    \prod_{h'\in \cS^\star \backslash \{h\}} \langle v_l^\head{h'}, v_{L+1}^\head{h'} \rangle , \\  
    \tilde g_{h,4} &:=  \biggl(\sum_{k=1}^d \frac{\ind(x = z = e_k)}{\mu^\pi(e_k)} - 1 \biggr)
    \prod_{h'\in \cS^\star \backslash \{h\}} \langle v^\hprime(Z), v^\hprime(X) \rangle .
\end{align}
Using the notation above, we have 
\begin{align}
    &\left|(g_{h,3} - g_{h,4})^\top \big( e_{M+1-h} - e_{M+1-i} \big) \right| \\
    &\quad  = \bigg| \bigg(\sum_{l=M+1}^{L} \frac{\EE_{X\mid\pi} \bigl[ \tilde g_{h,l} b_l^\top\bigr] }{L-M}  
    - \EE_{(x,X), (z,Z) \sim \mu^\pi \otimes \mu^\pi} \big[ \tilde g_{h,4} b(X,Z)^\top\big]\bigg) \big( e_{M+1-h} - e_{M+1-i} \big)  \bigg|.
\end{align}
Recall that 
\begin{align} 
    b_l^\top \left( e_{M+1-h} - e_{M+1-i} \right) &= \langle v_{L+1}^\head{h}, x_{l-h} -x_{l-i} \rangle - \langle v_{l}^\head{h},x_{L+1-h} -x_{L+1-i} \rangle, \\
    b(X, Z)^\top \left( e_{M+1-h} - e_{M+1-i} \right) &= \langle v^\head{h}(X), z_{-h} - z_{-i} \rangle - \langle v^\head{h}(Z), x_{-h} - x_{-i} \rangle.
\end{align}
We apply the triangular inequality to obtain that 
\begin{align}
    &\left|(g_{h,3} - g_{h,4})^\top \left( e_{M+1-h} - e_{M+1-i} \right) \right|  \\
    &\qquad \le \left| \frac{1}{L-M}\sum_{l=M+1}^{L} \EE_{X\mid\pi} \big[ \tilde g_{h,l} \langle v_{L+1}^\head{h}, x_{l-h} \rangle  \big] 
    - \EE_{(x,X), (z,Z) \sim \mu^\pi \otimes \mu^\pi} \big[ \tilde g_{h,4} \langle v^\h(Z), x_{-h} \rangle \big] \right| \\
    &\qquad\qquad + \left| \frac{1}{L-M}\sum_{l=M+1}^{L} \EE_{X\mid\pi} \big[ \tilde g_{h,l} \langle v_{L+1}^\head{h}, x_{l-i} \rangle  \big] 
    - \EE_{(x,X), (z,Z) \sim \mu^\pi \otimes \mu^\pi} \big[ \tilde g_{h,4} \langle v^\h(Z), x_{-i} \rangle \big] \right| \\
    &\qquad\qquad + \left| \frac{1}{L-M}\sum_{l=M+1}^{L} \EE_{X\mid\pi} \big[ \tilde g_{h,l} \langle v_{l}^\head{h}, x_{L+1-h} \rangle  \big] 
    - \EE_{(x,X), (z,Z) \sim \mu^\pi \otimes \mu^\pi} \big[ \tilde g_{h,4}\langle v^\h(X), z_{-h} \rangle \big] \right| \\
    &\qquad\qquad + \left| \frac{1}{L-M}\sum_{l=M+1}^{L} \EE_{X\mid\pi} \big[ \tilde g_{h,l} \langle v_{l}^\head{h}, x_{L+1-i} \rangle  \big] 
    - \EE_{(x,X), (z,Z) \sim \mu^\pi \otimes \mu^\pi} \big[ \tilde g_{h,4} \langle v^\h(X), z_{-h} \rangle \big] \right|.
\end{align}
Each term on the right-hand side can be bounded by \Cref{lem:approximation3}, where in the lemma we take $(\sigma^\h)_{h'\in\cS^\star} \in \RR^{M\times|\cS^\star|}$ and $((\sigma^\head{h'})_{h'\in\cS^\star\setminus\{h\}}, e_h)\in \RR^{M\times|\cS^\star|}$ as the two lists of vectors on the $M$-dimensional probability simplex for $\tilde\sigma$ and $\sigma$ respectively.
Consequently, we have
\begin{align}
    \left|(g_{h,3} - g_{h,4})^\top \left( e_{M+1-h} - e_{M+1-i} \right) \right| \le  \frac{8M}{L\gamma} + \frac{16}{L (1-\lambda) \gamma^{|\cS|/2+r_n/2+1}}.
\end{align}
\end{itemize}
Combining the above results and setting $\varepsilon = 1/\sqrt{L}$, $a = a(0) \leq O(1/L^{3/2})$ and together with the conditions in \Cref{eq:L_condition-1}, we have
\begin{align}
    \left|(g_{h,0} - g_{h,4})^\top \left( e_{M+1-h} - e_{M+1-i} \right) \right|  =  |\cE| = O\bigg(\frac{1}{\sqrt{L (1-\lambda)   \gamma^{r_n+2}}} \bigg), 
\end{align}
where $O(\cdot)$ hides universal constants independent of the parameters of the model.
We remark that while the left hand side is a function of $t$, the upper bound is independent of $t$.
Then, we can rewrite \eqref{eq:diff_stage2} as
\begin{align}
    &\partial_t w_{-h}^\h - \partial_t w_{-i}^\h \\
    &\quad = a \cdot \EEpifromP\bigg[  \sigma_{-i}^\h \cdot g_{h,4}^\top \left(e_{M+1-h} - e_{M+1-i} \right) + (\sigma_{-h}^\h - \sigma_{-i}^\h) \cdot \sum_{j=1}^M \sigma_{-j}^\h \cdot g_{h,4}^\top  (e_{M+1-h}-e_{M+1-j})   \bigg]\\
    &\qqquad \qquad \pm a \bigg(\sigma_{-i}^\h + (\sigma_{-h}^\h-\sigma_{-i}^\h)
        \sum_{j=1,j\neq h}^M \sigma_{-j}^\h\bigg) \cdot |\cE|.
    \label{eq:diff_stage2_2}
\end{align}

\paragraph{Lower Bound for The Difference $\partial_t w_{-h}^\h - \partial_t w_{-i}^\h$}
To show $\partial_t w_{-h}^\h - \partial_t w_{-i}^\h >0$, we first derive the lower bound of $\EEpifromP [ g_{h,4}^\top \left( e_{M+1-h} - e_{M+1-i}  \right)]$
for any $i \neq h$. 
Since $(x,X)$ and $(z,Z)$ are independent and identically distributed, by the definition of $b(X,Z)$,
\begin{align}
    &\EE_{\pi\sim\cP}\left[g_{h,4}^\top \left( e_{M+1-h} - e_{M+1-i}  \right)\right]\\
    &\quad= 2\EE_{\pi, (x,X), (z,Z) \sim \mu^\pi \otimes \mu^\pi} \biggl[ \sum_{k=1}^d\biggl(\frac{\ind(x = z = e_k)}{\mu^\pi(e_k)} - 1 \biggr) \prod_{h'\in \cS^\star \backslash \{h\}}\langle v^\hprime(Z), v^\hprime(X) \rangle \cdot  \langle v^\h(X), z_{-h} \rangle \biggr]\\
    &\qquad - 2\EE_{\pi, (x,X), (z,Z) \sim \mu^\pi \otimes \mu^\pi}\biggl[ \sum_{k=1}^d\biggl(\frac{\ind(x = z = e_k)}{\mu^\pi(e_k)} - 1 \biggr) \prod_{h'\in \cS^\star \backslash \{h\}} \langle v^\hprime(Z), v^\hprime(X) \rangle \cdot  \langle v^\h(X), z_{-i} \rangle \biggr]\\
    &\quad= 2\tau_{h,1} - 2\tau_{h,2},
\end{align}
where we introduce the following quantities for convenience: 
\begin{align}
    \tau_{h,1} &:= \EE_{\pi, (x,X), (z,Z) \sim \mu^\pi \otimes \mu^\pi} \bigg[ \sum_{k=1}^d\biggl(\frac{\ind(x = z = e_k)}{\mu^\pi(e_k)} - 1 \biggr)
    \prod_{h'\in \cS^\star \backslash \{h\}}  \langle v^\hprime(Z), v^\hprime(X) \rangle \cdot  \langle v^\h(X), z_{-h} \rangle \bigg], \\
    \tau_{h,2} &:= \EE_{\pi, (x,X), (z,Z) \sim \mu^\pi \otimes \mu^\pi}\bigg[ \sum_{k=1}^d\biggl(\frac{\ind(x = z = e_k)}{\mu^\pi(e_k)} - 1 \biggr)
    \prod_{h'\in \cS^\star \backslash \{h\}}  \langle v^\hprime(Z), v^\hprime(X) \rangle \cdot  \langle v^\h(X), z_{-i} \rangle \bigg].
\end{align}
The quantities  $\tau_{h,1}$ and $\tau_{h,2}$ can be further approximated. Specifically, by applying \Cref{lem:approximation4} to $\tau_{h,1}$, , where in the lemma we take $(\sigma^\h)_{h'\in\cS^\star} \in \RR^{M\times|\cS^\star|}$ and $((\sigma^\head{h'})_{h'\in\cS^\star\setminus\{h\}}, e_h)\in \RR^{M\times|\cS^\star|}$ as the two lists of vectors on the $M$-dimensional probability simplex for $\sigma$ and $\tilde\sigma$ respectively, and we obtain
\begin{align}
  \bigg| \tau_{h,1} - \prod_{h'\in \cS^\star \backslash \{h\}} (\sigma_{-h'}^\hprime)^2 \cdot \sigma_{-h}^\h \cdot \tilde I_{\chi^2}(\cS^\star) \bigg| \leq         
  \bigg(1 - \prod_{h'\in \cS^\star \backslash \{h\}} (\sigma_{-h'}^\hprime)^2 \cdot \sigma_{-h}^\h  \bigg) \tilde I_{\chi^2}(\cS^\star). \label{eq:tau_1_bound}
\end{align}
Drawing on the analagous reasoning as in the proof of \Cref{lem:approximation4}, we can approximate $\tau_{h,2}$ as follows: 
\begin{align}
\bigg| \tau_{h,2} - \prod_{h'\in \cS^\star \backslash \{h\}} (\sigma_{-h'}^\hprime)^2 \cdot \sigma_{-h}^\h \cdot \psi \bigg| \leq         
   \bigg(1 - \prod_{h'\in \cS^\star \backslash \{h\}} (\sigma_{-h'}^\hprime)^2 \cdot \sigma_{-h}^\h  \bigg) \tilde I_{\chi^2}(\cS^\star),\label{eq:tau_2_bound}
\end{align}
where 
\begin{align}
    \psi := \EE_{\pi, (x,X), (z,Z) \sim \mu^\pi \otimes \mu^\pi} \biggl[  \prod_{h'\in \cS^\star \backslash \{h\}} \ind(x_{-h'} = z_{-h'}) \cdot \ind(x_{-h}=z_{-i}) \cdot \biggl(\sum_{k=1}^d\frac{\ind(x = z = e_k)}{\mu^\pi(e_k)} - 1 \biggr) \biggr].
\end{align}

To establish the lower bound for $\tau_{h,1} - \tau_{h,2}$, let us begin by establishing an upper bound for $\psi$, which is approximately equal to $\tau_{h,2}$.   
We invoke \Cref{lem:cross_mutual} with $\cS = \cS^\star$ and $\cS' = \cS^\star \backslash \{h\} \cup \{i\}$ in the lemma to obtain  
\begin{align}
    \psi \le \frac{1}{2}\tilde I_{\chi^2}(\cS^\star) 
    + \frac{1}{2} \tilde I_{\chi^2}(\cS^\star\backslash \{h\}\cup \{i\})
    \le \tilde I_{\chi^2}(\cS^\star) - \frac{1}{2} \cdot \Delta \tilde I_{\chi^2}, \quad \forall i\neq h
\end{align}
Leveraging this for \eqref{eq:tau_1_bound} and \eqref{eq:tau_2_bound},
\begin{align}
    2\tau_{h,1} - 2\tau_{h,2} &\geq \prod_{h'\in \cS^\star \backslash \{h\}} (\sigma_{-h'}^\hprime)^2 \cdot \sigma_{-h}^\h \cdot \Delta \tilde I_{\chi^2} - 4 \bigg(1 - \prod_{h'\in \cS^\star \backslash \{h\}} (\sigma_{-h'}^\hprime)^2 \cdot \sigma_{-h}^\h  \bigg) \tilde I_{\chi^2}(\cS^\star) \\
    & \geq \prod_{h\in \cS^\star} (\sigma_{-h}^\h)^2  \cdot \Delta \tilde I_{\chi^2} - 4 \bigg(1 - \prod_{h\in \cS^\star} (\sigma_{-h}^\h)^2 \bigg ) \tilde I_{\chi^2}(\cS^\star), \label{eq:diff_tau}
\end{align}
where in the second line we multiply an additional $\sigma_{-h}^\h$ to the product as $\sigma_{-h}^\h\in [0,1]$.

Next, we provide a lemma showing that $\partial_t \sigma_{-h}^\h$ is growing for all time $t\geq t_1$, where $t_1$ is the starting time of the second stage.
\begin{lemma}[Reinforced Growth of $\sigma_{-h}^\h$]\label{lem:reinforced_growth}
    For all $h\in \cS^\star$, we have for all $i\neq h$ at any $t\geq t_1$:
    \begin{align} 
        \partial_t \sigma_{-h}^\h > 0, \quad \text{and}\quad \partial_t \log\sigma_{-h}^\h - \partial_t \log\sigma_{-i}^\h = \partial_t w_{-h}^\h - \partial_t w_{-i}^\h >0.
        \label{eq:sigma_growth-stage-2}
    \end{align} 
\end{lemma}
\begin{proof}
    See \Cref{sec:proof_stage2_add} for the proof.
\end{proof}
In the proof of \Cref{lem:reinforced_growth}, we will use the following useful proposition. 
\begin{proposition}\label{prop:dw-diff>0}
    Suppose $\prod_{h\in\cS^\star} (\sigma_{-h}^\h)^2 \geq 1/(1+(M-1)\exp(-\Delta w))^{2|\cS^\star|}$ with $\Delta w$ satisfying \eqref{eq:def_delta_w}, and $\sigma_{-h}^\h > \sigma_{-i}^\h$ for any $i\neq h, h\in\cS^\star$ at a given time $t$. 
    Suppose \Cref{asp:Markov_chain} holds and $L$ satisfies \eqref{eq:L_condition-1}. It holds that
    \begin{equation}
    \begin{gathered}\label{eq:diff_stage2_bound_2}
        \partial_t \log\sigma_{-h}^\h - \partial_t \log\sigma_{-i}^\h = \partial_t w_{-h}^\h - \partial_t w_{-i}^\h \ge \frac{a \Delta\tilde I_{\chi^2}}{2}
        \bigg(\sigma_{-i}^\h + (\sigma_{-h}^\h - \sigma_{-i}^\h) 
        \sum_{j=1,j\neq h}^M \sigma_{-j}^\h\bigg)>0, \\
        \partial_t \sigma_{-h}^\h >0, \quad \quad \forall i\neq h, \quad \forall h\in \cS^\star.
    \end{gathered}
\end{equation}
\end{proposition}
\begin{proof}
    See \Cref{sec:proof_stage2_add} for the proof.
\end{proof}

\Cref{lem:reinforced_growth} implies that during Stage \RNum{2}, for all $i\neq h$ and $ h\in\cS^\star$, we have $w_{-h}^\h > w_{-i}^\h$ and $\sigma_{-h}^\h >\sigma_{-i}^\h$ for all $t\geq t_1$.
In addition, as $\sigma_{-h}^\h$ is growing, all the conditions in \Cref{prop:dw-diff>0} are satisfied for any $t\geq t_1$, and hence all the conclusions in \eqref{eq:diff_stage2_bound_2}.

\paragraph{Convergence of $\sigma^\h$}
Finally, we characterize the convergence rate of \( \sigma^\h \).
For the convergence analysis, we adhere to the convention used in the previous stage, treating all model parameters as functions of the training time \( t \), where \( t = t_1 \) marks the start of the second stage.
With a slight abuse of notation, we denote by \( \sigma_{-i}^\h(t) \) the value of \( \sigma_{-i}(w^\h(t)) \) at time \( t \), where \( w^\h(t) \) is the input to the softmax function, and \( \sigma_{-i}(\cdot) \) refers to the \((M+1-i)\)-th element of the softmax probability.
For simplicity, we sometimes omit the time index \( t \) when the context makes it clear.

Note that $\partial_t \sigma_{-h}^\h > 0$ for all $h \in \cS^\star$.
Hence by the definition of the softmax operation, we have
\begin{align}
    \sigma_{-i}^\h 
    &= \sigma_{-h}^\h \cdot \exp(- (w_{-h}^\h - w_{-i}^\h)) \ge \sigma_{-h}^\h(t_1) \cdot \exp(- (w_{-h}^\h - w_{-i}^\h) \\
    &=\sigma_{-h}^\h(0) \cdot \exp(- (w_{-h}^\h - w_{-i}^\h)), \label{eq:sigmah_lb}
\end{align}
where the first inequality follows from the monotone growth of $\sigma_h^\h$, and 
the second line follows from the fact that the first attention layer is untouched during the first stage.
Note that here in \eqref{eq:sigmah_lb}, $\sigma^\h $ and $w^\h $ are functions of $t$. 
Now, putting together \eqref{eq:diff_stage2_bound_2} and \eqref{eq:sigmah_lb}, and also noting that $\sigma_{-h}^\h > \sigma_{-i}^\h$ for all $i\neq h$ and $h\in\cS^\star$,
it follows that 
\begin{align}
    \partial_t w_{-h}^\h - \partial_t w_{-i}^\h 
    \geq \frac{a\Delta\tilde I_{\chi^2}}{2} \sigma_{-i}^\h
    \geq \frac{a \Delta \tilde I_{\chi^2}}{2} \cdot  \sigma_{-h}^\h(0) \cdot \exp(- (w_{-h}^\h - w_{-i}^\h)) . 
\end{align}
Rearranging the terms, and using the fact that $w_{-h}^\h(t_1)-w_{-i}^\h(t_1) \geq \Delta w$ by \Cref{asp:initialization}, we get
\begin{align}
    \exp\Big(w_{-h}^\h(t) - w_{-i}^\h(t)\Big) &\geq \frac{a \Delta \tilde I_{\chi^2}\cdot \sigma_{-h}^\h(0)}{2} \cdot (t-t_1)+ \exp(\Delta w).
\end{align}
This yields a lower bound for $\sigma_{-h}^\h(t)$ as follows: 
\begin{align}
\sigma_{-h}^\h(t) = \frac{1}{1+ \sum_{i \neq h} \exp( w_{-i}^\h(t) - w_{-h}^\h(t) )} \geq \frac{1}{1+(M-1)\cdot (a\Delta \tilde I_{\chi^2}\cdot \sigma_{\min}(0)\cdot (t-t_1)/2 + \exp(\Delta w))^{-1}},
\end{align}
where we define $\sigma_{\min}(0) := \min_{h \in \cS^\star} \sigma_{-h}^\h(0)$. Consequently, we have
\begin{align}
    1 - \prod_{h\in \cS^\star} (\sigma_{-h}^\h(t))^2 
    &\leq 1- \bigg(\frac{1}{1+(M-1)\cdot (a\Delta \tilde I_{\chi^2}\cdot \sigma_{\min}(0)\cdot (t-t_1)/2 + \exp(\Delta w))^{-1}} \bigg)^{2|\cS^\star|} \\
    &= 1- \bigg( 1 - \frac{(M-1)}{(a\Delta \tilde I_{\chi^2}\cdot \sigma_{\min}(0)\cdot (t-t_1)/2 + \exp(\Delta w)) +(M-1) }\bigg)^{2|\cS^\star|}. 
\end{align}
Now, we consider large $t $ such that 
\begin{align}
    \frac{(M-1)}{(a\Delta \tilde I_{\chi^2}\cdot \sigma_{\min}(0)\cdot (t-t_1)/2 + \exp(\Delta w)) +(M-1) } < \frac{1}{2|\cS^\star|}.
\end{align}
Then, we can apply the inequality $(1 - x)^n \ge 1 - nx$ for $x \in [0,1/n]$ and $n \ge 1$ to obtain
\begin{align}
    1 - \prod_{h\in \cS^\star} (\sigma_{-h}^\h(t))^2 
    &\leq \frac{2|\cS^\star|\cdot (M-1)}{a\Delta \tilde I_{\chi^2}\cdot \sigma_{\min}(0)\cdot (t-t_1)/2 + \exp(\Delta w)+(M-1)}.
\end{align} 
Therefore, with training time $t_2 = 4L|\cS^\star|\cdot (M-1)/a\Delta \tilde I_{\chi^2}\cdot \sigma_{\min}(0) + t_1$, we can ensure that
\begin{align}
    1 - \prod_{h\in \cS^\star} (\sigma_{-h}^\h(t_2))^2 \leq  L^{-1}.
\end{align}
This completes the proof for Stage \RNum{2}.
\end{proof}

\subsubsection{Additional Proofs for Stage \RNum{2}}
\label{sec:proof_stage2_add}
We conclude this subsection with the proof of \Cref{lem:reinforced_growth} and \Cref{prop:dw-diff>0}.
\begin{proof}[Proof of \Cref{prop:dw-diff>0}]
    The condition $\prod_{h\in\cS^\star} (\sigma_{-h}^\h)^2 \geq 1/(1+(M-1)\exp(-\Delta w))^{2|\cS^\star|}$ with $\Delta w$ in \eqref{eq:def_delta_w} implies that
    \begin{align}
        \prod_{h\in\cS^\star} (\sigma_{-h}^\h)^2 \geq \biggl(1 + \frac{\Delta \tilde I_{\chi^2}}{14 \tilde I_{\chi^2}(\cS^\star)}\biggr)^{-1}\geq \frac{4\tilde I_{\chi^2}(\cS^\star) + \frac{2}{3}\Delta \tilde I_{\chi^2}}{ 4\tilde I_{\chi^2}(\cS^\star) + \Delta \tilde I_{\chi^2}}. \label{eq:sigma_lb2}
    \end{align}
    Combining \eqref{eq:diff_tau} and \eqref{eq:sigma_lb2} yields
    \begin{align}
        \EE_{\pi\sim\cP}\left[g_{h,4}^\top \left( e_{M+1-h} - e_{M+1-i}  \right)\right] =  2\tau_{h,1} - 2\tau_{h,2} \geq \frac{2}{3} \Delta \tilde I_{\chi^2} \label{eq:diff_tau_lower}
    \end{align}
    for any $i \neq h$. 
    Applying \eqref{eq:diff_tau_lower} to \eqref{eq:diff_stage2_2}, since each 
    $\sigma_{-i}^\h>0$ and $\sigma_{-h}^\h > \sigma_{-i}^\h$ at time $t$ for
    all $i\neq h, h\in\cS^\star$, it holds that
    \begin{align}\label{eq:diff_stage2_bound_1}
        \partial_t w_{-h}^\h - \partial_t w_{-i}^\h \geq 
        a\bigg(\sigma_{-i}^\h + (\sigma_{-h}^\h - \sigma_{-i}^\h) \cdot  
        \sum_{j=1,j\neq h}^M \sigma_{-j}^\h\bigg)\cdot 
        \bigg(\frac{2}{3}\Delta\tilde I_{\chi^2} - |\cE|\bigg).
    \end{align}
    Then since we assume a sufficiently large $L \geq 
    \Omega((\Delta \tilde I_{\chi^2}^2 (1-\lambda) \gamma^{r_n+2 })^{-1})$, it holds that
    $|\cE|\leq \Delta\tilde I_{\chi^2}/6$, we further have 
    \begin{align}
        \partial_t w_{-h}^\h - \partial_t w_{-i}^\h \geq 
        \frac{a \Delta\tilde I_{\chi^2}}{2}
        \bigg(\sigma_{-i}^\h + (\sigma_{-h}^\h - \sigma_{-i}^\h) 
        \sum_{j=1,j\neq h}^M \sigma_{-j}^\h\bigg)>0, \quad \forall i\neq h, \quad \forall h\in \cS^\star.
    \end{align}
    As $\partial_t \log\sigma_{-h}^\h - \partial_t \log\sigma_{-i}^\h = \partial_t w_{-h}^\h - \partial_t w_{-i}^\h>0$ by property of the softmax function, and $\sum_{i=1}^M \partial_t \sigma_{-i}^\h = 0$, we have $\partial_t \sigma_{-h}^\h > 0$ for all $h\in \cS^\star$.
    This completes the proof of \Cref{prop:dw-diff>0}.
\end{proof}

\begin{proof}[Proof of \Cref{lem:reinforced_growth}]
    We give a proof to \Cref{lem:reinforced_growth} by contradiction.
    Note that at the beginning of the second stage $t=t_1$, we have all the conditions for \Cref{prop:dw-diff>0} satisfied by the initialization conditions in \Cref{asp:initialization}.
    Then, by \Cref{eq:diff_stage2_bound_2} in \Cref{prop:dw-diff>0}, we have $\partial_t \log\sigma_{-h}^\h - \partial_t \log\sigma_{-i}^\h > 0$ and $\partial_t \sigma_{-h}^\h > 0$ for all $i\neq h$ and $h\in \cS^\star$ at $t=t_1$.

    Next, assume that $\tau>t_1$ is the smallest time such that at least $\partial_t \sigma_{-h}^\h \le 0$ or $\partial_t \log\sigma_{-h}^\h - \partial_t \log\sigma_{-i}^\h \le 0$ for some $i\neq h$ and $h\in \cS^\star$. 
    By definition of $\tau$, we have \eqref{eq:sigma_growth-stage-2} holds for any moment $t\in [t_1, \tau)$.
    As $\sigma_{-h}^\h$ and the gap $\sigma_{-h}^\h - \sigma_{-i}^\h$ are monotonically increasing, we have by the initialization condition and the boundedness of the gradient that at time $\tau$: 
    \begin{align}
        \prod_{h\in\cS^\star} (\sigma_{-h}^\h)^2 \geq 1/(1+(M-1)\exp(-\Delta w))^{2|\cS^\star|}, \quad\text{and}\quad \sigma_{-h}^\h > \sigma_{-i}^\h, \quad \forall i\neq h, \quad \forall h\in \cS^\star.
    \end{align}
    Hence, by \Cref{prop:dw-diff>0}, we have $\partial_t \log\sigma_{-h}^\h - \partial_t \log\sigma_{-i}^\h > 0$ and $\partial_t \sigma_{-h}^\h > 0$ for all $i\neq h$ and $h\in \cS^\star$ at time $\tau$, which contradicts the definition of $\tau$.
    This completes the proof of \Cref{lem:reinforced_growth}.
\end{proof}


\subsection{Analysis for Stage \RNum{3}}
\label{sec:proof_stage3}
In this section, we derive the dynamics of the second attention layer's weights $a$ in Stage \RNum{3}. 
We characterize the dynamics of $a$ when $a < O(\log L)$, where the signal term of the dynamics dominates the approximation error. 
We provide the growth rate of the weights for two regimes: when $a$ is either sufficiently small or large. 

\paragraph{Proof Strategy} 
We analyze the dynamics of $a$ via the following steps:
\begin{enumerate}
    \item \textbf{\color{OrangeRed}Dynamics Calculation.} First, we derive the explicit expression for the dynamics of $a$.
    \item \textbf{\color{OrangeRed}Dynamics Approximation.} We approximate the dynamics by exploiting the mixing properties of the Markov chain and the convergence of the weights from Stage \RNum{1} and \RNum{2}.
    \item \textbf{\color{OrangeRed}Lower and Upper Bound for The Growth Rate.} Finally, we establish the upper and lower bounds for the growth rate of $a$ when $a$ is either sufficiently small or large.
\end{enumerate}

For a set $\cS \subseteq [M]$, we denote $X_{l-\cS}:= (x_{l-s}: s\in\cS)$.
If $l=0$, we will ignore $l$ in the subscript and simply use $X_{-\cS}$. 
In this section, we abbreviate $p_\cS(t_1)$ after the first stage's training as $p_\cS$, and $\sigma_{-i}^\h(t_2)$ after the second stage's training as $\sigma_{-i}^\h$.

\begin{proof}[Proof of \Cref{thm:convergence}: Stage \RNum{3}]
We start with the explicit expression of the dynamics of $a$.
\paragraph{Calculation of The Dynamics of $a$}
First by the chain rule,
\begin{align}
    \frac{\partial\ell}{\partial a} = \sum_{l=M+1}^{L} \frac{\partial \ell}{\partial (a s_l)} \frac{\partial (a s_l)}{\partial a}  
    = -\sum_{l=M+1}^{L} \rbr{\frac{x_{L+1}}{y + \varepsilon \bm{1}}}^\top \rbr{x_l -y} \cdot \sigma_l(a s) \cdot s_l. 
\end{align}
where in the last equality we remind readers of the same procedure as we have used in the derivation of \eqref{eq:grad_s_l} in Stage \RNum{1}.
Then, taking expectation with respect to $X$ and $\pi$ and expanding 
$s_l= a \sum_{\cS\in\HleqD} p_{\cS} \prod_{h\in \cS} \langle v_l^\head{h}, v_{L+1}^\head{h} \rangle$, we have
\begin{align}
    \partial_t a &= -\frac{\partial\cL}{\partial a}
    =\EE\bigg[\sum_{l=M+1}^{L} \rbr{\frac{x_{L+1}}{y + \varepsilon \bm{1}}}^\top \rbr{x_l -y} \cdot \sigma_l\left(as\right) \cdot s_l \bigg] \\
    &= \EE\bigg[\sum_{\cS\in \HleqD}p_\cS \sum_{l=M+1}^{L} \sigma_l \sum_{k=1}^d \biggl(\frac{\ind(x_{L+1} = x_l = e_k)}{y(k) + \varepsilon} - \frac{y(k)\ind(x_{L+1} = e_k)}{y(k)+\varepsilon} \biggr) \prod_{h\in \cS} \langle v_l^\head{h}, v_{L+1}^\head{h} \rangle \bigg] =: f_0
\end{align}
We remind readers the shorthand $\sigma\equiv\sigma(as)$.
We denote the above quantity by $f_0$. 

\paragraph{Approximation of $\partial_t a$}
Similar to the analysis for the previous two stages, we develop a sequence of approximation steps that transforms 
$\partial_t a$ into a tractable quantity. We aim to decouple $x_{L+1} $ and $x_{l}$, approximate $s_{l}$ by a population version, and transform the expectation to one under the stationary distribution of the Markov chain. 
Specifically, the approximation involves the following steps:

\begin{itemize}
\item Our first step is to remove the summation over $\HleqD\backslash\{\cS^\star\}$ where $\cS^\star$ is the optimal set that maximizes the modified mutual information defined in \eqref{def:modified_chi_square_mi}.
This is because $c_{\cS^\star}$ dominates by the analysis of Stage \RNum{1}.
Specifically, we define
\begin{align}
    f_1 \defeq \EE\left[\sum_{l=M+1}^{L} \sigma_l \sum_{k=1}^d \biggl(\frac{\ind(x_{L+1} = x_l = e_k)}{y(k) + \varepsilon} - \frac{y(k)\ind(x_{L+1} = e_k)}{y(k)+\varepsilon} \biggr) \prod_{h\in \cS^\star} \langle v_l^\head{h}, v_{L+1}^\head{h} \rangle \right]. 
\end{align}
To bound $|f_0-f_1|$, note that for any $\cS\in \HleqD$, since each $v_l^\h$ has norm at most $1$, 
we can invoke \cref{lem:boundedness} with $C=1$ and obtain 
\begin{align}
    \left|\sum_{l=M+1}^{L} \sigma_l \sum_{k=1}^d \biggl(\frac{\ind(x_{L+1} = x_l = e_k)}{y(k) + \varepsilon} - \frac{y(k)\ind(x_{L+1} = e_k)}{y(k)+\varepsilon} \biggr) \prod_{h\in \cS^\star} \langle v_l^\head{h}, v_{L+1}^\head{h} \rangle \right| \le 2. 
\end{align} 

It follows that 
\begin{align}
    \!\!\! \!\!\! \!\!\! \left|f_0 - f_1\right| &= \EE\biggl[\sum_{\cS\in \HleqD\setminus\{\cS^\star\}}\!\!\!\!\!\!\! p_\cS \sum_{l=M+1}^{L} \sigma_l \sum_{k=1}^d \ind(x_{L+1} = e_k)\biggl(\frac{\ind(x_l = e_k)}{y(k) + \varepsilon} - \frac{y(k)}{y(k)+\varepsilon} \biggr) \prod_{h\in \cS} \langle v_l^\head{h}, v_{L+1}^\head{h} \rangle \biggr]\\
    &\quad + (1-p_{\cS^\star}) \left|\EE\biggl[\sum_{l=M+1}^{L} \sigma_l \sum_{k=1}^d \ind(x_{L+1} = e_k)\biggl(\frac{\ind(x_l = e_k)}{y(k) + \varepsilon} - \frac{y(k)}{y(k)+\varepsilon} \biggr) \prod_{h\in \cS^\star} \langle v_l^\head{h}, v_{L+1}^\head{h} \rangle \biggr]\right| \\
    &\leq 4(1 - p_{\cS^\star}(t_1)) 
    =2 \Delta_1, \quad \where \Delta_1\defeq (1 - p_{\cS^\star}(t_1)).
\end{align}
In summary, the difference between $f_0$ and $f_1$ is controlled by the convergence results from Stage \RNum{1}.

\item 
Our second step is to characterize the approximation error incurred by the difference between the ideal attention scores and the actual attention scores in the second attention layer.
Let us define $s_l^\star = \prod_{h\in\cS^\star} \ind(x_{l-h} = x_{L+1-h})$ as the ideal attention score for the second attention layer. We invoke \cref{lem:misspecification} to have for all $l\in[L]$, 
\begin{align}
    |s_l -  s_l^\star| \leq \Delta_1 + \Delta_2, \where \Delta_2 \defeq 1 - \prod_{h\in\cS^\star} (\sigma_{-h}^\h(t_2))^2.
    \label{eq:stage3_approx_1}
\end{align}
Corresponding to $\{s_l^\star\}_{l=M+1}^L$, we define
\begin{align}
    \sigma_l^\star := \frac{\exp\left(a \prod_{h\in\cS^\star}\ind(x_{l -  h} = x_{L+1-h})\right) }{\sum_{l'=M+1}^L \exp\left(a \prod_{h\in\cS^\star}\ind(x_{l' -  h} = x_{L+1-h})\right)}, \quad y^\star (k)\defeq \sum_{l=M+1}^{L} \sigma_l^\star \ind(x_l = e_k), \quad \forall k\in[d].
\end{align}
In the vector form, we have $y^\star = \sum_{l=M+1}^{L} \sigma_l^\star x_l$.  
Leveraging the above approximations, we define an approximation of $f_1$ as 
\begin{align}
    f_2  := \EE\bigg[\sum_{l=M+1}^{L} \sigma_l^\star \sum_{k=1}^d \biggl(\frac{\ind(x_{L+1} = x_l = e_k)}{y^\star(k) + \varepsilon} - \frac{y^\star(k)\ind(x_{L+1} = e_k)}{y^\star(k)+\varepsilon} \biggr) \prod_{h\in \cS^\star} \ind(x_{l-h}=x_{L+1-h}) \bigg]. 
\end{align}
Applying \cref{lem:approximation_error-3-1}, it holds that
\begin{align}
    |f_1 - f_2|\leq 12 \cdot (1  + a(t)\cdot \varepsilon^{-1}) \cdot (\Delta_1 + \Delta_2)
\end{align}
In summary, this error terms captures the difference between the ideal weights and the actual weights obtained by gradient flow at the end of Stage \RNum{2}.

\item Note that $y^\star(k)$ is also random due to the randomness in $\sigma_l^\star$, and as $L$ is sufficiently large, we want to replace $y^\star(k)$ with its population counterpart. 
Let $z\in \cX$ and $Z=(z_{-M}, \dots, z_{-1}) \in \cX^M$ be two random variables and we define similarly for $x\in \cX$ and $X=(x_{-M}, \dots, x_{-1}) \in \cX^M$.
To this end, we define a reweighed distribution
\begin{align}\label{eq:tilde_mu}
    \tilde \mu^\pi(z, Z\given X_{-\cS^\star}) = \frac{\mu^\pi(z, Z) \exp\left(
        a \prod_{h\in\cS^\star}\ind(z_{-  h} = x_{-h})
    \right)}{\sum_{z', Z'} \mu^\pi(z', Z') \exp\left(
        a \prod_{h\in\cS^\star}\ind(z_{-h}' = x_{-h})
    \right)}, 
\end{align}
where $\mu^\pi$ is the stationary distribution of the Markov chain over a window of size $M+1$. 
This can be viewed as a reweighting of the stationary distribution over $(z, Z)$ by an exponential term that depends on the sequence $X_{-\cS^\star}$.
We use  $\tilde \mu^\pi(z=e_k\given X_{L+1-\cS^\star}) $  to replace $y^\star (k)$ and 
define $f_3$ as
\begin{align}
    f_3 \defeq \EE\bigg[\sum_{l=M+1}^{L} \sigma_l^\star \sum_{k=1}^d \biggl(\frac{\ind(x_{L+1} = x_l = e_k)}{\tilde \mu^\pi(z = e_k\given X_{L+1-\cS^\star})} - {\ind(x_{L+1} = e_k)} \biggr) \prod_{h\in \cS^\star} \ind(x_{l-h}=x_{L+1-h}) \bigg].
\end{align}
One can immediately draw a connection to \cref{lem:approximation2} as both targets characterize the gap between the empirical and population distributions.
The only difference is that this time we have the distribution reweighed by some exponential term. 
For completeness, we provide the approximation result in \cref{lem:approximation5}, which bounds the difference between $f_2$ and $f_3$ as
\begin{align}
    |f_2 - f_3|\leq \frac{8 (1-\lambda)^{-1/2} (D_{\chi^2}(\mu_0\,\|\, \mu^\pi) + 1)^{1/4} + 8\sqrt{M}}{L^{1/2}\cdot \gamma^{|\cS^\star|+1}} + \frac{2d\varepsilon}{\gamma} \lesssim \frac{\sqrt M + d}{L^{1/2}(1-\lambda)^{1/2} \gamma^{|\cS^\star|+1 + r_n/4}}. 
\end{align}
where $\mu_0(\cdot)$ is the initial distribution for the first $r_n$ tokens in the Markov chain.
Here and in the sequel, we simply use $D_{\chi^2}(\mu_0 \,\|\, \mu^\pi)$ to denote $D_{\chi^2}(\mu_0(X_{1:r_n}=\cdot ) \,\|\, \mu^\pi(X_{1:r_n}=\cdot ))$ when it is clear from the context. 
In the last inequality, we use the fact that $D_{\chi^2}(\mu_0\,\|\, \mu^\pi) \le \gamma^{-r_n}$ by \eqref{eq:chi2_init_bound} and the condition $\varepsilon = L^{-1/2}$. 

\item 
Note that in the expression of $f_3$, each $\sigma_l^\star$ still implicitly depends on the actual value of the sequence $X$. 
Since $L$ is large and the Markov chain is well-mixed, we can approximate $\sum_{l=M+1}^{L} \sigma_l^\star \ind((x_l, X_{l-\cS^\star}) = (\cdot, \cdot))$ by $\tilde\mu^\pi(\cdot, \cdot\given X_{L+1-\cS^\star})$.
This gives rise to the following approximation of $f_3$:
\begin{align}
    \!\!\! f_4 
    &\defeq \EE_{\pi, X, Z\sim \tilde \mu^\pi(\cdot \given X_{L+1-\cS^\star})}\Biggl[\sum_{k=1}^d \biggl(\frac{\ind(x_{L+1} = z = e_k)}{\tilde \mu^\pi(z = e_k\given X_{L+1-\cS^\star})} - {\ind(x_{L+1} = e_k)}  \biggr) \cdot \ind(Z_{l-\cS^\star}=x_{L+1-\cS^\star})  \Biggr] \\
    &=\EE_{\pi, X, Z\sim \tilde \mu^\pi(\cdot \given X_{L+1-\cS^\star})}\Biggl[\sum_{k=1}^d \frac{\mu^\pi(x=e_k\given X_{-\cS^\star})\tilde\mu^\pi(z=e_k, Z_{-\cS^\star}=X_{-\cS^\star}\given X_{-\cS^\star})}{\tilde \mu^\pi(z=e_k\given X_{-\cS^\star})} \\
    & \hspace{8cm} -  \tilde\mu^\pi(Z_{-\cS^\star}=X_{-\cS^\star}\given X_{-\cS^\star}) \Biggr]
\end{align}
Applying \cref{lem:sigma_approx} yields
\begin{align}
    |f_3 - f_4| \le \!\!\!\!\! \sup_{\pi\in\supp(\cP)} \!\! \frac{8\gamma^{-1} (1-\lambda)^{-1/2} (D_{\chi^2}(\mu_0\,\|\, \mu^\pi) + 1)^{1/4} + 16\gamma^{-1} \sqrt{M}}{L^{1/2}\cdot \gamma^{|\cS^\star|+1}}\lesssim \frac{\sqrt M + d}{L^{1/2}(1-\lambda)^{1/2} \gamma^{|\cS^\star|+2+ r_n/4}}, 
\end{align}
where we use the fact that $D_{\chi^2}(\mu_0\,\|\, \mu^\pi) \le \gamma^{-r_n}$ by \eqref{eq:chi2_init_bound}.

\item 
Let $(z,Z)\sim \tilde\mu^\pi(\cdot \given X_{L+1-\cS^\star})$.
Since $L$ is large, the distribution of $(x_{L+1},X_{L+1-\cS^\star})$ is close to the stationary distribution $\mu^\pi$.
Thus, we introduce the following approximation of $f_4$:
\begin{align}
    f_5 &:= \EE_{\pi, (x, X_{-\cS^\star})\sim\mu^\pi,(z, Z)\sim \tilde\mu^\pi(\cdot\mid X_{-\cS^\star})} \left[\sum_{k=1}^d \biggl(\frac{\ind(x = z = e_k)}{\tilde \mu^\pi(e_k\given X_{-\cS^\star})} - {\ind(x = e_k)}  \biggr) \prod_{h\in \cS^\star} \ind(z_{-h}=x_{-h})  \right]\\
    &= \EE_{\pi, (x, X_{-\cS^\star})\sim \mu^\pi} \Biggl[\sum_{k=1}^d \frac{\mu^\pi(x=e_k\given X_{-\cS^\star})\tilde\mu^\pi(z=e_k, Z_{-\cS^\star}=X_{-\cS^\star}\given X_{-\cS^\star})}{\tilde \mu^\pi(z=e_k\given X_{-\cS^\star})} \\
    & \hspace{8cm} -  \tilde\mu^\pi(Z_{-\cS^\star}=X_{-\cS^\star}\given X_{-\cS^\star}) \Biggr].\label{eq:f_5_stage_3}
\end{align}
Note that 
\begin{align}
    &\left|\sum_{k=1}^d \frac{\mu^\pi(x=e_k\given X_{-\cS^\star})\tilde\mu^\pi(z=e_k, Z_{-\cS^\star}=X_{-\cS^\star}\given X_{-\cS^\star})}{\tilde \mu^\pi(z=e_k\given X_{-\cS^\star})}\right| \\
    &\quad =\left|\sum_{k=1}^d \mu^\pi(x=e_k\given X_{-\cS^\star})
    \cdot \tilde\mu^\pi(Z_{-\cS^\star}=X_{-\cS^\star}\given X_{-\cS^\star}, z=e_k)\right| \le \left|\sum_{k=1}^d \mu^\pi(x=e_k\given X_{-\cS^\star})\right| = 1,
\end{align}
and so is $|\tilde\mu^\pi(Z_{-\cS^\star}=X_{-\cS^\star}\given X_{-\cS^\star})|\le 1$.
The difference between $f_4$ and $f_5$ is thus bounded by $2\norm{p^\pi (x_{L+1}, X_{L+1-\cS^\star}=\cdot, \cdot ) - \mu^\pi (x_{L+1}, X_{L+1-\cS^\star}==\cdot)}_{\TV}$ and by 
the results in \eqref{eq:p-2-mu-4} of \cref{lem:p-2-mu}:
\begin{align}
    |f_4 -f_5| 
    \le 2\cdot \sup_{\pi\in\supp(\cP)} \lambda^{L - M} \sqrt{D_{\chi^2}(\mu_0\,\|\, \mu^\pi) + 1} \lesssim \frac{\lambda^{L-M}}{\gamma^{r_n/2}} \le L^{-1}, 
\end{align}
where we use $D_{\chi^2}(\mu_0\,\|\, \mu^\pi) \le \gamma^{-r_n}$ and the condition on $L$ in \eqref{eq:L_condition-2}.
\end{itemize}
Collecting all the above approximation steps, we obtain (where we use $\lesssim$ to hide absolute constants)
\begin{align}
    |f_0 - f_5| 
    &\lesssim \Delta_1 + (1  + a\cdot \varepsilon^{-1}) \cdot (\Delta_1 + \Delta_2) + L^{-1} + \frac{\sqrt M + d}{L^{1/2}(1-\lambda)^{1/2} \gamma^{|\cS^\star|+2+ r_n/4}} \\
    &\lesssim a\cdot L^{-1/2} + \frac{\sqrt M + d}{L^{1/2}(1-\lambda)^{1/2} \gamma^{|\cS^\star|+2+ r_n/4}}.
\end{align}
where the last line holds by moting that with sufficiently large $t_1$ and $t_2$  we have $\Delta_1 + \Delta_2 \le L^{-1}$, and $\varepsilon = L^{-1/2}$.
Here, express the error in terms of the trainable parameter $a$ and define 
\begin{align}
    \xi(a) \asymp \frac{\sqrt M + d}{L^{1/2}(1-\lambda)^{1/2} \gamma^{|\cS^\star|+2+ r_n/4}} + a\cdot L^{-1/2}.
\end{align}
In particular, we have for $a = O(\log L)$ that 
\begin{align}
    \xi(a) &
   =  O\bigg(\frac{\sqrt M + d}{L^{1/2}(1-\lambda)^{1/2} \gamma^{|\cS^\star|+2+ r_n/4}} +  \frac{\log L}{L^{1/2}} \bigg).
   \label{eq:xi_psi_ub}
\end{align}

In a nutshell, we conclude that when the weight $a$ satisfies $a < O(\log L)$, the dynamics of $a$ can be approximated by
\begin{align}
    \partial_t a = f_5 \pm \xi(a).
    \label{eq:dt_a-1}
\end{align}
The following proposition helps us reformulate $f_5$ in a form that facilitates the analysis of the dynamics of $a$.
\begin{proposition}\label{prop:f_5}
The term $f_5$ can be reformulated as
\begin{align}
    f_5 = \EE_{\pi, X_{-\cS^\star}\sim \mu^\pi}\left[ J(X_{-\cS^\star}; a, \pi) \cdot e^a \cdot \bigl( r^\pi(X_{-\cS^\star}) \bigr) ^3 \cdot \mu^\pi(X_{-\cS^\star})\right],
\end{align}
where $r^\pi(X_{-\cS^\star};a) = (1 + \mu^\pi(X_{-\cS^\star}) \cdot (e^a -1))^{-1}$ is the inverse of the normalization factor of $\tilde\mu^\pi$ in \eqref{eq:tilde_mu} and
\begin{align}
    J(X_{-\cS^\star}; a, \pi) = \sum_{k\in[d]} \frac{(\mu^\pi(x=e_k\given X_{-\cS^\star}) - \mu^\pi(x=e_k))^2 }{(1-r^\pi(X_{-\cS^\star};a)) \cdot \mu^\pi(x=e_k \given X_{-\cS^\star}) +  r^\pi(X_{-\cS^\star};a) \cdot \mu^\pi(x=e_k)}.
\end{align} 
\end{proposition}
\begin{proof}
    See \Cref{sec:add_proof_stage3} for the proof.
\end{proof}

Inspired by this form, we define an alternative function $\tilde J(\cdot; r, \pi)$ as
\begin{align}
    \tilde J(X_{-\cS^\star}; r, \pi) \defeq \sum_{k\in[d]} \frac{(\mu^\pi(x=e_k\given X_{-\cS^\star}) - \mu^\pi(x=e_k))^2 }{(1-r) \cdot \mu^\pi(x=e_k \given X_{-\cS^\star}) +  r \cdot \mu^\pi(x=e_k)}, \quad  r\in[0, 1]
    \label{eq:tilde_J}
\end{align}
where we replace $r^\pi(X_{-\cS^\star};a)$ by a parameter $r\in[0, 1]$.
As exactly calculating the inverse normalization factor $r^\pi(X_{-\cS^\star};a)$ is intractable, we instead seek to find an upper and lower bound for $r^\pi(X_{-\cS^\star};a)$ and plug them into $\tilde J(\cdot; r, \pi)$ to bound $f_5$
Suppose that $r^\pi(X_{-\cS^\star};a)$ enjoys the following parameter-dependent upper and lower bounds:
\begin{align}
    r_-(a) \le r^\pi(X_{-\cS^\star};a) \le r_+(a), \quad \forall X_{-\cS^\star}\in \cX^{|\cS^\star|}, \quad \forall \pi\in\supp(\cP).
\end{align}
Thus, an upper and lower bound to $J(X_{-\cS^\star}; a, \pi)$ can be given by
\begin{align}
    \inf_{r\in [r_-(a), r_+(a)]} \tilde J(X_{-\cS^\star}; r, \pi) \le J(X_{-\cS^\star}; a, \pi) \le \sup_{r\in [r_-(a), r_+(a)]} \tilde J(X_{-\cS^\star}; r, \pi).
\end{align}
In order to effectively tackle these bounds, we then study the properties of $\tilde J(\cdot; r, \pi)$ next.

\begin{proposition}\label{prop:J_properties}
    Define 
    \begin{align}
        D_+(X_{-\cS^\star}, \pi) &= \max\left\{
            D_{\chi^2} (\mu^\pi(\cdot)\,\|\, \mu^\pi(\cdot \given X_{-\cS^\star})),  D_{\chi^2} (\mu^\pi(\cdot \given X_{-\cS^\star})\,\|\, \mu^\pi(\cdot))
        \right\}.
    \end{align}
    The function $\tilde J(X_{-\cS^\star}; r, \pi)$ with $r\in[0, 1]$ defined in \eqref{eq:tilde_J} satisfies the following properties:
    \begin{enumerate}
        \item $\tilde J(X_{-\cS^\star}; r, \pi)$ is convex in $r$.
        \item $\tilde J(X_{-\cS^\star}; r, \pi) \le D_+(X_{-\cS^\star}, \pi)$.
        \item $\tilde J(X_{-\cS^\star}; r, \pi)$ is Lipschitz continuous in $r$ with Lipschitz constant $\gamma^{-1} D_+(X_{-\cS^\star}, \pi)$.
    \end{enumerate}
\end{proposition}
\begin{proof}
    See \Cref{sec:add_proof_stage3} for the proof.
\end{proof}


\paragraph{Upper and Lower Bounding $J(X_{-\cS^\star}; a, \pi)$}
Previously, we show via a reformulation of $f_5$ that it suffices to bound $J(X_{-\cS^\star}; a, \pi)$.
In the sequel, we let 
\begin{align}
    D_+(X_{-\cS^\star}, \pi) &= \max\left\{
        D_{\chi^2} (\mu^\pi(\cdot)\,\|\, \mu^\pi(\cdot \given X_{-\cS^\star})),  D_{\chi^2} (\mu^\pi(\cdot \given X_{-\cS^\star})\,\|\, \mu^\pi(\cdot))
    \right\},\\
    \rho &= \max\left\{\max_{X_{-\cS^\star}, \pi}\frac{D_+(X_{-\cS^\star}, \pi)}{D_{\chi^2} (\mu^\pi(\cdot)\,\|\, \mu^\pi(\cdot \given X_{-\cS^\star}))}, \max_{X_{-\cS^\star}, \pi}\frac{D_+(X_{-\cS^\star}, \pi)}{D_{\chi^2} (\mu^\pi(\cdot \given X_{-\cS^\star})\,\|\, \mu^\pi(\cdot))}\right\}.
\end{align}
It can be noticed that 
\begin{align}
    \rho &\le \max\left\{\max_{X_{-\cS^\star}, \pi}\frac{D_{\chi^2} (\mu^\pi(\cdot)\,\|\, \mu^\pi(\cdot \given X_{-\cS^\star}) )}{D_{\chi^2} (\mu^\pi(\cdot \given X_{-\cS^\star})\,\|\, \mu^\pi(\cdot))}, \max_{X_{-\cS^\star}, \pi}\frac{D_{\chi^2} (\mu^\pi(\cdot \given X_{-\cS^\star})\,\|\, \mu^\pi(\cdot))}{D_{\chi^2} (\mu^\pi(\cdot)\,\|\, \mu^\pi(\cdot \given X_{-\cS^\star}) )}\right\} \\
    &\le \max\left\{ \max_{X_{-\cS^\star}, \pi} \frac{\mu^\pi(\cdot)}{\mu^\pi(\cdot \given X_{-\cS^\star})}, \max_{X_{-\cS^\star}, \pi} \frac{\mu^\pi(\cdot \given X_{-\cS^\star})}{\mu^\pi(\cdot)}\right\} \le \gamma^{-1},
\end{align}
where the second inequality follows from noting that the $\chi^2$-divergence defined as $D_{\chi^2}(\mu\,\|\, \nu) = \sum_{x}  {(\mu(x) - \nu(x))^2}/{\nu(x)}$, and $D_{\chi^2}(\mu\,\|\, \nu)/D_{\chi^2}(\nu\,\|\, \mu) \le \sup_x \mu(x)/\nu(x)$.

Apparently, $r^\pi(X_{-\cS^\star};a)$ is a function of $a$ and enjoys the following parameter-dependent upper and lower bounds:
\begin{align}
    r_+(a) = (1 + \min_{X_{-\cS^\star}, \pi} \mu^\pi(X_{-\cS^\star}) (e^a - 1))^{-1}, \\
    r_-(a) = (1 + \max_{X_{-\cS^\star}, \pi} \mu^\pi(X_{-\cS^\star}) (e^a - 1))^{-1}.
\end{align}
If $a$ is small, we see that both $r_+(a)$ and $r_-(a)$ are close to $1$, and we directly have  
\begin{align}
    r_-(a) \le r^\pi(X_{-\cS^\star};a) \le 1, \where 1 - \max_{X_{-\cS^\star}, \pi} \mu^\pi(X_{-\cS^\star}) (e^a - 1) \le r_-(a) < 1.
\end{align}
This suggests an upper bound of $ J(X_{-\cS^\star}; a, \pi)$ as 
\begin{align}
    J(X_{-\cS^\star}; a, \pi) &\le \sup_{r\in [r_-(a), 1]} \tilde J(X_{-\cS^\star}; r, \pi) \le \tilde J(X_{-\cS^\star}; 1, \pi) + \gamma^{-1} \cdot D_+(X_{-\cS^\star}, \pi) \cdot (1 - r_-(a)) \\
    &\le D_{\chi^2}(\mu^\pi(\cdot \given X_{-\cS^\star})\,\|\, \mu^\pi(\cdot)) + \gamma^{-1} \cdot D_+(X_{-\cS^\star}, \pi) \cdot \max_{X_{-\cS^\star, \pi}} \mu^\pi(X_{-\cS^\star}) \cdot (e^a - 1) \\
    &\le D_{\chi^2}(\mu^\pi(\cdot \given X_{-\cS^\star})\,\|\, \mu^\pi(\cdot)) \cdot \left(1 + \gamma^{-2} \cdot \max_{X_{-\cS^\star, \pi}} \mu^\pi(X_{-\cS^\star}) \cdot (e^a - 1)\right), 
\end{align}
where the second line follows from the Lipschitz continuity property, and the last line holds because the ratio $D_+(X_{-\cS^\star}, \pi) / D_{\chi^2}(\mu^\pi(\cdot \given X_{-\cS^\star})\,\|\, \mu^\pi(\cdot))$ is upper bounded by $\rho$, and further by $\gamma^{-1}$.
A similar lower bound can be obtained by changing the sign of $\gamma^{-2} \cdot \max_{X_{-\cS^\star, \pi}} \mu^\pi(X_{-\cS^\star}) \cdot (e^a - 1)$. Hence, we h
\begin{align}
    J (X_{-\cS^\star}; a, \pi) = D_{\chi^2}(\mu^\pi(\cdot \given X_{-\cS^\star})\,\|\, \mu^\pi(\cdot)) \cdot \left(1 \pm \gamma^{-2} \cdot \max_{X_{-\cS^\star, \pi}} \mu^\pi(X_{-\cS^\star}) \cdot (e^a - 1)\right).\label{eq:bound_J_small a} 
\end{align}
On the other hand, when $a$ becomes large, we have both $r_+(a)$ and $r_-(a)$ close to $0$, and we have
\begin{align}
    0\le r^\pi(X_{-\cS^\star};a) \le r_+(a), \where 0 < r_+(a) \le \frac{1}{\min_{X_{-\cS^\star}, \pi} \mu^\pi(X_{-\cS^\star}) (e^a - 1)}.
\end{align}
In a similar fashion, we have the following  upper bound: 
\begin{align}
    J(X_{-\cS^\star}; a, \pi) &\le \sup_{r\in [0, r_+(a)]} \tilde J(X_{-\cS^\star}; r, \pi) \le \tilde J(X_{-\cS^\star}; 0, \pi) + \gamma^{-1} \cdot D_+(X_{-\cS^\star}, \pi) \cdot r_+(a)  \\
    &= D_{\chi^2}(\mu^\pi(\cdot)\,\|\, \mu^\pi(\cdot \given X_{-\cS^\star})) + \gamma^{-1} \cdot \frac{D_+(X_{-\cS^\star}, \pi)}{\min_{X_{-\cS^\star}, \pi} \mu^\pi(X_{-\cS^\star}) (e^a - 1)} \\
    &\le D_{\chi^2}(\mu^\pi(\cdot)\,\|\, \mu^\pi(\cdot \given X_{-\cS^\star})) \cdot \left(1 + \frac{\gamma^{-2}}{\min_{X_{-\cS^\star}, \pi} \mu^\pi(X_{-\cS^\star}) (e^a - 1)}\right).
\end{align}
We can similarly obtain a lower bound by changing the sign of the second term inside the bracket.
Hence, we have 
\begin{align}
    J (X_{-\cS^\star}; a, \pi) = D_{\chi^2}(\mu^\pi(\cdot)\,\|\, \mu^\pi(\cdot \given X_{-\cS^\star})) \cdot \left(1 \pm \frac{\gamma^{-2}}{\min_{X_{-\cS^\star}, \pi} \mu^\pi(X_{-\cS^\star}) (e^a - 1)}\right).
    \label{eq:bound_J_large a}
\end{align}

\paragraph{Divergence of $a$}

Recall that we have shown the dynamics of $a$ in  \eqref{eq:dt_a-1}, 
where $\xi(a) $ is negligible when $L$ goes to infinity. 
Thus, when $L$ is sufficiently large, 
we see by the nonnegativity of $f_5$ that $a(t)$ continues to increase as $t$ increases until it reaches a point where $f_5$ no longer dominates the approximation error.
To characterize the regime where $f_5 \ge \xi(a)$, we first note that for $a\le \log L$ it holds by \eqref{eq:xi_psi_ub} that 
\begin{align}
    \xi(a) = O(L^{-1/2} \log L) \approx L^{-1/2}, 
\end{align}
where $\approx$ hides logarithmic factors.
For $f_5$, we recall from \Cref{prop:f_5} that 
\begin{align}
    f_5 = \EE_{\pi, X_{-\cS^\star}\sim \mu^\pi}\bigg[ \frac{J(X_{-\cS^\star}) \cdot e^a}{\bigl( 1 + \mu^\pi(X_{-\cS^\star}) \cdot (e^a -1)\bigr)^3} \cdot \mu^\pi(X_{-\cS^\star})\bigg], 
\end{align}
where for small $a$ we have $f_5=\Omega(1)$ and for large $a$ we have $f_5 = \Omega(e^{-2a})$. 
Thus, $e^{-2a} \ge L^{-1/2}$ gives the condition for $f_5$ to dominate the approximation error, which gives $a = O(\log L)$.
In the sequel, we consider the dynamics for $a \le  (\log L)/8$ and give a more rigorous analysis. 


We use the notation $x=o(1) $ to denote that a term is much smaller than $1$, for example, $ (\log\log L)^{-1} = o(1)$.
For any $x_0 $ and $\delta$, we write $x = x_0 \pm \delta$ to indicate that $x$ is bounded within $[ x_0 - \delta , x_0 + \delta] $. 
In the following, we assume there exists $\delta$ satisfying $\delta \le \gamma^2/4 \land 1/8$ and 
\begin{align}
    \delta \cdot \EEpifromP \bigg[\sum_{X_{-\cS^\star}} D_{\chi^2} \big(\mu^\pi(\cdot \given X_{-\cS^\star})\,\|\, \mu^\pi(\cdot) \big) \cdot \bigl( \mu^\pi(X_{-\cS^\star})\bigr)^2\bigg] & \ge \xi(\log L), \\
    \delta \cdot \EEpifromP\bigg[\sum_{X_{-\cS^\star}} \frac{D_{\chi^2} \big (\mu^\pi(\cdot)\,\|\, \mu^\pi(\cdot \given X_{-\cS^\star}) \big) \cdot L^{-1/4}}{\mu^\pi(X_{-\cS^\star})}\bigg] & \ge \xi(\log L). 
\end{align}

Note that 
\begin{align}
    \xi(\log L) &\le O\bigg(\frac{\sqrt M + d}{L^{1/2}(1-\lambda)^{1/2} \gamma^{|\cS^\star|+2+ r_n/4}} +  \frac{\log L}{L^{1/2}} \bigg).
\end{align}
By additionally noting that $\mu^\pi(X_{-\cS^\star})\ge \gamma^{|\cS^\star|}$ thanks to the lower bound of the transition probability, 
we are able to find such a $\delta$ if we have
\begin{align}
    \frac{L}{(\log L)^4 } \ge \Omega\bigg(\frac{1}{\kappa^4 \gamma^{8 + 2|\cS^\star|}} \cdot \Big(\frac{\sqrt M + d}{(1-\lambda)^{1/2} \gamma^{|\cS^\star|+2+ r_n/4}}\Big)^4\bigg), 
\end{align}
where $\kappa$ is defined as 
\begin{align}
    \kappa \defeq \EE\left[
        D_{\chi^2} (\mu^\pi(\cdot)\,\|\, \mu^\pi(\cdot \given X_{-\cS^\star}))\right] \land \EE\left[D_{\chi^2} (\mu^\pi(\cdot \given X_{-\cS^\star})\,\|\, \mu^\pi(\cdot))
    \right] \land 1, 
\end{align}
and $\Omega (\cdot) $ only hides universal constants. 
Note that this is already guaranteed by the condition on $L$ in \eqref{eq:L_condition-2}.
In particular, we can just take $\delta = \gamma^2/4 \land 1/8$ in the following analysis.

\paragraph{Small $a$}
Consider the case where $a$ is small in the sense that $\mu^\pi(X_{-\cS^\star}) \cdot (e^a - 1) \le \delta$ for any $X_{-\cS^\star} $ and $ \pi\in\supp(\cP)$.
In fact, one can directly deduce from our previous results that $1- \delta \le r_-(a) \le r^\pi(X_{-\cS^\star};a) < 1$ and 
\begin{align}
    1-3\delta \le (r^\pi(X_{-\cS^\star};a) ) ^3 \le 1.
\end{align}
For $J(X_{-\cS^\star}; a, \pi)$, we combine the condition that  $\mu^\pi(X_{-\cS^\star}) \cdot (e^a - 1) \le \delta$ with \eqref{eq:bound_J_small a} to obtain that 
\begin{align}
    J(X_{-\cS^\star}; a, \pi) 
    &= \left(1\pm \gamma^{-2} \delta \right) \cdot D_{\chi^2} \big(\mu^\pi(\cdot \given X_{-\cS^\star})\,\|\, \mu^\pi(\cdot) \big), \where \gamma^{-2}\delta \le 1/4 .
\end{align}
Combining the above two results with \Cref{prop:f_5}, we have
\begin{align}
    f_5 &= \EE_{\pi, X_{-\cS^\star}\sim \mu^\pi}\left[ J(X_{-\cS^\star}; a, \pi) \cdot e^a \cdot \bigl( r^\pi(X_{-\cS^\star}) \bigr) ^3 \cdot \mu^\pi(X_{-\cS^\star})\right]\\
    &= 
    \left(1\pm (\gamma^{-2} + 3) \delta\right) \cdot \EEpifromP \bigg[\sum_{X_{-\cS^\star}} D_{\chi^2} \bigl (\mu^\pi(\cdot \given X_{-\cS^\star})\,\|\, \mu^\pi(\cdot) \big) \cdot \mu^\pi(X_{-\cS^\star})^2\bigg] \cdot e^a.
\end{align}
Also, the noise term $\xi + \psi(a)$ is upper bounded by
\begin{align}
    \xi + \psi(\log L)
    &\le \delta \cdot \EEpifromP \bigg[\sum_{X_{-\cS^\star}} D_{\chi^2} \bigl (\mu^\pi(\cdot \given X_{-\cS^\star})\,\|\, \mu^\pi(\cdot) \big) \cdot \mu^\pi(X_{-\cS^\star})^2\bigg] \\
    &\le \delta \cdot \EEpifromP \bigg[\sum_{X_{-\cS^\star}} D_{\chi^2} \bigl (\mu^\pi(\cdot \given X_{-\cS^\star})\,\|\, \mu^\pi(\cdot) \big) \cdot \mu^\pi(X_{-\cS^\star})^2\bigg] \cdot e^a
\end{align}
by the construction of $\delta$. 
Combining all the above results, we have the dynamics of $a$ as
\begin{align}
    \partial_t a =\left(1 \pm (\gamma^{-2} +4) \delta \right)  \cdot \EEpifromP \bigg[\sum_{X_{-\cS^\star}} D_{\chi^2} (\mu^\pi(\cdot \given X_{-\cS^\star})\,\|\, \mu^\pi(\cdot)) \cdot \mu^\pi(X_{-\cS^\star})^2\bigg] \cdot e^a .
\end{align}
A simple reformulation gives
\begin{align}
    -\partial_t {e^{-a}} =  \left(1 \pm (\gamma^{-2} +4)\delta \right) \cdot \EEpifromP \bigg[\sum_{X_{-\cS^\star}} D_{\chi^2} (\mu^\pi(\cdot \given X_{-\cS^\star})\,\|\, \mu^\pi(\cdot)) \cdot \mu^\pi(X_{-\cS^\star})^2\bigg], 
\end{align}
which implies that for small $a$, the growth follows 
\begin{align}
    a(t) \le  -\log\bigg(e^{-a(0)} - (1 + (\gamma^{-2} +4)\delta) \cdot \EEpifromP \bigg[\sum_{X_{-\cS^\star}} D_{\chi^2} (\mu^\pi(\cdot \given X_{-\cS^\star})\,\|\, \mu^\pi(\cdot)) \cdot \mu^\pi(X_{-\cS^\star})^2\bigg] \cdot t \bigg), \\
    a(t) \ge  -\log\bigg(e^{-a(0)} - (1 - (\gamma^{-2} +4)\delta) \cdot \EEpifromP \bigg[\sum_{X_{-\cS^\star}} D_{\chi^2} (\mu^\pi(\cdot \given X_{-\cS^\star})\,\|\, \mu^\pi(\cdot)) \mu^\pi(X_{-\cS^\star})^2\bigg] \cdot t\bigg). 
\end{align}
Therefore, in the beginning, $a(t)$ grows super exponentially fast. 

\paragraph{Large $a$}
As $a$ grows large such that $\mu^\pi(X_{-\cS^\star}) (e^a - 1) \ge \delta^{-1}$ for all $X_{-\cS^\star}$ and $\pi\in\supp(\cP)$, 
we conclude that $0 < r^\pi(X_{-\cS^\star};a) \le r_+(a) \le \delta$ and
\begin{align}
    \frac{r^\pi(X_{-\cS^\star};a)^3}{\left(\mu^\pi(X_{-\cS^\star}) e^a\right)^{-3}} = \frac{\left(\mu^\pi(X_{-\cS^\star}) e^a\right)^{3}}{(1 + \mu^\pi(X_{-\cS^\star}) (e^a - 1))^3} = \left(1 - \frac{1 - \mu^\pi(X_{-\cS^\star})}{{1 + \mu^\pi(X_{-\cS^\star}) (e^a - 1)}}\right)^{3}, 
\end{align}
which imples that
\begin{align}
    1 - 3 \delta \le \frac{r^\pi(X_{-\cS^\star};a)^3}{\left(\mu^\pi(X_{-\cS^\star}) e^a\right)^{-3}} \le  1.
\end{align}
For $J(X_{-\cS^\star}; a, \pi)$, we combine the condition that $\mu^\pi(X_{-\cS^\star}) \cdot (e^a - 1) \ge \delta^{-1}$ with \eqref{eq:bound_J_large a} to obtain that 
\begin{align}
    J(X_{-\cS^\star};a, \pi) = (1 \pm \gamma^{-2}\delta) \cdot D_{\chi^2} (\mu^\pi(\cdot)\,\|\, \mu^\pi(\cdot \given X_{-\cS^\star})), \where \gamma^{-2}\delta \le 1/4.
\end{align}
Combining the above two results with \Cref{prop:f_5}, we have
\begin{align}
    f_5 &= \EE_{\pi, X_{-\cS^\star}\sim \mu^\pi}\left[ J(X_{-\cS^\star}; a, \pi) \cdot e^a \cdot \bigl( r^\pi(X_{-\cS^\star}) \bigr) ^3 \cdot \mu^\pi(X_{-\cS^\star})\right]\\
    &= \left(1 \pm (\gamma^{-2} + 3) \delta\right) \cdot  \EEpifromP\bigg[\sum_{X_{-\cS^\star}} D_{\chi^2} (\mu^\pi(\cdot)\,\|\, \mu^\pi(\cdot \given X_{-\cS^\star})) \cdot \frac{e^{-2a} }{\mu^\pi(X_{-\cS^\star})}\bigg].
\end{align}
For the noise term $\xi + \psi(a)$, we have
\begin{align}
    \delta \cdot \EEpifromP\bigg[\sum_{X_{-\cS^\star}} D_{\chi^2} (\mu^\pi(\cdot)\,\|\, \mu^\pi(\cdot \given X_{-\cS^\star})) \cdot \frac{e^{-2a} }{\mu^\pi(X_{-\cS^\star})}\bigg] \ge \xi + \psi(a), 
\end{align}
which can be verified by the condition on $\delta$ as well as the fact that we are only considering $a \le (\log L)/8$. 
We thus have for the gradient that 
\begin{align}
    \partial_t a = (1 \pm (\gamma^{-2} + 4)\delta) \cdot \EEpifromP \bigg[\sum_{X_{-\cS^\star}} D_{\chi^2} (\mu^\pi(\cdot)\,\|\, \mu^\pi(\cdot \given X_{-\cS^\star})) \cdot \frac{e^{-2a} }{\mu^\pi(X_{-\cS^\star})}\bigg]. 
\end{align}
By rearranging the terms, we further have
\begin{align}
    \partial_t e^{2a} =  (1 \pm (\gamma^{-2} + 4)\delta) \cdot \EEpifromP \bigg[\sum_{X_{-\cS^\star}}  \frac{D_{\chi^2} (\mu^\pi(\cdot)\,\|\, \mu^\pi(\cdot \given X_{-\cS^\star})) \cdot}{2 \mu^\pi(X_{-\cS^\star})}\bigg]. 
\end{align}
Suppose this large $a$ regime starts at $t_0$ with value $a(t_0)$.
Thus, for large $a$, the growth rate is characterized by
\begin{align}
    a(t) = \frac 1 2 \log\bigg(
        (1 \pm (\gamma^{-2} + 4)\delta) \cdot \EEpifromP \bigg[\sum_{X_{-\cS^\star}}  \frac{D_{\chi^2} (\mu^\pi(\cdot)\,\|\, \mu^\pi(\cdot \given X_{-\cS^\star}))}{2 \mu^\pi(X_{-\cS^\star})}\bigg] \cdot (t-t_0) + e^{2 a(t_0)}
    \bigg), 
\end{align}
which is logarithmically fast.
This step ends until $a$ reaches the value $(\log L) /8$.
This concludes the proof.
\end{proof}

\subsubsection{Additional Proofs for Stage \RNum{3}}\label{sec:add_proof_stage3}
We conclude the proof of Stage \RNum{3} by providing the proof of \cref{prop:f_5} and \cref{prop:J_properties}.
\begin{proof}[Proof of \cref{prop:f_5}]
    In this paragraph, we aim to gain more insight in $f_5$.
    By the definition of $f_5$ in \eqref{eq:f_5_stage_3}, we can rewrite $f_5$ as
    \begin{align}
        f_5 &= \EE_{\pi, X_{-\cS^\star}\sim \mu^\pi}\bigg[\bigg(\sum_{k=1}^d\frac{\mu^\pi(x=e_k\given X_{-\cS^\star})^2 }{\tilde \mu^\pi(z=e_k\given X_{-\cS^\star})} - 1\bigg)\cdot \tilde\mu^\pi(Z_{-\cS^\star}=X_{-\cS^\star}\given X_{-\cS^\star}) \bigg] \\
        & = \EE_{\pi, X_{-\cS^\star}\sim \mu^\pi}\bigg[\sum_{k=1}^d \bigg(\frac{\mu^\pi(x=e_k\given X_{-\cS^\star}) }{\tilde \mu^\pi(z=e_k\given X_{-\cS^\star})} - 1\bigg)^2 \cdot \tilde \mu^\pi(z=e_k\given X_{-\cS^\star}) \cdot \tilde\mu^\pi(Z_{-\cS^\star}=X_{-\cS^\star}\given X_{-\cS^\star})\bigg], 
    \end{align}
    where in the last step, we use the simple fact 
    \begin{align}
       \sum_{x } \frac{p(X=x\given Y)^2}{q(X=x\given Y)} - 1 =  \sum_x \left(\frac{p(X=x\given Y)}{q(X=x\given Y)} - 1\right)^2 \cdot q(X=x\given Y).
    \end{align}
    In the definition of $f_5$, the key quantity we aim to understand is the reweighted distribution $\tilde\mu^\pi(z, Z\given X_{-\cS^\star})$.
    For the readers' convenience, we copy the definition of the reweighted distribution here:
    \begin{align}\label{eq:tilde_mu-copy}
        \tilde \mu^\pi(z, Z\given X_{-\cS^\star}) = \frac{\mu^\pi(z, Z) \exp\left(
            a \prod_{h\in\cS^\star}\ind(z_{-  h} = x_{-h})
        \right)}{\sum_{z', Z'} \mu^\pi(z', Z') \exp\left(
            a \prod_{h\in\cS^\star}\ind(z_{-h}' = x_{-h})
        \right)}, 
    \end{align}
    A key observation is that the reweighting only depends on the value of $Z_{-\cS^\star}$.
    Let $\bar \cS^\star = [M]\backslash \cS^\star$ and denote by $Z_{-\bar\cS^\star} = (z_{-h})_{h\in\bar\cS^\star}$.
    Following the above observation, we can additionally condition on $Z_{-\cS^\star}$ and conclude that 
    \begin{align} 
        \tilde\mu^\pi(z, Z_{-\bar\cS^\star}\given Z_{-\cS^\star}, X_{-\cS^\star}) 
        &= \frac{\tilde\mu^\pi(z, Z_{-\bar\cS^\star}, Z_{-\cS^\star}\given X_{-\cS^\star})}{\sum_{z', Z_{-\bar\cS^\star}'}\tilde\mu^\pi(z', Z_{-\bar\cS^\star}', Z_{-\cS^\star}\given X_{-\cS^\star})} \\
        &= \frac{\mu^\pi(z, Z_{-\bar\cS^\star}, Z_{-\cS^\star}) \exp\left(
            a \prod_{h\in\cS^\star}\ind(z_{-  h} = x_{-h})
        \right)}{\sum_{z', Z_{-\bar\cS^\star}'} \mu^\pi(z', Z_{-\bar\cS^\star}', Z_{-\cS^\star}) \exp\left(
            a \prod_{h\in\cS^\star}\ind(z_{-h} = x_{-h})
        \right)} \\
        &= \frac{\mu^\pi(z, Z_{-\bar\cS^\star}, Z_{-\cS^\star}) }{\mu^\pi(Z_{-\cS^\star})} = \mu^\pi(z, Z_{-\bar\cS^\star}\given Z_{-\cS^\star}) ,
    \end{align}
    as when fixing $Z_{-\cS^\star}$, the exponential reweighting terms cancel out in the numerator and denominator in the definition \eqref{eq:tilde_mu-copy}.
    Using the above identity, we are able to expand $\tilde\mu^\pi(z\given X_{-\cS^\star})$ as
    \begin{align}
        \tilde\mu^\pi(z\given X_{-\cS^\star}) &= \sum_{Z_{-\cS^\star}}\mu^\pi(z\given Z_{-\cS^\star}) \cdot \tilde\mu^\pi(Z_{-\cS^\star}\given X_{-\cS^\star}) \\
        & = \sum_{Z_{-\cS^\star}}\mu^\pi(z\given Z_{-\cS^\star}) \cdot \frac{
            \mu^\pi(Z_{-\cS^\star}) + \mu^\pi(X_{-\cS^\star})(e^a -1) \cdot \ind(Z_{-\cS^\star}= X_{-\cS^\star})
        }{1 + \mu^\pi(X_{-\cS^\star})(e^a -1)} \\
        & = \frac{
            \mu^\pi(z) + \mu^\pi(x=z\given X_{-\cS^\star}) \cdot \mu^\pi(X_{-\cS^\star}) \cdot (e^a -1)
        }{1 + \mu^\pi(X_{-\cS^\star}) \cdot (e^a -1)}.
    \end{align}
    where the second equality follows from the fact that the reweighing term in $\tilde\mu^\pi$ lifts the likelihood of $Z_{-\cS^\star} = X_{-\cS^\star}$ by a factor of $e^a$ relative to the base distribution $\mu^\pi(Z_{-\cS^\star})$, and the denominator is just the normalization constant.
    In the sequel, we let $r^\pi(X_{-\cS^\star};a) = (1 + \mu^\pi(X_{-\cS^\star}) \cdot (e^a -1))^{-1}$ be the inverse of the normalization constant.
    We then have 
    \begin{gather}
        \tilde\mu^\pi(z\given X_{-\cS^\star}) = r^\pi(X_{-\cS^\star};a) \cdot \mu^\pi(z) + (1 - r^\pi(X_{-\cS^\star};a) ) \cdot \mu^\pi(x=z\given X_{-\cS^\star}). \label{eq:tilde_mu_rewritten}
    \end{gather}
    On the other hand, by definition of $\tilde\mu^\pi$ in \eqref{eq:tilde_mu-copy}, we directly have
    \begin{align}
        \tilde\mu^\pi(Z_{-\cS^\star}= X_{-\cS^\star}\given X_{-\cS^\star}) 
        &= \frac{\mu^\pi( X_{-\cS^\star}) e^a}{\sum_{Z_{-\cS^\star}'} \mu^\pi(Z_{-\cS^\star}') \exp\left(
            a \prod_{h\in\cS^\star}\ind(z_{-h}' = x_{-h})
        \right)} \\
        &= e^a r^\pi(X_{-\cS^\star};a) \cdot \mu^\pi(X_{-\cS^\star}). \label{eq:tilde_mu_rewritten-2}
    \end{align}
    Combining both \eqref{eq:tilde_mu_rewritten} and \eqref{eq:tilde_mu_rewritten-2} we have for $f_5$ that 
    \begin{align}
        f_5 & = \EE_{\pi, X_{-\cS^\star}\sim\mu^\pi}\biggl[\sum_{k\in[d]}\left(\frac{\mu^\pi(x=e_k\given X_{-\cS^\star}) }{r^\pi(X_{-\cS^\star};a) \cdot \mu^\pi(x=e_k) + (1 - r^\pi(X_{-\cS^\star};a) ) \cdot \mu^\pi(x=e_k\given X_{-\cS^\star})} - 1\right)^2  \\
        &\hspace{6cm} \cdot \tilde \mu^\pi(z=e_k\given X_{-\cS^\star}) \cdot \tilde\mu^\pi(Z_{-\cS^\star}=X_{-\cS^\star}\given X_{-\cS^\star}) \biggr] \\
        & = \EE_{\pi, X_{-\cS^\star}\sim\mu^\pi}\biggl[ \sum_{k\in[d]}\left(\frac{\mu^\pi(x=e_k\given X_{-\cS^\star}) - \mu^\pi(x=e_k) }{r^\pi(X_{-\cS^\star};a) \cdot \mu^\pi(x=e_k) + (1 - r^\pi(X_{-\cS^\star};a) ) \cdot \mu^\pi(x=e_k\given X_{-\cS^\star})} \right)^2  \\
        &\hspace{6cm} \cdot \tilde \mu^\pi(z=e_k\given X_{-\cS^\star})\cdot  e^a r^\pi(X_{-\cS^\star};a)^3 \cdot \mu^\pi(X_{-\cS^\star})\biggr] \\
        &= \EE_{\pi, X_{-\cS^\star}\sim\mu^\pi}\biggl[\underbrace{\sum_{k\in[d]} \frac{(\mu^\pi(x=e_k\given X_{-\cS^\star}) - \mu^\pi(x=e_k))^2 }{\tilde \mu^\pi(z=e_k\given X_{-\cS^\star})}}_{\ds J(X_{-\cS^\star}; a, \pi)}\cdot  e^a r^\pi(X_{-\cS^\star};a)^3 \cdot \mu^\pi(X_{-\cS^\star})\biggr]. 
        \label{eq:f_5_stage_3-1}
    \end{align}
    Here, we note that $J(\cdot; a, \pi)$ is a function depending on both $a$ and $\pi$, and can be expanded as 
    \begin{align}
        J(X_{-\cS^\star}; a, \pi) = \sum_{k\in[d]} \frac{(\mu^\pi(x=e_k\given X_{-\cS^\star}) - \mu^\pi(x=e_k))^2 }{(1-r^\pi(X_{-\cS^\star};a)) \mu^\pi(x=e_k \given X_{-\cS^\star}) +  r^\pi(X_{-\cS^\star};a) \mu^\pi(x=e_k)}.
    \end{align}
    Hence, we complete the proof of \Cref{prop:f_5}.
    \end{proof}

    \begin{proof}[Proof of \Cref{prop:J_properties}]
        Also, note that $\tilde J(X_{-\cS^\star}; r, \pi)$ 
        is convex in $r$, as by taking the derivative of $\tilde J(X_{-\cS^\star}; r, \pi)$ with respect to $r$, we have
        \begin{align}
            \frac{\partial \tilde J(X_{-\cS^\star}; r, \pi)}{\partial r} &= \sum_{k\in[d]} \frac{(\mu^\pi(x=e_k\given X_{-\cS^\star}) - \mu^\pi(e_k))^3 }{\left((1-r) \mu^\pi(x=e_k \given X_{-\cS^\star}) +  r \mu^\pi(e_k)\right)^2} , \\
            \frac{\partial^2 \tilde J(X_{-\cS^\star}; r, \pi)}{\partial r^2} &= 2 \cdot \sum_{k\in[d]} \frac{(\mu^\pi(x=e_k\given X_{-\cS^\star}) - \mu^\pi(e_k))^4 }{\left((1-r) \mu^\pi(x=e_k \given X_{-\cS^\star}) +  r \mu^\pi(e_k)\right)^3} \ge 0.
        \end{align}
        Hence, a naive upper bound for $\tilde J(X_{-\cS^\star}; r, \pi)$ is 
        \begin{align}
            \tilde J(X_{-\cS^\star}; r, \pi)
            &\le \max\{\tilde J(X_{-\cS^\star}; 0, \pi), \tilde J(X_{-\cS^\star}; 1, \pi)\} \\
            &\le 
            \max\left\{
                D_{\chi^2} (\mu^\pi(\cdot)\,\|\, \mu^\pi(\cdot \given X_{-\cS^\star})),  D_{\chi^2} (\mu^\pi(\cdot \given X_{-\cS^\star})\,\|\, \mu^\pi(\cdot))
            \right\}, 
        \end{align}
        where we remind the readers that $D_{\chi^2}(\mu\,\|\, \nu) = \sum_{x}  {(\mu(x) - \nu(x))^2}/{\nu(x)}$.
        Next, we show that $\tilde J(X_{-\cS^\star}; r, \pi)$ is Lipschitz continuous in $r$:
        \begin{align}
            \bigg|\frac{\partial \tilde J(X_{-\cS^\star}; r, \pi)}{\partial r}\bigg| &= \bigg|\sum_{k\in[d]} \frac{(\mu^\pi(x=e_k\given X_{-\cS^\star}) - \mu^\pi(e_k))^3 }{\left((1-r) \mu^\pi(x=e_k \given X_{-\cS^\star}) +  r \mu^\pi(e_k)\right)^2}\bigg| \\
            &\le \sum_{k\in[d]} \frac{(\mu^\pi(x=e_k\given X_{-\cS^\star}) - \mu^\pi(e_k))^2 }{(1-r) \mu^\pi(x=e_k \given X_{-\cS^\star}) +  r \mu^\pi(e_k)} \cdot \left| \frac{\mu^\pi(x=e_k\given X_{-\cS^\star}) - \mu^\pi(e_k) }{(1-r) \mu^\pi(x=e_k \given X_{-\cS^\star}) +  r \mu^\pi(e_k)} \right| \\
            &\le \tilde J(X_{-\cS^\star};r,\pi) \cdot \max\bigg\{ 
                \frac{\mu^\pi(x=e_k \given X_{-\cS^\star})}{\mu^\pi(e_k)}, 
                \frac{\mu^\pi(e_k)}{\mu^\pi(x=e_k \given X_{-\cS^\star})}
                \bigg\} \\
            &\le \gamma^{-1} \cdot \max\left\{
                D_{\chi^2} (\mu^\pi(\cdot)\,\|\, \mu^\pi(\cdot \given X_{-\cS^\star})),  D_{\chi^2} (\mu^\pi(\cdot \given X_{-\cS^\star})\,\|\, \mu^\pi(\cdot))
            \right\}, 
        \end{align}
        where we use both the upper bound for $\tilde J(X_{-\cS^\star}; r, \pi)$ and the lower bound for the transition kernel that both $\mu^\pi(\cdot \given X_{-\cS^\star})$ and $\mu^\pi(\cdot)$ are bounded between $\gamma$ and $1$.
\end{proof}


\subsection{Lemma on GIH Approximation Error}
\label{sec:proof_of_gih_approx}
Now given the convergence result for the training dynamics, 
the natural question to ask is how well the learned model implements the GIH mechanism.
In the following part of this section, we state the lemma on the approximation error and also present a formal proof of the lemma. 
\begin{lemma}\label{lem:gih_approx}
    Suppose \cref{asp:Markov_chain}  holds and consider training a transformer model $\mathtt{TF}(M,H, d, D)$ with $H = M$. 
    Let 
    \begin{align}
        \Delta_1 \defeq 1  - p_{\cS^\star}(t_1), \quad 
        \Delta_2 \defeq 1 - \prod_{h\in\cS^\star} (\sigma_{-h}^\h (t_2))^2,   
    \end{align}
    where $t_1$ and $t_2$ are the ending time for the first two stages of the training, respectively.
    Suppose the error $\Delta_1, \Delta_2 = O(L^{-1})$ after the first two stages' training, and $a = \Theta(\log L)$ after the last stage's training.
    Let $y$ be the output of the model in \eqref{eq:transformer} after the training and $y^\star$ be the output of the GIH mechanism $\mathtt{GIH}(x_{1:L}; M, D)$ defined in \Cref{def:gih}.
    Then for any $\pi\in\supp(\cP)$ and with high probability $1 - O(L^{-1})$ , it holds that
    \begin{align}
        \left\| y^\star - y \right\|_1 = O(L^{-a/\log L}).
    \end{align}
\end{lemma}
\begin{proof}[Proof of \Cref{lem:gih_approx}]
    Let $s^\star_l = \prod_{h\in\cS^\star} \ind(x_{l-h} = x_{L+1-h})$ and $s_l = \langle u_{L+1}, u_{l} \rangle$.
    Invoking \cref{lem:misspecification}, the model misspecification error is bounded by 
    \begin{align}
        \max_{M < l\le L} \left|s^\star_l - s_l\right| \le (\Delta_1 +\Delta_2) \defeq \Delta. 
        \label{eq:gih_approx-1}
    \end{align}
    We note that the second layer's attention weight $a$ can be as large as $(\log L) / 8$. 
    We are comparing the output of the model with the GIH mechanism $\mathtt{GIH}(x_{1:L}; M, D)$. 
    Let $N = \sum_{l> M} \prod_{h\in\cS^\star} \ind(x_{l-h} = x_{L+1-h})$. 
    The output of this GIH mechanism is given by 
    \begin{align}
        y^\star \defeq \begin{cases}
            N^{-1} \cdot {\sum_{l=M+1}^L x_l \cdot \prod_{h\in\cS^\star} \ind(x_{l-h} = x_{L+1-h}) }, \quad \text{if}\quad  N \ge 1, \\
            (L-M)^{-1} \cdot  \sum_{l=M+1}^L x_l, \quad \text{otherwise}.
        \end{cases}
    \end{align}
    We define 
    \begin{align}
        \sigma_l^\star = \begin{cases}
            N^{-1} \cdot \prod_{h\in\cS^\star} \ind(x_{l-h} = x_{L+1-h}) , \quad \text{if}\quad  N \ge 1, \\
            (L-M)^{-1} , \quad \text{otherwise}, 
        \end{cases}
    \end{align}
    with $\sigma^\star = (\sigma_l^\star)_{l>M}$.
    Since $\norm{x_l}_1=1$, the $\ell$-1 norm of the difference between $y^\star$ and the model's actual output is given by 
    \begin{align}
        \left\| y^\star - y \right\|_1 \le \left\| \sigma^\star - \sigma \right\|_1.
    \end{align}
    Let us define the set $\Gamma = \{L\ge l > M: \prod_{h\in\cS^\star} \ind(x_{l-h} = x_{L+1-h}) = 1\}$ and $\bar\Gamma = \{L\ge l > M: \prod_{h\in\cS^\star} \ind(x_{l-h} = x_{L+1-h}) = 0\}$.
    Using \eqref{eq:gih_approx-1}, 
    for $l\in\Gamma$, we have $1 \ge s_l \ge s_l^\star - \Delta = 1 -\Delta$ and for $l\in\bar\Gamma$, we have $0\le s_l \le s_l^\star + \Delta = \Delta$.
    Consider the normalization factor in the softmax function. 
    \begin{align}
        \cZ \defeq \sum_{l=M+1}^L \exp(a \cdot s_l). 
    \end{align}
    By the split of the set $\Gamma$ and $\bar\Gamma$ and noting that $|\Gamma|=N$, the normalization factor is lower and upper bounded by 
    \begin{align}
        \cZ &\ge N \exp(a \cdot (1-\Delta)) + (L-M - N)\cdot \eqdef \cZ_-, 
        \\ 
        \cZ &\le N \exp(a) + (L-M - N)\cdot \exp(a \cdot \Delta) \eqdef \cZ_+.
    \end{align}
    Let us consider the event $N\ge 1$ in the following.
    We then have for $l\in\Gamma$ that 
    \begin{align}
        |\sigma_l^\star - \sigma_l| 
        &= \bigg| \frac{\exp(a \cdot s_l)}{\cZ} - \frac 1 N\bigg| \le \bigg| \frac{\exp(a)}{\cZ_-} - \frac 1 N\bigg| \bigvee \left| \frac{\exp(a\cdot (1-\Delta))}{\cZ_+} - \frac 1 N\right| \\
        &\le \left| \frac{1}{N \exp(-a \Delta) + (L-M - N) \cdot \exp(-a)} - \frac 1 N\right| \\
        &\qquad \bigvee \left| \frac{\exp(-2 a \Delta)}{N \exp(- a \Delta)  + (L-M - N)\exp(-a)} - \frac 1 N\right|\\
        &\le \frac{N\cdot (1 - \exp(-a\Delta)) + (L-M - N) \cdot \exp(-a)}{(N \exp(- a \Delta) + (L-M - N) \cdot \exp(-a)) \cdot N } \le \frac{1 - \exp(-a\Delta)}{N \exp(-a\Delta)} + \frac{L \cdot \exp(-a)}{N^2 \exp(-a\Delta)}. 
    \end{align}
    Note that $ a \Delta = o(1)$ due to the assumption that $\Delta = O(L^{-1})$ and $a = o(L)$.
    The right hand side is upper bounded by 
    $O(a \Delta/N) + O(L\exp(-a)/N^2)$.
    For $l\in\bar\Gamma$, we have
    \begin{align}
        |\sigma_l^\star - \sigma_l| = \sigma_l \le \frac{\exp(a\Delta)}{\cZ_-} \le \frac{\exp(a\cdot (2\Delta - 1))}{N} = O\left(\frac{\exp(-a)}{N}\right).
    \end{align}
    In summary, 
    \begin{align}
        \left\| y^\star - y \right\|_1 
        &\le \left\| \sigma^\star - \sigma \right\|_1 \le \sum_{l\in\Gamma} \left| \sigma_l^\star - \sigma_l \right| + \sum_{l\in\bar\Gamma} \sigma_l \\
        &\le N \cdot O\left(\frac{a \Delta N +  L\exp(-a)}{N^2} \right) + L \cdot O\left(\frac{\exp(-a)}{N}\right) 
        \le O\left(a \Delta   + \frac{L\exp(-a)}{N}\right).\label{eq:gif_approx-1}
    \end{align}
    The above inequality holds whenever $N\ge 1$. 
    Now we aim to upper bound the probability that $N = 0$.
    Note that $N = \sum_{l=M+1}^L \ind(X_{l-\cS^\star} = X_{L+1-\cS^\star})$. 
    We consider the following second moment:
    \begin{align}
        \EE\bigg[
            \bigg((L-M)^{-1} \sum_{l=M+1}^L \ind(X_{l-\cS^\star} = E) - \mu^\pi(E)\bigg)^2
        \bigg] 
        &\le D_{\chi^2} \biggl(
            (L-M)^{-1} \sum_{l=M+1}^L  \ind(X_{l-\cS^\star} = \cdot) \,\Big\|\, \mu^\pi(\cdot)
         \biggr) \\
        &\lesssim \frac{M}{L (1-\lambda) \cdot \gamma^{|\cS^\star|/2}}, \quad \forall E\in\cX^{|\cS^\star|},  
    \end{align}
    where the first inequality holds by noting that $D_{\chi^2}(\mu \,\|\, \nu) = \sum_x (\mu(x) - \nu(x))^2/\nu(x)$ and the 
    last inequality holds by \cref{lem:convergence chi-square}.
    Therefore, by the Chebyshev's inequality, we have 
    \begin{align}
        \PP\bigg(
            \left|L^{-1} \sum_{l=1}^L \ind(X_{l-\cS^\star} = E) - \mu^\pi(E) \right|\ge t
        \bigg) \le \frac{1}{L (1-\lambda) \cdot \gamma^{|\cS^\star|} \cdot t^2}. 
    \end{align}
    We can take $t = \min_{E\in\cX^{|\cS^\star|}} \mu^\pi(E)/2$ and by also taking a union bound over $E\in \cX^{|\cS^\star|}$  (which gives a $d^{|\cS^\star|}$ factor), we conclude that with high probability $\tilde O(1-L^{-1})$ it holds that 
    \(
        N \ge t L = L\cdot \min_{E\in\cX^{|\cS^\star|}} \mu^\pi(E)/2.
    \)
    Thus, it follows from \eqref{eq:gif_approx-1} that with high probability
    \begin{align}
        \left\| y^\star - y \right\|_1  \le  O\left(a \Delta + \frac{\exp(-a)}{\min_{E\in\cX^{|\cS^\star|}} \mu^\pi(E)/2} \right) = O\left(L^{-1} \log L + L^{-a/\log L}\right).
    \end{align}
    Hence, we complete the proof of \Cref{lem:gih_approx}.
\end{proof}

\newpage

\section{Auxiliary Lemmas and Their Proofs}\label{sec:auxiliary_results}

In this appendix, we present the auxiliary lemmas used to derive the approximation of the gradient flow dynamics in the proof of \cref{thm:convergence}, which is presented in the previous appendix. The proofs of these lemmas are presented right below their statements. 

\subsection{Useful Inequalities}
The following lemma provides a bound on the model misspecification error, which is the difference between the model's output and the ideal output.
\begin{lemma}[Model Misspecification Error]
    \label{lem:misspecification}
    Let $u_{L+1}$ be the output feature after the FFN \& Normalization layer. Then, the model misspecification error defined as 
    \begin{align}
        \max_{l\in[L]} \bigg|\langle u_{L+1}, u_{l} \rangle - \prod_{h\in\cS^\star} \ind(x_{l-h} = x_{L+1-h})\bigg|
    \end{align}
    is bounded by $\Delta_1 + \Delta_2$, where $\Delta_1$ and $\Delta_2$ are the errors after the training of the  first and second stages, respectively, and are defined respectively as 
    \begin{align}
        \Delta_1 \defeq 1 - p_{\cS^\star}, \qquad
        \Delta_2 \defeq 1 - \prod_{h\in\cS^\star} (\sigma_{-h}^\h)^2.
    \end{align}
\end{lemma}
\begin{proof} [Proof of \cref{lem:misspecification}]
    By definition of the output feature $u_l$ after the FFN \& Normalization layer:
    \begin{align}
        &\langle u_{L+1}, u_{l} \rangle   = \sum_{\cS\in\HleqD} p_\cS \cdot \prod_{h\in\cS} \langle v_l^\h, v_{L+1}^\h \rangle .
    \end{align}
    As each $v_l^\h$ is a convex combination of $X_{\cM(l)}$ where $\cM(l) = \{ l - M, \ldots, l-1\}$, $\norm{v_l^\h}_2\le 1$.
    Thus,
    \begin{align}
        \bigg|\langle u_{L+1}, u_{l} \rangle - \prod_{h\in\cS^\star} \langle v_l^\h, v_{L+1}^\h \rangle \bigg|
        & = \bigg| \sum_{\cS\in\HleqD} p_{\cS} \cdot \prod_{h\in \cS} \langle v_l^\head{h}, v_{L+1}^\head{h} \rangle - \prod_{h\in \cS^\star} \langle v_l^\head{h}, v_{L+1}^\head{h} \rangle \bigg| \\
        & \le \bigg|- (1 - p_{\cS^\star}) \prod_{h\in \cS^\star} \langle v_l^\head{h}, v_{L+1}^\head{h} \rangle + \sum_{\cS\in\HleqD \backslash\{\cS^\star\}} p_{\cS} \cdot \prod_{h\in \cS} \langle v_l^\head{h}, v_{L+1}^\head{h} \rangle \bigg| \\
        &\le \max\bigg\{ 1 - p_{\cS^\star}, \sum_{\cS\in\HleqD \backslash\{\cS^\star\}} p_{\cS} \bigg\} = 1 - p_{\cS^\star} \eqdef \Delta_1, 
    \end{align}
    where $\Delta_1$ is the error after the training of the first stage.
    Since $v_l^\h = \sum_{j\in M} \sigma_{-j}^\h x_{l-j}$, we have
    \begin{align}
        \prod_{h\in\cS^\star} \langle v_l^\h, v_{L+1}^\h \rangle 
        &= \prod_{h\in\cS^\star}  \bigg( \sum_{i, j\in [M]^2} \sigma_{-i}^\h \sigma_{-j}^\h \langle x_{l-i}, x_{L+1-j} \rangle \bigg) \\
        &= \sum_{  \{ (i_h, j_h) \}  _{h\in\cS^\star} \in [M]^{2|\cS^\star|}} \prod_{h\in\cS^\star} \sigma_{-i_h}^\h \sigma_{-j_h}^\h \ind(x_{l-i_h} = x_{L+1-j_h}). 
    \end{align}
Here in the second equality, we exchange the order of summation and product. 
The last term of the second equality can be understood as follows. 
We first pick $|\cS^\star |$  index pairs $\{ (i_h, j_h)\}_{h\in \cS^\star }$ arbitrarily, with each $i_h, j_h \in [H]$. 
Then we evaluate the product $\prod_{h\in\cS^\star} \sigma_{-i_h}^\h \sigma_{-j_h}^\h \ind(x_{l-i_h} = x_{L+1-j_h})$ given these indices. Then we sum over all possible values that $\{ (i_h, j_h)\}_{h\in \cS^\star }$ can take.

    The above equation  implies that
    \begin{align}
        &\left|\prod_{h\in\cS^\star} \langle v_l^\h, v_{L+1}^\h \rangle - \prod_{h\in\cS^\star} (\sigma_{-h}^\h)^2 \ind(x_{l-h} = x_{L+1-h})\right|\\
        &\quad = \bigg|\sum_{\{(i_h,j_h)\}_{h \in \cS^\star} \neq\{(h, h)\}_{h\in\cS^\star}}\prod_{h\in\cS^\star}\sigma_{-i_h}^\h \sigma_{-j_h}^\head{h} \ind(x_{l-i_h} = x_{L+1-j_h})\bigg|\\
        &\quad \le \sum_{\{ (i_h,j_h)\} _{h \in \cS^\star} \neq \{ (h, h)\} _{h\in\cS^\star}}\prod_{h\in\cS^\star} \sigma_{-i_h}^\h \sigma_{-j_h}^\head{h} \le 1 - \prod_{h\in\cS^\star} (\sigma_{-h}^\h)^2 \eqdef \Delta_2,
        \label{eq:misspecification-1} 
    \end{align}
    where the last inequality follows from the fact that 
    \begin{align}
        \sum_{(i_h,j_h)_{h \in \cS^\star}}\prod_{h\in\cS^\star} \sigma_{-i_h}^\h \sigma_{-j_h}^\head{h} = \prod_{h\in\cS^\star}  \bigg( \sum_{i,j\in [M]^2} \sigma_{-i}^\h \sigma_{-j}^\head{h} \bigg) = \prod_{h\in\cS^\star}  \bigg( \sum_{i\in [M]} \sigma_{-i}^\h  \bigg) ^2  = 1.
    \end{align}
    Here the summation sign in the right-hand side of the second equality indicates that in the last line of \eqref{eq:misspecification-1} we sum over all possible values that $\{ (i_h, j_h)\}_{ h\in \cS^*}$ can take, except for the only case where $(i_h, j_h) = (h,h)$ for all $h\in [H]$.

    In summary, by triangle inequality,  we have shown that 
    \begin{align}
        &\bigg|\langle u_{L+1}, u_{l} \rangle - \prod_{h\in\cS^\star} \ind(x_{l-h} = x_{L+1-h})\bigg| \le \Delta_1 + \Delta_2.
    \end{align}
    The proof is completed.
\end{proof}
Next, in \Cref{lem:boundedness}, we establish a  uniform bound for the quantity involved in the gradient.

\begin{lemma}\label{lem:boundedness}
    Let $y(k) = \sum_{l=M+1}^{L} \sigma_l \ind(x_l = e_k)$ for each $k\in[d]$ where $\sum_{l=M+1}^{L} \sigma_l = 1$ and $\sigma_l \ge 0$ for all $l\in[L]$. 
    Let $\varepsilon$ and $C$ be two positive numbers. 
    For any $C$-bounded function $f: \cX^{L+1} \to [-C, C] $, we have 
    \begin{align}
        \left|\sum_{l=M+1}^{L} \sigma_l \cdot \sum_{k=1}^d \biggl(\frac{\ind(x_{L+1} = x_l = e_k)}{y(k) + \varepsilon} - \frac{y(k)\ind(x_{L+1} = e_k)}{y(k)+\varepsilon} \biggr) \cdot f(X)\right| \le 2C. 
    \end{align} 
\end{lemma}
\begin{proof}[Proof of \cref{lem:boundedness}]
By the triangular inequality, we have
\begin{align}
    &\left|\sum_{l=M+1}^{L} \sigma_l  \cdot \sum_{k=1}^d \biggl(\frac{\ind(x_{L+1} = x_l = e_k)}{y(k) + \varepsilon} - \frac{y(k)\ind(x_{L+1} = e_k)}{y(k)+\varepsilon} \biggr) \cdot f(X)\right| \\
    &\quad \le C\cdot \left|\sum_{k=1}^d \sum_{l=M+1}^{L} \sigma_l \cdot \ind(x_l = e_k) \cdot \frac{\ind(x_{L+1} = e_k) }{y(k) + \varepsilon} \right|  + C\cdot \left|\sum_{k=1}^d \sum_{l=M+1}^{L} \sigma_l \cdot  \frac{y(k)\cdot \ind(x_{L+1} = e_k)}{y(k)+\varepsilon} \right| \\
    &\quad = 2C \cdot \left|\sum_{k=1}^d \frac{y(k) \cdot \ind(x_{L+1} = e_k)}{y(k)+\varepsilon} \right| \le 2C \label{eq:indicator_upp}, 
\end{align}
where in the equality, we use the definition $y(k) = \sum_{l=M+1}^{L} \sigma_l \ind(x_l = e_k)$ and $\sum_{l=M+1}^{L} \sigma_l =1$ . Now we conclude the proof of this lemma. 
\end{proof}

\subsection{Approximation Errors for Dynamics Analysis}
\label{sec:approx_error_dynamics}
Next, \cref{lem:approximation1} addresses the approximation error induced by $\sigma_l \approx 1/L$ in the transformer model.
The approximation error will be for $g_{0, \cS}$ to $g_{1, \cS}$ for Stage \RNum{1} and $g_{h, 1}$ to $g_{h, 2}$ for Stage \RNum{2}. 
\begin{lemma}\label{lem:approximation1}
For the transformer model defined in \eqref{eq:transformer} and any bounded function $f: \cX^{L+1}\to\RR$ such that $\sup_{x\in\cX^L} |f(x)| \leq C$ for a constant $C>0$, define two quantities $A$ and $B$ as
\begin{align}
    A &:= \sum_{l=M+1}^{L} \EE_{X\mid\pi} \bigg[ \sigma_l(a s) \cdot \sum_{k\in[d]}\biggl(\frac{\ind(x_{L+1} = x_l = e_k)}{y(k) + \varepsilon} - \frac{y(k)\ind(x_{L+1} = e_k)}{y(k)+\varepsilon} \biggr) \cdot f(X)\bigg],\\
    B &:= \frac{1}{L-M} \sum_{l=M+1}^{L}  \EE_{X\mid\pi} \biggl[ \biggl(\sum_{k\in[d]}\frac{\ind(x_{L+1} = x_l = e_k)}{\bar y(k) + \varepsilon} - \frac{\bar y(k)\ind(x_{L+1} = e_k)}{\bar y(k)+\varepsilon} \biggr) \cdot f(X) \biggr],
\end{align}
where $s=u_{L+1}^\top U_{1:L}^\top$ and $\bar y = (L-M)^{-1} \sum_{l=M+1}^{L} x_l$.
Then, for all $a \in [0,1]$ and $\varepsilon \in (0,1]$, it holds that 
\begin{align}
    |A - B| \leq \frac{8Cad}{\varepsilon^2 }.
\end{align}
\end{lemma}

\begin{proof}[Proof of \Cref{lem:approximation1}]
    By triangular inequality, we have 
    \begin{align}
        |A - B| &\leq \sum_{l=M+1}^{L} \EE \biggl[ \sum_{k\in[d]} \biggl\{ \biggl| \sigma_l \big(a \cdot s\big) -\frac{1}{L-M} \biggr|\cdot   \biggl|\frac{\ind(x_{L+1} = x_l = e_k)}{y(k) + \varepsilon}  \biggr|  \\
        &\qqquad + \frac{1}{L-M}  \biggl| \frac{\ind(x_{L+1} = x_l = e_k)}{y(k) + \varepsilon} - \frac{\ind(x_{L+1} = x_l = e_k)}{\bar y(k) + \varepsilon}  \biggr| \\
        &\qqquad + \biggl| \sigma_l \big(a \cdot s\big) -\frac{1}{L-M} \biggr| 
         \cdot \biggl|\frac{y(k) \ind(x_{L+1} = e_k)}{y(k) + \varepsilon}  \biggr|, \\
        &\qqquad  + \frac{1}{L-M}  \biggl| \frac{y(k) \ind(x_{L+1} = e_k)}{y(k) + \varepsilon} - \frac{\bar y(k) \ind(x_{L+1} = e_k)}{\bar y(k) + \varepsilon}  \biggr|
        \biggr\} \cdot f(X) \biggr] .
    \end{align}
    Note that  $ 0 \leq s_l \leq 1$ for all $l =M+1, \ldots, L$ thanks to the layer normalization. Then, for the softmax operation, we have
    \begin{align}
        \frac{1}{ 1 + (L-M-1)\exp(a)}  \leq \sigma_l \big(a \cdot s \big) \leq \frac{\exp(a)}{L-M-1+ \exp(a)}, 
    \end{align}
    which implies that 
    \begin{align} 
        \biggl| \sigma_l \big(a \cdot s\big) -\frac{1}{L-M} \biggr| 
        &\leq \max \cbr{\frac{1}{L-M} -\frac{1}{ 1 + (L-M-1)\exp(a)}, \ \frac{\exp(a)}{L-M-1+ \exp(a)}- \frac{1}{L-M}}\\
        & \leq \frac{\exp(a)-1}{L-M-1}.\label{eq:softmax_error}
    \end{align}
    Since indicator functions are bounded above by $1$, we have  
    \begin{align} \label{eq:sm_error1}
        \biggl|\frac{\ind(x_{L+1} = x_l = e_k)}{y(k) + \varepsilon}  \biggr| \leq \frac{1}{\varepsilon},  \quad \biggl|\frac{y(k) \ind(x_{L+1} = e_k)}{y(k) + \varepsilon}  \biggr| \leq \frac{1}{\varepsilon}, \\
    \end{align}
    For the second term, we have
    \begin{align} 
        \biggl| \frac{\ind(x_{L+1} = x_l = e_k)}{y(k) + \varepsilon} - \frac{\ind(x_{L+1} = x_l = e_k)}{\bar y(k) + \varepsilon}  \biggr| &\leq \frac{|\bar y(k)- y(k) |}{\varepsilon^2} \leq \frac{\sum_{l=M+1}^{L} |\sigma_l \big(a \cdot s^\top \big) -(L-M)^{-1}|}{\varepsilon^2} \notag \\ 
        &\leq \frac{\exp(a)-1}{\varepsilon^2},\label{eq:sm_error2}
    \end{align}
    where the last inequality follows from \eqref{eq:softmax_error}.
    Similarly, the following bound can be derived: 
    \begin{align} \label{eq:sm_error3}
        \biggl| \frac{y(k) \ind(x_{L+1} = e_k)}{y(k) + \varepsilon} - \frac{\bar y(k) \ind(x_{L+1} = e_k)}{\bar y(k) + \varepsilon}  \biggr|  \leq \frac{\exp(a)-1}{\varepsilon}.
    \end{align}
    Combining \eqref{eq:softmax_error}, \eqref{eq:sm_error1}, \eqref{eq:sm_error2} and \eqref{eq:sm_error3}, it holds that 
    \begin{align}
       |A - B| \leq \sum_{l=M+1}^L \EE \biggl[ 4\sum_{k \in [d]} \frac{\exp(a)-1}{\varepsilon^2 (L-M)} \cdot f(X) \biggr] \leq \frac{4Cd(\exp(a)-1)}{\varepsilon^2}  \leq \frac{8Cad}{\varepsilon^2},
    \end{align} 
    where the last inequality follows from $\exp(x)-1 \leq 2x$ for $0 \leq x \leq 1$. This concludes the proof of the lemma. 
\end{proof} 
    
\cref{lem:approximation2} provides the approximation error introduced by $\mu^\pi(e_k) \approx \bar y(k)$ in the transformer model.
    \begin{lemma}\label{lem:approximation2}
        For the transformer model defined in \eqref{eq:transformer} and any bounded function $f: \cX^L\to\RR$ such that $\sup_{x\in\cX^L} |f(x)| \leq C$ for a constant $C>0$, define two quantities $A$ and $B$ as
        \begin{align}
            A &:= \frac{1}{L-M} \sum_{l=M+1}^{L}  \EE_{X\mid\pi} \biggl[ \biggl(\sum_{k\in[d]}\frac{\ind(x_{L+1} = x_l = e_k)}{\bar y(k) + \varepsilon} - \frac{\bar y(k)\ind(x_{L+1} = e_k)}{\bar y(k)+\varepsilon} \biggr) \cdot f(X) \biggr],   \\
            B &:= \frac{1}{L-M} \sum_{l=M+1}^{L}  \EE_{X\mid\pi} \biggl[ \biggl(\sum_{k\in[d]}\frac{\ind(x_{L+1} = x_l = e_k)}{\mu^\pi(e_k)} - 1 \biggr) \cdot f(X) \biggr],  
        \end{align}
        where $\bar y = (L-M)^{-1} \sum_{l=M+1}^L x_l$.
        Under \Cref{asp:Markov_chain}, it holds that
        \begin{align}
            |A - B| &\leq  4C \cdot \frac{(1-\lambda)^{-1/2}(D_{\chi^2}(\mu_0\,\|\, \mu^\pi) + 1)^{1/4} + 2\sqrt{M}}{L^{1/2}\gamma} + C\gamma^{-1} \varepsilon. 
        \end{align}
        where $\mu_0(\cdot)$ is the initial distribution over the first $r_n$ tokens $X_{1:r_n}$. Here we let $D_{\chi^2}(\mu_0 \,\|\, \mu^\pi)$ to denote $D_{\chi^2}(\mu_0(X_{1:r_n}=\cdot ) \,\|\, \mu^\pi(X_{1:r_n}=\cdot ))$, i.e., the $\chi^2$-divergence between $\mu_n$ and the distribution over the first $r_n$ tokens under the stationary distribution $\mu^{\pi}$.
        \end{lemma}
    
        \begin{proof}[Proof of \cref{lem:approximation2}]
            Let us use $\bar y_X(\cdot)$ to remind the readers that $\bar y(\cdot)$ is also a function of $X$. 
            We simplify the expectation $\EE_{X\mid\pi}$ by $\EE$ in this proof. 
            By rearranging the terms, we have
            \begin{align}
                |A - B|  &= \biggl|\frac{1}{L-M} \sum_{l=M+1}^{L}  \EE \biggl[ \biggl(\sum_{k\in[d]}  \frac{\ind(x_{L+1} = x_l = e_k)}{\bar y_X(k) + \varepsilon}  - \sum_{k\in[d]}\frac{\ind(x_{L+1} = x_l = e_k)}{\mu^\pi(e_k)}\\
                & \hspace{5cm}- \sum_{k\in[d]}\frac{\bar y_X(k) \cdot \ind(x_{L+1} = e_k)}{\bar y_X(k)+\varepsilon} + 1 \biggr) \cdot  f(X)  \biggr] \biggr| \\
                &= \biggl|\frac{1}{L-M} \sum_{l=M+1}^{L}  \EE \biggl[ \biggl(\sum_{k\in[d]} \Bigl(  \frac{\mu^\pi(e_k) -\bar y_X(k) }{(\bar y_X(k)+\varepsilon) \cdot \mu^\pi(e_k)}  - \frac{\varepsilon}{(\bar y_X(k)+\varepsilon)\cdot \mu^\pi(e_k)} \Bigr)\cdot \ind(x_{L+1} = x_l = e_k) \\
                &\hspace{7cm} - \sum_{k\in[d]}\frac{\varepsilon\ind(x_{L+1} = e_k) }{\bar y_X(k)+\varepsilon}\biggr) \cdot  f(X)  \biggr] \biggr|. 
            \end{align}
            Here, we have three terms to control. 
            For the  first error term, we define 
            \begin{align}
                \err_1 &\defeq \biggl|\frac{1}{L-M} \sum_{l=M+1}^{L}  \EE \biggl[  \sum_{k\in[d]} \frac{\mu^\pi(e_k) -\bar y_X(k) }{(\bar y_X(k)+\varepsilon) \cdot \mu^\pi(e_k)}  \cdot \ind(x_{L+1} = x_l = e_k) \cdot f(X) \biggr]\biggr|\\
                &\le \frac{C}{L-M} \sum_{l=M+1}^{L}  \EE \biggl[  \sum_{k\in[d]} \frac{|\mu^\pi(e_k) -\bar y_X(k)|}{(\bar y_X(k)+\varepsilon) \cdot \mu^\pi(e_k)}  \cdot \ind(x_{L+1} = x_l = e_k)\biggr] \\
                & \le C \cdot \EE \biggl[  \sum_{k\in[d]} \frac{|\mu^\pi(e_k) -\bar y_X(k)|}{\mu^\pi(e_k)}\cdot \ind(x_{L+1}=e_k)\biggr]. 
            \end{align}
            The first inequality above holds by noting that $\sup_X |f(X)| \le C$ and the last inequality holds by noting that $\bar y_X(e_k) = (L-M)^{-1} \sum_{l=M+1}^L \ind(x_l = e_k)$.
            Using Cauchy-Schwarz inequality, we arrive at 
            \begin{align}
                \err_1 
                &\le C\cdot \bigg( \EE\bigg[\sum_{k\in[d]} \Big(\frac{\mu^\pi(e_k) - \bar y_X(k)}{\sqrt{\mu^\pi(e_k)}}\Big)^2 \bigg] \cdot \EE\bigg[\sum_{k\in[d]} \frac{\ind(x_{L+1}=e_k)}{\mu^\pi(e_k)}\bigg] \bigg)^{1/2} \\
                &\le C \gamma^{-1/2}\cdot \sqrt{\EE\left[D_{\chi^2}\left(\bar y_X(\cdot) \,\|\, \mu^\pi(x_{L+1}=\cdot)\right)\right]}.
            \end{align}
            For the second term, we similarly have
            \begin{align}
                \err_2  &= \bigg|\frac{1}{L-M} \sum_{l=M+1}^{L} \sum_{k\in [d]} \EE\left[
                \frac{\varepsilon}{(\bar y_X(k)+\varepsilon)\mu^\pi(e_k)} \cdot \ind(x_{L+1} = x_l = e_k) \cdot f(X)
                \right] \bigg| \\
                &\le C\bigg|\sum_{k\in [d]} \EE\bigg[
                \frac{\varepsilon \cdot \ind(x_{L+1} = e_k)}{\mu^\pi(e_k)} 
                \bigg] \bigg| \le C\gamma^{-1} \varepsilon. 
            \end{align} 
            Lastly, we have the error term 
            \begin{align}
                \err_3 &\defeq \frac{1}{L-M} \sum_{l=M+1}^{L} \EE\bigg[\sum_{k\in[d]}\frac{\varepsilon\ind(x_{L+1} = e_k) }{\bar y_X(k)+\varepsilon} \cdot f(X) \bigg] \le C \cdot \EE\bigg[\sum_{k\in[d]}\frac{\varepsilon\ind(x_{L+1} = e_k) }{\bar y_X(k)+\varepsilon} \bigg] \\
                &\le C\cdot \bigg|\EE\bigg[\sum_{k\in[d]}\frac{\varepsilon\ind(x_{L+1} = e_k) }{\mu^\pi(e_k) + \varepsilon}\bigg]\bigg| + C \cdot 
                \bigg|\sum_{k\in[d]}\EE\bigg[\frac{\varepsilon(\bar y_X(k) - \mu^\pi(e_k)) \cdot \ind(x_{L+1}=e_k)}{(\mu^\pi(e_k) + \varepsilon)(\bar y_X(k) + \varepsilon)} \bigg]\bigg|.
            \end{align}
            Here, the first term is upper bounded by $C\gamma^{-1}\varepsilon$, and for the second term we have by Cauchy-Schwartz that 
            \begin{align}
            & C\cdot \biggl|\sum_{k\in[d]}\EE\bigg[\frac{\varepsilon(\bar y_X(k) - \mu^\pi(e_k))\cdot \ind(x_{L+1}=e_k)}{(\mu^\pi(e_k) + \varepsilon)(\bar y_X(k) + \varepsilon)} \bigg]\biggr| \\
            &\quad \le C\cdot \sqrt{\EE\bigg[\sum_{k\in[d]}
                \frac{(\bar y_X(k) - \mu^\pi(e_k))^2 }{\mu^\pi(e_k)}
            \bigg] \cdot \EE\bigg[ \sum_{k\in[d]}\frac{ \varepsilon^2 \ind(x_{L+1}=e_k)}{(\bar y_X(k) + \varepsilon)^2 \mu^\pi(e_k)}\bigg]} \\
            &\quad \le C\gamma^{-1/2} \cdot \sqrt{\EE_X D_{\chi^2}(\bar y_X(\cdot )\,\|\, \mu^\pi(x_{L+1}=\cdot))},
            \end{align}
            which shares a similar upper bound as $\err_1$. 
            Now we invoke \Cref{lem:convergence chi-square} to conclude that 
            \begin{align}
                |A-B| &\le 
                \err_1 + \err_2 + \err_3 
                \le 2C \gamma^{-1/2} \cdot \sqrt{\EE_X D_{\chi^2}(\bar y_X(\cdot )\,\|\, \mu^\pi(x_{L+1}=\cdot))} + C\gamma^{-1}\varepsilon \\
                &\le 2C \gamma^{-1/2} \bigg( \frac{4(1-\lambda)^{-1}\sqrt{D_{\chi^2}(\mu_0\,\|\, \mu^\pi) + 1} + 16 M}{L \cdot \min_{x_{L+1}} \mu^\pi(x_{L+1})} \bigg)^{1/2} + C\gamma^{-1} \cdot \varepsilon \\
                &\le 2C \gamma^{-1} \cdot \frac{2(1-\lambda)^{-1/2}(D_{\chi^2}(\mu_0\,\|\, \mu^\pi) + 1)^{1/4} + 4 \sqrt{M}}{L^{1/2}} + C\gamma^{-1} \varepsilon.
            \end{align}
            Hence, we complete our proof of \Cref{lem:approximation2}.
        \end{proof}

        \cref{lem:approximation3} covers the approximation error due to the mixing property of the Markov chain.

\begin{lemma}\label{lem:approximation3}
Let $\cS \in \HleqD $ be a fixed set. 
For any $h \in \cS$, let $\tilde \sigma^\h$ and $\sigma^\h$ be two fixed probability distributions over $[M]$. That is, for any $i, j \in [M]$, we have  $\tilde \sigma_{-i}^\h , \sigma_{-j}^\h \in [0,1]$, and    $\sum_{i=1}^M \tilde \sigma_{-i}^\h =  \sum_{j=1}^M \sigma_{-j}^\h = 1$. 
Given these distributions over $[M]$, we define 
$$
\tilde v_{L+1}^\head{h} := \sum_{i  \in [M]} \tilde \sigma_{-i }^\head{h} \cdot x_{L+1-i}, \qquad \text{and} \qquad  v_l^\head{h} := \sum_{j  \in [M]} \sigma_{- j}^\head{h} \cdot x_{l-j},
$$
where we let $x_{l}\in\cX$ denote the $l$-th token in the Markov chain for all $l \in [L+1] $. 
Moreover, with slight abuse of notation, 
we let $(z, Z)=(z, z_{-1}, \ldots, z_{-M})\in\cX^{M+1}$  and $(x, X)=(x, x_{-1}, \ldots, x_{-M})\in\cX^{M+1}$ be two independent  random variables sampled from the stationary distribution $\mu^{\pi}$. 
We define random variables $\tilde v^\h (Z)$ and $v^\h (X) $ as 
\begin{align}
   \tilde v^\head{h}(Z) := \sum_{i  \in [M]} \tilde \sigma_{-i }^\head{h} \cdot  z_{-i} , \qquad \textrm{and} \qquad  v^\head{h}(X) := \sum_{j  \in [M]} \sigma_{-j }^\head{h} \cdot  x_{-j }.
\end{align}
Using $\tilde v_{L+1}^\h$, $v_{l}^\h$, $\tilde v^\h (Z)$, and $v^\h (X)$, we define two quantities $A$ and $B$ as
\begin{align}
A &:= \frac{1}{L-M} \sum_{l=M+1}^{L}  \EE_{X\mid\pi} \biggl[ \biggl(\sum_{k\in[d]}\frac{\ind(x_{L+1} = x_l = e_k)}{\mu^\pi(e_k)} - 1 \biggr) \cdot \prod_{h\in \cS} \langle v_l^\head{h}, \tilde v_{L+1}^\head{h} \rangle \biggr],  \\
B &:= \EE_{(x, X), (z, Z) \sim \mu^\pi \times \mu^\pi} \biggl[ \biggl(\sum_{k\in[d]}\frac{\ind(x = z = e_k)}{\mu^\pi(e_k)} - 1 \biggr) \cdot \prod_{h\in \cS} \langle \tilde v^\h(Z), v^\h(X) \rangle \biggr],
\end{align}
where $\EE_{X \given \pi}$ means that the expectation is taken with respect to the randomness of the Markov chain with transition $\pi$. 
Then, under \cref{asp:Markov_chain}, we have 
\begin{align}
|A - B| \leq \frac{8M}{L\gamma} + \frac{16\sqrt{D_{\chi^2}(\mu_0\,\|\, \mu^\pi) + 1}}{L (1-\lambda) \gamma^{|\cS|/2+1}},
\end{align}
where $\mu_0(\cdot)$ is the initial distribution over the first $r_n$ tokens $X_{1:r_n}$ and $D_{\chi^2}(\mu_0 \,\|\, \mu^\pi)$  is a short-hand notation of  $D_{\chi^2}(\mu_0(X_{1:r_n}=\cdot ) \,\|\, \mu^\pi(X_{1:r_n}=\cdot ))$. 
\end{lemma}

\begin{proof}[Proof of \cref{lem:approximation3}]
By triangular inequality, we have
\begin{align}
    |A - B|  & \leq \biggl| \frac{1}{L-M} \sum_{l=M+1}^{L}  \EE _{X\mid\pi} \biggl[ \biggl(\sum_{k\in[d]}\frac{\ind(x_{L+1} = x_l = e_k)}{\mu^\pi(e_k)} \biggr) \cdot \prod_{h\in \cS} \langle v_l^\head{h}, \tilde v_{L+1}^\head{h} \rangle \biggr] \\
    & \qquad -\EE_{(x,X), (z,Z)\sim \mu^\pi \times \mu^\pi} \biggl[ \biggl(\sum_{k\in[d]}\frac{\ind(x = z = e_k)}{\mu^\pi(e_k)}\biggr) \cdot \Bigl( \prod_{h\in \cS} \langle \tilde v^\head{h}(Z), v^\head{h}(X)  \rangle \Bigr)   \biggr] \biggr| \\
    & \qquad + \biggl| \frac{1}{L-M} \sum_{l=M+1}^{L}  \EE _{X\mid\pi} \biggl[ \prod_{h\in \cS} \langle v_l^\head{h}, v_{L+1}^\head{h} \rangle \biggr]  - \EE_{ (x,X), (z,Z)\sim \mu^\pi \times \mu^\pi}  \biggl[  \biggl( \prod_{h\in \cS} \langle \tilde v^\head{h}(Z), v^\head{h}(X)  \rangle   \biggr) \biggr] \biggr|. 
\end{align}
We will establish the upper bounds for each of the absolute value terms. We first focus on the first absolute value term.

\paragraph{Bounding the First Absolute Value Term}
Let $p^\pi(X)$ denote the joint distribution of the whole sequence $X$ under kernel $\pi$.
By the definitions of $\tilde v_{L+1}^\head{h}$ and  $ v_l^\head{h} $, we have 
\begin{align}
\langle v_l^\head{h}, v_{L+1}^\head{h} \rangle & =   \sum_{i_h, j_h \in [M]} \sigma_{-i_h}^\head{h} \cdot  \sigma_{-j_h}^\head{h} \cdot \langle  x_{L+1 - i_h} , x_{l-j_h} \rangle \\
& =  \sum_{i_h, j_h \in [M]}  \sum_{ k \in [d]}  \sigma_{-i_h}^\head{h} \cdot  \sigma_{-j_h}^\head{h} \cdot \ind(x_{L+1 - i_h} = x_{l-j_h} = e_k),
\end{align}
where we use $(i_h, j_h)$ as the indices to highlight that they are associated with head $h$. And we use $k_h\in[d]$ to index all the possible common values for $x_{l-i_h}$ and $x_{L+1-j_h}$. 
Then plugging this equality into $\prod_{h\in \cS} \langle v_l^\head{h}, v_{L+1}^\head{h} \rangle $ and exchanging the order of product and summation, we have 
\begin{align} \label{eq:expand_prod}
    &  \biggl(\sum_{k\in[d]}\frac{\ind(x_{L+1} = x_l = e_k)}{\mu^\pi(e_k)} \biggr) \cdot \prod_{h\in \cS} \langle v_l^\head{h}, v_{L+1}^\head{h} \rangle   \\
    & \qquad =   \sum_{\{ (i_h, j_h) \} _{h\in \cS }}  \sum_{ \{ k_h \}_{h\in \cS }, k \in [d] }   \frac{\ind(x_{L+1} = x_l = e_k)}{\mu^\pi(z = e_k)}  
    \cdot  \bigg( \prod _{h \in \cS} \sigma_{-i_h}^\head{h} \cdot  \sigma_{-j_h}^\head{h} \cdot \ind(x_{L+1 - i_h} = x_{l-j_h} = e_{k_h} ) \bigg) ,
\end{align}
where the summation means that we sum over all possible values that $\{i_h, j_h, k_h\}_{ h\in \cS} $ and $k$ can take. Specifically, 
each $i_h$ and $j_h$ take values in $[M]$, and  each $k_h$ and $k$ takes values in $[d]$. 
Moreover, using the property of indicator functions, we can further simplify \eqref{eq:expand_prod} by gathering all indicators: 
\begin{align}
    \label{eq:expand_prod2}
    &  \biggl(\sum_{k\in[d]}\frac{\ind(x_{L+1} = x_l = e_k)}{\mu^\pi(e_k)} \biggr) \cdot \prod_{h\in \cS} \langle v_l^\head{h}, v_{L+1}^\head{h} \rangle   \\
   &  \qquad =   \sum_{\{ (i_h, j_h )\} _{h\in \cS }}   \bigg( \prod _{h \in \cS} \sigma_{-i_h}^\head{h} \cdot  \sigma_{-j_h}^\head{h} \bigg) \cdot \bigg( \sum_{ \{ k_h \}_{h\in \cS }, k \in [d] }   \frac{\ind(x_{L+1} = x_l = e_k, x_{L+1 - i_h} = x_{l-j_h} = e_{k_h}, \forall h \in \cS )}{\mu^\pi(z = e_k)}  
     \bigg) . 
\end{align}
Now we take expectations with respect to the randomness of $X$ on both ends of \eqref{eq:expand_prod2}
and get  
\begin{align}
    & \frac{1}{L-M} \sum_{l=M+1}^{L}  \EE  \biggl[ \biggl(\sum_{k\in[d]}\frac{\ind(x_{L+1} = x_l = e_k)}{\mu^\pi(e_k)} \biggr) \cdot \prod_{h\in \cS} \langle v_l^\head{h}, v_{L+1}^\head{h} \rangle \biggr]\\
    & =   \!\!\! \sum_{\{ (i_h, j_h)\} _{h\in\cS}}  \!\!\!\left(\prod_{h\in\cS}\tilde \sigma_{-i_h}^\h \sigma_{-j_h}^\head{h} \right) \cdot \sum_{\{k_h\} _{h\in\cS}, k \in [d]} \frac{\sum_{l=M+1}^{L} p^\pi(x_{L+1} = x_l = e_k, x_{L+1-i_h} = x_{l-j_h} =e_{k_h}, \forall h \in \cS) }{(L-M) \cdot \mu^\pi(z = e_k)}.
\end{align}

To further simplify the above equality, 
we define a new probability distribution over $X_{L+1-M:L+1}$ and another subsequence of length $M+1$. 
Note that $X_{L+1-M:L+1}   $ contains is a subsequence with $M+1$ tokens. 
We let $(z, Z) =  (z, z_{-1}, \ldots, z_{-1} , z_{-M})$ denote a random token sequence of size $M+1$ in reverse order. We define a joint distribution $\hat p^{\pi}$ over $X_{L+1-M:L+1} $ and $(z, Z)$ as follows. 
Let  $E=(E_{0}, E_{-1}, \ldots, E_{-M} )$ and $E'= (E'_{0}, E'_{-1}, \ldots,   E'_{-M})$ be two elements in  $\cX^{M+1}$. That is, each component of $E$ and $E'$ are in $\cX$. The probability mass function of $\hat p^{\pi}$ 
is defined as 
\begin{align}\label{def:hat_p_pi}
&  \hat p^\pi\big ((x_{L+1},  x_{L},\ldots, x_{L+1-M} )=E, (z, Z)=E' \big )  \\
& \qquad  = \frac{1}{L-M} \sum_{l=M+1}^{L} p^\pi\bigl ((x_{L+1},  x_{L},\ldots, x_{L+1-M} )=E , (x_l, x_{l-1}, \ldots, x_{l = M })=E' \bigr ).
\end{align} 
That is, $\hat p^{\pi} $ can be viewed as the joint distribution of $X_{L+1-M: L+1} $ with an averaged distribution of the history. 
When $L$ is sufficiently large, by the mixing property of the Markov chain, we expect that, under $\hat p^{\pi}$, $(z,Z)$ is approximately independent of $X_{L+1-M: L+1} $, and the marginal distributions of $(z, Z) $ and $X_{L+1-M: L+1} $ are both  close to the stationary distribution $\mu^{\pi}$. We will translate this intuition into a rigorous argument in \Cref{lem:hat p-approx-TV}, which bounds the total-variation distance between $\hat p^{\pi}$ and the product distribution  $\mu^{\pi} \times \mu^{\pi} $. 

With $\hat p^{\pi} $ defined in \eqref{def:hat_p_pi}, we can rewrite the expectation above as 
\begin{align} \label{eq:g2_first_term}
& \frac{1}{L-M} \sum_{l=M+1}^{L}  \EE  \biggl[ \biggl(\sum_{k\in[d]}\frac{\ind(x_{L+1} = x_l = e_k)}{\mu^\pi(e_k)} \biggr) \cdot \prod_{h\in \cS} \langle v_l^\head{h}, v_{L+1}^\head{h} \rangle \biggr]\\
    & \qquad =   \!\!\! \sum_{\{ (i_h, j_h)\} _{h\in\cS}}  \!\!\!\left(\prod_{h\in\cS}\tilde \sigma_{-i_h}^\h \sigma_{-j_h}^\head{h} \right) \cdot \sum_{\{  k_h \} _{h\in\cS}, k\in[d]  } \frac{\hat p^\pi(x_{L+1} = z = e_k, x_{L+1-i_h} =z_{-j_h} = e_{k_h}, \forall h \in \cS) }{\mu^\pi(z = e_k)}.  \notag 
\end{align}

Similarly, by the definitions of $\tilde v^\h (Z) $ and $v^\h (Z)$, we can write $\langle \tilde v^\h (Z) , v^\h (X) \rangle  $ as 
$$
\langle\tilde v^\h (Z) , v^\h (X) \rangle =  \sum_{i_h, j_h \in [M]}  \sum_{ k_h \in [d]}  \sigma_{-i_h}^\head{h} \cdot  \sigma_{-j_h}^\head{h} \cdot \ind(z_{ - i_h} = x_{ j_h} = e_k).
$$
Then, multiplying these terms with $h \in \cS$, we can write 
\begin{align}
 & 
\biggl(\sum_{k\in[d]}\frac{\ind(x = z = e_k)}{\mu^\pi(e_k)}\biggr) \cdot 
 \prod_{h\in \cS} \langle\tilde v^\h (Z) , v^\h (X) \rangle \notag \\
 &  \qquad =   \sum_{\{ (i_h, j_h) \} _{h\in \cS }}   \bigg( \prod _{h \in \cS} \sigma_{-i_h}^\head{h} \cdot  \sigma_{-j_h}^\head{h} \bigg) \cdot \bigg( \sum_{ \{ k_h \}_{h\in \cS }, k \in [d] }   \frac{\ind(z  = x  = e_k, z_{ - i_h} = x_{-j_h} = e_{k_h}, \forall h \in \cS )}{\mu^\pi(z = e_k)}  
     \bigg) . \label{eq:second_term_prod}
\end{align}
Recall that here $(z, Z) = (z, z_{-1}, \ldots, z_{-M})$ and $(x, X) = (x, x_{-1}, \ldots, x_{-M})$ are independently sampled from the stationary distribution $\mu^{\pi}$. Taking the expectation under $\mu^{\pi} $, we have  


\begin{align}
    &\EE_{(x,X), (z,Z)\sim \mu^\pi \times \mu^\pi} \biggl[ \biggl(\sum_{k\in[d]}\frac{\ind(x = z = e_k)}{\mu^\pi(e_k)}\biggr)\cdot \biggl( \prod_{h\in \cS} \langle \tilde v^\head{h}(Z), v^\head{h}(X)  \rangle \biggr)    \biggr] \label{eq:g3_first_term} \\
    &\qquad = \!\!\!\sum_{\{(i_h, j_h)\} _{h\in\cS}} \!\!\!\left(\prod_{h\in\cS}\tilde \sigma_{-i_h}^\h \sigma_{-j_h}^\head{h}\right) \cdot \!\!\!\sum_{\{  k_h \} _{h\in\cS}, k\in[d]  } \!\!\! \frac{\mu^\pi(x = e_k, x_{-i_h} = e_{k_h}, \forall h \in \cS)\cdot \mu^\pi(z = e_k, z_{-j_h} = e_{k_h}, \forall h \in \cS) }{\mu^\pi(z = e_k)} . 
\end{align}

To bound the first absolute value term in the upper bound on $|A - B|$, we aim to compare \eqref{eq:g2_first_term} and \eqref{eq:g3_first_term}.
To this end,  let us fix collections of index pairs $ (i_h, j_h) _{h\in\cS}$.
Let $\cS_1 = \{ i_h: h\in\cS\}$ and $\cS_2 = \{ j_h:h\in\cS \}$ be the unique values in  $ ( i_h)_{h\in \cS} $ and $( j_h )_{h\in \cS} $. 
Since there might exists two elements $h$ and $h'$ in $\cS$ such that $i_h = i_{h'}$ or $j_h = j_{h'}$,  $|\cS_1|$ and  $|\cS_{2}|$ might be strictly less than $|\cS|$.  
As a result, $\hat p^\pi(x_{L+1} = z = e_k, x_{L+1-i_h} =z_{-j_h} = e_{k_h}, \forall h \in \cS) $ only involves random variables $x_{L+1}$, $X_{L+1 - \cS_1} = \{ x_{L+1 - i}\}_{i\in \cS_1}$,  $z$, $Z_{-\cS_2} = \{ z_{-j}\}_{j   \in \cS_{2}} $, which are a subset of the random variables defined in \eqref{def:hat_p_pi}.  
Similarly, 
$$
\mu^\pi(x = e_k, x_{-i_h} = e_{k_h}, \forall h \in \cS)\cdot \mu^\pi(z = e_k, z_{-j_h} = e_{k_h}, \forall h \in \cS) 
$$ 
only involves a subset of random variables  $x$, $X_{-\cS_1} = \{ x_{-i}\}_{i\in \cS_1}$, $z$, and $Z_{-\cS_2}$. 
Let us define  $\bar E =(E_0, (E_{-i})_{i\in\cS_1})\in \cX^{|\cS_1|+1}$ and $  \bar E'=(E_0', (E_{-j}')_{j\in\cS_2}) \in \cX^{|\cS_2|+1}$. 
By enumerating $\bar E $  in $\cX^{|\cS_1|+1}$ and $\bar E'$ in $\cX^{|\cS_2|+1}$, we equivalently enumerate all possible values the above random variables can take. 
Therefore, by comparing \eqref{eq:g2_first_term} with \eqref{eq:g3_first_term}, we have 
\begin{align}
    &\sum_{ \{ k_h\}  _{h\in\cS} , k\in[d]}
    \bigl|\hat p^\pi(x_{L+1} = z = e_k, x_{L+1-i_h} =z_{-j_h} = e_{k_h}, \forall h \in \cS) \\ 
    &\hspace{3cm} - \mu^\pi(x = e_k, x_{-i_h} = e_{k_h}, \forall h \in \cS)\cdot \mu^\pi(z = e_k, z_{-j_h} = e_{k_h}, \forall h \in \cS)\bigr| \\ 
    &\qquad = \sum_{\bar E, \bar E'} \bigl|\hat p^\pi \big ((x_{L+1}, X_{L+1 - \cS_1})=\bar E, (z, Z_{-\cS_2})= \bar E' \big ) - \mu^\pi \bigl ((x_{L+1}, X_{L+1-\cS_1})=\bar E\bigr )\cdot \mu^\pi \big ((z, Z_{-\cS_2})=E' \bigr )\bigr| \\
    &\hspace{3cm} \cdot \ind(E_{0}=E'_{0}, E_{-i_h}=E'_{-j_h}, \forall h\in \cS) \\
    &\qquad \le 2 \norm{\hat p^\pi(Y=\cdot, Y'=\cdot) - \mu^\pi(Y=\cdot) \times \mu^\pi(Y'=\cdot)}_{\TV}, 
    \label{eq:approximation3-1}
\end{align}
where in the last line, we use $Y$ and $Y'$ as placeholders for the random variables $(x_{L+1}, X_{L+1 - \cS_1})$ and $(z, Z_{-\cS_2})$ respectively. 
In the first equality, we sum over $\bar E \in  \cX^{|\cS_1|+1}$ and $\bar E' \in \cX^{|\cS_2|+1}$, and the last inequality follows from the definition of total variation  distance and dropping the indicator. 
By \Cref{lem:hat p-approx-TV}, this total variation distance is bounded by
\begin{align}
    &2 \norm{\hat p^\pi(Y=\cdot, Y'=\cdot) - \mu^\pi(Y=\cdot) \times \mu^\pi(Y'=\cdot)}_{\TV} \\
    &\quad \leq \frac{4M}{L} + \frac{8\sqrt{D_{\chi^2}(\mu_0\,\|\, \mu^\pi) + 1}}{L (1-\lambda)\cdot \sqrt{\min_{x_{L+1}, X_{L+1-\cS_1}} \mu^\pi(x_{L+1}, X_{L+1-\cS_1})}} \\
    &\quad \leq \frac{4M}{L} + \frac{8\sqrt{D_{\chi^2}(\mu_0\,\|\, \mu^\pi) + 1}}{L (1-\lambda) \cdot  \gamma^{(|\cS|+1)/2}},
    \label{eq:approximation3-2}
\end{align}
where the last inequality holds by \Cref{cor:station_lb} and the fact that $|\cS_1| \le |\cS|$. Specifically,  \Cref{cor:station_lb} implies that the density function of the joint distribution of $x_{L+1}$ and $X_{L+1 - \cS_1}$ is lower bounded by $\gamma ^{|\cS_1 | +1} \geq \gamma ^{|\cS| + 1} $. 
Thus, combining \eqref{eq:approximation3-1} and \eqref{eq:approximation3-2}, we have
\begin{align} 
    & \biggl| \sum_{ \{ k_h\}  _{h\in\cS} , k\in[d]} \frac{\hat p^\pi(x_{L+1} = z = e_k, x_{L+1-i_h} =z_{-j_h} = e_{k_h}, \forall h \in \cS) }{\mu^\pi(z = e_k)} \\
    & \qquad \qquad - \sum_{ \{ k_h\}  _{h\in\cS} , k\in[d]}\frac{\mu^\pi(x = e_k, x_{-i_h} = e_{k_h}, \forall h \in \cS)\cdot \mu^\pi(z = e_k, z_{-j_h} = e_{k_h}, \forall h \in \cS) }{\mu^\pi(z = e_k)}  \biggr| \\
    & \qquad \leq \frac{1}{\gamma } \cdot \left(\frac{4M}{L} + \frac{8\sqrt{D_{\chi^2}(\mu_0\,\|\, \mu^\pi) + 1}}{L (1-\lambda) \cdot  \gamma^{(|\cS|+1)/2}}\right).
    \label{eq:g2_g3error}
\end{align}

Therefore, to bound the first absolute value term, we combine \eqref{eq:g2_first_term},  \eqref{eq:g3_first_term}, and \eqref{eq:g2_g3error} and use triangle inequality to get 
\begin{align}
     &   \biggl| \frac{1}{L-M} \sum_{l=M+1}^{L}  \EE _{X\mid\pi} \biggl[ \biggl(\sum_{k\in[d]}\frac{\ind(x_{L+1} = x_l = e_k)}{\mu^\pi(e_k)} \biggr) \cdot \prod_{h\in \cS} \langle v_l^\head{h}, \tilde v_{L+1}^\head{h} \rangle \biggr] \notag \\
    & \qquad\qquad -\EE_{(x,X), (z,Z)\sim \mu^\pi \times \mu^\pi} \biggl[ \biggl(\sum_{k\in[d]}\frac{\ind(x = z = e_k)}{\mu^\pi(e_k)}\biggr) \cdot \Bigl( \prod_{h\in \cS} \langle \tilde v^\head{h}(Z), v^\head{h}(X)  \rangle \Bigr)   \biggr] \biggr| \notag \\
    & \qquad \leq \frac{1}{\gamma } \cdot \left(\frac{4M}{L} + \frac{8\sqrt{D_{\chi^2}(\mu_0\,\|\, \mu^\pi) + 1}}{L (1-\lambda) \cdot  \gamma^{(|\cS|+1)/2}}\right) \cdot \sum_{\{ (i_h, j_h)\} _{h\in\cS}} \bigg(\prod_{h\in\cS}\tilde \sigma_{-i_h}^\h \sigma_{-j_h}^\head{h} \bigg) .\label{eq:bound_first_abs}
\end{align}
Furthermore,   recall  that $\tilde \sigma^\h  $ and $\sigma^\h $ are probability distributions over $[M]$ for all $h \in \cS$. 
By going over all possible values that $\{ (i_h, j_h)\}_{h\in\cS}$ can take, we have 
\begin{align}
    \sum_{\{ (i_h, j_h)\} _{h\in\cS}} \bigg(\prod_{h\in\cS}\tilde \sigma_{-i_h}^\h \sigma_{-j_h}^\head{h} \bigg)  = \prod_{h\in\cS} \bigg(\sum_{k\in [M]} \tilde \sigma_{-k}^\h \bigg) \cdot \bigg(\sum_{k\in [M]} \sigma_{-k}^\head{h} \bigg) = 1. \label{eq:prod_sum_1}
\end{align}
Plugging this equality into \eqref{eq:bound_first_abs}, we show that 
the upper bound in \eqref{eq:bound_first_abs}  can be reduced to the right-hand side of \eqref{eq:g2_g3error}.

\paragraph{Bounding the Second Absolute Value Term}
For the second absolute value term, an analogous argument can be applied. In fact, the proof is simpler because we only need to handle $\langle v^\h _{l}, v_{L+1}^\h $ and $\langle \tilde v^\h (Z) , v^\h (Z) \rangle $ and do not have indicators $\ind (x_{L+1} = x_{l}) $ and $\ind (x = z)$. 

Similar to the derivation in \eqref{eq:g2_first_term} and \eqref{eq:g3_first_term},\begin{align}
    & \biggl| \frac{1}{L-M} \sum_{l=M+1}^{L}  \EE  \biggl[ \prod_{h\in \cS} \langle v_l^\head{h}, \tilde v_{L+1}^\head{h} \rangle \biggr]  - \EE_{ (x,X), (z,Z)\sim \mu^\pi \times \mu^\pi}  \biggl[  \biggl( \prod_{h\in \cS} \langle \tilde v^\head{h}(Z), v^\head{h}(X)  \rangle   \biggr) \biggr] \biggr| \\
    &\qquad   = \biggl | \sum_{\{(i_h, j_h)\} _{h\in\cS}} \!\!\!\bigg(\prod_{h\in\cS}\tilde \sigma_{-i_h}^\h \sigma_{-j_h}^\head{h}\bigg) \cdot \!\!\!\sum_{\{  k_h \} _{h\in\cS}   } \!\!\! \Bigl ( \hat p^\pi(  x_{L+1-i_h} =z_{-j_h} = e_{k_h}, \forall h \in \cS) \label{eq:abs_term2_1} \\
    & \qquad \qquad \qquad \qquad\qquad \qquad \qquad \qquad  -  
    \mu^\pi(  x_{-i_h} = e_{k_h}, \forall h \in \cS)\cdot \mu^\pi(  z_{-j_h} = e_{k_h}, \forall h \in \cS)  \Big) \bigg|   .\notag  
\end{align}
Similar to  \eqref{eq:approximation3-1}, for any fixed collection of index pairs $  (i_h, j_h) _{h \in \cS}$, we let   $\cS_1 = \{ i_h: h\in\cS\}$ and $\cS_2 = \{ j_h:h\in\cS \}$ denote the unique values in  $(i_h)_{h\in \cS} $ and $ ( j_h )_{h\in \cS} $. 
By \Cref{lem:hat p-approx-TV}, we have 
\begin{align} 
    & \biggl| \sum_{\{  k_h \} _{h\in\cS}   } \!\!\! \Bigl ( \hat p^\pi(  x_{L+1-i_h} =z_{-j_h} = e_{k_h}, \forall h \in \cS)   -  
    \mu^\pi(  x_{-i_h} = e_{k_h}, \forall h \in \cS)\cdot \mu^\pi(  z_{-j_h} = e_{k_h}, \forall h \in \cS)  \Big) \bigg|    \notag \\
    & \qquad \leq 2 \big \|\hat p^\pi(   \tilde Y=\cdot, \tilde Y'=\cdot) - \mu^\pi(\tilde Y=\cdot) \times \mu^\pi(\tilde Y'=\cdot)\big \|_{\TV} 
 \leq  \frac{4M}{L} + \frac{8\sqrt{D_{\chi^2}(\mu_0\,\|\, \mu^\pi) + 1}}{L (1-\lambda) \cdot  \gamma^{|\cS|/2} }. \label{eq:abs_term2_2}
\end{align}
Here we use $\tilde Y$ and $\tilde Y'$ as placeholders for random variables $X_{L+1 - \cS_1} $ and $Z_{-\cS_2}$. We note that \Cref{lem:hat p-approx-TV} can be applied to any subsets of $X_{L+1-M:L+1}$ and  $(z, Z)$. 
Therefore, combining \eqref{eq:prod_sum_1}, \eqref{eq:abs_term2_1}, and \eqref{eq:abs_term2_2}, we conclude that 
\begin{align}
    & \biggl| \frac{1}{L-M} \sum_{l=M+1}^{L}  \EE  \biggl[ \prod_{h\in \cS} \langle v_l^\head{h}, \tilde v_{L+1}^\head{h} \rangle \biggr]  - \EE_{ (x,X), (z,Z)\sim \mu^\pi \times \mu^\pi}  \biggl[  \biggl( \prod_{h\in \cS} \langle \tilde v^\head{h}(Z), v^\head{h}(X)  \rangle   \biggr) \biggr] \biggr| \\
    &\quad \leq  \frac{4M}{L} + \frac{8\sqrt{D_{\chi^2}(\mu_0\,\|\, \mu^\pi) + 1}}{L (1-\lambda) \gamma^{|\cS|/2}}.
\end{align}
Note that the second upper bound is dominated by the previous one. 
This completes the proof of \Cref{lem:approximation3}.
\end{proof}

\cref{lem:approximation4} provides an approximation result using the definition of the modified $\chi^2$-mutual information. 
\begin{lemma}
    \label{lem:approximation4}
    Consider a fixed set $\cS \in \HleqD $.
    For any $h \in \cS$, let $\tilde \sigma^\h$ and $\sigma^\h$ be two probability distributions over $[M]$. That is, for any $i, j \in [M]$, we have  $\tilde \sigma_{-i}^\h , \sigma_{-j}^\h \in [0,1]$, and    $\sum_{i=1}^M \tilde \sigma_{-i}^\h =  \sum_{j=1}^M \sigma_{-j}^\h = 1$. 
    Moreover, 
    we let $(z, Z)=(z, z_{-M}, \ldots, z_{-1})\in\cX^{M+1}$  and $(x, X)=(x, x_{-M}, \ldots, x_{-1})\in\cX^{M+1}$ be two independent  random variables sampled from the stationary distribution $\mu^{\pi}$. 
    We define random variables $\tilde v^\h (Z)$ and $v^\h (X) $ as 
    \begin{align}
    \tilde v^\head{h}(Z) := \sum_{i  \in [M]} \tilde \sigma_{-i }^\head{h} \cdot  z_{-i} , \qquad \textrm{and} \qquad  v^\head{h}(X) := \sum_{j  \in [M]} \sigma_{-j }^\head{h} \cdot  x_{-j }.
    \end{align}
   Let  $  (i_h^\star,j_h^\star) _{h\in\cS} $ be any fixed collection of index pairs, where    $i_h^\star\in[M]$ and $j_h^\star\in[M]$ for all $h \in \cS$.  We define quantities $A$ and $B$ as
    \begin{align}
        A &:= \EE_{\pi, (x,X), (z,Z) \sim \mu^\pi \times \mu^\pi} \biggl[ \biggl(\sum_{k\in[d]}\frac{\ind(x = z = e_k)}{\mu^\pi(e_k)} - 1 \biggr) \cdot \prod_{h\in \cS} \langle \tilde v^\h(Z), v^\h(X) \rangle, \biggr], \\
        B &:= \EE_{\pi, (x,X), (z,Z) \sim \mu^\pi \times \mu^\pi} \biggl[  \prod_{h\in \cS} \ind(x_{-i_h^\star} = z_{-j_h^\star}) \cdot \biggl(\sum_{k=1}^d\frac{\ind(x = z = e_k)}{\mu^\pi(e_k)} - 1 \biggr) \biggr].
    \end{align} 
    Under \Cref{asp:Markov_chain}, it holds that 
    \begin{align}
        \Bigl| \EE_\pi \left[ A \right] - \prod_{h\in\cS} \sigma_{-i_h^\star}^\h \tilde\sigma_{-j_h^\star}^\head{h} \cdot B \Bigr| \leq         
            \Bigl(1 - \prod_{h\in\cS}\sigma_{-i_h^\star}^\h \tilde\sigma_{-j_h^\star}^\head{h} \Bigr) \cdot I_{\chi^2}(\cS^\star).
    \end{align}
\end{lemma}

\begin{proof}[Proof of \cref{lem:approximation4}]
To simplify the notation, we define a signal set $\Gamma(\cS)$ and an error set $\bar\Gamma(\cS)$ as  $$\Gamma(\cS) := \cbr{ (i_h^\star,j_h^\star)_{h\in\cS} }, \qquad \bar\Gamma(\cS) := \cbr{ (i_h,j_h)_{h\in\cS} \in ([M]\times[M])^{|\cS|}} \backslash \Gamma(\cS). $$
Similar to \eqref{eq:second_term_prod}, we can write 
$\EE_\pi [A]$  as 
\begin{align}
    \EE_\pi [A]  =\EE_{\pi, (x,X), (z,Z)\sim \mu^\pi \times \mu^\pi} \biggl[  \sum_{\{ (i_h,j_h)\} _{h \in \cS}}\prod_{h\in\cS}\sigma_{-i_h}^\h \tilde \sigma_{-j_h}^\head{h} \cdot \ind(x_{-i_h} = z_{-j_h}) \cdot \biggl(\sum_{k\in[d]}\frac{\ind(x = z = e_k)}{\mu^\pi(e_k)} - 1 \biggr) \biggr] ,
\end{align}
where we exchange the order of product and summation. 
Using the notation $\bar\Gamma(\cS)$, we can split the summation into two parts: 
    \begin{align}
        \EE_\pi [A] 
        & =  \EE_{\pi, (x,X), (z,Z)\sim \mu^\pi \times \mu^\pi}  \biggl[  \prod_{h\in\cS} \sigma_{-i_h^\star}^\h \tilde \sigma_{-j_h^\star}^\h \cdot  \ind(x_{-i_h^\star} = z_{-j_h^\star}) \cdot \biggl(\sum_{k\in[d]}\frac{\ind(x = z = e_k)}{\mu^\pi(e_k)} - 1 \biggr) \biggr] \\
        & \qquad  +  \EE_{\pi, (x,X), (z,Z)\sim \mu^\pi \times \mu^\pi} \biggl[  \sum_{\{ (i_h, j_h)\}_{h \in \cS} \in \bar\Gamma(\cS)}\prod_{h\in\cS}\sigma_{-i_h}^\h \tilde\sigma_{-j_h}^\head{h} \cdot \ind(x_{-i_h} = z_{-j_h}) \biggl(\sum_{k\in[d]}\frac{\ind(x = z = e_k)}{\mu^\pi(e_k)} - 1 \biggr) \biggr] \\
        & =  \prod_{h\in\cS} \sigma_{-i_h^\star}^\h \tilde\sigma_{-j_h^\star}^\h \cdot B \\
        & \qquad  +  \EE_{\pi, (x,X), (z,Z)\sim \mu^\pi \times \mu^\pi} \biggl[  \sum_{(i_h, j_h)_{h \in \cS} \in \bar\Gamma(\cS)}\prod_{h\in\cS}\sigma_{-i_h}^\h \tilde\sigma_{-j_h}^\head{h} \cdot \ind(x_{-i_h} = z_{-j_h}) \biggl(\sum_{k\in[d]}\frac{\ind(x = z = e_k)}{\mu^\pi(e_k)} - 1\biggr) \biggr].
    \end{align}
    Here  last equality holds by the definition of the $B$ and the fact that $\tilde \sigma^\h$ and $\sigma^\h $ are fixed vectors.

Therefore, to prove this lemma, it suffices to upper bound the second term above. 
To this end, we apply \cref{lem:cross_mutual} stated below for any fixed set of indices $ (i_h, j_h) _{h\in \cS} \in \bar \Gamma (\cS)$.  Specifically,  let  $\cS_1 = \{i_h \colon      h\in\cS\} $ and $\cS_2 = \{j_h  \colon    h\in\cS\} $  denote the unique values of $( i_h )_{h\in \cS} $ and $( j_h )_{h\in \cS}$. 
\cref{lem:cross_mutual} implies that 
    \begin{align} \label{eq:lem_d7}
        \EE_{\pi, (x,X), (z,Z)\sim \mu^\pi \times \mu^\pi} \bigg[\prod_{h\in\cS} \ind(x_{-i_h} = z_{-j_h}) \cdot  \biggl(\sum_{k\in[d]}\frac{\ind(x = z = e_k)}{\mu^\pi(z = e_k)} - 1 \biggr) \bigg] \le I_{\chi^2}(\cS^\star).
    \end{align}
    Combining  \cref{eq:lem_d7} with the fact  that $$\sum_{(i_h, j_h)_{h \in \cS} \in \bar\Gamma(\cS)}\prod_{h\in\cS}\sigma_{-i_h}^\h \tilde\sigma_{-j_h}^\head{h} = 1 - \prod_{h\in\cS} \sigma_{-i_h^\star}^\h \tilde\sigma_{-j_h^\star}^\head{h}, $$ 
    the desired term  is  bounded above by $(1 - \prod_{h\in\cS} \sigma_{-i_h^\star}^\h \tilde\sigma_{-j_h^\star}^\head{h} )  \cdot I_{\chi^2}(\cS^\star)$, which concludes the proof.
\end{proof}

\begin{lemma} \label{lem:cross_mutual}
Let  $\cS \in \HleqD$ be a fixed subset and let   $\{(i_h, j_h)\}_{h\in\cS}$  be a fixed collection of index pairs, where $i_h, j_h\in[M]$ for all $h \in \cS$. 
    Let $\cS_1 = \{i_h  \colon  h\in\cS\} $ and $\cS_2 = \{j_h  \colon    h\in\cS\}$ denote the unique values of $(i_h )_{h\in \cS} $ and $ (j_h)_{h \in \cS}$. 
    We let $(z, Z)=(z, z_{-M}, \ldots, z_{-1})\in\cX^{M+1}$  and $(x, X)=(x, x_{-M}, \ldots, x_{-1})\in\cX^{M+1}$ be two independent  random variables sampled from the stationary distribution $\mu^{\pi}$, where $\pi$ is the transition kernel of the Markov chain and is sampled from prior $\cP$.
    If \Cref{asp:Markov_chain} holds, it follows that
    \begin{align}
        \EE_{\pi\sim\cP, (x,X), (z,Z)\sim \mu^\pi \times \mu^\pi} & \bigg[\prod_{h \in \cS} \ind(x_{-{i_h}} = z_{-{j_h}}) \cdot \biggl(\sum_{k\in[d]}\frac{\ind(x = z = e_k)}{\mu^\pi(z = e_k)} - 1 \biggr) \bigg] \\ 
        & \le \frac{1}{2}\left( \tilde I_{\chi^2}(\cS_1) + \tilde I_{\chi^2}(\cS_2) \right)
        \le \tilde I_{\chi^2}(\cS^\star), 
    \end{align}
    where $\tilde I_{\chi^2}(\cS)$ is the modified $\chi^2$-mutual information defined in \Cref{def:modified_chi_square_mi} and $\cS^\star = \argmax_{\cS\in\HleqD} \tilde I_{\chi^2}(\cS)$.
\end{lemma}
\begin{proof}[Proof of \cref{lem:cross_mutual}]
    We first note that it is allowed $|\cS_1|\neq|\cS_2|$ as there could be duplicate values in both $ ( i_h)  _{h\in\cS}$ and $ ( j_h) _{h\in\cS}$, while $\cS_1$ and $\cS_2$ are the unique values.
    In the sequel, we let    $X_{-\cS_1} $ denote $\{ x_{-i_h}\} _{h\in \cS} $ and let $Z_{-\cS_2} $ denote $\{ z_{-j_h}\} _{h\in \cS } $, where repeated elements are removed. 
    Moreover, we let 
    $\{ X_{-\cS_1} = Z_{-\cS_2}\} $ be 
 the event that $x_{-i_h} = z_{-j_h}$ for all $h \in \cS$. 
 Notice that 
 $\prod_{h \in \cS} \ind(x_{-{i_h}} = z_{-{j_h}}) = \ind (X_{-\cS_1}= Z_{-\cS_2}) $.
 Then, we have
    \begin{align}
        &\EE_{\pi, (x,X), (z,Z)\sim \mu^\pi \times \mu^\pi} \biggl[\prod_{h \in \cS} \ind(x_{-{i_h}} = z_{-{j_h}}) \cdot  \biggl(\sum_{k\in[d]}\frac{\ind(x = z = e_k)}{\mu^\pi(z = e_k)} - 1 \biggr) \biggr] \label{eq:cross_mutual-0}\\
        & \quad = \EE_{\pi, (x,X), (z,Z)\sim \mu^\pi \times \mu^\pi}
        \biggl[ \ind(X_{-\cS_1}= Z_{-\cS_2}) \cdot \biggl(\sum_{k\in[d]}\frac{\ind(x = z = e_k)}{\mu^\pi(z = e_k)} - 1 \biggr) \biggr] \\
        & \quad =  \EE_{\pi, (x,X), (z,Z)\sim \mu^\pi \times \mu^\pi} \biggl[  \biggl(\sum_{k\in[d]}\frac{\mu^\pi(x = e_k | X_{-\cS_1})\cdot \mu^\pi(z = e_k | Z_{-\cS_2})}{\mu^\pi(z = e_k)} - 1 \biggr)\cdot \ind (X_{-\cS_1} = Z_{-\cS_2})  \biggr] \\
        &\quad =\EE_{\pi, (X, Z)\sim \mu^\pi \times \mu^\pi} \biggl[ \sum_{k\in[d]} \biggl(\frac{\mu^\pi(x = e_k | X_{-\cS_1} )}{\mu^\pi(x = e_k)} - 1 \biggr)\cdot  
        \biggl( \frac{\mu^\pi(z = e_k | Z_{-\cS_2})}{\mu^\pi(z = e_k)} - 1 \biggr)\cdot  \mu^\pi(z= e_k) \cdot  \ind (X_{-\cS_1} = Z_{-\cS_2}) \biggr].
    \end{align}
    Here in the second equality, we take a conditional expectation given $X_{-\cS_1}$ and $Z_{-\cS_2}$. 
    The last equality can be verified by direct computation. 
To simplify the expectation above, we aim to transform the indicator of $\ind (X_{-\cS_1} = Z_{-\cS_2})$ into probabilities involving $X_{-\cS_1}$ and $Z_{-\cS_2}$.
To this end, we need to explicitly enumerate all possible values that $X_{-\cS_1}$ and $Z_{-\cS_2}$ can take. 
This is challenging, as there may be duplicated values in both $i_h$ and $j_h$, and thus $X_{-\cS_1}$ and  $Z_{-\cS_2}$ can have different sizes. 
However, since $\cS_1$ is a ``reduction'' of $\{ i_h\} _{h \in \cS}$, we can revert to the original space and consider $E = (E_{h})_{h\in \cS}  \in \cX^{|\cS|}$ that \emph{respects} the reduction from $\{i_h\}_{h \in \cS}$ to $\cS_1$. 
Here each $E_{h}$ is the value $x_{i_h}$ takes. 
In other words, with $ (i_h)_{h \in \cS}$ that might have duplicated values, we consider the values taken by $(  x_{i_h})_{h\in \cS} $, with duplicates allowed. And $E$ has the same duplication structure as $(  x_{i_h})_{h\in \cS} $. 
 In the following, we describe these values by introducing the  notion of compatibility.

\begin{definition}[Compatible Value Set]
    We say that $E \in \cX^{|\cS|}$ is \emph{compatible} with $( i_h)  _{h \in \cS}$ if, for any $h \neq h'$ such that $i_h = i_{h'}$, we have $E_h = E_{h'}$. 
    In other words, the \emph{unique} values in $E$ can be indexed by $\{i_h\}_{h\in\cS} = \cS_1$ if $E$ is compatible with $(i_h)_{h\in\cS}$.
\end{definition}
By this definition,   $E$ is compatible with $( i_h) _{h \in \cS}$ if it respect duplication pattern of $(x_{i_h})_{h\in \cS}$. 
If $i_h = i_{h'}$, then we know that $x_{i_h}$ and $x_{i_{h'}}$ is the same token. Since $x_{i_h}$ and $x_{i_{h'}}$ take values  $E_{h} $ and $E_{h'} $, we must have $E_{h}  = E_{h'}$. 
As a concrete example, 
suppose $\cS = \{1, 2, 3\}$, 
and the values of $(i_h)_{h \in \cS}$ are given by  $(i_1, i_2, i_3) = (1, 2, 1)$. 
Therefore, we have $\cS_1=\{1, 2\}$, which contains  the unique values of $(i_1, i_2, i_3)$.
Now, let $E = (E_1, E_2, E_3)$. For $E$ to be \emph{compatible} with $(i_h)_{h \in \cS}$, we must have $E_1 = E_3$ since $i_1 = i_3$. 
There is no restriction on $E_2$. 
So, a compatible value set for this example could be $E = (a, b, a)$, where $a$ and $b$ are elements of $\cX$.

    In the sequel, we define $\cE$ as the set of vectors in $\cX^{|\cS|}$ that are compatible with both $\{i_h\}_{h\in\cS}$ and $\{j_h\}_{h\in\cS}$, i.e.,
    \begin{align}
        \cE = \cbr{E \in \cX^{|\cS|} \ | \ E \text{ is compatible with both } (i_h)_{h\in\cS} \text{ and }  (j_h)_{h\in\cS}}.
    \end{align}
    The compatibility condition allows us to assign $x_{-i_h}, z_{-j_h}$ the value $E_h$ for all $h\in\cS$ when  $E\in\cE$. 
    Under this assignment, the constraint $X_{-\cS_1} = Z_{-\cS_2}$ is automatically satisfied.
We use the notation $\{ X_{-\cS_1} = E\} $ to denote the event that $x_{-i_h} = E_h$ for all $h\in\cS$, and similarly for $Z_{-\cS_2} = E$.
    In particular, we are able to rewrite the indicator $\ind(X_{-\cS_1} = Z_{-\cS_2})$ as $\sum_{E\in\cE}\ind(X_{-\cS_1} = E, Z_{-\cS_2} = E)$. 
  Then we  can rewrite \eqref{eq:cross_mutual-0} by separating $X_{-\cS_1}$ and $Z_{-\cS_2}$ as 
    \begin{align}
        &\EE_{\pi, (x,X), (z,Z)\sim \mu^\pi \times \mu^\pi} \biggl[\prod_{h \in \cS} \ind(x_{-{i_h}} = z_{-{j_h}}) \cdot  \biggl(\sum_{k\in[d]}\frac{\ind(x = z = e_k)}{\mu^\pi(z = e_k)} - 1 \biggr) \biggr]\\
        &\quad =  \EE_{\pi} \biggl[ \sum_{E\in \cE}\sum_{k\in[d]} \biggl(\frac{\mu^\pi(x = e_k | X_{-\cS_1} =E)}{\mu^\pi(x = e_k)} - 1 \biggr) 
        \cdot \biggl( \frac{\mu^\pi(z = e_k | Z_{-\cS_2}=E)}{\mu^\pi(z = e_k)} - 1 \biggr) \\
        &\hspace{6.5cm} \cdot \mu^\pi(X_{-\cS_1} = E) \cdot \mu^\pi(Z_{-\cS_2}=E)\cdot \mu^\pi(z= e_k)  \biggr] \\
        &\quad \le \frac{1}{2}\EE_{\pi } \biggl[\sum_{E\in\cE}  \sum_{k\in[d]} \biggl(\frac{\mu^\pi(x = e_k | X_{-\cS_1} )}{\mu^\pi(x = e_k)} - 1 \biggr)^2 
        \cdot   \mu^\pi(z= e_k) \cdot \bigl( \mu^\pi(X_{-\cS_1}=E) \big) ^2 \biggr] \\
        & \qqquad + \frac{1}{2} \EE_{\pi} \biggl[\sum_{E\in\cE}  \sum_{k\in[d]} \biggl( \frac{\mu^\pi(z = e_k | Z_{-\cS})}{\mu^\pi(z = e_k)} - 1 \biggr)^2 \cdot  \mu^\pi(z= e_k)  \cdot \bigl( \mu^\pi(Z_{-\cS_2} = E) \bigr ) ^2 \biggr]. 
        \label{eq:lem:cross_mutual-1}
    \end{align}
    where in the last inequality, we apply $ab \leq a^2 + b^2/2$.
    
    Next, for each $E \in\cE$, consider $E'= (E'_{i})_{i\in\cS_1}$ such that
    \begin{align}
        E_{i_h}' = E_{h}, \quad \forall h\in\cS. \label{eq:cross-mutual-marginal}
    \end{align}
    Note that for each $E\in\cE$, $E'$ must exist and is unique. The existence follows from the compatibility definition, which allows us to index all the unique values in $E$ by restricting the indices to the set $\cS_1$.
    The uniqueness is due to the fact that \eqref{eq:cross-mutual-marginal} completely determines all the values in $E'$ because  enumerating over $i_h$ for $h\in\cS$ is just the same as enumerating over $i$ for $i\in\cS_1$.
    In fact, $E'$ contains all the unique values of $E$. 
In the above example, we have $\cS_1 = \{1,2\}$ and thus $E' = ( a, b)$ when $E = (a, b, a)$.

Since $E'$ is uniquely defined based on $E$,  we are able to define an operator $\cJ_1$ that maps $E\in\cE$ to $E' \in \cX^{|\cS_1|}$ according to the mapping given  in  \eqref{eq:cross-mutual-marginal}.
    Let   $\cJ_1(\cE)$ be  the image of $\cE$ under $\cJ_1$.
    It is important to note that for each $E'\in\cJ_1(\cE)$, there is also a unique pre-image $E\in\cE$ such that $\cJ_1(E) = E'$ according to the rule \eqref{eq:cross-mutual-marginal}.
    \textbf{\emph{Therefore, $\cJ_1$ is an one-to-one mapping from $\cE$ to $\cJ_1(\cE)$.}}
    In the following, for any  $E'\in\cJ_1(\cE)$, we denote by $\{ X_{-\cS_1} = E'\}$ the event where $x_{-i} = E_{i}'$ for all $i\in\cS_1$.
    Equivalently, we have $x_{-i_h} = E_{i_h}' = E_h $ for all $h \in \cS$. Thus,  the event $\{ X_{-\cS_1} = E'\}$ is exactly the same as $\{ X_{-\cS_1} = E\}$ introduced above. 
    Therefore, the first term on the right hand side of \eqref{eq:lem:cross_mutual-1} can be reformulated as 
    \begin{align}
        &\frac{1}{2} \cdot \EE_{\pi } \biggl[\sum_{E\in\cE}  \sum_{k\in[d]} \biggl(\frac{\mu^\pi(x = e_k | X_{-\cS_1} )}{\mu^\pi(x = e_k)} - 1 \biggr)^2 
        \cdot   \mu^\pi(z= e_k) \cdot \mu^\pi(X_{-\cS_1}=E)^2 \biggr] \\
        &\quad = \frac{1}{2} \cdot \EE_{\pi } \biggl[\sum_{E'\in \cJ_1(\cE)}  \sum_{k\in[d]} \biggl(\frac{\mu^\pi(x = e_k | X_{-\cS_1} )}{\mu^\pi(x = e_k)} - 1 \biggr)^2 
        \cdot   \mu^\pi(z= e_k) \cdot \bigl(\mu^\pi(X_{-\cS_1}=E') \bigr) ^2 \biggr]\\
        &\quad \le \frac{1}{2} \cdot \EE_{\pi } \biggl[\sum_{E'\in \cX^{|\cS_1|}}  \sum_{k\in[d]} \biggl(\frac{\mu^\pi(x = e_k | X_{-\cS_1} )}{\mu^\pi(x = e_k)} - 1 \biggr)^2 
        \cdot   \mu^\pi(z= e_k) \cdot \bigl(\mu^\pi(X_{-\cS_1}=E') \big) ^2 \biggr] = \frac{1}{2} \tilde I_{\chi^2}(\cS_1), 
    \end{align}
    where the equality follows from the 
    bijection between $\cE$ and $\cJ_1(\cE)$, and the last inequality holds by noting that $\cJ_1(\cE) \subseteq \cX^{|\cS_1|}$.
    The last equality follows from the definition of the modified mutual information.
    The argument for the second term on the right hand side of \eqref{eq:lem:cross_mutual-1} is similar, and we hence conclude that
    \begin{align}
        &\EE_{\pi, (x,X), (z,Z)\sim \mu^\pi \times \mu^\pi} \biggl[\prod_{h \in \cS} \ind(x_{-{i_h}} = z_{-{j_h}}) \cdot  \biggl(\sum_{k\in[d]}\frac{\ind(x = z = e_k)}{\mu^\pi(z = e_k)} - 1 \biggr) \biggr] \le \frac{1}{2}\tilde I_{\chi^2}(\cS_1) + \frac{1}{2}\tilde I_{\chi^2}(\cS_2). 
    \end{align}
    Lastly, note that $\tilde I_{\chi^2}(\cS) \le \tilde I_{\chi^2}(\cS^\star)$ for any $\cS\in\HleqD$ by the optimality of $\cS^\star $. 
    Hence, we complete the proof of \cref{lem:cross_mutual}.
\end{proof}

\cref{lem:approximation_error-3-1} quantifies the approximation error from $\sigma_l \approx \sigma_l^\star$ and $y(k) \approx y^\star(k)$ for Stage \RNum{3}.
    \begin{lemma}
        \label{lem:approximation_error-3-1}
        For the transformer model defined in \eqref{eq:transformer}, define two quantities $f_1$ and $f_2$ as
        \begin{align}
            f_1 &\defeq \EE\left[\sum_{l=M+1}^{L} \sigma_l \sum_{k=1}^d \biggl(\frac{\ind(x_{L+1} = x_l = e_k)}{y(k) + \varepsilon} - \frac{y(k)\ind(x_{L+1} = e_k)}{y(k)+\varepsilon} \biggr)  \cdot \prod_{h\in \cS^\star} \langle v_l^\head{h}, v_{L+1}^\head{h} \rangle \right], \\
            f_2 &\defeq \EE\left[\sum_{l=M+1}^{L} \sigma_l^\star \sum_{k=1}^d \biggl(\frac{\ind(x_{L+1} = x_l = e_k)}{y^\star(k) + \varepsilon} - \frac{y^\star(k)\ind(x_{L+1} = e_k)}{y^\star(k)+\varepsilon} \biggr) \cdot \prod_{h\in \cS^\star} \ind(x_{l-h}=x_{L+1-h}) \right], 
        \end{align}
        where the expectation is taken over all the randomness in the data, and 
        \begin{align}
            \sigma_l^\star := \frac{\exp\left(a \cdot \prod_{h\in\cS^\star}\ind(x_{l -  h} = x_{L+1-h})\right) }{\sum_{l'=1}^L \exp\left(a \cdot \prod_{h\in\cS^\star}\ind(x_{l' -  h} = x_{L+1-h})\right)}, \quad y^\star (k) \defeq \sum_{l=M+1}^{L} \sigma_l^\star \ind(x_l = e_k), 
        \end{align}
        with $\cS^\star$ is the optimal information set.
        Under \Cref{asp:Markov_chain}, it holds that
        \begin{align}
            |f_1 - f_2| \le 12 \cdot (1  + a \varepsilon^{-1}) \cdot (\Delta_1 + \Delta_2), 
        \end{align}
        where $\Delta_1 \defeq 1 - p_{\cS^\star}$ and $\Delta_2 \defeq 1 - \prod_{h\in \cS^\star} (\sigma_{-h}^\h)^2$.
    \end{lemma}
    \begin{proof}
    We separate the approximation error into three parts
    \(
        |f_1 - f_2| \le \err_1 + \err_2 + \err_3, 
    \)
    which are explained in detail as follows.

    \paragraph{The First Error Term}
    Here, the first error $\err_1$ captures the error of replacing $ \prod_{h\in \cS^\star} \ind(x_{l-h}=x_{L+1-h})$ with $\prod_{h\in \cS^\star} \langle v_l^\head{h}, v_{L+1}^\head{h} \rangle$ in $f_2$:
    \begin{align}
        \err_1 &\defeq \bigg|\EE\bigg[\sum_{l=M+1}^{L} \sigma_l \!\cdot\! \sum_{k\in[d]}\Bigl(\frac{\ind(x_{L+1} = x_l = e_k)}{y^\star(k) + \varepsilon} - \frac{y^\star(k)\ind(x_{L+1} = e_k)}{y^\star(k)+\varepsilon} \Bigr)  . \\
        &\hspace{3cm}  \cdot \Bigl(\prod_{h\in \cS^\star} \langle v_l^\head{h}, v_{L+1}^\head{h} \rangle - \prod_{h\in \cS^\star} \ind(x_{l-h}=x_{L+1-h})\Bigr )\bigg]\bigg|,
    \end{align}
    Using \cref{lem:boundedness}, we have 
    \begin{align}
        \biggl|\sum_{l=M+1}^{L} \sigma_l \!\cdot\! \sum_{k\in[d]}\Bigl(\frac{\ind(x_{L+1} = x_l = e_k)}{y^\star(k) + \varepsilon} - \frac{y^\star(k)\ind(x_{L+1} = e_k)}{y^\star(k)+\varepsilon} \Bigr)\biggr| \le 2.
    \end{align}
    Then using \cref{lem:misspecification}, we conclude that 
    \begin{align}
        \err_1 \le 2 \sup_{l\in [L]} \biggl|\prod_{h\in \cS^\star} \langle v_l^\head{h}, v_{L+1}^\head{h} \rangle - \prod_{h\in \cS^\star} \ind(x_{l-h}=x_{L+1-h})\biggr| \le 2(\Delta_1 + \Delta_2).
    \end{align}
    
    \paragraph{The Second Error Term}
    The second error term characterizes the difference in $\sigma_l$ and $\sigma_l^\star$:
    \begin{align}
        \err_2 &\defeq \bigg|\EE\bigg[\sum_{l=M+1}^{L} (\sigma_l^\star - \sigma_l) \!\cdot\! \sum_{k\in[d]}\biggl(\frac{\ind(x_{L+1} = x_l = e_k)}{y^\star(k) + \varepsilon} - \frac{y^\star(k)\ind(x_{L+1} = e_k)}{y^\star(k)+\varepsilon} \biggr) \! \cdot \! \prod_{h\in \cS^\star} \langle v_l^\head{h}, v_{L+1}^\head{h} \rangle \bigg]\bigg|. 
    \end{align}
    To characterize such an error, 
    we invoke equation (53) in Lemma 5.1 of \citet{chen2022adaptive}. This lemma states that for $\sigma$ and $\sigma^\star$ being the output of the softmax function with scaling parameters $a$ for $s$ and $s^\star$  respectively, i.e., 
    $$\sigma = \frac{\exp(a s)}{\sum_{l=M+1}^{L} \exp(a s_l)}, \qquad \text{and} \qquad \sigma^\star = \frac{\exp(a s^\star)}{\sum_{l=M+1}^{L} \exp(a s_l^\star)}, $$
    it holds that
    \(
        \left\| \sigma - \sigma^\star\right\|_1 \le 4 a \cdot \norm{s - s^\star}_\infty. 
    \)
    Consequently, we have $\norm{\sigma - \sigma^\star}_{1} \le 4 a \cdot (\Delta_1 + \Delta_2)$ by \cref{lem:misspecification}.
    We notice that
    \begin{align}
        \bigg|\sum_{k\in[d]}\biggl(\frac{\ind(x_{L+1} = x_l = e_k)}{y^\star(k) + \varepsilon} - \frac{y^\star(k)\ind(x_{L+1} = e_k)}{y^\star(k)+\varepsilon} \biggr) \! \cdot \! \prod_{h\in \cS^\star} \langle v_l^\head{h}, v_{L+1}^\head{h} \rangle \bigg| \le \max\{\varepsilon^{-1}, 1\} = \varepsilon^{-1}.
    \end{align}
    Thus, $\err_2 \le   4 a\varepsilon^{-1} \cdot (\Delta_1 + \Delta_2) $.
    
    \paragraph{The Third Error Term}
    The last error term characterizes the difference between $y^\star$ and $y$:
    \begin{align}
        \err_3 &\defeq \bigg|\EE\bigg[\sum_{l=M+1}^{L} \sigma_l \!\cdot\! \sum_{k\in[d]}\Big(\frac{1}{y^\star(k) + \varepsilon} - \frac{1}{y(k) + \varepsilon} \Big) \cdot \ind(x_{L+1} = x_l = e_k) \! \cdot \! \prod_{h\in \cS^\star} \langle v_l^\head{h}, v_{L+1}^\head{h} \rangle \bigg]\bigg| \\
        &
        \qquad + \bigg|\EE\bigg[\sum_{l=M+1}^{L} \sigma_l \!\cdot\! \sum_{k\in[d]}\Big(\frac{y^\star(k)}{y^\star(k)+\varepsilon} - \frac{y(k)}{y(k)+\varepsilon} \Big) \cdot \ind(x_{L+1} = e_k) \! \cdot \! \prod_{h\in \cS^\star} \langle v_l^\head{h}, v_{L+1}^\head{h} \rangle \bigg]\bigg|.
    \end{align}
    By noting that $y^\star = \sum_{l=M+1}^{L} \sigma_l^\star x_l$ and $y = \sum_{l=M+1}^{L} \sigma_l x_l$, we have
    \begin{align}
        \left\| y^\star - y\right\|_1 
        &= \bigg\| \sum_{l=M+1}^{L} (\sigma_l^\star - \sigma_l) x_l\bigg\|_1 \le \sum_{l=M+1}^{L} \left| \sigma_l^\star - \sigma_l\right|_1 \cdot \left\| x_l\right\|_1 
        \le \norm{\sigma - \sigma^\star}_1 \le 4 a \cdot (\Delta_1 + \Delta_2).
    \end{align}
    The first term of $\err_3$ can be bounded by 
    \begin{align}
        &\bigg|\EE\bigg[\sum_{l=M+1}^{L} \sigma_l \!\cdot\! \sum_{k\in[d]}\Big(\frac{1}{y^\star(k) + \varepsilon} - \frac{1}{y(k) + \varepsilon} \Big) \cdot \ind(x_{L+1} = x_l = e_k) \! \cdot \! \prod_{h\in \cS^\star} \langle v_l^\head{h}, v_{L+1}^\head{h} \rangle \bigg]\bigg|\\
        &\quad \le \sum_{k\in[d]}\frac{|y(k) - y^\star(k)|}{(y^\star(k)+\varepsilon)(y(k)+\varepsilon)} \cdot y(k) \ind(x_{L+1}=e_k) 
        \le \norm{y - y^\star}_1 \cdot \varepsilon^{-1} \le 4 a \varepsilon^{-1} (\Delta_1 + \Delta_2). 
    \end{align}
    Moreover, for the second term of $\err_3$, we have 
    \begin{align}
        &\bigg|\EE\bigg[\sum_{l=M+1}^{L} \sigma_l \!\cdot\! \sum_{k\in[d]}\biggl(\frac{y^\star(k)}{y^\star(k)+\varepsilon} - \frac{y(k)}{y(k)+\varepsilon} \biggr) \cdot \ind(x_{L+1} = e_k) \! \cdot \! \prod_{h\in \cS^\star} \langle v_l^\head{h}, v_{L+1}^\head{h} \rangle \bigg]\bigg| \\
        &\quad \le \sum_{k\in[d]}\frac{|y(k) - y^\star(k)|}{(y^\star(k)+\varepsilon)(y(k)+\varepsilon)} \cdot \varepsilon \cdot\ind(x_{L+1}=e_k)  \le 4 a \varepsilon^{-1} (\Delta_1 + \Delta_2). 
    \end{align}
    It then holds that 
    \begin{align}
        |f_1 - f_2| 
        &\le \err_1 + \err_2 + \err_3 \le 2(\Delta_1 + \Delta_2) + 4 a \varepsilon^{-1} (\Delta_1 + \Delta_2) + 8 a \varepsilon^{-1} (\Delta_1 + \Delta_2) \\
        &= 12 \cdot (1  + a \varepsilon^{-1}) \cdot (\Delta_1 + \Delta_2). 
    \end{align}
    Therefore, we complete the proof of \cref{lem:approximation_error-3-1}.
    \end{proof}

    \begin{lemma}\label{lem:approximation5}
        Let us define for brevity, 
        \begin{align}
            \tilde \mu_X^\pi(z, Z) = \tilde \mu^\pi(z, Z\given X_{L+1-\cS^\star}) = \frac{\mu^\pi(z, Z) \exp\left(
                a \cdot \prod_{h\in\cS^\star}\ind(z_{-  h} = x_{L+1-h})
            \right)}{\sum_{z', Z'} \mu^\pi(z', Z') \exp\left(
                a \cdot \prod_{h\in\cS^\star}\ind(z_{-h}' = x_{L+1-h})
            \right)}, 
        \end{align}
        where $Z = (z_{-M}, \dots, z_{-1})$ and $\mu^\pi$ is the stationary distribution of the Markov chain over a window of size $M+1$. 
        We denote by $\tilde \mu_X^\pi(e_k) = \tilde\mu_X^\pi(z=e_k)$ where $\tilde\mu_X^\pi(z)$ is the marginal distribution for $z$ and serves as the population counterpart for $y^\star = \sum_{l=M+1}^{L} \sigma_l^\star x_l$.
        We define quantity $A$ and $B$ as
        \begin{align}
            A &\defeq \EE\bigg[\sum_{l=M+1}^{L} \sigma_l^\star \sum_{k=1}^d \biggl(\frac{\ind(x_{L+1} = x_l = e_k)}{y^\star(k) + \varepsilon} - \frac{y^\star(k)\ind(x_{L+1} = e_k)}{y^\star(k)+\varepsilon} \biggr) \prod_{h\in \cS^\star} \ind(x_{l-h}=x_{L+1-h}) \bigg].\\
            B &\defeq \EE\bigg[\sum_{l=M+1}^{L} \sigma_l^\star \sum_{k=1}^d \biggl(\frac{\ind(x_{L+1} = x_l = e_k)}{\tilde \mu_X^\pi(e_k)} - {\ind(x_{L+1} = e_k)} \biggr) \prod_{h\in \cS^\star} \ind(x_{l-h}=x_{L+1-h})  \bigg].
        \end{align}
        Under \Cref{asp:Markov_chain}, we have
        \begin{align}
            |A - B| \le  \frac{8 (1-\lambda)^{-1/2} (D_{\chi^2}(\mu_0\,\|\, \mu^\pi) + 1)^{1/4} + 8\sqrt{M}}{L^{1/2}\cdot \gamma^{|\cS^\star|+1}} + \frac{2d\varepsilon}{\gamma}.
        \end{align} 
    \end{lemma}
    \begin{proof}[Proof of \cref{lem:approximation5}]
        The proof follows the same arguments as \cref{lem:approximation2}.
        We remind the readers that $y^\star(k)$ is also a function of the whole chain $X$.
        We note that
        \begin{align}
            |A - B|  &= \biggl|\EE \biggl[ \sum_{l=M+1}^{L} \sigma_l^\star  \cdot \biggl(\sum_{k\in[d]}  \frac{\ind(x_{L+1} = x_l = e_k)}{y^\star(k) + \varepsilon}  - \sum_{k\in[d]}\frac{\ind(x_{L+1} = x_l = e_k)}{\tilde\mu_X^\pi(e_k)}\\
            & \hspace{5cm}- \sum_{k\in[d]}\frac{y^\star(k)\ind(x_{L+1} = e_k)}{y^\star(k)+\varepsilon} + 1 \biggr) \cdot \prod_{h\in \cS^\star} \ind(x_{l-h}=x_{L+1-h})  \biggr] \biggr| \\
            &= \biggl| \EE \biggl[ \sum_{l=M+1}^{L} \sigma^\star_l  \cdot 
     \biggl(\sum_{k\in[d]} \Bigl(  \frac{\tilde\mu_X^\pi(e_k) -y^\star(k) }{(y^\star(k)+\varepsilon) \cdot 
     \tilde\mu_X^\pi(e_k)}  - \frac{\varepsilon}{(y^\star(k)+\varepsilon) \cdot 
     \tilde\mu_X^\pi(e_k)} \Bigr)\cdot \ind(x_{L+1} = x_l = e_k) \\
            &\hspace{5cm} - \sum_{k\in[d]}\frac{\varepsilon\ind(x_{L+1} = e_k) }{y^\star(k)+\varepsilon}\biggr) \cdot \prod_{h\in \cS^\star} \ind(x_{l-h}=x_{L+1-h})  \biggr] \biggr|. 
        \end{align}
        To handle this error, we define three error terms as 
        \begin{align}
            \err_1 &\defeq \bigg|\EE\bigg[\sum_{k\in[d]}  \frac{\tilde\mu_X^\pi(e_k) -y^\star(k) }{(y^\star(k)+\varepsilon) \cdot 
     \tilde\mu_X^\pi(e_k)}  \cdot \sum_{l=M+1}^{L} \sigma_l^\star  \cdot 
     \ind(x_{L+1} = x_l = e_k) \cdot \prod_{h\in \cS^\star} \ind(x_{l-h}=x_{L+1-h})  \bigg] \bigg|, \\
            \err_2 &\defeq \bigg|\EE\bigg[\sum_{k\in[d]}  \frac{\varepsilon}{(y^\star(k)+\varepsilon) \cdot 
     \tilde\mu_X^\pi(e_k)}  \cdot \sum_{l=M+1}^{L} \sigma_l^\star  \cdot 
     \ind(x_{L+1} = x_l = e_k) \cdot \prod_{h\in \cS^\star} \ind(x_{l-h}=x_{L+1-h})  \bigg]\bigg|, \\
            \err_3 &\defeq \bigg|\EE\bigg[\sum_{k\in[d]}  \frac{\varepsilon}{y^\star(k)+\varepsilon}  \cdot \ind(x_{L+1} = e_k) \cdot \sum_{l=M+1}^{L} \sigma_l^\star \cdot  \prod_{h\in \cS^\star} \ind(x_{l-h}=x_{L+1-h})  \bigg]\bigg|.
        \end{align}

        For the first error term, we have that
        \begin{align}
            \err_1 
            &\le \EE \bigg[
                \sum_{k\in[d]} \frac{|\tilde\mu_X^\pi(e_k) -y^\star(k)|}{(y^\star(k)+\varepsilon)} \cdot 
                \sum_{l=M+1}^{L} \frac{ \sigma_l^\star \ind(x_l = e_k)}{\tilde\mu_X^\pi(e_k)}
            \bigg] \\
            &= \EE \bigg[
                \sum_{k\in[d]} \frac{|\tilde\mu_X^\pi(e_k) -y^\star(k)|}{(y^\star(k)+\varepsilon)} \cdot 
                \frac{y^\star(k)}{\tilde\mu_X^\pi(e_k)}
            \bigg] \le \gamma^{-1} \cdot 
            \EE\bigg[
                \sum_{k\in[d]} |\tilde\mu_X^\pi(e_k) -y^\star(k)|
            \bigg], 
        \end{align}
        where we recall that by assumption, $\gamma $ provides a lower bound for $\pi(\cdot\given X_\pa)$, hence also a lower bound for $\tilde\mu_X^\pi(e_k)$.
        Next, we invoke \cref{prop:weighted_approximation} which provides
         an upper bound for the difference between the empirical and population distributions in terms of the $\ell_1$-norm:
        \begin{align}
            \EE\bigg[
            \bigg\|\tilde\mu_X^\pi(z=\cdot) -y^\star(\cdot)\bigg\|_{1}
            \bigg] 
            &\le \frac{4 \big ( (1-\lambda)^{-1}\sqrt{D_{\chi^2}(\mu_0\,\|\, \mu^\pi) + 1} + 4 M \big)^{1/2} } {L^{1/2}\cdot \min_{\pi, x_{L+1}, X_{L+1 -\cS^\star}} \mu^\pi(x_{L+1}, X_{L+1 -\cS^\star})} \\
            &\le \frac{4  (1-\lambda)^{-1/2} (D_{\chi^2}(\mu_0\,\|\, \mu^\pi) + 1)^{1/4} + 8\sqrt{M}}{L^{1/2}\cdot \gamma^{|\cS^\star|+1}}.
            \label{eq:approximation5-1}
        \end{align}
        Hence, we control the first error term. 
    
        For the second error term, we follow the same procedure and obtain an upper bound as
        \begin{align}
            \err_2 &\le \EE\bigg[
                \sum_{k\in[d]} \frac{\varepsilon}{\tilde\mu_X^\pi(e_k)} \cdot  \sum_{l=M+1}^{L} \frac{ \sigma_l^\star \ind(x_l = e_k)}{(y^\star(k) + \varepsilon)}
            \bigg] \le \EE\bigg[
                \sum_{k\in[d]} \frac{\varepsilon}{\tilde\mu_X^\pi(e_k)} 
            \bigg] 
            \le \gamma^{-1} d \varepsilon.
        \end{align}
    
        For the last error term, it holds that 
        \begin{align}
            \err_3 & \le  \EE\bigg[
                \sum_{k\in[d]}  \frac{\varepsilon}{y^\star(k)+\varepsilon}  \cdot \ind(x_{L+1} = e_k) 
            \bigg]\\
            & \le \bigg| \EE\bigg[
                \sum_{k\in[d]} \frac{\varepsilon \ind(x_{L+1} = e_k)}{\tilde\mu_X^\pi(e_k) + \varepsilon}
            \bigg] \bigg|
            + \bigg| \sum_{k\in[d]} \EE\bigg[
                \frac{\varepsilon(y^\star(k) - \tilde\mu_X^\pi(e_k)) \cdot \ind(x_{L+1}=e_k)}{(\tilde\mu_X^\pi(e_k)+\varepsilon) (y^\star(k) + \varepsilon)}
            \bigg] \bigg| \\
            & \le \frac{\varepsilon}{\gamma} + 
                \EE\bigg[
                    \sum_{k\in[d]} \frac{|y^\star(k) - \tilde\mu_X^\pi(e_k)| }{\gamma}
                \bigg] \le \frac{\varepsilon}{\gamma} + \frac{4  (1-\lambda)^{-1/2} (D_{\chi^2}(\mu_0\,\|\, \mu^\pi) + 1)^{1/4} + 8\sqrt{M}}{L^{1/2}\cdot \gamma^{|\cS^\star|+1}}.
        \end{align}
        where the last inequality follows directly from \eqref{eq:approximation5-1}.
    
        In summary, the difference between $f_2$ and $f_3$ is bounded by 
        \begin{align}
            |f_2 -f_3| &\le \err_1 + \err_2 + \err_3 \le   \frac{8 (1-\lambda)^{-1/2} (D_{\chi^2}(\mu_0\,\|\, \mu^\pi) + 1)^{1/4} + 8\sqrt{M}}{L^{1/2}\cdot \gamma^{|\cS^\star|+1}} + \frac{2d\varepsilon}{\gamma}, 
        \end{align}
        which completes our proof of \cref{lem:approximation5}.
    \end{proof}

The following lemmas are for analyzing the error $|f_3-f_4|$ for Stage \RNum{3}.
\begin{lemma}
    \label{lem:sigma_approx}
    We define 
    \begin{align}
        A &\defeq \EE\bigg[\sum_{l=M+1}^{L} \sigma_l^\star \cdot  \sum_{k=1}^d \biggl(\frac{\ind(x_{L+1} = x_l = e_k)}{\tilde \mu_X^\pi(e_k)} - {\ind(x_{L+1} = e_k)} \biggr)\cdot  \prod_{h\in \cS^\star} \ind(x_{l-h}=x_{L+1-h})  \bigg], \\
        B &\defeq \EE_{X, (z, Z)\sim \tilde \mu_X^\pi}\bigg[\sum_{k=1}^d \biggl(\frac{\ind(x_{L+1} = z = e_k)}{\tilde \mu_X^\pi(e_k)} - {\ind(x_{L+1} = e_k)}  \biggr) \cdot 
 \prod_{h\in \cS^\star} \ind(z_{l-h}=x_{L+1-h})  \bigg], 
    \end{align}
    where
    \begin{align}
        \sigma_l^\star &\defeq \frac{\exp\left(a \cdot \prod_{h\in\cS^\star}\ind(x_{l -  h} = x_{L+1-h})\right) }{\sum_{l'=1}^L \exp\left(a \cdot \prod_{h\in\cS^\star}\ind(x_{l' -  h} = x_{L+1-h})\right)}, \\
        \tilde \mu_X^\pi(z, Z) &\defeq \tilde \mu^\pi(z, Z\given X_{L+1-\cS^\star}) =  \frac{\mu^\pi(z, Z) \exp\left(
            a \cdot \prod_{h\in\cS^\star}\ind(z_{-  h} = x_{L+1-h})
        \right)}{\sum_{z', Z'} \mu^\pi(z', Z') \exp\left(
            a \cdot \prod_{h\in\cS^\star}\ind(z_{-h}' = x_{L+1-h})
        \right)}.
    \end{align}
    Under \Cref{asp:Markov_chain}, we have
    \begin{align}
        |A - B| \le \frac{8\gamma^{-1} (1-\lambda)^{-1/2} (D_{\chi^2}(\mu_0\,\|\, \mu^\pi) + 1)^{1/4} + 16\gamma^{-1} \sqrt{M}}{L^{1/2}\cdot \gamma^{|\cS^\star|+1}}.
    \end{align}
\end{lemma}
\begin{proof}[Proof of \cref{lem:sigma_approx}]
    For $Z = (z_{-M}, \dots, z_{-1})$ and $Z' = (z_{-M}', \dots, z_{-1}')$, we let $Z_{-\cS^\star} = (z_{-h})_{h\in\cS^\star}$, we define
    \begin{gather}
        \hat \mu_X^\pi(z, Z) = \frac{1}{L-M} \sum_{l=M+1}^{L} \ind(x_l = z, X_{l-M:l-1}=Z),\\
        R(Z, X_{L+1-\cS^\star}) = \exp\bigg(
            a \cdot \prod_{h\in\cS^\star} \ind(z_{-h} = x_{L+1-h})
        \bigg). 
    \end{gather}
    Using these notations, we can rewrite the normalizing factor in $\tilde\mu_X^\pi$ and $\sigma_l^\star$ respectively as
    \begin{align}
        \Phi = \sum_{z, Z} \mu^\pi(z, Z) \cdot R(Z, X_{L+1-\cS^\star}), \quad \hat\Phi = \sum_{z, Z} \hat\mu_X^\pi(z, Z) \cdot R(Z, X_{L+1-\cS^\star}).
    \end{align}
    We also define 
    \begin{gather}
        \phi(z, Z_{-\cS^\star}) = \mu^\pi(z, Z_{-\cS^\star}) \cdot R(Z_{-\cS^\star}, X_{-\cS^\star}) , \quad \hat\phi(z, Z_{-\cS^\star}) = \hat\mu_X^\pi(z, Z_{-\cS^\star})\cdot  R(Z_{-\cS^\star}, X_{L+1-\cS^\star}).
    \end{gather}
    If we further define $\hat\nu_X^\pi(z, Z_{-\cS^\star}) = \sum_{l=M+1}^L \ind(x_l = z, X_{l-\cS^\star} = Z_{-\cS^\star})$, then we have
    \begin{gather}
        \hat\nu_X^\pi(z, Z_{-\cS^\star}) = \frac{ \hat\mu_X^\pi(z, Z_{-\cS^\star})\cdot  R(Z_{-\cS^\star}, X_{L+1-\cS^\star}) }{\hat\Phi} = \frac{\hat\phi(z, Z_{-\cS^\star})}{\hat\Phi},\\
        \tilde\mu_X^\pi(z, Z_{-\cS^\star}) = \frac{\mu^\pi(z, Z_{-\cS^\star}) \cdot R(Z_{-\cS^\star}, X_{L+1-\cS^\star}) }{\Phi} = \frac{\phi(z, Z_{-\cS^\star})}{\Phi}.
    \end{gather}
    Using the above definitions and relationship, $A$ and $B$ can be rewritten as 
    \begin{align}
        A = \EE \bigg[ \sum_{k=1}^d \frac{\hat\phi(e_k, X_{L+1-\cS^\star})}{\hat \Phi \cdot \tilde\mu_X^\pi(e_k)} - \frac{\hat\phi(X_{L+1-\cS^\star})}{\hat\Phi}\bigg], \quad B= \EE\bigg[
            \sum_{k=1}^d \frac{\phi(e_k, X_{L+1-\cS^\star})}{\Phi \cdot \tilde\mu_X^\pi(e_k)} - \frac{\phi(X_{L+1-\cS^\star})}{\Phi}
        \bigg]. 
    \end{align}
    Therefore, the difference between $A$ and $B$ is given by
    \begin{align}
        |A - B| 
        &\le \frac{2}{\gamma} \cdot \EE \bigg[
            \sum_{z, Z_{-\cS^\star}} \bigg|\frac{\phi(z, Z_{-\cS^\star})}{\Phi} - \frac{\hat\phi(z, Z_{-\cS^\star})}{\hat\Phi} \bigg| 
        \bigg] \le \frac{2}{\gamma} \cdot \EE \bigg[
            \sum_{z, Z_{-\cS^\star}} \bigg| \tilde\mu_X^\pi(z, Z_{-\cS^\star}) - \hat\nu_X^\pi(z, Z_{-\cS^\star}) \bigg| 
        \bigg] \\
        &\le \frac{8\gamma^{-1} \cdot \bigl(  (1-\lambda)^{-1}\sqrt{D_{\chi^2}(\mu_0\,\|\, \mu^\pi) + 1} + 4 M \bigr) ^{1/2} }{L^{1/2}\cdot \min_{x_{L+1}, X_{L+1 -\cS^\star}} \mu^\pi(x_{L+1}, X_{L+1 -\cS^\star})}. 
    \end{align}where the last inequality follows from the result in \cref{prop:weighted_approximation}.
    Invoking the lower bound $\mu^\pi(x_{L+1}, X_{L+1 -\cS^\star}) \ge \gamma^{|\cS^\star|+1}$, we complete the proof of \cref{lem:sigma_approx}.
\end{proof}

\subsection{Lemmas on Concentration of Markov Chain} \label{sec:MC_convergence}

Recall that we previously define $X= (x_{1}, \ldots, x_L)$ as the observed sequence and $x_{L+1}$ as the value at time $L+1$ to be predicted.
For generality, we will use $X=(x_1, \ldots, x_{L+1})$ to denote the whole sequence in the following proof. 
We denote by $p^\pi(\cdot)$ the joint distribution for the sequence $X$ with kernel $\pi$.
Recall that we have the parent set $\pa=\{-r_1, \ldots, -r_n\}$, and as the start of a chain, we sample the first $r_n$ tokens by $(x_1, \ldots, x_{r_n})\sim \mu_0$. 

In the sequel, we will study concentration properties of the Markov chain $X$ for a window of tokens with window size at most $M$, where $M > r_n$.
To proceed, let us consider a fixed set $\cS\subseteq [M]$.
For any $l \in [M+1, L+1 ]$, we 
define $Y_l = (x_l, X_{l-\cS})$ as a new vector  containing the token at position $l$ and also the tokens in the past $\cS$ positions prior to $x_{l}$.
Here, we follow the convention that $X_{l-\cS}=(X_{l-i})_{i\in \cS}$.
We also consider another fixed  subset $\cS' \subseteq [M]$ and similarly define $Y_l' = (x_l, X_{l-\cS'})$.


The concentration properties of the Markov chain are rooted in the fact that when conditioning on all the parents, the current token is independent of all the past tokens.
Given the parent set structure $\pa=\{-r_1, \ldots, -r_n\}$, we aim to make $Y_{L+1}$ approximately independent of $Y_l$ by conditioning on some intermediate parent sets.
To this end, we define $A = (x_{L+1-M}, \ldots, x_{L-M+r_n}) \in \cX^{r_n}$ and $B_l = (x_{l-r_n+1}, \ldots, x_{l}) \in \cX^{r_n}$ as these intermediate parent sets.
By the Markov property and the parent set structure, we have the following conditional independence relations:
\begin{align}
    Y_{L+1} \indep (B_l, Y_l) \given A, \quad (Y_{L+1}, A) \indep Y_{l} \given B_l, \quad  \forall l = M+1, \ldots, L - M+r_n.
\end{align}
To illustrate, let us consider the first condition $Y_{L+1} \indep (B_l, Y_l) \given A$.
When $l \leq L - M + r_n$, the $B_l$ and $Y_l$ are both contained in the history $\{ x_{k} \colon k \leq L- M + r_n\}=A\cup \{ x_{k} \colon k \leq L- M\}$. 
When conditioning on $A$, the randomness of $(B_l, Y_l)$ is measurable by the $\sigma$-algebra generated by the ``past''  $\{ x_{k} \colon k \leq L - M  \}$.
Moreover, the randomness of $Y_{L+1}$ is measurable by the $\sigma$-algebra generated by the ``future'' $\{ x_{k} \colon k \in [ L+1 - M +r_n, L + 1]  \}$ when conditioning on $A$.
Notice that the parent to the any element in the future $\{ x_{k} \colon k \in [ L+1 - M +r_n, L + 1]  \}$ is either contained in $A$, or can be generated conditioned on $A$ without touching further history $\{ x_{k} \colon k \leq L- M\}$. 
Thus, by the Markov property, conditioning on $A$,  $Y_{L+1}$  is independent of the past $\{ x_{k} \colon k \leq L- M\}$, and in particular, $(B_l, Y_l)$.  
Similarly, 
since  $B $ contains the parent of $x_{l+1}$, conditioning on $B$, $Y_l$ is independent of $x_{l+1} $ and later tokens.  
Moreover, given $B$, the randomness of $Y_{l} $ comes from the randomness of $x_{l-M}, \ldots, x_{l-r_n}$. 
Since $l \leq L-M + r_n$, we have $L+1 - M \geq l +1 - r_n $. As a result, conditioning on $B$, the randomness of $(Y_{L+1}, A)$ comes from tokens generated no earlier than $x_{l+1}$. Therefore, $(Y_{L+1}, A)$ and $Y_{l}$ are conditionally independent given $B_{l}$.  We visualize the definition of $Y_{L+1}$, $A$, $B_l$, and $Y_l$ in \Cref{fig:markov_proof}

\begin{figure}[h]
    \centering
    \includegraphics[width=0.9\textwidth]{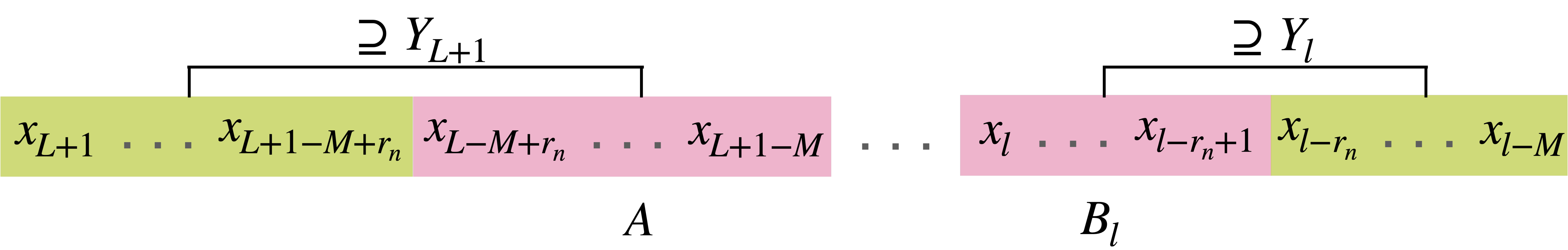}
    \caption{Illustration of the definition of $Y_{L+1}$, $A$, $B_l$, and $Y_l$.
    When conditioned on $A$, $Y_{L+1}$ is independent of $(B_l, Y_l)$.
    When conditioned on $B_l$, $Y_l$ is independent of $(A, Y_{L+1})$.}
    \label{fig:markov_proof}
\end{figure}

Similarly, for  $Y_{l}' = (x_{l}, X_{l - \cS'})$ defined using the subset $\cS'$, we also parallel conditional independence relations: 
\begin{align}
    Y_{L+1}' \indep (B_l, Y_l') \given A, \quad (Y_{L+1}', A) \indep Y_{l}' \given B_l, \quad  \forall l = M+1, \ldots, L - M+r_n.
\end{align}
In particular, we also have 
\begin{align}\label{eq:cond_indep_yy}
    Y_{L+1}  \indep (B_l, Y_l') \given A, \quad (Y_{L+1}, A) \indep Y_{l}' \given B_l, \quad  \forall l = M+1, \ldots, L - M+r_n.
\end{align}

Using $\{ Y_{l}, Y_{l}'\}$, we define a joint distribution  $\hat p^\pi$ over $2+ |\cS| + |\cS'|$ tokens as follows. 
For any  $E \in \cX^{|\cS| + 1}$ and $ E'  \in \cX^{|\cS'| + 1}$, the probability mass function of $\hat p^{\pi} $ is defined as 
\begin{align}
    &\hat p^\pi (Y_{L+1}=E, Y'=E') \\
    &\quad \defeq \frac{1}{L-M} \sum_{l=M+1}^L p^\pi(Y_{L+1}=E, Y_l'=E') 
    \\
    &\quad  = \frac{1}{L-M} \sum_{l=M+1}^L \sum_{A, B_l} \mu^\pi(Y_{L+1} = E \given A) \cdot P_\pi^{L -M+r_n - l} (A \given B_l)\cdot p^\pi(Y_l' = E'\given B_l) \cdot p^\pi(B_l).\label{def:hat-p}
\end{align}
Here, $Y'$ is just a placeholder for $Y_{l}'$
as $\hat p$ takes an average over $l$ and does not depend on any specific position index.
The summation $\sum_{  A,  B_{l}} $ means we sum over all possible values that $A$ and $B_{l}$ can take. 
In the last line of \eqref{def:hat-p}, we decompose the joint distribution $p^\pi(Y_{L+1}=E, Y_l'=E')$ into the product of the conditional distributions by the Markov property in \eqref{eq:cond_indep_yy}. 
That is, 
\begin{align}
p^\pi(Y_{L+1}=E, Y_l'=E')  & = \sum_{A, B_{l} } p^\pi(Y_{L+1}=E,  Y_l'=E', A , B_{l} ) \\
& =   \sum_{A, B_{l} } p^\pi(B_{l} )  \cdot p^\pi(Y_{L+1}=E, A \given B_{l}) \cdot  p^\pi(Y_{l}=E'  \given B_{l} ) \notag \\
& =   \sum_{A, B_{l} } p^\pi(Y_{L+1}=E \given  A) \cdot  p^\pi (A \given B_{l}) \cdot  p^\pi(Y_{l}=E'  \given B_{l} ) \cdot p^\pi(B_{l} ) .
\end{align}
Here the second equality follows from the fact that 
$(Y_{L+1}, A) \indep Y_{l}' \given B_l$  and the last equality follows from the fact that $ Y_{L+1}  \indep (B_l, Y_l') \given A$, which implies $p^\pi(Y_{L+1}=E \given  A, B_{l}) = p^\pi(Y_{L+1}=E \given  A)$. 
Moreover, we  denote by $P_\pi^{i}$ the $i$-step transition kernel of the chain, which corresponds to the $i$-th power of the transition matrix $P_\pi$.
Here, we are following the convention in the main text that 
\begin{align}
    P_\pi(Z',  Z) = \pi(z'_{l} \mid Z_{\mathtt{pa}(l)}) \cdot \mathbf{1}(Z'_{l-r_n + 1:-1} = Z_{l-r_n+1:-1}).
    \label{eq:transition_matrix-1}
\end{align}
In the following, we always consider a fixed transition kernel $\pi$ and omit the superscript/subscript $\pi$ in the matrix notation. We denote the transition matrix by $ P_\pi$ and the stationary distribution by $\mu^\pi$ for a window of length $r_n$.
For the transition matrix, we index each row by the next $r_n$-window $Z'$ and each column by the current $r_n$-window $Z$.
Under this notation, since both $A$ and $B_l$ have lengths $r_n$, we have 
\begin{align}
    \label{eq:transition_ab}
    p^\pi (A \given B_{l}) = P_\pi^{L -M+r_n - l} (A, B_l).
\end{align}
Here $P_\pi^{L -M+r_n - l} (A \given B_l)$ corresponds to the $( A, B_{l} )$-entry of the matrix $(P_{\pi})^{L- M + r_n - l} $. 
Combining \eqref{eq:cond_indep_yy} and \eqref{eq:transition_ab}, we obtain the last equality in \eqref{def:hat-p}. 

In the sequel, to simplify the notation, we write $P_{\pi}$ and $\mu^{\pi}$ as $P$ and $\mu$ respectively. 
Let us consider the reweighted transition kernel 
\begin{align}
    K \defeq \diag\bigl(\sqrt{\mu}\bigr)^{-1} \cdot P \cdot \diag\left(\sqrt{\mu}\right), 
\end{align}
where $\sqrt\mu$ is the element-wise square root of $\mu$.
Since the transition matrix is \emph{primitive} by assumption and having only one eigenvalue with value one  on its spectral circle, we also have for $K$ that the leading eigenvalue is one with eigenvector $\sqrt{\mu}$, i.e.
$\sqrt\mu = K \sqrt\mu$ and $\sqrt\mu^\top = \sqrt\mu^\top K$.
However, the projection in the leading eigenspace (or the Perron projection) is not of our interest. 
The following property of $K$ will be useful in the subsequent proof.
\begin{proposition}
    \label{prop:K}
    For the reweighted transition matrix $K$, we have for any integer $i \ge 0$
    \begin{align}
        P^i - \mu \vone^\top = \diag\left(\sqrt{\mu}\right) \cdot \bigl(K - \sqrt \mu \sqrt \mu^\top\bigr)^i \cdot \diag\bigl(\sqrt{\mu}^{-1}\bigr)
    \end{align}
\end{proposition}
\begin{proof}[Proof of \Cref{prop:K}]
\begin{align}
    P^i - \mu \vone^\top 
    &= \left(\diag\bigl(\sqrt{\mu}\bigr) \cdot K \cdot \diag\bigl(\sqrt{\mu}\bigr)^{-1}\right)^i - \mu \vone^\top \\
    &= \diag\bigl(\sqrt{\mu}\bigr) \cdot \left(K^i  - \sqrt\mu \sqrt\mu^\top  \right)\cdot \diag\bigl(\sqrt{\mu}\bigr)^{-1}  \\
    &= \diag\left(\sqrt{\mu}\right) \cdot \bigl(K - \sqrt \mu \sqrt \mu^\top\bigr)^i \cdot \diag\bigl(\sqrt{\mu}^{-1}\bigr), 
\end{align}
where the last equality holds by noting that $K - \sqrt \mu \sqrt \mu^\top$ project $\sqrt\mu$ to the zero vector, and for any $v\perp \sqrt\mu$, we have $(K - \sqrt\mu\sqrt\mu^\top) v = K v$.
Thus for any test vector $x$:
\begin{align}
    (K - \sqrt\mu \sqrt\mu^\top)^i x 
    &= (K - \sqrt\mu \sqrt\mu^\top)^i (x - \langle \sqrt\mu, x\rangle \cdot \sqrt\mu)\\ 
    &= K^i (x - \langle \sqrt\mu, x\rangle \cdot \sqrt\mu) = (K^i - \sqrt\mu \sqrt\mu^\top) x.
\end{align}
This completes the proof of \Cref{prop:K}.
\end{proof}
Indeed, the second largest eigenvalue of $K$ (in magnitude) determines the mixing rate of the chain.
Let $\lambda$ denote the eigenvalue of $K$ with the second largest magnitude.

Furthermore, if the transition kernel $\pi$ admits a lower bound $\gamma>0$, then we can guarantee that both $p^\pi$ and $\mu^\pi$ admit a uniform lower bound. 
\begin{proposition}[Uniform Lower Bound]\label{prop:transit_lb}
    Suppose $\pi(\cdot\given X_{\pa}) \ge \gamma$ uniformly for some $\gamma > 0$ and $\pa=\{-r_1, \ldots, -r_n\}$. 
    Suppose $X_{1:r_n}\sim \mu_0(\cdot)$ where $\mu_0 \in \Delta(\cX^{r_n})$.
    Then for any $S$ tokens $x_{l_1}, x_{l_2}, \ldots, x_{l_S}$ such that $l_s\ge r_n$ for any $s\in[S]$, we have 
    \begin{align}
        p^\pi(x_{l_1}, \ldots, x_{l_S}) \ge \gamma^{S}.
    \end{align}
\end{proposition}
Using \Cref{prop:transit_lb}, we show that the transition matrix $P_\pi$ is primitive.
\begin{corollary}[Uniform Lower Bound Implies Primitive Transition]\label{cor:primitive}
    Under the condition of \Cref{prop:transit_lb}, with $\pi(\cdot\given X_{\pa}) \ge \gamma>0$, the transition matrix defined in \eqref{eq:transition_matrix-1} is primitive. 
\end{corollary}
\begin{proof}[Proof of \Cref{cor:primitive}]
If the initial distribution is set to be any one-hot vector in $\Delta(\cX^{r_n})$, and taking $x_{l_1}, \ldots, x_{l_S}$ in \Cref{prop:transit_lb} to be $x_{r_n+1}, \ldots, x_{2r_n}$, we conclude that $p^\pi( X_{r_n+1:2r_n}  \given X_{1:r_n} ) >0$ holds for any $ X_{r_n+1:2r_n}, X_{1:r_n}\in\cX^{r_n}$.
Recall from the definition that for a primitive matrix $P$, we can find some positive integer $k$ such that $P^k$ has all positive entries. 
For our case, we can set $k=r_n$ and everything  follows by noting that $p^\pi( X_{r_n+1:2r_n}  \given X_{1:r_n} ) = P_{\pi}^{r_n}( X_{r_n+1:2r_n} ,X_{1:r_n} )$. 
\end{proof}
Another corollary of \Cref{prop:transit_lb} is that, if we take $\mu_0=\mu^\pi$, which is the stationary distribution, we can replace $p^\pi$ in \Cref{prop:transit_lb} by $\mu^{\pi}$.

\begin{corollary} \label{cor:station_lb}
Suppose $\pi(\cdot\given X_{\pa}) \ge \gamma$ uniformly for some $\gamma > 0$ and $\pa=\{-r_1, \ldots, -r_n\}$. 
For the stationary distribution $\mu^\pi$ and $S$ tokens $x_{l_1}, x_{l_2}, \ldots, x_{l_S}$ such that $l_s\ge r_n$ for any $s\in [S]$, we have  $\mu^\pi(x_{l_1}, \ldots, x_{l_S})\ge \gamma^S$.
\end{corollary}

We prove \Cref{prop:transit_lb} as follows. 
\begin{proof}
[Proof of \Cref{prop:transit_lb}]
    Without loss of generality, suppose that $M\le l_1 < l_2 <\ldots < l_S$. 
    We will prove the statement by induction on the number of tokens $S$. 
    If $S=1$, we can rewrite
    \begin{align}
        p^\pi(x_{l_1}) = \sum_{X_{\pa(l_1)}} \pi(x_{l_1}\given X_{\pa(l_1)}) p^\pi(X_{\pa(l_1)})  \ge \sum_{X_{\pa(l_1)}} \gamma \cdot p^\pi(X_{\pa(l_1)}) \ge \gamma. 
    \end{align}
    Now, suppose the statement holds for $1, 2, \ldots, S-1$. 
    Let $Y = x_{l_1}, \ldots, x_{l_{s-1}}$.
    Then, we have 
    \begin{align}
        p^\pi(x_{l_1}, \ldots, x_{l_S}) &= \sum_{X_{\pa(l_S)}\setminus Y} \pi(x_{l_S} \given X_{\pa(l_S)}) \cdot p^\pi(Y) \cdot p^\pi(X_{\pa(l_S)}\setminus Y) \\
        &\ge \sum_{X_{\pa(l_S)}\setminus Y} \gamma \cdot p^\pi(Y) \cdot p^\pi(X_{\pa(l_S)}\setminus Y)= \gamma \cdot p^\pi(Y) \ge \gamma^{S}, 
    \end{align}
    where the last inequality holds by the induction condition. Hence, we finish the proof. 
\end{proof}

Before analyzing $\hat{p}^\pi$, we first study a simpler convergence result: quantifying the closeness between $\sum_{l=M+1}^{L} \eta^{L-l} p^\pi(B_l=b) / \sum_{l=M+1}^{L} \eta^{L-l}$ and $\mu^\pi(b)$ for certain values of $\eta \in (0, 1]$.
\begin{lemma}
    \label{lem:p-2-mu}
    Following the notations introduced above, for the Markov chain with parent set $\pa=\{-r_1, \ldots, -r_n\}$,
    let $D_{\chi^2}(\mu_0\,\|\, \mu^\pi)$ be the $\chi^2$-divergence between the initial distribution $\mu_0$ and the stationary distribution $\mu^\pi$ over the first $r_n$ tokens.
    Take any $\cS\subseteq [M]$ and let $Y_l = (x_l, X_{l-\cS})$ for $l = M+1, \ldots, L+1$.
    Suppose $L/2 \ge M\ge r_n$.
    We have 
    \begin{align}
        \bigg\|\frac{\sum_{l=M+1}^L p^\pi(Y_l =\cdot)}{L - M} - \mu^\pi(Y_{L+1}=\cdot)\bigg\|_{\mathrm{TV}} 
        &\le \frac{2\sqrt{D_{\chi^2}(\mu_0\,\|\, \mu^\pi) + 1}}{L(1 - \lambda)}, 
        \label{eq:windows-TV}\\
        \norm{p^\pi (Y_{L+1}=\cdot) - \mu^\pi (Y_{L+1}=\cdot)}_{\TV} &\le \lambda^{L - M} \sqrt{D_{\chi^2}(\mu_0\,\|\, \mu^\pi) + 1}.\label{eq:p-2-mu-4}
    \end{align}
\end{lemma}
\begin{proof}[Proof of \cref{lem:p-2-mu}]
    Let $c_l = \eta^{L-l} / \sum_{l=r_n}^{L-M+r_n} \eta^{L-l}$, where $\eta \in [0, 1]$ is a constant to be determined. 
    Denote by $\mu_0$, a vector of length $|\cX|^{r_n}$, the initial distribution of the chain.
    We begin by quantifying the total variation (TV) distance:
    \begin{align}
        \bigg\| \frac{\sum_{l=r_n}^{L-M+r_n} \lambda^{L-l} p^\pi(B_l=\cdot)}{\sum_{l=r_n}^{L-M+r_n} \lambda^{L-l}} - \mu^\pi(\cdot)\bigg\|_{\mathrm{TV}} = \bigg \| \sum_{l=r_n}^{L-M+r_n} c_l \cdot \left(p^\pi(B_l=\cdot) - \mu^\pi(\cdot)\right)\bigg\|_{\mathrm{TV}}.
    \end{align}
    Let $b \in \cX^{r_n}$, representing the value for a length-$r_n$ window. Using matrix notation, we have:
    \begin{align}
        \sum_{l=r_n}^{L-M+r_n} c_l \left(p^\pi(B_l=b) - \mu^\pi(b)\right)
        &= \sum_{l=r_n}^{L-M+r_n} c_l \cdot \vone_b^\top   P^{l - r_n}  (\mu_0 - \mu)= \sum_{l=r_n}^{L-M+r_n} c_l \cdot \vone_b^\top  (P^{l - r_n}  - \mu \vone^\top)\mu_0 \\
        &= \sum_{l=r_n}^{L-M+r_n} c_l \cdot \vone_B^\top   \diag\left(\sqrt{\mu}\right)  \bigl(K - \sqrt \mu \sqrt \mu^\top\bigr)^{l - r_n} \diag\left(\sqrt{\mu}\right)^{-1} \mu_0, 
    \end{align}
    where $\vone_b$ is the indicator vector corresponding to $b$. The last equality follows from \Cref{prop:K}. 
    For any test vector $u \in \{0, 1\}^{|\cX|^{r_n}}$, using the variational representation of TV distance:
    \begin{align}
        &\bigg \|\sum_{l=r_n}^{L-M+r_n} c_l \left(p^\pi(B_l=\cdot) - \mu^\pi(\cdot)\right)\bigg\|_{\mathrm{TV}} 
        = \max_{u \in \{0, 1\}^{|\cX|^{r_n}}} u^\top  \sum_{l=r_n}^{L-M+r_n} c_l \left(p^\pi(B_l=\cdot) - \mu^\pi(\cdot)\right) \\
        &\quad =  \max_{u \in \{0, 1\}^{|\cX|^{r_n}}}\sum_{l=r_n}^{L-M+r_n} c_l \cdot  u^\top \diag\left(\sqrt{\mu}\right)  \cdot \bigl(K - \sqrt \mu \sqrt \mu^\top\bigr)^{l - r_n} \cdot \diag\left(\sqrt{\mu}\right)^{-1} \cdot \mu_0 \\
        &\quad \le \sum_{l=r_n}^{L-M+r_n} c_l \cdot \lambda^{l - r_n} \cdot \left\|{\diag\left(\sqrt{\mu}\right)^{-1} \cdot \mu_0}\right\|_2 = \sum_{l=r_n}^{L-M+r_n} c_l \cdot \lambda^{l - r_n} \cdot \sqrt{D_{\chi^2}(\mu_0\,\|\, \mu^\pi) + 1}, 
        \label{eq:p-2-mu-1}
    \end{align}
    where the inequality holds by $\norm{u^\top \diag\left(\sqrt{\mu}\right)}_2 \le \norm{\sqrt \mu}_2 =1$ and $K- \sqrt \mu \sqrt \mu^\top$ has leading eigenvalue with magnitude $\lambda$.
    The last identity follows directly from the definition of the $\chi^2$-divergence that $D_{\chi^2}(\mu_0\,\|\, \mu^\pi) +1= \sum_{b} { \mu_0(b)^2}/{\mu^\pi(b)} $.

    Substituting the definition of $c_l$, we have 
    \begin{align}
        \left\| \frac{\sum_{l=r_n}^{L-M+r_n} \eta^{L-l} p^\pi(B_l=b)}{\sum_{l=r_n}^{L-M+r_n} \eta^{L-l}} - \mu^\pi(A=b)\right\|_{\mathrm{TV}} 
        &\le  \frac{ \sum_{l=r_n}^{L-M+r_n} \eta^{L - l} \cdot \lambda^{l- r_n} \cdot \sqrt{D_{\chi^2}(\mu_0\,\|\, \mu^\pi) + 1}}{\sum_{l=r_n}^{L-M+r_n} \eta^{L-l}}. 
    \end{align}
    We consider two special cases. In the first case, we set $\eta = \lambda$, which gives us
    \begin{align}
        \left\| \frac{\sum_{l=r_n}^{L-M+r_n} \lambda^{L-l} p^\pi(B_l=b)}{\sum_{l=r_n}^{L-M+r_n} \lambda^{L-l}} - \mu^\pi(A=b)\right\|_{\mathrm{TV}} 
        &\le \frac{ \sum_{l=r_n}^{L-M+r_n} \lambda^{L- r_n} \cdot \sqrt{D_{\chi^2}(\mu_0\,\|\, \mu^\pi) + 1}}{(1 - \lambda^{L-M})/(1-\lambda)} \\
        &\le \frac{L \cdot \lambda^{L - r_n} \cdot (1-\lambda)}{1-\lambda^{L-M}} \cdot \sqrt{D_{\chi^2}(\mu_0\,\|\, \mu^\pi) + 1}. 
    \end{align}
    In the second case, we set $\eta = 1$, which gives us
    \begin{align}
        \left\| \frac{\sum_{l=r_n}^{L-M+r_n} p^\pi(B_l=\cdot)}{L -M} - \mu^\pi(A=\cdot)\right\|_{\mathrm{TV}} 
        \!\!\! \le  \frac{\sum_{l=r_n}^{L-M+r_n} \lambda^{l- r_n}  \sqrt{D_{\chi^2}(\mu_0\,\|\, \mu^\pi) + 1}}{L-M} \le \frac{\sqrt{D_{\chi^2}(\mu_0\,\|\, \mu^\pi) + 1}}{(L-M)(1 - \lambda)}. 
    \end{align}

    Note that the TV distance is an $f$-divergence.
    Thus, we can use the data processing inequality to obtain the desired result for $Y_l$ from the above inequality.
    To do so, note that $\sum_{l=M+1}^L p^\pi(Y_l=\cdot) / (L-M)$ and $ \mu^\pi(Y_{L+1}=\cdot)$ can be transformed from $\sum_{l=r_n}^{L-M+r_n-1} p^\pi(B_l=\cdot) / (L-M)$ and $\mu^\pi(A=\cdot)$ by the same emission kernel 
    \begin{align}
        p^\pi(Y_{L+1}=\cdot \given A = \cdot) = p^\pi(Y_{l}=\cdot \given B_{l-M+r_n} = \cdot) = \mu^\pi(Y_{L+1}=\cdot \given A = \cdot)= \mu^\pi(Y_{l}=\cdot \given B_{l-M+r_n} = \cdot).
    \end{align}
    Therefore, by the data processing inequality, it holds that
    \begin{align}
        \bigg\|\frac{\sum_{l=M+1}^L p^\pi(Y_l =\cdot)}{L-M} - \mu^\pi(Y_{L+1}=\cdot)\bigg\|_{\mathrm{TV}} 
        \!\!\!\!\!\!\leq \bigg \|\frac{\sum_{l=r_n}^{L-M+r_n} p^\pi(B_{l} =\cdot)}{L-M } - \mu^\pi(A=\cdot)\bigg\|_{\mathrm{TV}} \!\!\!\!\!\!\le \frac{\sqrt{D_{\chi^2}(\mu_0\,\|\, \mu^\pi) + 1}}{(L-M)(1 - \lambda)}.
    \end{align}
    Similarly for $p^\pi (Y_{L+1} = \cdot)$ and $\mu^\pi(\cdot)$, we have
    \begin{align}
        &\norm{p^\pi (Y_{L+1}=\cdot) - \mu^\pi (Y_{L+1}=\cdot)}_{\TV} 
        \le \norm{p^\pi (A=\cdot) - \mu^\pi (A=\cdot)}_{\TV} \\
        &\quad \le \max_{u\in\{0, 1\}^{|\cX|^{r_n}}} u^\top \cdot   \diag\left(\sqrt{\mu}\right) \cdot \bigl(K - \sqrt \mu \sqrt \mu^\top\bigr)^{L - M} \cdot \diag\left(\sqrt{\mu}\right)^{-1} \cdot \mu_0 
        \le \lambda^{L - M} \sqrt{D_{\chi^2}(\mu_0\,\|\, \mu^\pi) + 1}, 
    \end{align}
    where the latter two inequality follows from the same arguments as in \eqref{eq:p-2-mu-1}.
    Hence, the proof is completed.
\end{proof}

We have established that the average $\sum_{l=M+1}^L p^\pi(Y_l=\cdot)/(L-M)$ converges to $\mu^\pi(A=\cdot)$ in total variation distance.
This represents a ``first-order'' convergence since it involves the average of the marginal distribution of $Y_l$.
However, the quantity of interest in \eqref{def:hat-p} is the average of the joint distribution of $Y_{L+1}$ and $Y_l$, which concerns ``second-order'' convergence.
This is studied in the following lemma.
\begin{lemma}
    \label{lem:hat p-approx-TV}
    Following the notations introduced above, for the Markov chain with parent set $\pa=\{-r_1, \ldots, -r_n\}$,
    let $D_{\chi^2}(\mu_0\,\|\, \mu^\pi)$ be the $\chi^2$-divergence between the initial distribution $\mu_0$ and the stationary distribution $\mu^\pi$ over the first $r_n$ tokens.
    Take any $\cS, \cS'\subseteq [M]$ and let $Y_l = (x_l, X_{l-\cS})$ and $Y_l' = (x_l, X_{l-\cS'})$ for $l = M+1, \ldots, L+1$.
    Suppose $L/2 \ge M\ge r_n$.
    For $\hat p^\pi$ defined in \eqref{def:hat-p}, we have
    \begin{align}
        \left\|\hat p^\pi (Y_{L+1}=\cdot, Y'=\cdot) - \mu^\pi(Y_{L+1}=\cdot) \times \mu^\pi(Y'=\cdot)\right\|_{\mathrm{TV}}\le \frac{2M}{L} + \frac{4\sqrt{D_{\chi^2}(\mu_0\,\|\, \mu^\pi) + 1}}{L (1-\lambda)\cdot \sqrt{\min_{E} \mu^\pi(Y_{L+1}=E)}}. 
    \end{align}
    In particular, we have 
    \begin{align}
        &\bigg\|\hat p^\pi (Y_{L+1}=\cdot, Y'=\cdot) - \mu^\pi(Y_{L+1}=\cdot) \times \bigg(
        \frac{1}{L-M} \sum_{l=M+1}^L p^\pi(Y_l' = \cdot )
        \bigg) \bigg\|_{\TV} \\
        &\quad \le \frac{2M}{L} + \frac{2\sqrt{D_{\chi^2}(\mu_0\,\|\, \mu^\pi) + 1}}{L (1-\lambda)\cdot \sqrt{\min_{E} \mu^\pi(Y_{L+1}=E)}}.
    \label{eq:hat p-approx-TV-mid}
    \end{align}
\end{lemma}
\begin{proof}[Proof of \cref{lem:hat p-approx-TV}]
Let us take 
$ \mu^\pi(E) \cdot  
    (L-M)^{-1} \sum_{l=M+1}^L p^\pi(Y_l = E')
$ as the intermediate distribution, and we have by \eqref{def:hat-p} that
\begin{align}
    &\hat p^\pi (Y_{L+1}=E, Y'=E') - \mu^\pi(Y_{L+1}=E) \cdot \bigg(
        \frac{1}{L-M} \sum_{l=M+1}^L p^\pi(Y_l = E')
    \bigg)  \\
    &  = \underbrace{\frac{1}{L-M} \sum_{l=M+1}^{L-M+r_n}  \sum_{A, B_l} \mu^\pi(Y_{L+1} = E \given A)\cdot \bigl(P^{L - l - (M - r_n)} (A \given B_l) - \mu^\pi(A) \bigr)\cdot  p^\pi(Y_l' = E'\given B_l) \cdot  p^\pi(B_l)}_{\dr (\RNum{1})} \\
    & \quad + \underbrace{\frac{1}{L-M} \sum_{l=L-M+r_n+1}^L \left(p^\pi(Y_{L+1} = E, Y_l' = E') - \mu^\pi(Y_{L+1} = E) p^\pi(Y_l' = E')\right)}_{\dr (\RNum{2})}.\label{eq:hat p-2-product}
\end{align}
where we use the fact that $\sum_A \mu^\pi(Y_{L+1} = E \given A) \mu^\pi(A) = \mu^\pi(Y = E)$ for the first line.
The second term on the right hand side can be easily controlled as we already have an $L^{-1}$ factor. We let $\mathrm{TV}_0$ be the total variation distance of the second term.
It is easy to see that
\begin{align}
    \mathrm{TV}_0 \defeq \frac 1 2 \sum_{E, E'} |\text{(\RNum{2})}| \le \frac{M-r_n}{L-M} \le \frac{M}{L-M}, 
\end{align}
where we remark that (\RNum{2}) is a function of both $E$ and $E'$, and the total variation distance is just taking the sum of the absolute values of the differences.
Here, we also use the fact that $L\ge 2M$.
Using \Cref{prop:K},
we can also rewrite the first term on the right hand side of \eqref{eq:hat p-2-product} in the matrix form as
\begin{align}
    {\dr (\RNum{1})}&=  \frac{1}{L-M} \sum_{l=M+1}^{L-M+r_n}  \mu^\pi(Y_{L+1}=\cdot \given A=\cdot) \cdot \diag\left(\sqrt{\mu}\right) \cdot \bigl(K - \sqrt \mu \sqrt \mu^\top\bigr)^{L-l - (M-r_n)} \cdot \diag\left(\sqrt{\mu}\right)^{-1} \\
    &\autoquad{5} \cdot \diag(p^\pi(B_l = \cdot)) \cdot p^\pi(Y_l' = \cdot\given B_l=\cdot)^\top .
\end{align}
When considering the $\ell_1$-norm of the above term, we introduce a test matrix $U$ of shape $|\cX|^{|Y_{L+1}|} \times |\cX|^{|Y_{L+1}|}$ with each element of $U$ chosen from $\{0, 1\}$.
Let $\mathrm{TV}_1 $ be the total variation distance of the first term (\RNum{1}). 
Then, we have
\begin{align}
    \mathrm{TV}_1 & \le \max_{U} \trace\biggl[\frac{1}{L-M} \sum_{l=M+1}^{L-M+r_n}  \mu^\pi(Y_{L+1}=\cdot \given A=\cdot) \cdot \diag\left(\sqrt{\mu}\right) \cdot \bigl(K - \sqrt \mu \sqrt \mu^\top\bigr)^{L-l - (M-r_n)}  \\
    &\autoquad{5} \cdot \diag\left(\sqrt{\mu}\right)^{-1}\cdot \diag(p^\pi(B_l = \cdot)) \cdot p^\pi(Y_l' = \cdot\given B_l=\cdot)^\top \cdot U(\cdot, \cdot)^\top \biggr]. 
\end{align}
To upper bound this quantity, we consider each row of $U$ as $U(E, \cdot) = u(\cdot \given E)^\top$. 
Note that $u(\cdot \given E)$ is also a $\{0, 1\}$-valued vector.
By expanding the trace, we have
\begin{align}
    \mathrm{TV}_1 & \le 
    \sum_{E}\max_{u(\cdot \given E)} \frac{1}{L-M} \sum_{l=M+1}^{L-M+r_n}  \mu^\pi(Y_{L+1}=E \given A=\cdot) \cdot \diag\left(\sqrt{\mu}\right)  \\
    &\quad \cdot \bigl(K - \sqrt \mu \sqrt \mu^\top\bigr)^{L-l - (M-r_n)}  \cdot \diag\left(\sqrt{\mu}\right)^{-1}\cdot \diag(p^\pi(B_l = \cdot)) \cdot p^\pi(Y_l' = \cdot\given B_l=\cdot)^\top \cdot u(\cdot \given E). 
\end{align}
Note that the $\ell_2$-norm of the vector in the last line can be upper bounded by
\begin{align}
    \!\! &\left\|\bigl(K - \sqrt \mu \sqrt \mu^\top\bigr)^{L-l - (M-r_n)} \cdot \diag\left(\sqrt{\mu}\right)^{-1} \cdot \diag(p^\pi(B_l = \cdot)) \cdot p^\pi(Y_l' = \cdot\given B_l=\cdot)^\top \cdot u(\cdot \given E)\right\|_2 \\
    &\quad \le \left\|\lambda^{L-l - (M-r_n)} \cdot \diag\left(\sqrt{\mu}\right)^{-1}\cdot \diag(p^\pi(B_l = \cdot)) \cdot \vone\right\|_2  
    = \lambda^{L-l - (M-r_n)} \bigl\|\diag\left(\sqrt{\mu}\right)^{-1}\cdot p^\pi(B_l = \cdot)\bigr\|_2 \\
    &\quad = \lambda^{L-l - (M-r_n)} \sqrt{D_{\chi^2}(p^\pi(B_l =\cdot )\,\|\, \mu^\pi(B_l=\cdot)) + 1} \le \lambda^{L-l - (M-r_n)} \sqrt{D_{\chi^2}(\mu_0\,\|\, \mu^\pi) + 1},
    \label{eq:hat p-approx-TV-1}
\end{align}
where the first inequality holds by noting that $p^\pi(Y_l' = \cdot\given B_l=\cdot)^\top \cdot u(\cdot \given E)$ is a vector with element within $[0, 1]$, and also invoking the operator norm of the matrix $\tilde K - \sqrt \mu \sqrt \mu^\top$.
The second identity follows from the definition of the $\chi^2$-divergence that $D_{\chi^2}(p^\pi(B_l =\cdot )\,\|\, \mu^\pi(\cdot)) + 1 = \sum_{b} { p^\pi(B_l = b)^2}/{\mu^\pi(b)}$.
The last inequality is the data processing inequality as $p^\pi(B_l = \cdot)$ can be transformed from $\mu_0(B_{r_n})$ and $\mu^\pi(B_{l})$ can be transformed from $\mu^\pi(B_{r_n})$ by the same emission kernel $\mu^\pi(B_l = \cdot \given B_{r_n}=\cdot)$.
Consequently, we have for the TV distance that
\begin{align}
    \mathrm{TV}_1 & \le \frac{1}{L-M} \sum_{l=M+1}^{L-M+r_n} \lambda^{L-l - (M-r_n)}\cdot \sqrt{D_{\chi^2}(\mu_0\,\|\, \mu^\pi) + 1}  \\
    &\autoquad{3} \cdot  \max_{\{v(\cdot \given E)\}_{E}:\; \norm{v(\cdot \given E)}_2\le 1} \sum_{E, A} \mu^\pi(Y_{L+1}=E \given A) \cdot \sqrt{\mu^\pi(A)} \cdot v(A \given E)\\
    &\le \frac{\sqrt{D_{\chi^2}(\mu_0\,\|\, \mu^\pi) + 1}}{ (L - M)(1-\lambda)} \cdot \max_{\{v(\cdot \given E)\}_{E}:\; \norm{v(\cdot \given E)}_2\le 1}  \sum_{A, E} \frac{\mu^\pi(A\given Y_{L+1}= E)}{\sqrt{\mu^\pi(A)}} \cdot v(A\given E) \cdot \mu^\pi(Y_{L+1}=E) \\
    &\le \max_{\{v(\cdot \given E)\}_{E}:\; \norm{v(\cdot \given E)}_2\le 1} \frac{\sqrt{D_{\chi^2}(\mu_0\,\|\, \mu^\pi) + 1}}{(L-M) (1-\lambda)} \cdot \sqrt{I_{\chi^2}(A; Y_{L+1}) + 1} \cdot  \sqrt{\sum_{A, E} v(A\given E)^2 \cdot \mu^\pi(Y_{L+1}=E)}. 
\end{align}
where in the first equality, we use the variational form of the $\ell_2$-norm for vector $\mu^\pi(Y_{L+1}=E \given A=\cdot) \cdot \diag(\sqrt{\mu})$. 
In the second inequality, we apply \eqref{eq:hat p-approx-TV-1} and use the Bayes rule. 
The last inequality follows from the Cauchy-Schwarz inequality.
Here, the mutual information $I_{\chi^2}(A; Y_{L+1}) + 1$ can be upper bounded by 
\begin{align}
    I_{\chi^2}(A; Y_{L+1}) + 1 
    = \sum_{A, E} \frac{\mu^\pi(Y_{L+1}=E\given A)}{\mu^\pi(Y_{L+1}= E)} \cdot \mu^\pi(Y_{L+1}=E,  A) \le \frac{1}{\min_{E} \mu^\pi(Y_{L+1}=E)}, 
\end{align}
and the last term involving $v(A\given E)$ can be upper bounded by $1$ thanks to the constraint on $v(\cdot \given E)$.
In conclusion, 
\begin{align}
    \mathrm{TV}_1 \le \frac{\sqrt{D_{\chi^2}(\mu_0\,\|\, \mu^\pi) + 1}}{(L-M)(1 - \lambda)\cdot \sqrt{\min_{E} \mu^\pi(Y_{L+1}=E)}}.   
\end{align}
Lastly, let us relate the intermediate distribution to the final distribution $\mu^\pi(Y=\cdot) \times \mu^\pi(Y'=\cdot)$, where we define the total variation distance $\mathrm{TV}_2$ as
\begin{align}
    \mathrm{TV}_2 &\defeq \bigg\|\mu^\pi(\cdot) \cdot \bigg(
        \frac{1}{L-M} \sum_{l=M+1}^{L} p^\pi(Y_l' = \cdot)
    \bigg) - \mu^\pi(\cdot) \cdot \mu^\pi(\cdot)\bigg\|_{\mathrm{TV}} = \bigg\|\bigg(
        \frac{1}{L-M} \sum_{l=M+1}^{L} p^\pi(Y_l' = \cdot)
    \bigg) - \mu^\pi(\cdot)\bigg\|_{\mathrm{TV}}.
\end{align}
Invoking \eqref{eq:windows-TV} of \cref{lem:p-2-mu}, we have this quantity upper bounded by 
\begin{align}
    \mathrm{TV}_2 & \le \frac{\sqrt{D_{\chi^2}(\mu_0\,\|\, \mu^\pi) + 1}}{(L-M)(1 - \lambda)}.
\end{align}
Using the triangular inequality for the total variation distance, we have
\begin{align}
    &\left\|\hat p^\pi (Y_{L+1}=\cdot, Y'=\cdot) - \mu^\pi(Y_{L+1}=\cdot) \times \mu^\pi(Y'=\cdot)\right\|_{\mathrm{TV}} \\
    &\quad \le \mathrm{TV}_0 + \mathrm{TV}_1 + \mathrm{TV}_2 \\
    &\quad \le \frac{M}{L-M} + \frac{2\sqrt{D_{\chi^2}(\mu_0\,\|\, \mu^\pi) + 1}}{(L-M) (1-\lambda)\cdot \sqrt{\min_{E} \mu^\pi(Y_{L+1}=E)}} \\
    &\quad \le \frac{2M}{L} + \frac{4\sqrt{D_{\chi^2}(\mu_0\,\|\, \mu^\pi) + 1}}{L (1-\lambda)\cdot \sqrt{\min_{E} \mu^\pi(Y_{L+1}=E)}}, 
\end{align}
and the upper bound for \eqref{eq:hat p-approx-TV-mid} follows by the same arguments.
Hence, the proof is completed.
\end{proof}

In the following, we use a similar technique as in \cref{lem:hat p-approx-TV} to derive a bound for the chi-square divergence. 
\begin{lemma}
    \label{lem:convergence chi-square}
    For the $\chi^2$-divergence between the empirical distribution $(L-M)^{-1} \sum_{l=M+1}^L \allowbreak  \ind(Y_l = \cdot)$ and the stationary distribution $\mu^\pi(\cdot)$, we have 
    \begin{align}
        \EE\bigg[D_{\chi^2} \bigg(
            \frac{1}{L-M} \sum_{l=M+1}^{L}  \ind(Y_l = \cdot) \,\Big\|\, \mu^\pi(Y_{L+1}=\cdot)
        \bigg)\bigg] \le \frac{4 (1-\lambda)^{-1}\sqrt{D_{\chi^2}(\mu_0\,\|\, \mu^\pi) + 1} + 16 M}{L \cdot \min_{E} \mu^\pi(Y_{L+1}=E)}, 
    \end{align}
    where the expectation is with respect to $X\sim p^\pi$. 
\end{lemma}
\begin{proof}[Proof of \cref{lem:convergence chi-square}]
    By definition of the $\chi^2$-divergence, what we aim to bound is just 
    \begin{align}
        &\EE \bigg[
            \sum_{E}  \bigg(
                (L-M)^{-1} \sum_{l=M+1}^L  \ind(Y_l = E) - \mu^\pi(E)
                \bigg)^2 \bigg / \mu^\pi(E) 
                \bigg]  \\
        &\quad = \EE \bigg[
                    \sum_{E} \frac{
                        (L-M)^{-2} \sum_{l, l'= M+1}^L  \ind(Y_l = Y_{l'} = E) - \mu^\pi(E)^2
                    }{\mu^\pi(E)}
                \bigg]\\
        &\quad = \EE\bigg[
            \sum_E  \sum_{l, l'=M+1}^L  \frac{\ind(Y_l = Y_{l'} = E)}{(L-M)^{2} \mu^\pi(E)}  - 1
        \bigg] = \sum_E  \sum_{l, l'=M+1}^L  \frac{p^\pi(Y_l = Y_{l'} = E)}{(L-M)^{2} \mu^\pi(E)}  - 1.
    \end{align}
    To study the above quantity, for $l \ge 2M-r_n+2$, we define
    \begin{align}
        J_1(l) \defeq 
            \sum_E \sum_{l'=M+1}^{l - M + r_n-1}  \frac{p^\pi(Y_l = Y_{l'} = E)}{(L-M)^{2} \mu^\pi(E)} - \frac{l-2M+r_n}{(L-M)^2}.
    \end{align}
    Following our convention, we let $A_l=X_{l-M:l-M+r_n-1}$ and $B_{l'}=X_{l'-r_n+1:l'}$ be two length-$r_n$ window and by the Markov property, we have
    \begin{align}
        Y_{l+1} \indep (B_{l'}, Y_{l'}) \given A_l, \quad (Y_{l+1}, B_{l}) \indep Y_{l'} \given B_{l'}. 
    \end{align}
    Let us fix an index $l \ge 2M-r_n+2$ and take a summation over $M+1 \le l' \le l - M+r_n -1$.
    Expanding the joint distribution, we have
    \begin{align}
        J_1(l) &\defeq \frac{1}{(L-M)^2} \sum_{l'=M+1}^{l- M+r_n-1}  \sum_{E, A_l, B_{l'}} \mu^\pi(Y_{l} = E \given A_l)  \cdot \bigl(P^{l - l' - M+r_n-1} (A_l \given B_{l'}) - \mu^\pi(A_l) \bigr) \\
        &\hspace{3cm} \cdot p^\pi(Y_{l'} = E\given B_{l'}) \cdot p^\pi(B_{l'}) \cdot \mu^\pi(Y_{l'}=E)^{-1}\\
        &=  \frac{1}{(L-M)^2} \sum_{l'=M+1}^{l- M+r_n-1}  \trace\Bigl[\mu^\pi(Y_{l}=\cdot \given A_l=\cdot) \cdot \diag\left(\sqrt{\mu}\right) \cdot \bigl(K - \sqrt \mu \sqrt \mu^\top\bigr)^{l-l'- M + r_n - 1}  \\
        &\hspace{3cm} \cdot \diag\bigl(\sqrt{\mu}\bigr)^{-1} \cdot \diag(p^\pi(B_{l'} = \cdot)) \cdot p^\pi(Y_{l'} = \cdot\given B_{l'}=\cdot)^\top \cdot \diag(\mu^\pi(Y_{l'}=\cdot)^{-1}) \Bigr] \\
        &=  \frac{1}{(L-M)^2} \sum_{l'=M+1}^{l- M+r_n-1}  \trace\Bigl[\diag(\mu^\pi(Y_{l'}=\cdot)^{-1/2}) \cdot \mu^\pi(Y_{l}=\cdot \given A_l=\cdot) \cdot \diag\left(\sqrt{\mu}\right) \\
        &\hspace{3cm} \cdot \bigl(K - \sqrt \mu \sqrt \mu^\top\bigr)^{l-l'- M + r_n - 1}  \cdot \diag\bigl(\sqrt{\mu}\bigr)^{-1} \cdot \diag(p^\pi(B_{l'} = \cdot)) \\
        &\hspace{4cm}\cdot p^\pi(Y_{l'} = \cdot\given B_{l'}=\cdot)^\top \cdot \diag(\mu^\pi(Y_{l'}=\cdot)^{-1/2}) \Bigr],
    \end{align}
    where the first identity follows from the fact that 
    \begin{align} 
        &\sum_{E, A_l, B_l'} \mu^\pi(Y_l=E\given A_l) \cdot \mu^\pi(A_l) \cdot p^\pi(Y_{l'}=E\given B_{l'}) \cdot p^\pi(B_{l'}) \cdot \mu^\pi(Y_{l'}=E)^{-1}\\
       &\quad = \sum_{E} p^\pi(Y_{l'}=E) \cdot \mu^\pi(Y_{l'}=E) \cdot \mu^\pi(Y_{l'}=E)^{-1} = 1, 
    \end{align}
    and the second identity follows from \Cref{prop:K}.
    We next invoke the Cauchy-Schwarz inequality for trace, i.e., $\trace(W^\top V)^2\le \trace(W^\top W) \trace(V^\top V)$, where we take
    \begin{align}
        W^\top &= \diag(\mu^\pi(Y_{l}=\cdot)^{-1/2}) \cdot \mu^\pi(Y_{l}=\cdot \given A_l=\cdot) \cdot \diag(\sqrt{\mu}) \cdot \bigl(K - \sqrt \mu \sqrt \mu^\top\bigr)^{l-l' - M+r_n-1},  \\
        V & = \diag\bigl(\sqrt{\mu}\bigr)^{-1} \cdot \diag(p^\pi(B_{l'} = \cdot)) \cdot p^\pi(Y_{l'} = \cdot\given B_{l'}=\cdot)^\top \cdot \diag(\mu^\pi(Y_{l'}=\cdot)^{-1/2}) \\
        &= \diag\bigl(\sqrt{\mu}\bigr) \cdot p^\pi(Y_{l'} = \cdot\given B_{l'}=\cdot)^\top \cdot \diag(\mu^\pi(Y_{l'}=\cdot)^{-1/2})
    \end{align}
    Note that 
    \begin{align}
        \sqrt{\trace(W^\top W)} 
        &\le \lambda^{l-l'-M+r_n-1} \cdot \sqrt{\trace\left(\diag(\mu^\pi(Y_{l}=\cdot)^{-1})  \mu^\pi(Y_{l}=\cdot \given A=\cdot)  \diag\left(\mu \right)  \mu^\pi(Y_{l}=\cdot \given A=\cdot)^\top\right)}\\
        &=\lambda^{l-l'-M+r_n-1} \cdot \sqrt{\sum_{A_l, Y_l}\frac{\mu^\pi(Y_l , A_l)^2}{\mu^\pi(Y_l ) \cdot \mu^\pi(A_l )}}.
    \end{align}
    Following the same calculation, we have
    \begin{align}
        \sqrt{\trace(V^\top V)} 
        &= \sqrt{\sum_{Y_{l'}, B_{l'}} \frac{p^\pi(Y_{l'} , B_{l'})^2}{\mu^\pi(Y_{l'}) \mu^\pi(B_{l'})}}.
    \end{align}
    Therefore, 
    \begin{align}
        J_1(l)\le \frac{1}{(L-M)^2} \sum_{l'=M+1}^{l- M+r_n-1} \lambda^{l-l'- M+r_n-1} \cdot 
        \sqrt{\sum_{A_l, Y_l}\frac{\mu^\pi(Y_l , A_l)^2}{\mu^\pi(Y_l ) \cdot \mu^\pi(A_l )} \cdot \sum_{Y_{l'}, B_{l'}} \frac{p^\pi(Y_{l'} , B_{l'})^2}{\mu^\pi(Y_{l'}) \mu^\pi(B_{l'})}}.
    \end{align}
    We further have
    \begin{align}
        \sum_{Y_{l'}, B_{l'}} \frac{p^\pi(Y_{l'} , B_{l'})^2}{\mu^\pi(Y_{l'}) \mu^\pi(B_{l'})}
        &\le \max_{Y_{l'}, B_{l'}} \left\{\frac{p^\pi(Y_{l'} \given B_{l'} )}{\mu^\pi(Y_{l'} )}\right\} \cdot \sum_{B_{l'}}\frac{p^\pi(B_{l'})^2 }{\mu^\pi(B_{l'})} \\
        &\le \frac{D_{\chi^2} (p^\pi(B_{l'}=\cdot) \,\|\, \mu^\pi(B_{l'}=\cdot))  + 1}{\min_{E} \mu^\pi(Y_{L+1}=E)} 
        \le \frac{D_{\chi^2}(\mu_0 \,\|\, \mu^\pi) + 1}{\min_{E} \mu^\pi(Y_{L+1}=E)}, 
    \end{align}
    where the last inequality holds by the data processing inequality. 
    Similarly, we have 
    \begin{align}
        \sum_{A_l, Y_l}\frac{\mu^\pi(Y_l , A_l)^2}{\mu^\pi(Y_l ) \cdot \mu^\pi(A_l )} 
        &\le \max_{Y_l, A_l} \left\{ \frac{\mu^\pi(Y_l\given A_l)}{\mu^\pi(Y_l)}\right\} \le \frac{1}{\min_{E} \mu^\pi(Y_{L+1}=E)}.
    \end{align}
    Therefore, we conclude that 
    \begin{align}
        J_1(l) \le \frac{\sqrt{D_{\chi^2}(\mu_0 \,\|\, \mu^\pi) + 1}}{(L-M)^2 (1-\lambda) \cdot \min_{E} \mu^\pi(Y_{L+1}=E)},
    \end{align}
    and 
    \begin{align}
        2\sum_{l=2M-r_n+2}^{L} J_1(l) \le \frac{2\sqrt{D_{\chi^2}(\mu_0 \,\|\, \mu^\pi) + 1}}{(L-M) (1-\lambda) \cdot \min_{E} \mu^\pi(Y_{L+1}=E)}, 
    \end{align}
    where we double the value as $l>l'$ only contributes to half of the terms in the double summation. 
    Note that in the above summation for $l> l'$, we only include terms satisfying $l - l' \ge M-r_n+1$ and $l - (M+1) \ge M -r_n +1$. 
    For the remaining $(l, l')$ not included above, each term is bounded above by 
    \begin{align}
        \bigg|\frac{1}{(L-M)^{2}} \bigg(\sum_E  \frac{p^\pi(Y_l = Y_{l'} = E)}{\mu^\pi(E)} -1\bigg)\bigg| \le \frac{1}{(L-M)^{2} \min_{E} \mu^\pi(Y_{L+1}=E)}, 
    \end{align}
    and we have no more than $4 L (M -r_n +1)$ of these terms in total. 
    As a result, we conclude with $L/2\ge M \ge r_n$ that
    \begin{align}
        J_1 \le \frac{4 (1-\lambda)^{-1}\sqrt{D_{\chi^2}(\mu_0\,\|\, \mu^\pi) + 1} + 16 M}{L \cdot \min_{E} \mu^\pi(Y_{L+1}=E)}. 
    \end{align}
    Hence, we complete the proof of \Cref{lem:convergence chi-square}.
\end{proof}

\begin{proposition}
    \label{prop:weighted_approximation}
    Let us define
    \begin{align}
        \tilde \mu_X^\pi(z, Z_{-\cS^\star}) = \frac{\mu^\pi(z, Z_{-\cS^\star}) \exp\left(
            a \cdot \prod_{h\in\cS^\star}\ind(z_{-  h} = x_{L+1-h})
        \right)}{\sum_{z', Z_{-\cS^\star}'} \mu^\pi(z', Z_{-\cS^\star}') \exp\left(
            a \cdot \prod_{h\in\cS^\star}\ind(z_{-h}' = x_{L+1-h})
        \right)}, 
    \end{align}
    where $Z_{-\cS^\star} = (z_{-h})_{h\in\cS^\star}$ and $\mu^\pi$ is the stationary distribution of the Markov chain over a window of size $M+1$. 
    We also treat $\tilde \mu_X^\pi(\cdot)$ as a length $|\cX|$ vector where $\cX$ is the state space of the Markov chain.
    Let $\hat \nu_X^\pi(z, Z_{-\cS^\star}) = \sum_{l=M+1}^L \sigma_l^\star \ind(x_l = z, X_{l-\cS^\star} = Z_{-\cS^\star})$ where 
    \begin{align}
        \sigma_l^\star = \frac{\exp(a \cdot \prod_{h\in\cS^\star} \ind(x_{l-h} = x_{L+1-h}) )}{\sum_{l'=M+1}^L \exp(a \cdot \prod_{h\in\cS^\star} \ind(x_{l'-h} = x_{L+1-h}) )}.
    \end{align}
    Then, we have
    \begin{align}
        \EE_X\left[
        \left\|\tilde\mu_X^\pi(z=\cdot, Z_{-\cS^\star}=\cdot) -\hat \nu_X^\pi(z=\cdot, Z_{-\cS^\star}=\cdot)\right\|_{1}
    \right] 
    &\le \frac{4 \bigl(  (1-\lambda)^{-1}\sqrt{D_{\chi^2}(\mu_0\,\|\, \mu^\pi) + 1} + 4 M \bigr) ^{1/2}}{L^{1/2}\cdot \min_{x_{L+1}, X_{L+1 -\cS^\star}} \mu^\pi(x_{L+1}, X_{L+1 -\cS^\star})}. 
    \end{align}
\end{proposition}
\begin{proof}[Proof of \cref{prop:weighted_approximation}]
    To unify the notations, we let $Z=(z_{-M},\dots, z_{-1})$ and define
    \begin{gather}
        \hat \mu_X^\pi(z, Z_{-\cS^\star}) = \frac{1}{L-M} \sum_{l=M+1}^L \ind(x_l = z, X_{l-\cS^\star}=Z_{-\cS^\star}), \\
        R(Z_{-\cS^\star}, X_{L+1-\cS^\star}) = \exp\left(
            a \cdot  \ind(Z_{-\cS^\star} = x_{L+1-\cS^\star})
        \right). 
    \end{gather}
    Using these notations, we can define the normalizing factor in $\tilde\mu_X^\pi$ and $y_X^\star$ respectively as
    \begin{align}
        \Phi = \sum_{z, Z_{-\cS^\star}} \mu^\pi(z, Z_{-\cS^\star}) \cdot R(Z_{-\cS^\star}, X_{L+1-\cS^\star}), \quad \hat\Phi = \sum_{z, Z_{-\cS^\star}} \hat\mu_X^\pi(z, Z_{-\cS^\star}) \cdot R(Z_{-\cS^\star}, X_{L+1-\cS^\star}).
    \end{align}
    We also define 
    \begin{align}
        \phi(z, Z_{-\cS^\star}) = \mu^\pi(z, Z_{-\cS^\star}) \cdot R(Z_{-\cS^\star}, X_{L+1-\cS^\star}), \quad \hat\phi(z, Z_{-\cS^\star}) = \hat\mu_X^\pi(z, Z_{-\cS^\star}) \cdot R(Z_{-\cS^\star}, X_{L+1-\cS^\star}).
    \end{align}
    We can then rewrite the objective as
    \begin{align}
        &\left\|\tilde\mu_X^\pi(z=\cdot, Z_{-\cS^\star}=\cdot) -\hat \nu_X^\pi(z=\cdot, Z_{-\cS^\star}=\cdot)\right\|_{1} \\
        &\quad = \sum_{z, Z_{-\cS^\star}} \bigg |\frac{\phi(z, Z_{-\cS^\star})}{\Phi} - \frac{\hat\phi(z, Z_{-\cS^\star})}{\hat\Phi}\bigg|  \le \!\!\! \sum_{z, Z_{-\cS^\star}}\!\!\!  \frac{\hat \phi(z, Z_{-\cS^\star}) \cdot |\hat \Phi  - \Phi| + |\phi(z, Z_{-\cS^\star}) - \hat\phi(z, Z_{-\cS^\star})| \cdot \hat \Phi}{\Phi \cdot \hat\Phi} \\
        &\quad = \frac{|\hat\Phi - \Phi| + \sum_{z, Z_{-\cS^\star}} |\phi(z, Z_{-\cS^\star}) - \hat\phi(z, Z_{-\cS^\star})|}{\Phi}
        \le \frac{2\sum_{z, Z_{-\cS^\star}} |\phi(z, Z_{-\cS^\star}) - \hat\phi(z, Z_{-\cS^\star})|}{\Phi}. \label{eq:weighted_approximation-1}
    \end{align}
    Furthermore, notice that 
    \begin{align}
        &\frac{\sum_{z, Z_{-\cS^\star}} |\phi(z, Z_{-\cS^\star}) - \hat\phi(z, Z_{-\cS^\star})|}{\Phi} 
        = \frac{\sum_{z, Z_{-\cS^\star}} \left|(\mu^\pi(z, Z_{-\cS^\star}) - \hat \mu_X ^\pi(z, Z_{-\cS^\star}) ) \cdot R(Z_{-\cS^\star}, X_{L+1-\cS^\star}) \right|}{ \sum_{z, Z_{-\cS^\star}} \mu^\pi(z, Z_{-\cS^\star}) \cdot R(Z_{-\cS^\star}, X_{L+1-\cS^\star})} \\
        &\quad  \le \frac{
            \sum_{z, Z_{-\cS^\star}} \left|(\mu^\pi(z, Z_{-\cS^\star}) - \hat\mu_X^\pi(z, Z_{-\cS^\star}))\right| + (e^a - 1)\sum_{z, Z_{-\cS^\star}\in\Gamma_X} \left|\mu^\pi(z, Z_{-\cS^\star}) - \hat\mu_X^\pi(z, Z_{-\cS^\star})\right|
        }{1 + (e^a - 1) \cdot \sum_{z, Z_{-\cS^\star}\in\Gamma_X} \mu^\pi(z, Z_{-\cS^\star})}\\
        &\quad \le \sum_{z, Z_{-\cS^\star}} \left|\mu^\pi(z,Z_{-\cS^\star} ) - \hat\mu_X^\pi(z, Z_{-\cS^\star})\right| + \frac{\sum_{z, Z_{-\cS^\star}\in\Gamma_X} \left|(\mu^\pi(z, Z_{-\cS^\star}) - \hat\mu_X^\pi(z, Z_{-\cS^\star}))\right|}{\sum_{z, Z_{-\cS^\star}\in\Gamma_X}\mu^\pi(z, Z_{-\cS^\star})}
        \label{eq:weighted_approximation-2}.
    \end{align}
    where we define $\Gamma_X = \{Z_{-\cS^\star}: Z_{-\cS^\star}=X_{L+1 - \cS^\star}\}$. 
    Note that when $Z_{-\cS^\star}\in\Gamma_X$, we have $R(Z_{-\cS^\star}, X_{L+1-\cS^\star}) = e^a$ and when $Z_{-\cS^\star}\notin \Gamma_X$, we have $R(Z_{-\cS^\star}, X_{L+1-\cS^\star}) = 1$.
    For the first term on the right-hand side of \eqref{eq:weighted_approximation-2}, we have by Cauchy-Schwarz that
    \begin{align}
        \EE_X\bigg[
            \sum_{z, Z_{-\cS^\star}} \left|\mu^\pi(z, Z_{-\cS^\star}) - \hat\mu_X^\pi(z, Z_{-\cS^\star})\right|
        \bigg]
        &\le \biggl( \EE_X\bigg[
            \sum_{z, Z_{-\cS^\star}} \frac{(\mu^\pi(z, Z_{-\cS^\star}) - \hat\mu_X^\pi(z, Z_{-\cS^\star}))^2}{\mu^\pi(z, Z_{-\cS^\star})}
        \bigg] \bigg)^{1/2} \\
        &\le \bigg( \frac{4 (1-\lambda)^{-1}\sqrt{D_{\chi^2}(\mu_0\,\|\, \mu^\pi) + 1} + 16 M}{L \cdot \min_{x_{L+1}, X_{L+1-\cS^\star}} \mu^\pi(x_{L+1}, X_{L+1-\cS^\star})} \bigg)^{1/2} , 
    \end{align}
    where in the last inequality, we invoke \cref{lem:convergence chi-square} where we take $Y_l = x_l$ in the lemma. 
    For the second term on the right hand of \eqref{eq:weighted_approximation-2}, we note that 
    \begin{align}
        & \EE_X\bigg[\frac{\sum_{z, Z_{-\cS^\star}\in\Gamma_X} \left|(\mu^\pi(z, Z_{-\cS^\star}) - \hat\mu_X^\pi(z, Z_{-\cS^\star}))\right|}{\sum_{z, Z_{-\cS^\star}\in\Gamma_X}\mu^\pi(z, Z_{-\cS^\star})}\bigg] \\
        &\quad \le \sum_{E, z} \EE_X\left[\frac{\left|\mu^\pi(z, Z_{-\cS^\star} = E) - \hat\mu_X^\pi(z, Z_{-\cS^\star} = E)\right|}{\mu^\pi(Z_{-\cS^\star} = E)}  \cdot \ind(X_{L+1 - \cS^\star} = E)\right] \\
        &\quad \le \sum_{E, z} \bigg( \EE_X\bigg[
            \Big(\frac{\mu^\pi(z, Z_{-\cS^\star} = E) - \hat\mu_X^\pi(z, Z_{-\cS^\star} = E)}{\sqrt{\mu^\pi(Z_{-\cS^\star} = E)}}\Big)^2 
        \bigg] \cdot 
            \frac{p^\pi(X_{L+1-\cS^\star} = E)}{\mu^\pi(Z_{-\cS^\star}=E)} \bigg)^{1/2}  \\
        &\quad \le \bigg ( \EE_X\bigg[
            \sum_{E, z} \frac{\left(\mu^\pi(z, Z_{-\cS^\star} = E) - \hat\mu_X^\pi(z, Z_{-\cS^\star} = E)\right)^2 }{\mu^\pi(Z_{-\cS^\star} = E)}
        \bigg] \cdot 
        \sum_{E, z} \frac{p^\pi(X_{L+1-\cS^\star} = E)}{\mu^\pi(Z_{-\cS^\star}=E)} \bigg)^{1/2} , 
        \label{eq:weighted_approximation-3}
    \end{align}
    where the last two inequalities 
 follow from the Cauchy-Schwarz inequality.
    We have an upper bound for the second term on the right-hand side of \eqref{eq:weighted_approximation-3} that
    \begin{align}
        \bigg( \sum_{E, z} \frac{p^\pi(X_{L+1-\cS^\star} = E)}{\mu^\pi(Z_{-\cS^\star}=E)}\bigg)^{1/2 } \le \sqrt{\frac{1}{\min_E \mu^\pi(Z_{-\cS^\star}=E)}}.
        \label{eq:weighted_approximation-4}
    \end{align}
    We can also apply \cref{lem:convergence chi-square} to the first term with $Y_{L+1} = (x_{L+1}, X_{L+1 -\cS^\star})$
    and conclude that 
    \begin{align}
        &\biggl( \EE_X \bigg[
            \sum_{E, z} \frac{\left(\mu^\pi(z, Z_{-\cS^\star} = E) - \hat\mu_X^\pi(z, Z_{-\cS^\star} = E)\right)^2 }{\mu^\pi(Z_{-\cS^\star} = E) } \bigg]\bigg)^{1/2}
         \\
        &\quad \le \bigg( \frac{4 (1-\lambda)^{-1}\sqrt{D_{\chi^2}(\mu_0\,\|\, \mu^\pi) + 1} + 16 M}{L \cdot \min_{x_{L+1}, X_{L+1 -\cS^\star}} \mu^\pi(x_{L+1}, X_{L+1 -\cS^\star})}\bigg)^{1/2} . 
        \label{eq:weighted_approximation-5}
    \end{align}
    In summary, we have 
    \begin{align}
        &\EE_X\left[
            \left\|\tilde\mu_X^\pi(e_k) -y^\star(k)\right\|_{1}
        \right] \\ 
        &\quad\le \frac{2}{\min_{x_{L+1}, X_{L+1 -\cS^\star}} \mu^\pi(x_{L+1}, X_{L+1 -\cS^\star})}\cdot \bigg( \frac{ (1-\lambda)^{-1}\sqrt{D_{\chi^2}(\mu_0\,\|\, \mu^\pi) + 1} + 4 M}{L }\bigg)^{1/2}  \\
        &\qqquad + 2\bigg( \frac{(1-\lambda)^{-1}\sqrt{D_{\chi^2}(\mu_0\,\|\, \mu^\pi) + 1} + 4 M}{L \cdot \min_{x_{L+1}, X_{L+1-\cS^\star}} \mu^\pi(x_{L+1}, X_{L+1-\cS^\star}))} \bigg)^{1/2} .
    \end{align}
   Note that the second term is dominated by the first term. Thus, we conclude the proof of \Cref{prop:weighted_approximation}.
\end{proof}

\end{document}